\documentclass[10pt,onecolumn]{IEEEtran}

\usepackage{times}
\usepackage{epsfig}
\usepackage{graphicx}
\usepackage{amsmath}
\usepackage{amssymb}
\usepackage{morefloats}

\usepackage{times}
\usepackage{graphicx}
\usepackage{amsmath}
\usepackage{amssymb}
\usepackage{wrapfig}

\usepackage{mathrsfs}
\usepackage{amsfonts}
\usepackage{amsthm}

\usepackage{multirow}

\usepackage[lined,boxed]{algorithm2e}
\usepackage{color}

\usepackage{bm}
\usepackage{setspace}
\usepackage[section]{placeins}

\newcommand{\reffont}{\fontsize{7.5pt}{7.5pt}\selectfont}
\newcommand{\authorfont}{\fontsize{12pt}{12pt}\selectfont}

\usepackage[rm]{subfigure}
\usepackage{xspace,algorithmic,algorithm,bibspacing}
\graphicspath{{./}{./figs/}{./caltech/}{./Author_Photo/}}

\newcommand{\comment}[1]{}

\renewcommand{\comment}[1]{ {\color{red}{COMMENTS}:}{\color{red}{ #1} }}

\def\T{{\!\top}}

\newtheorem{theorem}{Theorem}[section]

\def\bp{ {\bf p } }

\def\bM{ {\bf M } }

\usepackage{caption}
\usepackage[compress,nosort]{cite}

\renewcommand{\vspace}[1]{{}}

\begin{document}

\begin{spacing}{1.2}

\title{Online Metric-Weighted Linear Representations for Robust Visual Tracking}

\author{ \authorfont
  Xi Li${^{\dag,\ddag}}$,
  Chunhua Shen$^{\ddag}$\thanks{Correspondence should be addressed to C. Shen (email: chunhua.shen@adelaide.edu.au).},
  Anthony Dick$^{\ddag}$,
  Zhongfei Zhang$^{\dag}$,
  Yueting Zhuang$^{\dag}$\\
$^{\dag}$College of Computer Science, Zhejiang University, China\\
$^{\ddag}$School of Computer Science, University of Adelaide, Australia

}

\IEEEcompsoctitleabstractindextext{%

\begin{abstract}

In this paper, we propose a visual tracker based on
a metric-weighted linear representation of appearance.
In order to capture the interdependence of
different feature dimensions, we develop two
online distance metric learning methods using
proximity comparison information and structured output learning.
    The learned metric is then incorporated into a linear representation of appearance.
We show that online distance metric learning
significantly improves the robustness of the
tracker, especially on those sequences exhibiting drastic
appearance changes.
In order to bound growth in the number of training samples,
we design a time-weighted reservoir sampling method.

Moreover,  we enable our tracker
to automatically perform object identification
during the process of object tracking,
by introducing a collection of static template samples
belonging to several object classes of interest.
Object identification results for
an entire video sequence are achieved by
systematically combining the tracking information
and visual recognition at each frame.
Experimental results on challenging video sequences demonstrate
the effectiveness of the method for both inter-frame
tracking and object identification.

\end{abstract}

\begin{keywords}
Visual tracking, linear representation, structured metric learning, reservoir sampling
\end{keywords}}

\maketitle

\IEEEdisplaynotcompsoctitleabstractindextext
\IEEEpeerreviewmaketitle

\section{Introduction}

Visual tracking is an important and challenging
problem in computer vision, with widespread application domains.
Its goal is to consistently locate an object of interest in multiple images captured at successive time steps.
Despite great progress in recent years, visual tracking
remains a challenging problem because of the
complicated appearance changes caused by factors including
illumination variation, shape deformation, occlusion,
pose variation, background clutter, sophisticated
object motion, and scene blurring.
To address these factors, a variety of tracking approaches
have been proposed
to improve the robustness, speed, or accuracy of visual tracking.
In order to effectively capture the dynamic spatio-temporal
information on object appearance, these tracking approaches aim to learn
generative or discriminative appearance models using a variety of statistical
learning techniques,
including
hidden Markov model~\cite{park2012robust},
mixture models~\cite{Jepson-Fleet-Yacoob5}, subspace
learning~\cite{Limy-Ross17,Kwon-Lee-CVPR2010,zhang2013robust},
linear regression~\cite{Meo-Ling-ICCV09, Li-Shen-Shi-cvpr2011,zhang2013robust,wang2013online},
covariance learning~\cite{wu2012real},
compressive tracking~\cite{zhang2012real},
SVMs~\cite{Avidan-2004,bai2012robust,harestruck_iccv2011,khanloo2012large},
boosting~\cite{Grabner-Grabner-Bischof-BMVC2006,Babenko-Yang-Belongie-cvpr2009}, random
forest~\cite{Santner-Leistner-Saffari-Pock-Bischof-cvpr2010}, spatial attention
learning~\cite{Fan-Wu-Dai-ECCV2010}, metric learning~\cite{wang2010discriminative,tsagkatakisonline_csvt}, and
tracking-learning-detection~\cite{kalal2012tracking}.

Linear representations, in which the object is
represented as a linear combination of basis samples, are often used to build
such appearance models.
Suppose that we have a set of basis samples denoted as: $\mathcal{T}=[\textbf{T}_{1} \ldots \textbf{T}_{q}] \in \mathcal{R}^{d\times q}$.
Using these basis samples,
a new sample $\textbf{y}$ can be approximated
by the following linear combination:
$\textbf{y} \approx \mathcal{T}\textbf{c}=c_{1}\textbf{T}_{1} + c_{2}\textbf{T}_{2} + \cdots + c_{q}\textbf{T}_{q}$,
where $\textbf{c}=(c_{1}, c_{2}, \ldots, c_{q})^{\T}$ is a coefficient vector.
This gives rise to the following reconstruction error norm:
$D_{\textbf{y}} = \|\textbf{y} - (c_{1}\textbf{T}_{1} + c_{2}\textbf{T}_{2} + \cdots + c_{q}\textbf{T}_{q})\|_{2}$.
In this case, the smaller $D_{\textbf{y}}$ is, the more likely $\textbf{y}$ is generated from the
subspace spanned by $\mathcal{T}$. During visual tracking,  the subspace is likely to vary
dynamically as new data arrive, so $\mathcal{T}$ should be accordingly adjusted to the new data.
In addition, the feature representation for each sample
is typically extracted from local image patches, whose appearance is often correlated because of their spatial proximity.
Therefore, the elements of the feature representation in different dimensions often
contain intrinsic spatial correlation information.
In general, such correlation information is stable
in the case of complicated appearance changes (e.g., partial occlusion, illumination variation, and shape deformation),
and thus
plays an important role in robust visual tracking.
However, most existing linear representation-based trackers
(e.g., \cite{Meo-Ling-ICCV09,Li-Shen-Shi-cvpr2011})  build
linear regressors that treat feature dimensions independently
and ignore the correlation between them.

Therefore, we address the following three key issues for tracking
robustness and  efficiency:
i)
how to capture
the intrinsic correlation between different feature dimensions;
ii)
how to maintain and update a limited-sized
basis sample buffer $\mathcal{T}$ that effectively adapts during tracking;
and iii)
given this correlation information and a dynamically changing basis  $\mathcal{T}$, how to efficiently compute the optimal
coefficient vector $\textbf{c}$ at each frame.

The core of our approach to these problems is online metric learning, which learns and updates a distance metric matrix over time.  Within this framework, the three issues listed above are solved as follows.
For i), inter-dimensional correlation is captured by the learned Mahalanobis distance metric matrix $\mathbf{M}$
such that $\mathbf{M}=\mathbf{L}^{\T}\mathbf{L}$, where $\mathbf{L}$ projects the feature vector to
a more discriminative feature space.
In other words, online metric learning aims to
find a linear mapping (i.e., a set of linear combinations
over the correlated feature elements from different feature dimensions),
which projects the original samples to
a more discriminative feature space for robust visual tracking.
We compare two different metric learning methods, one of which uses structured learning while the other is based only on pairwise sample proximity.
For ii), we design a time-weighted reservoir sampling method to maintain and update limited-sized sample buffers in the metric learning procedure.
In addition,
we prove
that metric learning based on our reservoir sampling method is statistically close to metric learning using all observed training samples.
For iii), we pose the calculation of $\textbf{c}$ as a least-square optimization problem, which admits an extremely simple and efficient closed-form solution.
We also demonstrate that, with the emergence of new data, the solution can be efficiently updated by a sequence of simple matrix operations.

Therefore, the main contributions of this work are two-fold:
1) We propose a novel
online metric-weighted linear representation for visual tracking.
The linear representation is associated with
a metric-weighted least-square optimization problem, which admits
an extremely simple and efficient closed-form solution.
The metric used in the linear representation is updated online
in a max-margin optimization framework
using proximity comparison or structured learning.
Different from the similarity metric learning developed in~\cite{chechik2010large}, our work
is formulated as a max-margin optimization problem for
learning a Mahalanobis distance metric.
Furthermore, we introduce the idea of
structured learning  to
the online metric learning process, which is also novel in the visual tracking literature.
2) We design a time-weighted reservoir sampling method to maintain and update
limited-sized sample buffers in the linear representation. The method is able to
effectively maintain sample buffers
that not only retain some old samples to avoid tracker drift,
but also adapt to recent changes.
In addition, we theoretically prove
that metric learning based on our reservoir sampling method is
statistically close to
metric learning using all available training samples.
{\em  This is the first time that reservoir sampling is used
in an  online metric learning setting that is tailored for robust
visual tracking.}

We note that, if the template samples all represent the same object, this same procedure can be used to identify the object in the presence of multiple objects.
The goal of object identification is naturally achieved
by combining the tracking information
and the linear regression-based visual
recognition together. We obtain promising results
of pedestrian identification on multi-view
video sequences.

Compared to previous systems, we fully exploit the linear representation of object appearance in a consistent and principled manner. By using online metric learning, our similarity measure is better maintained despite changing conditions. Reservoir sampling allows us to represent object appearance over time as a linear combination of samples. Our incremental solution update allows rapid update of object appearance coefficients, either for object tracking or identification.

\section{Related work}

Our work builds on recent progress in several related fields:
i) linear representations; ii) distance metric learning;
iii) reservoir sampling; and iv) structured tracking. We give a brief overview of the most relevant work in each of these areas.

{\bf Linear representations} Mei and Ling~\cite{Meo-Ling-ICCV09} propose a tracker
based on a sparse linear representation
obtained by  solving an $\ell_{1}$-regularized minimization problem.
With the sparsity constraint, this tracker can adaptively select a small number of
relevant templates to optimally approximate the given test samples.
The main limitation is its computational expense due to solving
an $\ell_{1}$-norm convex problem.
To speed up the tracking,
Bao \emph{et al.}~\cite{bao2012real} take advantage of a fast numerical solver
called  accelerated proximal gradient, which
solves the $\ell_{1}$-regularized minimization problem
with guaranteed quadratic convergence.
An alternative is to solve the $\ell_{1}$-regularized minimization problem
in an approximate way.
For example,
Li \emph{et al.}~\cite{Li-Shen-Shi-cvpr2011}
propose to  approximately solve the sparse
optimization problem using orthogonal matching pursuit (OMP).
Recently,
research has revealed that the $ \ell_1 $-norm induced sparsity
does not necessarily help improve the accuracy of image classification;
and non-sparse representations are typically orders of
magnitude faster to compute than their sparse counterparts, with
competitive
accuracy
\cite{rigamonti2011sparse,FaceICCV2011}.
Subsequently, Li \emph{et al.}~\cite{li2012incremental}
propose a 3D discrete cosine transform based
multilinear representation for visual tracking.
The representation models the spatio-temporal
properties of object appearance from
the perspective of signal compression and reconstruction.
However, the above trackers
construct linear regressors that are defined on independent feature dimensions
(mutually independent raw pixels in both \cite{Meo-Ling-ICCV09}
and \cite{Li-Shen-Shi-cvpr2011}). In other words, the correlation
information between different feature dimensions is not exploited.
Such correlation information
can play an important role in robust visual tracking.

{\bf Distance metric learning}
The goal of distance metric learning is to seek
an effective and discriminative metric space, where
both intra-class compactness
and inter-class separability are maximized.
In general, distance metric learning~\cite{weinberger2006distance,mensink2012metric,kedem2012non,ying2012distance}
is a popular and powerful tool for many applications.
For example, in \cite{weinberger2006distance}, a
Mahalanobis distance metric is learned using positive semidefinite
programming.
Chechik \emph{et al.} \cite{chechik2010large} propose
a cosine similarity metric learning method using proximity
comparison for large-scale image retrieval.
Discriminative metric learning has also been successfully applied
to visual tracking
\cite{wang2010discriminative,jiang2011adaptive,tsagkatakisonline_csvt}.
These works learn a distance
metric mainly for object matching across adjacent frames, and
the tracking is not carried out in the framework of linear
representations.
In addition, Hong \emph{et al.}~\cite{hong2012dual}
learn a discriminative distance metric
in a max-margin framework, where the average
inter-class distance is maximized while minimizing
the average intra-class distance.
The distance metric learning approach is implemented
in a batch mode learning scheme, which does not allow for online updating required for visual tracking.

{\bf Reservoir sampling} Visual tracking is a time-varying process which deals with a
        dynamic stream data in an online manner.
        Due to memory limitations,
        it is often impractical for trackers to store all the video stream data.
        To address this issue, reservoir sampling is a means of
        maintaining and updating limited-sized
        data buffers. However, the conventional reservoir
        sampling in \cite{vitter1985random, zhaoICML2011} can only
        accomplish the task of uniform random sampling, which assumes all samples are equally important.
        Due to temporal coherence, visual tracking in the current frame usually
        relies more on recently received samples than old samples.
        Hence, time-weighted reservoir sampling is required
        for robust visual tracking.

{\bf Structured tracking} The objective of
structured tracking is to utilize the
intrinsic structural information on object appearance
for robust object tracking. For instance,
Jia \emph{et al.}~\cite{jia2012visual}
construct a structured
sparse appearance model, which performs
alignment pooling on the sparse
 coefficient vectors for
local patches within the object.
Similar to~\cite{jia2012visual},
Zhong \emph{et al.}~\cite{zhong2012robust}
also propose a structured appearance model
that carries out average pooling
on the sparse
 coefficient vectors
for local image patches within the object.
Likewise, Li \emph{et al.}~\cite{li2013visual} present a local block-division
appearance model that comprises a set of
block-specific SVM classification models.
The Dempster-Shafer evidence theory is further used to
fuse the block-specific SVM discriminative information
for object localization.
In addition, structured output learning (e.g., structured SVM~\cite{harestruck_iccv2011}) is applied
to visual tracking. Its key idea is to learn a classification model in
a max-margin optimization framework, which involves an infinite number of
constraints containing structured information (e.g., VOC overlap score~\cite{harestruck_iccv2011}).

{\bf Tracking and identification}
Recent studies have demonstrated the effectiveness
of combining object identification and
tracking together.
For example, Edwards \emph{et al.}~\cite{edwards1999improving}
present an adaptive framework that improves
the performance of face tracking and recognition
by adaptively combining the motion information
from the video sequence.
Following this work, Zhou \emph{et al.}~\cite{zhou2004visual}
study the problem of simultaneous tracking and recognizing human
faces in a particle filter framework.
Moreover, Mei and Ling~\cite{Meo-Ling-ICCV09}
formulate simultaneous
vehicle tracking and identification
as a $\ell_{1}$-regularized sparse representation problem.

Compared with the previous work on appearance modeling,
the advantages of this work are as follows.
First, this work constructs a simple but effective
appearance model based on a metric-weighted
linear regression problem, which admits
an extremely simple and efficient closed-form solution
with online updating.
Second, this work naturally embeds useful discriminative information
within the process of
appearance modeling using online metric learning,
which finds an effective metric space
(obtained by discriminative linear mappings)
for metric-weighted
linear regression.
Third, this work effectively maintains
limited-sized sample buffers for
online metric learning by
time-weighted reservoir sampling.
The maintained buffers can not only
retain some old samples with a long lifespan
for avoiding the tracker drift, but also
adapt to recent appearance changes.
A preliminary conference version of this work appears in~\cite{li2012non}.

\vspace{-0.45cm}
\section{The proposed visual tracking algorithm}
\vspace{-0.1cm}
\subsection{Particle filtering for tracking}
 \label{sec:Motion_model}
\vspace{-0.1cm}
At the top level, visual tracking is posed as
a sequential object state estimation problem,
which is often solved in a particle filtering framework~\cite{Isard-Blake-ECCV1996}.
The particle filter
can be divided into prediction and update steps: \vspace{-0.28cm}
\[
 p(\mathbf{Z}_{t}\hspace{-0.05cm}\mid \hspace{-0.05cm}\mathcal
{O}_{t-1} )\hspace{-0.0cm}\propto \int
\hspace{-0.1cm}p(\mathbf{Z}_{t}\hspace{-0.1cm}\mid
\hspace{-0.1cm}\mathbf{Z}_{t-1})p(\mathbf{Z}_{t-1}\hspace{-0.1cm}\mid \hspace{-0.1cm}
\mathcal {O}_{t-1} )d\mathbf{Z}_{t-1}, \hspace{0.25cm}
p(\mathbf{Z}_{t}\hspace{-0.1cm}\mid \hspace{-0.1cm}\mathcal
{O}_{t} )\hspace{-0.0cm}\propto\hspace{-0.0cm}
p(\mathbf{o}_{t}\hspace{-0.1cm}\mid \hspace{-0.1cm}\mathbf{Z}_{t})p(\mathbf{Z}_{t}\hspace{-0.1cm}\mid \hspace{-0.1cm}\mathcal
{O}_{t-1} ),\hspace{-0.0cm} \vspace{-0.21cm}
\]
where $\mathcal
{O}_{t}=\{\mathbf{o}_{1},\ldots,\mathbf{o}_{t}\}$ are observation variables,
and $\mathbf{Z}_{t}= (\mathcal{X}_{t}, \mathcal{Y}_{t}, \mathcal{S}_{t})$
denotes the motion parameters including
$\mathcal{X}$ translation, $\mathcal{Y}$ translation, and scaling.
The key distributions are
$p(\mathbf{o}_{t}\hspace{-0.1cm}\mid \hspace{-0.1cm}\mathbf{Z}_{t})$ denoting the
observation model, and $p(\mathbf{Z}_{t}\hspace{-0.1cm}\mid
\hspace{-0.1cm}\mathbf{Z}_{t-1})$ representing the state transition model.
Usually, the motion between two consecutive frames is assumed to conform to a Gaussian distribution:
$p(\mathbf{Z}_{t}|\mathbf{Z}_{t-1}) = \mathcal{N}(\mathbf{Z}_{t}; \mathbf{Z}_{t-1}, \Sigma)$,
where $\Sigma$ denotes a diagonal covariance matrix with diagonal
elements: $\sigma_{\mathcal{X}}^{2}$, $\sigma_{\mathcal{Y}}^{2}$, and $\sigma_{\mathcal{S}}^{2}$.
For each state $\mathbf{Z}_{t}$, there is a corresponding image region $\mathbf{o}_{t}$ that is normalized  by image scaling.
The optimal object state $\mathbf{Z}_{t}^{\ast}$ at time $t$ can be determined by solving the following maximum a posterior (MAP) problem:
$\mathbf{Z}_{t}^{\ast} = \arg \max_{\mathbf{Z}_{t}} \thinspace p(\mathbf{Z}_{t}\hspace{-0.1cm}\mid \hspace{-0.1cm}\mathcal
{O}_{t})$. Therefore, efficiently constructing an effective observation model $p(\mathbf{o}_{t}\hspace{-0.1cm}\mid \hspace{-0.1cm}\mathbf{Z}_{t})$
plays a critical role in robust visual tracking. Motivated by this observation,
we design a metric-weighted linear representation that
captures the intrinsic object appearance properties in a discriminative distance metric space.

\begin{figure}[t]
\vspace{-0.4cm}
\includegraphics[scale=0.155]{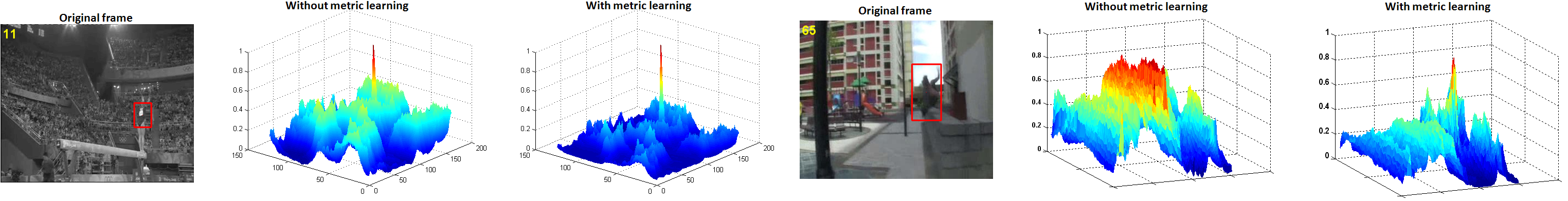}
\vspace{-0.8cm}
 \caption{\footnotesize Illustration of our discriminative criterion based on metric-weighted linear representation. The first column displays
 the original frames; the second column shows the corresponding confidence maps without metric learning (i.e., $\mathbf{M}$ is an identity matrix); and the third column exhibits
 the corresponding confidence maps with metric learning. Clearly, our metric-weighted criterion is more discriminative.}
 \label{fig:confidence_map} \vspace{-1.1cm}
\end{figure}

\subsection{Problem formulation}
\label{sec:malso}

Modeling the observed appearance of an object $p(\mathbf{o}_{t}\hspace{-0.1cm}\mid \hspace{-0.1cm}\mathbf{Z}_{t})$ is more complex than modelling its motion.
This is often posed as
a problem of linear representation and
reconstruction, which corresponds
to a $\ell_p$-norm regularized
least-square optimization problem (e.g., solved in~\cite{Meo-Ling-ICCV09, Li-Shen-Shi-cvpr2011}).
These optimization problems usually
ignore the relative importance of
individual feature dimensions as well as
the correlation between feature dimensions.
During tracking, such information
plays a critical role in
robust object/non-object classification
with complicated appearance variations.
Motivated by this observation, we propose a
metric-weighted linear representation
that is
capable of capturing the varying correlation information between
feature dimensions. As shown in Fig.~\ref{fig:confidence_map},
metric learning results in a linear representation that is more discriminative
for object/non-object classification.

{\bf Metric-weighted linear representation.}
More specifically, given a set of basis samples $\mathbf{P} =
(\mathbf{p_{i}})_{i=1}^{N} \in \mathcal{R}^{d\times N}$ and a test
sample $\mathbf{y} \in \mathcal{R}^{d\times 1}$, we aim to discover
a linear combination of $\mathbf{P}$ to optimally approximate the test
sample $\mathbf{y}$ by solving the following optimization problem: \vspace{-0.26cm}
\begin{equation}
 \underset{\mathbf{x}}{\min} \thinspace g(\mathbf{x}; \mathbf{M},
\mathbf{P}, \mathbf{y})= \underset{\mathbf{x}}{\min} \thinspace
(\mathbf{y}-\mathbf{P}\mathbf{x})^{\T}\mathbf{M}
(\mathbf{y}-\mathbf{P}\mathbf{x}),
\label{eq:linear_regression_metric} \vspace{-0.3cm}
\end{equation}
where $\mathbf{x} \in \mathcal{R}^{N \times 1}$ and
$\mathbf{M}$ is a symmetric distance metric matrix
that can be decomposed as $\mathbf{M}=\mathbf{L}^{\T}\mathbf{L}$.
In principle, the idea of the metric-weighted linear representation is to linearly reconstruct
the given test sample $\mathbf{y}$ using the basis samples  $(\mathbf{p_{i}})_{i=1}^{N}$
within a distance metric space (characterized by the Mahalanobis metric matrix $\mathbf{M}$).
The aforementioned linear regression problem is equivalent to the following form:
$\underset{\mathbf{x}}{\min} \thinspace
(\mathbf{L}\mathbf{y}-\mathbf{L}\mathbf{P}\mathbf{x})^{\T}
(\mathbf{L}\mathbf{y}-\mathbf{L}\mathbf{P}\mathbf{x})$. In other words, we perform the linear reconstruction task
on the transformed sample $\mathbf{L}\mathbf{y}$ with respect to the transformed basis samples
$(\mathbf{L}\mathbf{p_{i}})_{i=1}^{N}$. When $\mathbf{L}$ is an identity matrix, our metric-weighted regression
problem degenerates to a standard least square regression problem.

The optimization problem~\eqref{eq:linear_regression_metric} has an analytical solution that can
be  computed as: \vspace{-0.25cm}
\begin{equation}
\mathbf{x}^{\ast} =
(\mathbf{P}^{\T}\mathbf{M}\mathbf{P})^{-1}\mathbf{P}^{\T}\mathbf{M}\mathbf{y}.
\label{eq:regression_optimal_solution} \vspace{-0.25cm}
\end{equation}
If $\mathbf{P}^{T}\mathbf{M}\mathbf{P}$ is a singular matrix, we
use its pseudoinverse to compute $\mathbf{x}^{\ast}$.

{\bf Tracking application.}
During tracking, we typically want to classify a candidate sample as either foreground or background. We are therefore interested in the relative similarity of the sample to a set of foreground and background samples.
To address this problem, we
obtain the foreground and background linear regression
solutions as follows:
$\mathbf{x}^{\ast}_{f} = \arg \min_{\mathbf{x}_{f}}  g(\mathbf{x}_{f};
\mathbf{M}, \mathbf{P}_{f}, \mathbf{y})$ and
$\mathbf{x}^{\ast}_{b} = \arg \min_{\mathbf{x}_{b}}  g(\mathbf{x}_{b};
\mathbf{M}, \mathbf{P}_{b}, \mathbf{y})$,
where $\mathbf{P}_{f}$ and $\mathbf{P}_{b}$ are foreground and
background basis samples, respectively.

Thus, we can define a discriminative criterion for measuring the
similarity of the test sample $\mathbf{y}$ to the foreground
class: \vspace{-0.26cm}
\begin{equation}
\mathcal{S}(\mathbf{y}) \hspace{-0.06cm} =
\hspace{-0.06cm}\sigma\left[\exp(-\theta_{f}/\gamma_{f}) -
\rho \exp(-\theta_{b}/\gamma_{b})\right],
\label{eq:particle_liki_model} \vspace{-0.26cm}
\end{equation}
where $\gamma_{f}$ and $\gamma_{b}$ are two scaling factors,
$\theta_{f} $ = $ g(\mathbf{x}_{f}^{\ast}; $ $ \mathbf{M},
\mathbf{P}_{f}, \mathbf{y})$, $\theta_{b} =
g(\mathbf{x}_{b}^{\ast}; \mathbf{M}, \mathbf{P}_{b},
\mathbf{y})$,
$\rho$ is a trade-off control factor, and $\sigma[\cdot]$ is the
sigmoid function.
Here, the term $\exp(-\theta_{f}/\gamma_{f})$ reflects the reconstruction
similarity relative to the foreground class, while $\exp(-\theta_{b}/\gamma_{b})$ determines the
similarity with the background class. Greater $\exp(-\theta_{f}/\gamma_{f})$ with a smaller
$\exp(-\theta_{b}/\gamma_{b})$ indicates a stronger confidence for foreground prediction.

During tracking, the similarity score
$\mathcal{S}(\cdot)$ is associated with the observation model
of the particle filter such that
$p(\mathbf{o}_{t}\hspace{-0.1cm}\mid \hspace{-0.1cm}\mathbf{Z}_{t}) \propto \mathcal{S}(\mathbf{o}_{t})$.

Implementing this framework involves three main challenges: i) maintaining a representative pool of foreground and background samples; ii) efficiently updating the solution when foreground or background samples are updated; iii) learning and updating the metric matrix $\mathbf{M}$. These are addressed in the following four sections.

\begin{algorithm}[t]
\begin{minipage}[ctb]{16cm}
\footnotesize
\caption{Metric-weighted linear representation}
\label{alg:online_linear_regression}
\KwIn
{
 The current distance metric matrix $\mathbf{M}$, the basis samples
$\mathbf{P} = (\mathbf{p_{i}})_{i=1}^{N} \in \mathcal{R}^{d\times N}$,
any test sample $\mathbf{y} \in \mathcal{R}^{d\times 1}$.
}
\KwOut
{
 The optimal linear representation solution $\mathbf{x}^{\ast}$ of sample $\mathbf{y}$.
}
\begin{enumerate}\itemsep=-0pt

   \item Build the optimization problem in
 Equ.~\eqref{eq:linear_regression_metric}:
        \[\underset{\mathbf{x}}{\min} \thinspace g(\mathbf{x};
 \mathbf{P}, \mathbf{y})= \underset{\mathbf{x}}{\min} \thinspace
 (\mathbf{y}-\mathbf{P}\mathbf{x})^{T}\mathbf{M}
 (\mathbf{y}-\mathbf{P}\mathbf{x})
         \vspace{-0.15cm}\]
   \item Compute the optimal solution $\mathbf{x}^{\ast} =
 (\mathbf{P}^{T}\mathbf{M}\mathbf{P})^{-1}\mathbf{P}^{T}\mathbf{M}\mathbf{y}$.
 If samples are added to or removed from $\mathbf{P}$,
$(\mathbf{P}^{T}\mathbf{M}\mathbf{P})^{-1}$
         can be efficiently updated in an online manner:
         \begin{itemize} \itemsep=-0pt
         \item Use Equ.~\eqref{eq:incremental} to compute the
 incremental inverse.
         \item Use Equ.~\eqref{eq:decremental} to calculate the
 decremental inverse.
         \item Obtain the replacement inverse based on the incremental and decremantal inverses.
         \end{itemize}
   \item Return the optimal  solution $\mathbf{x}^{\ast}$.
\end{enumerate}
\end{minipage}
\end{algorithm}

\vspace{-0.3cm}
\subsection{Online solution update}

The main computational time of
Equ.~\eqref{eq:regression_optimal_solution} is spent on the calculation
of
$(\mathbf{P}^{T}\mathbf{M}\mathbf{P})^{-1}$.
For computational efficiency, we need to incrementally or decrementally
update the inverse when a sample is added to or removed from $\mathbf{P}$, for a fixed
metric $\mathbf{M}$.
Motivated by this observation,
we design an online update scheme that deals with the following three cases:
1) the basis sample matrix $\mathbf{P} $ is incrementally expanded by one column
such that $\mathbf{P_{n}} = (\mathbf{P} \thickspace \Delta \mathbf{p})$;
2) the basis sample matrix $\mathbf{P} $ is decrementally reduced by one column
such that $\mathbf{P_{o}}$ is the reduced matrix after removing the
$i$-th column of $\mathbf{P}$; and 3) one column of $\mathbf{P}$ is replaced by
a new sample.

{\bf Incremental case.} Let $\mathbf{P_{n}} = (\mathbf{P} \thickspace \Delta \mathbf{p})$
denote the expanded matrix of $\mathbf{P}$.
Clearly, the following relation holds:
\[
(\mathbf{P_{n}})^{T}\mathbf{M}\mathbf{P_{n}}
=\left(
\begin{matrix}
\mathbf{P}^{T}\mathbf{M}\mathbf{P} & \mathbf{P}^{T}\mathbf{M} \Delta
\mathbf{p}\\
(\Delta \mathbf{p})^{T}\mathbf{M} \mathbf{P} & (\Delta
\mathbf{p})^{T}\mathbf{M} \Delta \mathbf{p}
\end{matrix}
\right).
\]
For simplicity, let $\mathbf{H}=(\mathbf{P}^{T}\mathbf{M}\mathbf{P})^{-1}$,
$\mathbf{c} = \mathbf{P}^{T}\mathbf{M} \Delta \mathbf{p}$, and $r =
(\Delta \mathbf{p})^{T}\mathbf{M} \Delta \mathbf{p}$.
Since $\mathbf{M}$ is a symmetric matrix,
$\mathbf{c}^{T} = (\Delta \mathbf{p})^{T}\mathbf{M} \mathbf{P}$.
The inverse of
 $(\mathbf{P_{n}})^{T}\mathbf{M}\mathbf{P_{n}}$ can be computed
as~\cite{jennings1992matrix}:
\begin{equation}
((\mathbf{P_{n}})^{T}\mathbf{M}\mathbf{P_{n}})^{-1}
=\left(
\begin{matrix}
\mathbf{H}+\frac{\mathbf{H}\mathbf{c}\mathbf{c}^{T}\mathbf{H}}{r-\mathbf{c}^{T}\mathbf{H}\mathbf{c}}
& -\frac{\mathbf{H}\mathbf{c}}{r-\mathbf{c}^{T}\mathbf{H}\mathbf{c}}\\
-\frac{\mathbf{c}^{T}\mathbf{H}}{r-\mathbf{c}^{T}\mathbf{H}\mathbf{c}}
& \frac{1}{r-\mathbf{c}^{T}\mathbf{H}\mathbf{c}}
\end{matrix}
\right).
\label{eq:incremental}
\end{equation}

{\bf Decremental case.} Let
$\mathbf{P_{o}}$ denote the reduced matrix of $\mathbf{P}$ after
removing the $i$-th column such that $1\leq i \leq N$.
Based on~\cite{jennings1992matrix}, the inverse of
$(\mathbf{P_{o}})^{T}\mathbf{M}\mathbf{P_{o}}$ can be computed
as: \vspace{-0.082cm}
\begin{equation}
((\mathbf{P_{o}})^{T}\mathbf{M}\mathbf{P_{o}})^{-1}
=\mathbf{H}(\mathcal{I}_{i},\mathcal{I}_{i})- \frac{\mathbf{H}(\mathcal{I}_{i},
i)\mathbf{H}(i,\mathcal{I}_{i})}{\mathbf{H}(i,i)},\\
\label{eq:decremental} \vspace{-0.082cm}
\end{equation}
where $\mathcal{I}_{i}=\{1, 2, \ldots, N\}\backslash{\{i\}}$ stands for
the index set except $i$.

{\bf Replacement case.} For adapting to object appearance changes, it is necessary for
trackers to replace an old sample from the buffer with a new
sample.
Sample replacement is implemented in two
stages: 1) old sample removal; and 2) new sample addition, corresponding to the decremental and incremental cases described above.

The complete optimization procedure including online sample update
is summarized in
Algorithm~\ref{alg:online_linear_regression}.

\begin{algorithm}[t]
\begin{minipage}[ctb]{16cm}
\footnotesize
\caption{Online metric learning using proximity comparison}
\label{alg:online_metric_learning}
\KwIn
{
 The current distance metric matrix $\mathbf{M}^{k}$ and a new
triplet $(\mathbf{p}, \mathbf{p}^{+}, \mathbf{p}^{-})$.
}

\KwOut
{
 The updated distance metric matrix $\mathbf{M}^{k+1}$.
}
\begin{enumerate}\itemsep=0pt
  \item Calculate $\mathbf{a}_{+} = \mathbf{p} - \mathbf{p}^{+}$ and
$\mathbf{a}_{-} = \mathbf{p} - \mathbf{p}^{-}$
  \item Compute the optimal step length
  $\eta \hspace{-0.08cm}= \hspace{-0.08cm}\min\left\{C, \max \left\{0,\frac{1 + \mathbf{a}_{+}^{T}\mathbf{M}^{k}\mathbf{a}_{+} -
\mathbf{a}_{-}^{T}\mathbf{M}^{k}\mathbf{a}_{-}}{2\mathbf{a}_{-}^{T}\mathbf{U}\mathbf{a}_{-} \hspace{-0.02cm}- \hspace{-0.02cm}
2\mathbf{a}_{+}^{T}\mathbf{U}\mathbf{a}_{+} \hspace{-0.02cm}- \hspace{-0.02cm}\|\mathbf{U}\|_{F}^{2}}\right\}\hspace{-0.1cm}\right\}$
  with $\mathbf{U} = \mathbf{a}_{-}\mathbf{a}_{-}^{T} -
\mathbf{a}_{+}\mathbf{a}_{+}^{T}$.
  \item $\mathbf{M}^{k+1} \leftarrow \mathbf{M}^{k} + \eta
(\mathbf{a}_{-}\mathbf{a}_{-}^{T} -
\mathbf{a}_{+}\mathbf{a}_{+}^{T})$.
\end{enumerate}
\end{minipage}
\end{algorithm}

\begin{algorithm}[t]
\begin{minipage}[ctb]{16cm}
\footnotesize
\caption{Online structured distance metric learning}
\label{alg:online_metric_learning_structure}
\KwIn
{
 The current distance metric matrix $\mathbf{M}^{k}$,
the current tracking bounding box $\mathbf{R}_{t}$ and its
associated feature vector $\mathbf{p}_{t}$.
}

\KwOut
{
 The updated distance metric matrix $\mathbf{M}^{k+1}$.
}
\begin{enumerate}\itemsep=-0pt
\item Most violated constraint set $\mathcal{P}= \emptyset$
\item Sample a number of
         bounding boxes around $\mathbf{R}_{t}$ to construct a constraint set for the optimization problem~\eqref{eq:online_optimization_function_structure}.
\item Compute the most violated constraint $(\mu, \nu)$
as shown in Equ.~\eqref{eq:most_violated_constraints}.
\item Add $(\mathbf{p}_{t}^{\mu}\circ \mathbf{R}_{t}^{\mu}, \mathbf{p}_{t}^{\nu}\circ \mathbf{R}_{t}^{\nu})$ to
the most violated constraint set such that $\mathcal{P} \leftarrow \mathcal{P} \bigcup \{(\mathbf{p}_{t}^{\mu}\circ \mathbf{R}_{t}^{\mu}, \mathbf{p}_{t}^{\nu}\circ \mathbf{R}_{t}^{\nu})\}$.
\item Solve the optimization problem~\eqref{eq:eta_linear_solver} to obtain the optimal step length vector $\bm{\eta}^{\ast}$.
\item Compute the updated metric matrix $\mathbf{M}$ according to Sec.~\ref{sec:oml}.
\item Repeat Steps 3--6 until convergence (restricted by a maximum iteration number).
\item Return $\mathbf{M}^{k+1} \leftarrow \mathbf{M}$.
\end{enumerate}
\end{minipage}
\end{algorithm}

\vspace{-0.25cm}
\subsection{Online proximity based metric learning}
\label{sec:oml}

Having introduced the metric-weighted linear representation  in Sec.~\ref{sec:malso}, we now address the key issue of calculating the metric matrix $\mathbf{M}$.
$\bf M$ should ideally be learned from the visual data, and should be dynamically updated as conditions change throughout a video sequence.

\subsubsection{\bf Triplet-based ranking loss}
   Suppose that we have a set of sample triplets
   $\{(\mathbf{p}, \mathbf{p}^{+},\mathbf{p}^{-})\}$
    with
    $ \mathbf{p},
    \mathbf{p}^{+},
    \mathbf{p}^{-} \in \mathcal{R}^{d}
    $.  These triplets encode the proximity comparison information.
   In each triplet, the distance between $  \bf p $
   and $ {\bf p}^+  $ should be smaller than the distance between
   $ \bf p $ and $ {\bf p}^-  $.

The Mahalanobis distance under  metric $ \bf M$ is defined as:
\vspace{-0.153cm}
\begin{equation}
D_{\mathbf{M}}(\mathbf{p}, \mathbf{q}) = (\mathbf{p}
-\mathbf{q})^{T}\mathbf{M}(\mathbf{p}-\mathbf{q}). \vspace{-0.126cm}
\end{equation}
Clearly, $ \bf M $ must be a symmetric and positive semidefinite
matrix. It is equivalent to learn a projection matrix $ \bf L $ such
that $ {\bf M} = {\bf L} {\bf L}^T$.
In practice,
   we generate the triplets set as:
   $ \bp $ and $ \bp^+$ belong to the same class
   and $ \bp $ and $ \bp^-$ belong to different classes.
So we want the constraints
$ D_\bM ( \bp, \bp^+ ) < D_\bM ( \bp, \bp^- ) $ to be satisfied as
well as possible.
By putting it into a large-margin learning framework, and using the
soft-margin hinge loss,
the loss function for each triplet is: \vspace{-0.016cm}
\begin{equation}
l_{\mathbf{M}}(\mathbf{p}, \mathbf{p}^{+}, \mathbf{p}^{-}) =
       \max\{0, 1 + D_{\mathbf{M}}(\mathbf{p}, \mathbf{p}^{+}) -
D_{\mathbf{M}}(\mathbf{p}, \mathbf{p}^{-})\}. \vspace{-0.01cm}
\label{eq:local_hinge_loss}
\end{equation}

\subsubsection{\bf Large-margin metric learning}
To obtain the optimal distance metric matrix $\mathbf{M}$, we need to
minimize the global loss $L_{\mathbf{M}}$
that takes the sum of hinge losses~\eqref{eq:local_hinge_loss} over
all possible triplets from the training
set: \vspace{-0.21cm}
\begin{equation}
L_{\mathbf{M}} = \underset{(\mathbf{p}, \mathbf{p}^{+},
\mathbf{p}^{-})\in \mathcal{Q}}{\sum} l_{\mathbf{M}}(\mathbf{p},
\mathbf{p}^{+}, \mathbf{p}^{-}), \vspace{-0.21cm}
\label{eq:total_loss}
\end{equation}
where $\mathcal{Q}$ is the triplet set.
   To sequentially optimize the above objective function $L_{\mathbf{M}}$
   in an online fashion,
   we design an iterative algorithm to solve the following
   convex problem:
\begin{equation}
\begin{array}{l}
\mathbf{M}^{k+1} = \underset{\mathbf{M}}{\arg\min}
\frac{1}{2}\|\mathbf{M}-\mathbf{M}^{k}\|^{2}_{F} + C\xi,\\
\mbox{s.t.} \thickspace  D_{\mathbf{M}}(\mathbf{p}, \mathbf{p}^{-})
- D_{\mathbf{M}}(\mathbf{p}, \mathbf{p}^{+})  \geq 1 -  \xi,
\thickspace \xi \geq  0,
\end{array}
\label{eq:online_optimization_function}
\end{equation}
where $\|\cdot\|_{F}$ denotes the Frobenius norm, $\xi$ is a slack
variable, and $C$ is a positive factor controlling
the trade-off between the smoothness term
$\frac{1}{2}\|\mathbf{M}-\mathbf{M}^{k}\|^{2}_{F}$ and the loss  term
$\xi$.
Following the passive-aggressive mechanism used
in~\cite{chechik2010large, crammer2006online}, we only update the
metric matrix $\mathbf{M}$ when $l_{\mathbf{M}}(\mathbf{p},
\mathbf{p}^{+}, \mathbf{p}^{-})>0$.

\subsubsection{\bf Optimization of ${\mathbf M}$}
We optimize the function in Equ.~\eqref{eq:online_optimization_function}
with Lagrangian regularization:\vspace{-0.12cm}
\begin{equation}
\begin{array}{l}
\mathcal{L}(\mathbf{M}, \eta, \xi, \beta) =
\frac{1}{2}\|\mathbf{M}-\mathbf{M}^{k}\|^{2}_{F} + C\xi - \beta \xi
+ \eta(1-\xi
+
D_{\mathbf{M}}(\mathbf{p},
\mathbf{p}^{+})-D_{\mathbf{M}}(\mathbf{p}, \mathbf{p}^{-})),
\end{array}
\label{eq:Lagrange_loss} \vspace{-0.12cm}
\end{equation}
where $\eta\geq 0$ and $\beta \geq 0$ are Lagrange multipliers.
The optimization procedure is carried out in the following two alternating
steps.

\begin{itemize}
\item {\it Update $\mathbf{M}$}.
By setting
$ \frac{\partial\mathcal{L}(\mathbf{M}, \eta, \xi, \beta)}{\partial
\mathbf{M}} = 0$, we arrive at the update rule
\begin{equation}
\mathbf{M}^{k+1} = \mathbf{M}^{k} + \eta \mathbf{U}
\label{eq:eq:metric_update} \vspace{-0.12cm}
\end{equation}
where
$\mathbf{U} = \mathbf{a}_{-}\mathbf{a}_{-}^{T} - \mathbf{a}_{+}\mathbf{a}_{+}^{T}$ and $\mathbf{a}_{+} = \mathbf{p} - \mathbf{p}^{+}$,
$\mathbf{a}_{-} = \mathbf{p} - \mathbf{p}^{-}$.

\item {\it Update $\eta$}. Subsequently, we take the derivative of the
Lagrangian~\eqref{eq:Lagrange_loss} w.r.t. $\eta$ and set it
to zero, leading to the update rule:
\begin{equation}
\hspace{-0.0cm}
\eta = \min\left\{C, \max \left\{0,\frac{1 + \mathbf{a}_{+}^{T}\mathbf{M}^{k}\mathbf{a}_{+} -
\mathbf{a}_{-}^{T}\mathbf{M}^{k}\mathbf{a}_{-}}{2\mathbf{a}_{-}^{T}\mathbf{U}\mathbf{a}_{-} - 2\mathbf{a}_{+}^{T}\mathbf{U}\mathbf{a}_{+} - \|\mathbf{U}\|_{F}^{2}}\right\}\right\}
\label{eq:eta_final}
\end{equation}

\end{itemize}
The full derivation of each step can be found in the supplementary file (as shown in Sec.~\ref{sec:oml_supp}).
The complete procedure of online distance metric learning is
summarized in Algorithm~\ref{alg:online_metric_learning}.

\subsubsection{\bf  Online matrix inverse update}
When updated according to
Algorithm~\ref{alg:online_metric_learning}, $\mathbf{M}$ is modified
by rank-one
additions
such that $\mathbf{M}\longleftarrow \mathbf{M} + \eta
(\mathbf{a}_{-}\mathbf{a}_{-}^{T} -
\mathbf{a}_{+}\mathbf{a}_{+}^{T})$ where $\mathbf{a}_{+} = \mathbf{p} - \mathbf{p}^{+}$ and
$\mathbf{a}_{-} = \mathbf{p} - \mathbf{p}^{-}$ are
two vectors (defined in Equ.~\eqref{eq:eq:metric_update}) for triplet construction, and
$\eta$ is a step-size factor (defined in Equ.~\eqref{eq:eta_final}).
As a result, the original $\mathbf{P}^{T}\mathbf{M}\mathbf{P}$
becomes $\mathbf{P}^{T}\mathbf{M}\mathbf{P}
+ (\eta\mathbf{P}^{T}\mathbf{a}_{-})(\mathbf{P}^{T}\mathbf{a}_{-})^{T}
+  (-\eta\mathbf{P}^{T}\mathbf{a}_{+})(\mathbf{P}^{T}\mathbf{a}_{+})^{T}$.
When $\mathbf{M}$ is modified by a rank-one addition, the inverse of
$\mathbf{P}^{T}\mathbf{M}\mathbf{P}$ can be updated
according to the theory of~\cite{householder1964theory,powell1969theorem}:
\begin{equation}
(\mathbf{J}+\mathbf{u}\mathbf{v}^{T})^{-1} = \mathbf{J}^{-1} -
\frac{\mathbf{J}^{-1}\mathbf{u}\mathbf{v}^{T}\mathbf{J}^{-1}}{1+\mathbf{v}^{T}\mathbf{J}^{-1}\mathbf{u}}.
\label{eq:rank_one_update} \vspace{-0.02cm}
\end{equation}
Here, $\mathbf{J}=\mathbf{P}^{T}\mathbf{M}\mathbf{P}$, $\mathbf{u}=\eta\mathbf{P}^{T}\mathbf{a}_{-}$ (or $\mathbf{u}=-\eta\mathbf{P}^{T}\mathbf{a}_{+}$),
and $\mathbf{v}=\mathbf{P}^{T}\mathbf{a}_{-}$ (or $\mathbf{v}=\mathbf{P}^{T}\mathbf{a}_{+}$).

\vspace{-0.2cm}
\subsection{Online structured metric learning}

Metric learning based on sample proximity comparisons leads to an efficient online learning algorithm, but requires pre-defined sets of positive and negative samples. In tracking, these usually correspond to target/non-target image patches. The boundary between these classes typically occurs where sample overlap with the target drops below a threshold, but this can be difficult to evaluate exactly and thus introduces some noise into the algorithm.

In this Section, we replace the proximity based metric learning module with an online structured metric learning method for learning  $\bf M$. The main advantage of this method is that it directly learns the metric from measured sample overlap, and therefore does not require the separation of samples into positive and negative classes.

{\bf Structured ranking}
Let $\mathbf{p}_{t}$ and $\mathbf{p}_{t}^{i}$
 denote two feature vectors extracted from two image patches,
 which are respectively associated
 with two bounding boxes
 $\mathbf{R}_{t}$ and $\mathbf{R}_{t}^{i}$ from frame $t$.
 Without loss of generality, let us assume that
 $\mathbf{R}_{t}$  corresponds to the
 bounding box obtained by the current tracker
 while $\mathbf{R}_{t}^{i}$ is associated with
 a bounding box from the area surrounding $\mathbf{R}_{t}$.
 As in~\cite{harestruck_iccv2011}, the structural affinity relationship
 between $\mathbf{p}_{t}$ and $\mathbf{p}_{t}^{i}$
 is captured by the following overlap function: $s^{o}_{t}(\mathbf{R}_{t}, \mathbf{R}_{t}^{i}) = \frac{\mathbf{R}_{t}\bigcap\mathbf{R}_{t}^{i}}{\mathbf{R}_{t}\bigcup\mathbf{R}_{t}^{i}}.$
 As a result, we define the following optimization problem for structured metric learning:  \vspace{-0.18cm}
 \begin{equation}
 \begin{array}{l}
 \mathbf{M}^{k+1} = \underset{\mathbf{M}}{\arg\min}\thinspace
 \frac{1}{2}\|\mathbf{M}-\mathbf{M}^{k}\|^{2}_{F} + C\xi,\\
 \mbox{s.t.}
 \thickspace D_{\mathbf{M}}(\mathbf{p}_{t}, \mathbf{p}_{t}^{j}) - D_{\mathbf{M}}(\mathbf{p}_{t}, \mathbf{p}_{t}^{i})
  \geq
 \Delta_{ij}  -  \xi, \forall i,j\\
 \end{array}
 \label{eq:online_optimization_function_structure} \vspace{-0.18cm}
 \end{equation}
 where $\xi \geq  0$
 and  $\Delta_{ij} = s^{o}_{t}(\mathbf{R}_{t}, \mathbf{R}_{t}^{i}) - s^{o}_{t}(\mathbf{R}_{t}, \mathbf{R}_{t}^{j})$.
 Clearly, the number of constraints in the optimization problem~\eqref{eq:online_optimization_function_structure}
 is
 exponentially large or even
 infinite, making it difficult to optimize.
 Our approach to this optimization problem differs
 from~\cite{harestruck_iccv2011} in four main aspects:
 i) our approach aims to learn a distance metric while
 \cite{harestruck_iccv2011} seeks  a SVM classifier;
 ii)
 we optimize an online max-margin objective function
 while \cite{harestruck_iccv2011} solves a batch-mode
 optimization problem; iii) our optimization problem involves nonlinear constraints on
 triplet-based Mahalanobis distance differences, while the optimization problem in~\cite{harestruck_iccv2011}
 comprises linear constraints on doublet-based SVM classification score differences;
 and iv) our approach directly solves the primal optimization problem while \cite{harestruck_iccv2011}
 optimizes the dual problem.

{\bf Structured optimization}
 Inspired by the cutting-plane method,
 we iteratively construct a constraint set (denoted as $\mathcal{P}$)
 containing the most violated constraints
 for the optimization problem~\eqref{eq:online_optimization_function_structure}.
 In our case,
 the most violated constraint is selected according to
 the following criterion:  \vspace{-0.18cm}
 \begin{equation}
 (\mu, \nu)
 =
 \underset{
 (i, j)}{\arg\max} \hspace{0.2cm}
 \Delta_{ij} + D_{\mathbf{M}}(\mathbf{p}_{t}, \mathbf{p}_{t}^{i}) - D_{\mathbf{M}}(\mathbf{p}_{t}, \mathbf{p}_{t}^{j}),
 \label{eq:most_violated_constraints}  \vspace{-0.18cm}
 \end{equation}
 For notational simplicity, let $l_{\mathbf{M}}(\mathbf{p}_{t}\circ \mathbf{R}_{t}, \mathbf{p}_{t}^{j} \circ \mathbf{R}_{t}^{j},\mathbf{p}_{t}^{i} \circ \mathbf{R}_{t}^{i} )$ denote the loss term
 $\Delta_{ij} + D_{\mathbf{M}}(\mathbf{p}_{t}, \mathbf{p}_{t}^{i}) - D_{\mathbf{M}}(\mathbf{p}_{t}, \mathbf{p}_{t}^{j})$.
 Note that the violated constraints generated from~\eqref{eq:most_violated_constraints}
 are used if and only if $l_{\mathbf{M}}(\mathbf{p}_{t}\circ \mathbf{R}_{t}, \mathbf{p}_{t}^{j} \circ \mathbf{R}_{t}^{j},\mathbf{p}_{t}^{i} \circ \mathbf{R}_{t}^{i} )$ is greater than zero.
 Subsequently, we add the most violated constraint
 to the optimization problem~\eqref{eq:online_optimization_function_structure} in an iterative manner,
 that is, $\mathcal{P} \leftarrow \mathcal{P} \bigcup \{(\mathbf{p}_{t}^{\mu}\circ \mathbf{R}_{t}^{\mu}, \mathbf{p}_{t}^{\nu}\circ \mathbf{R}_{t}^{\nu})\}$.
 The corresponding Lagrangian is formulated as:  \vspace{-0.18cm}
 \begin{equation}
 \begin{array}{l}
 \mathcal{L} =
 \frac{1}{2}\|\mathbf{M}-\mathbf{M}^{k}\|^{2}_{F} +  (C - \beta)\xi
 + \sum_{\ell=1}^{|\mathcal{P}|}\eta_{\ell}[\Delta_{\mu_{\ell}\nu_{\ell}} - \xi
 +D_{\mathbf{M}}(\mathbf{p}_{t}, \mathbf{p}_{t}^{\mu_{\ell}}) - D_{\mathbf{M}}(\mathbf{p}_{t}, \mathbf{p}_{t}^{\nu_{\ell}})],
 \end{array}
 \label{eq:Lagrange_loss_structure}  \vspace{-0.18cm}
 \end{equation}
 where $\beta \geq 0$ and $\eta_{\ell}\geq 0$ are Lagrange multipliers.
 The optimization procedure is once again carried out in two alternating
steps:
 
 \begin{itemize}
 \item {\it Update $\mathbf{M}$.} 
 By setting $\frac{\partial\mathcal{L}}{\partial
 \mathbf{M}}$ to zero, we obtain an updated $\mathbf{M}$ defined as:
\begin{equation}
 \mathbf{M}^{k+1} = \mathbf{M}^{k} + \sum_{\ell=1}^{|\mathcal{P}|}\eta_{\ell}\mathbf{U}_{\ell}
\end{equation}
where 
$\mathbf{U}_{\ell} = \mathbf{a}_{t}^{\nu_{\ell}}(\mathbf{a}_{t}^{\nu_{\ell}})^{\T} - \mathbf{a}_{t}^{\mu_{\ell}}(\mathbf{a}_{t}^{\mu_{\ell}})^{\T}$, and $\mathbf{a}_{t}^{n}$ denotes $\mathbf{p}_{t} - \mathbf{p}_{t}^{n}$.

 \item {\it Update $\eta_{\ell}$.}
 To obtain the optimal solution for all Lagrange multipliers $\eta_{\ell}$, we set $\frac{\partial \mathcal{L}}{\partial \eta_{\ell}} = 0$ for all $\ell$, leading to the following optimization problem:  
 \begin{equation}
 \begin{array}{l}
 \bm{\eta}^{\ast} = \underset{\bm{\eta}}{\arg \min} \thinspace \|\mathbf{B}\bm{\eta} - \mathbf{f}\|_{1}, \hspace{0.15cm}
 \mbox{s.t.} \thinspace \bm{\eta} \succeq 0; \mathbf{1}^{\T}\bm{\eta} \leq C.
 \end{array}
 \label{eq:eta_linear_solver}  \vspace{-0.25cm}
 \end{equation}
where 
where $\bm{\eta} = (\eta_{1}, \eta_{2}, \ldots, \eta_{|\mathcal{P}|})^{\T}$,
 $\mathbf{f} = (f_{1}, f_{2}, \ldots, f_{|\mathcal{P}|})$ with
 $f_{\ell}$ being $-[\Delta_{\mu_{\ell}\nu_{\ell}} + (\mathbf{a}_{t}^{\mu_{\ell}})^{\T}\mathbf{M}^{k}\mathbf{a}_{t}^{\mu_{\ell}} - (\mathbf{a}_{t}^{\nu_{\ell}})^{\T}\mathbf{M}^{k}\mathbf{a}_{t}^{\nu_{\ell}}]$,
 and $\mathbf{B} = (b_{\ell m})_{|\mathcal{P}|\times |\mathcal{P}|}$ with
 $b_{\ell m}$ being
 $\mathbf{1}^{\T}(\mathbf{U}_{\ell} \circ \mathbf{U}_{m})\mathbf{1} + (\mathbf{a}_{t}^{\mu_{m}})^{\T}\mathbf{U}_{m}\mathbf{a}_{t}^{\mu_{m}}
 - (\mathbf{a}_{t}^{\nu_{m}})^{\T}\mathbf{U}_{m}\mathbf{a}_{t}^{\nu_{m}}
 + (\mathbf{a}_{t}^{\mu_{m}})^{\T}\mathbf{U}_{\ell}\mathbf{a}_{t}^{\mu_{m}}
 - (\mathbf{a}_{t}^{\nu_{m}})^{\T}\mathbf{U}_{\ell}\mathbf{a}_{t}^{\nu_{m}}$.
\end{itemize}

 As before, the optimal $\mathbf{M}$ is updated
 as a sequence of rank-one
 additions: $\mathbf{M}\longleftarrow \mathbf{M} + \eta_{\ell}[\mathbf{a}_{t}^{\nu_{\ell}}(\mathbf{a}_{t}^{\nu_{\ell}})^{\T} - \mathbf{a}_{t}^{\mu_{\ell}}(\mathbf{a}_{t}^{\mu_{\ell}})^{\T}]$.
 As a result, the original $\mathbf{P}^{\T}\mathbf{M}\mathbf{P}$
 becomes $\mathbf{P}^{\T}\mathbf{M}\mathbf{P}
 + (\eta_{\ell}\mathbf{P}^{\T}\mathbf{a}_{t}^{\nu_{\ell}})(\mathbf{P}^{\T}\mathbf{a}_{t}^{\nu_{\ell}})^{\T}
 +  (-\eta_{\ell}\mathbf{P}^{\T}\mathbf{a}_{t}^{\mu_{\ell}})(\mathbf{P}^{\T}\mathbf{a}_{t}^{\mu_{\ell}})^{\T}$.
 When $\mathbf{M}$ is modified by a rank-one addition, the inverse of
 $\mathbf{P}^{\T}\mathbf{M}\mathbf{P}$ can again be updated
 according to the theory of~\cite{householder1964theory,powell1969theorem}.

Algorithm~\ref{alg:online_metric_learning_structure}
outlines the procedure of online structured metric learning.
The complete derivation of these results is given in the supplementary file (as shown in Sec.~\ref{sec:structured_metric_learning_supp}).

\begin{figure*}[t]
 \vspace{-0.3cm}
\centering
\includegraphics[scale=0.6]{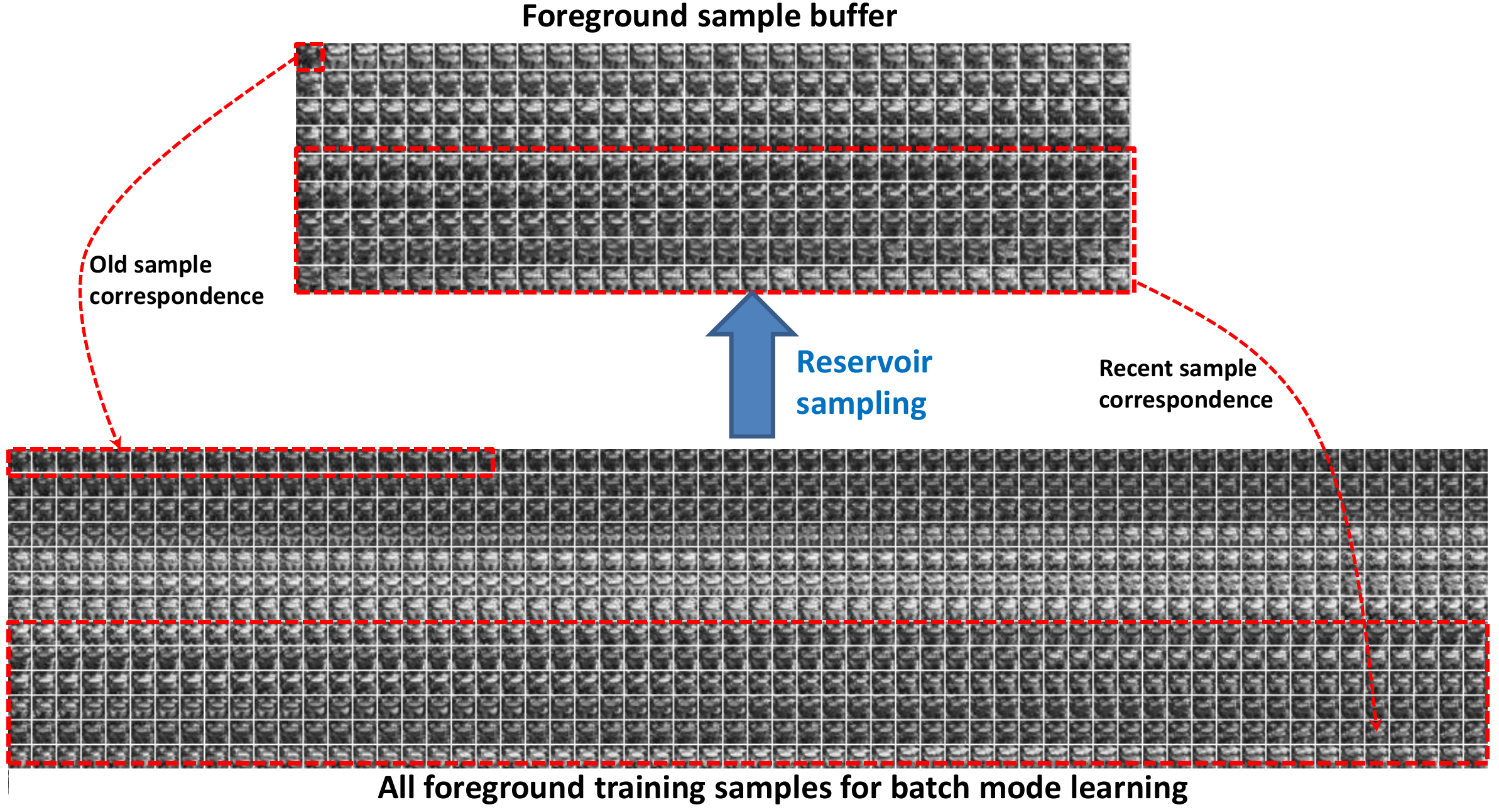}
\vspace{-0.0cm}
 \caption{Intuitive illustration of time-weighted reservoir sampling.
 The upper part corresponds to the foreground samples stored in the foreground buffer during tracking,
 and the lower part is associated with all the foreground samples collected in the entire
 tracking process. Clearly, time-weighted reservoir sampling encourages more recent samples to
 appear in the buffer, and meanwhile retain some old samples with a long lifespan.}
\label{fig:weighted_sample_example} \vspace{-0.5cm}
\end{figure*}

\subsection{Time-weighted reservoir sampling}
\label{sec:twrs}

In order to construct a metric-weighted
linear representation (referred to in Equ.~\eqref{eq:linear_regression_metric}) for visual tracking,
we need to learn a discriminative Mahalanobis metric matrix $\mathbf{M}$ by minimizing the total
ranking losses (referred to in Equ.~\eqref{eq:total_loss}) over a set of training triplets $\mathcal{Q}=  \{(\mathbf{p},
\mathbf{p}^{+}, \mathbf{p}^{-})\}$, which are generated from the training data (collected incrementally frame by frame) by
proximity comparison.
As tracking proceeds, the amount of collected training data increases, which  leads to an exponential growth
of the triplet set size (i.e., $|\mathcal{Q}|$). As a result, the optimization required for metric learning
(referred to in Equ.~\eqref{eq:total_loss}) becomes computationally intractable. %
To address this issue, a practical solution is to maintain a limited-sized buffer to store only selected
training triplets. However, using the training triplets from the limited-sized buffer (instead of all training data) for metric learning usually leads to
discriminative information loss. Therefore, how to effectively reduce such information loss is our focus in this work.

Inspired by the idea of reservoir sampling~\cite{vitter1985random,
zhaoICML2011, Kolonko04sequentialreservoir, efraimidis2006weighted} (i.e., sequential
random sampling for statistical learning),
we propose a sampling scheme to maintain and update the limited-sized buffer
while preserving the discriminative information on the ranking losses as much as possible.
Moreover, since the training data for tracking have to be collected frame by frame,
the limited-sized buffer needs to be updated sequentially.
Therefore, we seek a sequential sampling mechanism to online maintain and update the buffer,
in such a way that the ranking losses for metric learning is as close as possible to those using all the received training samples. Reservoir sampling
is one approach to this problem.

The classical version of
reservoir sampling simulates the process of
uniform random sampling~\cite{vitter1985random, zhaoICML2011} from a large population of sequential samples.
However, this is inappropriate for visual
tracking because the samples are dynamically
distributed as time progresses. Usually, recent samples should have
a greater influence on the current tracking process than those
appearing a long time ago. Therefore, larger weights should be assigned
to recent samples
while smaller weights should be attached to old samples.
Based on weighted reservoir sampling~\cite{Kolonko04sequentialreservoir,efraimidis2006weighted},
our sample scheme further takes into account the time-varying properties of
visual tracking, by incorporating time-related weight information into the weighted
reservoir sampling process.

More specifically, we design a time-weighted reservoir
sampling (TWRS)
method for randomly drawing samples according to their
time-varying properties, as listed in
Algorithm~\ref{alg:Weighted_reserior_sampling}.
In the algorithm, each new sample is associated with
a time-related weight $w = q^{\mathbb{I}_{}}$
with $\mathbb{I}_{}$ being the frame index  number corresponding to
$\mathbf{p}_{}$ and $q > 1$ being fixed. Using this time-related weight,
a random key for indexing the new sample is
generated by $k_{}=u_{}^{1/w_{}}$ with $u_{}\sim rand(0,1)$.
After that, a weighted sampling procedure~\cite{efraimidis2006weighted}
is adopted to update the existing
foreground or background sample buffer.
Fig.~\ref{fig:weighted_sample_example} gives an
intuitive illustration of the way that TWRS retains useful old samples
while keeping sample adaptability to recent changes.
{\em Note that it is the first time that
time-weighted reservoir sampling is used for
visual tracking.}

As described in the supplementary file, Theorem~\ref{theo:Weighted_reservoir_sampling_supp} states the relationship between
the ranking losses respectively from the reservoir sampling-based buffer and all training data seen to date.
This theorem shows that the sum of the ranking losses $\{l_{\mathbf{M}}(\mathbf{p},\mathbf{p}^{+}, \mathbf{p}^{-})\}$ over the foreground and background buffers
is probabilistically close to the sum of
the empirical ranking losses $\{\l_{\mathbf{M}}(\mathbf{p}, \mathbf{p}^{c_{+}}_{i}, \mathbf{p}^{c_{-}}_{j})\}$
over all the received training data.

Therefore,
statistical learning based on our reservoir sampling method with limited-sized sample buffers can effectively approximate statistical learning
using all the received training samples.
In our case,  reservoir sampling is used to maintain and update the foreground and background basis samples for discriminative distance metric learning.
Hence, the TWRS method extends the reservoir sampling
method~\cite{efraimidis2006weighted} to
cope with the online metric learning
problem using two sample buffers. It is a version of reservoir sampling tailored for online triplet-based
metric learning during visual tracking.

The key benefit of TWRS is to effectively generate and maintain
the limited-sized sample buffers, which encourage the recent samples
and meanwhile retain some old samples with a long lifespan.
In this way, online metric learning using the
limited-sized sample buffers approximates that of batch-mode learning (i.e., retaining all the samples during tracking),
which balances the effectiveness and efficiency.
The key difference from other online approaches (e.g., using forgetting factors)
is that the total learning costs for TWRS are derived from
the sequentially generated limited-sized sample buffers (retaining the old and recent samples simultaneously).
In contrast, the online learning approaches using forgetting factors
only store the recent samples and discard
the previously generated samples.
The costs of using the previously generated samples decay
recursively using a forgetting factor.
As a result, the online learning approaches using forgetting factors
may suffer from the model drift problem.

\setlength{\textfloatsep}{8pt}
\begin{algorithm}[t]
\begin{minipage}[ctb]{15cm}
\footnotesize
\caption{Time-weighted reservoir sampling}
\label{alg:Weighted_reserior_sampling}
\KwIn
{
 \hspace{-0.2cm} Current buffers $\mathcal{B}_{f}$ and  $\mathcal{B}_{b}$ together with their corresponding keys, a new
training sample $\mathbf{p}_{}$, maximum buffer size $\Omega$,
time-weighted factor $q$.
}
\KwOut
{
 \hspace{-0.1cm}Updated buffers $\mathcal{B}_{f}$ and  $\mathcal{B}_{b}$ together with their corresponding keys. \vspace{-0cm}
}
\begin{enumerate}\itemsep=-0pt
\item Obtain the samples $\mathbf{p}^{\ast}_{f}\in \mathcal{B}_{f}$ and $\mathbf{p}^{\ast}_{b} \in \mathcal{B}_{b}$ with the
smallest \mbox{\hspace{0.08cm}keys $k^{\ast}_{f}$ and $k^{\ast}_{b}$ from $\mathcal{B}_{f}$ and $\mathcal{B}_{b}$, respectively}.
  \item Compute the time-related weight $w = q^{\mathbb{I}_{}}$
with $\mathbb{I}_{}$ being the corresponding frame index  number of
$\mathbf{p}_{}$. \vspace{-0.05cm}
  \item Calculate a key $k_{}=u_{}^{1/w_{}}$ where $u_{}\sim rand(0,1)$.
  \item
  \textbf{Case:} $\mathbf{p}_{}\in$ foreground.
                 If $|\mathcal{B}_{f}|<\Omega$, $\mathcal{B}_{f}=\mathcal{B}_{f}\bigcup \{\mathbf{p}_{}\}$;
                 otherwise, $\mathbf{p}^{\ast}_{f}$ is replaced with $\mathbf{p}_{}$ when \mbox{\hspace{0.08cm}$k_{}>k^{\ast}_{f}$}.\\
        \textbf{Case:} $\mathbf{p}_{}\in$ background. If $|\mathcal{B}_{b}|<\Omega$, $\mathcal{B}_{b}=\mathcal{B}_{b}\bigcup \{\mathbf{p}_{}\}$;
                 otherwise, $\mathbf{p}^{\ast}_{b}$ is replaced with $\mathbf{p}_{}$ when $k_{}>k^{\ast}_{b}$.
  \item Return $\mathcal{B}_{f}$ and  $\mathcal{B}_{b}$ together with their corresponding keys.
\end{enumerate}
\end{minipage}
\end{algorithm}

\section{Experimental evaluation of our baseline tracker}

\subsection{Experimental configurations}

{\it 1) Implementation details.}
For the sake of computational efficiency, we simply consider the
object state information in 2D translation and
scaling in the particle filtering module, where
the corresponding variance parameters are set to (10, 10, 0.1).
The number of particles is set to 200.
In practice, the video sequences used for experiments
consists of the targets with relatively slow motion and progressive
scale variation in most cases. As a result,
such settings for particle number and translational variances
are adequate for robust visual tracking.
Of course, in the case of fast motion or drastic object motion
the parameter settings with larger variance parameters or
more particles should be adopted.
For each particle, there is a corresponding image region
represented as a HOG feature descriptor (\cite{Dalal-Triggs-CVPR2005}) with $3\times 3$ cells
(each cell is represented by a 9-dimensional histogram vector)
in the five spatial block-division modes (~\cite{lixi-cvpr2008}),
resulting in
a 405-dimensional feature vector for the image region.
The number of triplets used for online metric learning
is chosen as 500.
The maximum buffer size $\Omega$ and time-weighted factor $q$ in
Algorithm~\ref{alg:Weighted_reserior_sampling}
is set as 300 and 1.6, respectively.
Similarly to~\cite{Babenko-Yang-Belongie-cvpr2009}, we take a spatial distance-based strategy
for training sample selection.
The scaling factors $\gamma_{f}$
and $\gamma_{b}$ in Equ.~\eqref{eq:particle_liki_model} are
chosen as 1. The trade-off control factor $\rho$ in
Equ.~\eqref{eq:particle_liki_model} is set as 0.1.
Note that the aforementioned parameters are fixed throughout all the
experiments.
If the proposed tracker is implemented in Matlab
on a workstation
with an Intel Core 2 Duo 2.66GHz processor and 3.24G RAM,
the average running time of the proposed tracker is about
0.55 second per frame (because of the slow for-loop operations in Matlab). In contrast,
if the proposed tracker is carried out in C++
with multi-thread parallel computation (for greatly speeding up the for-loop operations),
the running time will be greatly reduced
(for about 0.08 second per frame).

\begin{figure}[t]
\centering
\includegraphics[width=4.15cm, height = 3.35cm]{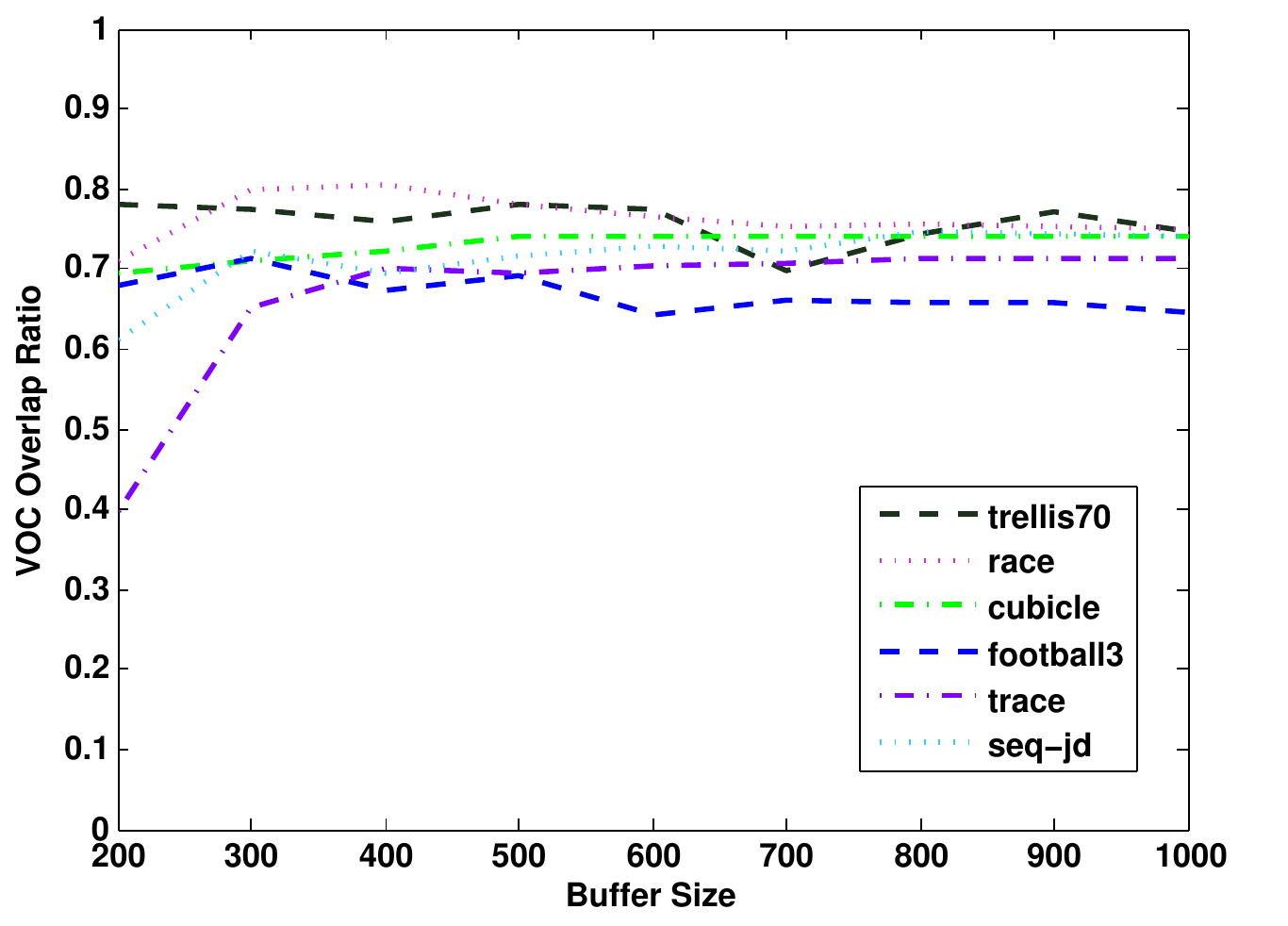}
\hspace{-0.5cm}
\includegraphics[width=4.15cm, height = 3.35cm]{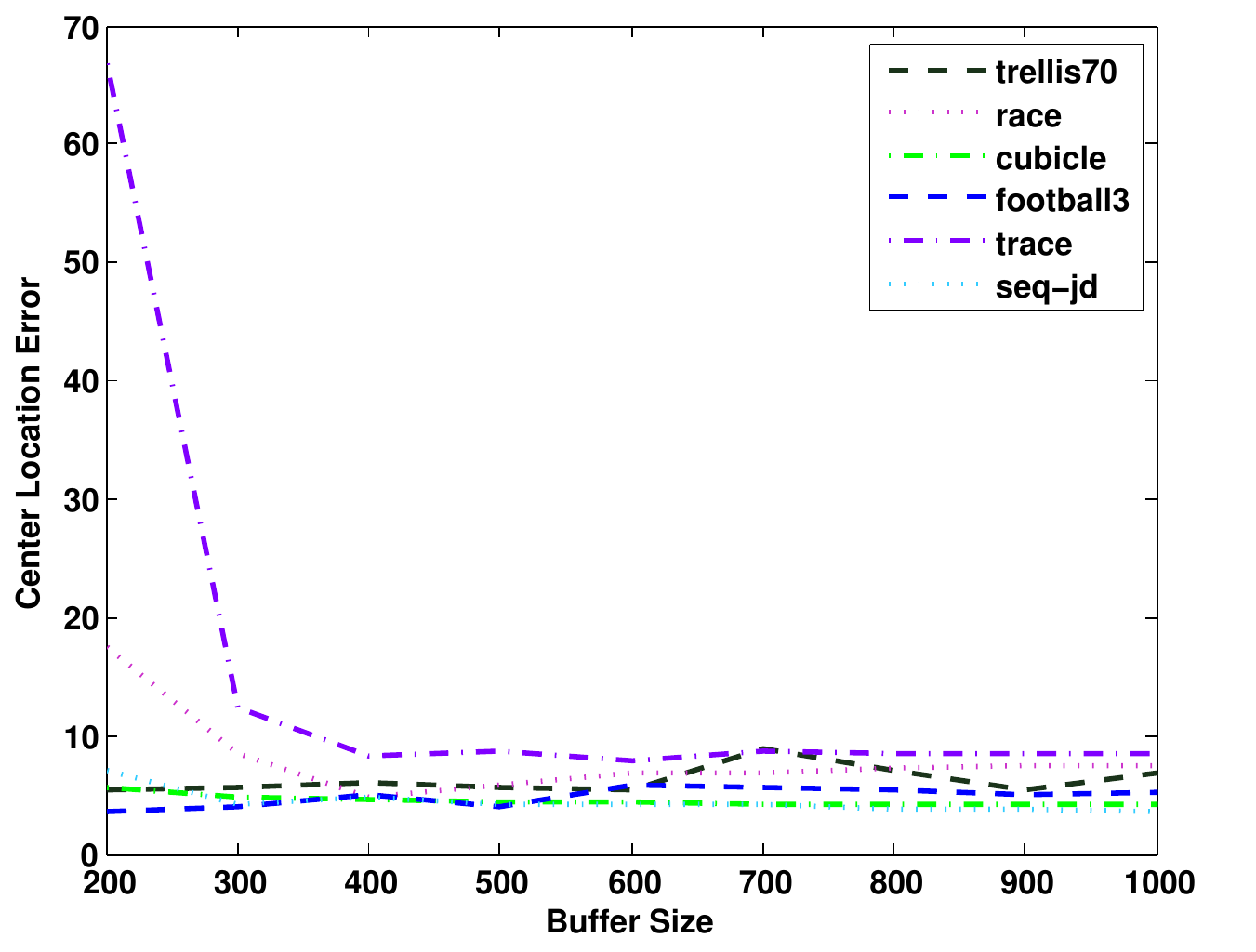}
\hspace{-0.5cm}
\includegraphics[width=4.15cm, height = 3.35cm]{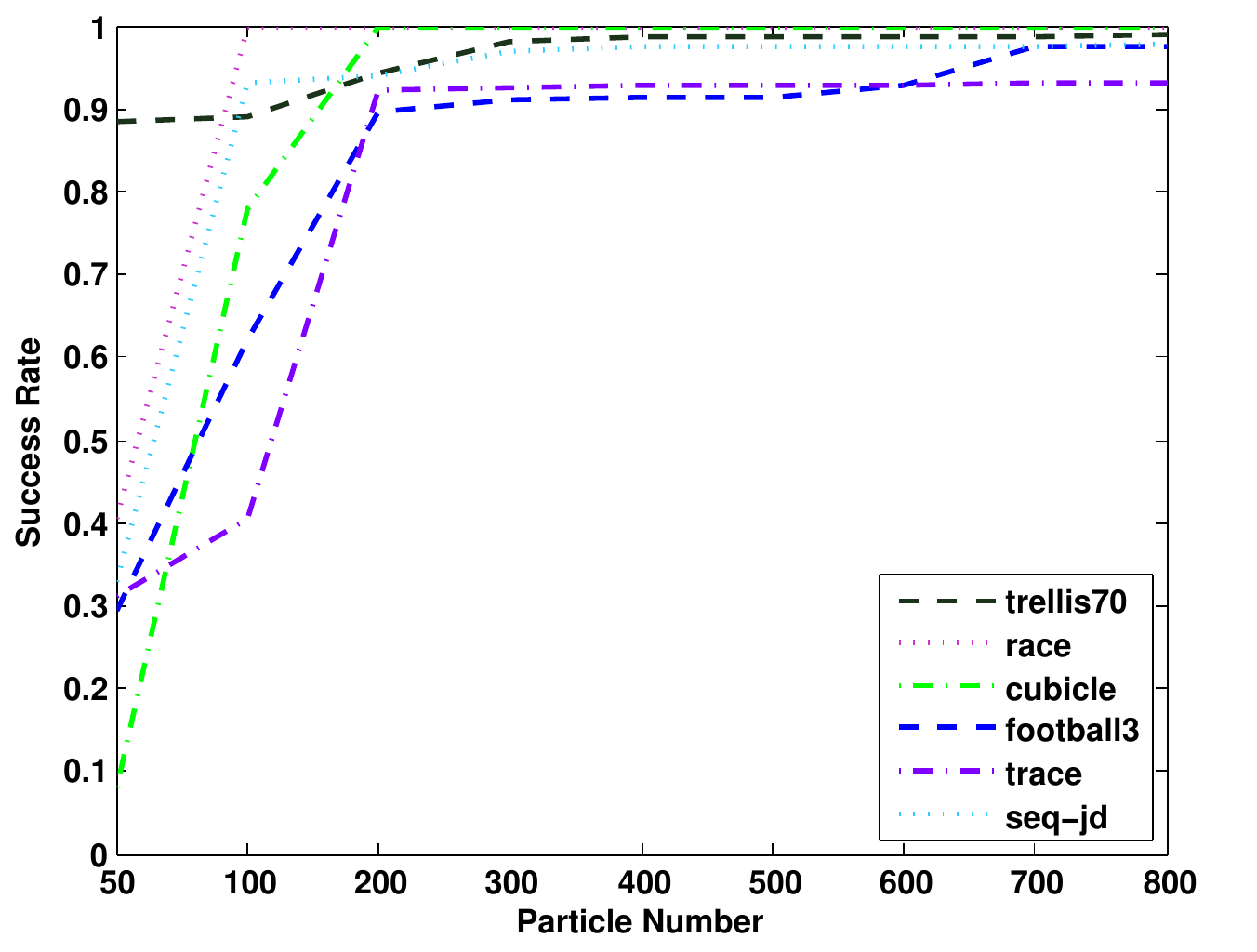}
\hspace{-0.5cm}
\includegraphics[width=4.15cm, height = 3.35cm]{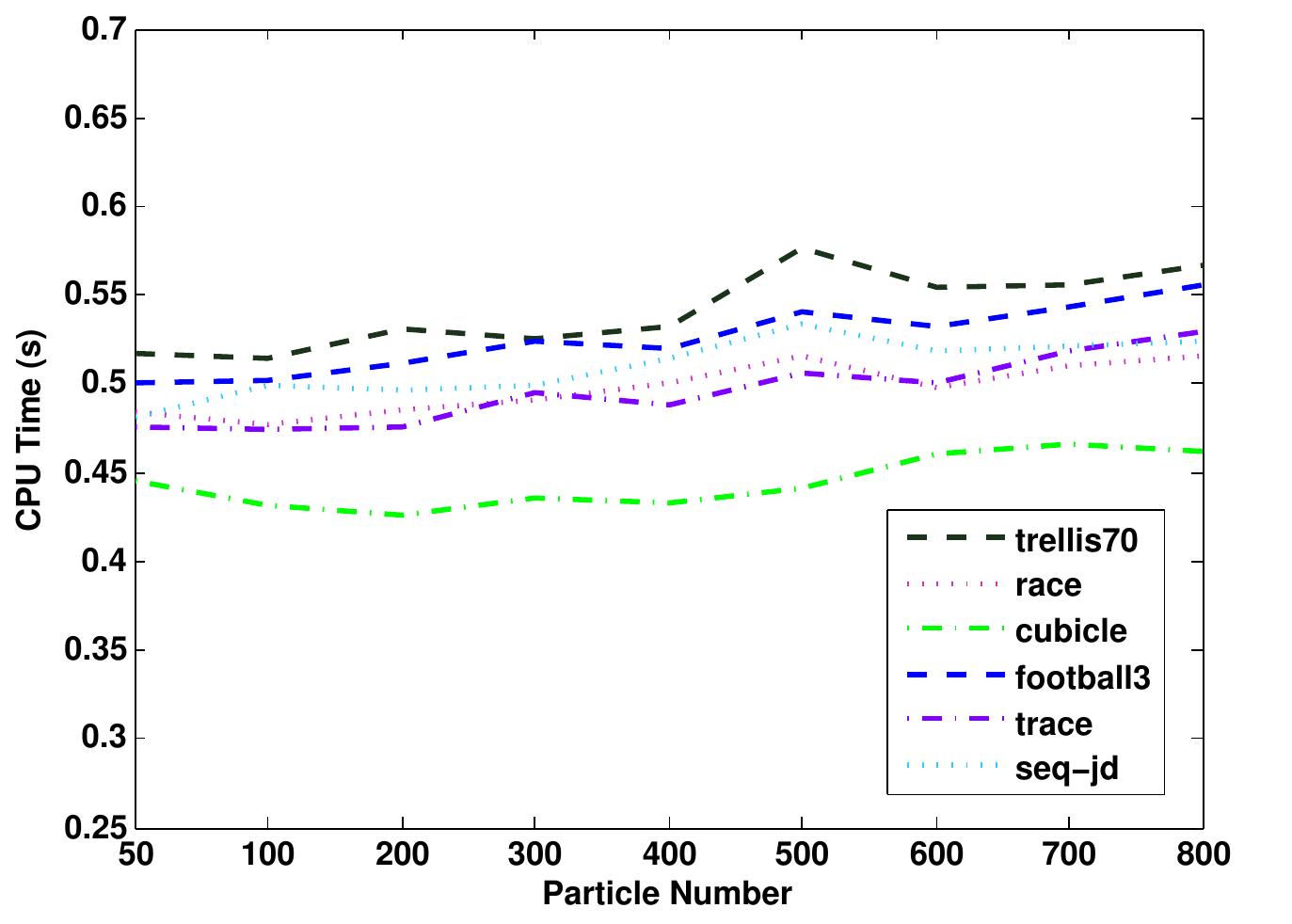}\vspace{-0.1cm}\\
\footnotesize (a) \hspace{3.6cm}(b) \hspace{3.3cm} (c) \hspace{3.3cm} (d) \hspace{-0.3cm}\\
\vspace{-0.16cm}
 \caption{
 Quantitative evaluation of the proposed tracker using
different buffer sizes  and  particle numbers.
The left half corresponds to the tracking results with different buffer sizes,
while the right half is associated with the tracking results with different
particle numbers.
 \vspace{-0.0cm}}
 \label{fig:buffersize_particlenum}
\end{figure}

{\it 2) Datasets and evaluation criteria}
A set of experiments are conducted on
eighteen challenging video sequences, which consist of
8-bit grayscale images.
These video sequences are
captured from different scenes,
and contain different types of object motion events (e.g., human walking and
car running), which are illustrated in the supplementary file.
For quantitative performance comparison, two popular evaluation
criteria are introduced, namely,
center location error (CLE) and VOC overlap ratio (VOR)
between the predicted bounding
box $B_p$ and ground truth bounding box $B_{gt}$
such that ${\rm VOR} = \frac{area(B_p\bigcap B_{gt})}{area(B_p \bigcup B_{gt})}$.
If the
VOC overlap ratio
is larger than $0.5$, then tracking is considered  successful in that frame.

\vspace{-0.25cm}
\subsection{Empirical analysis of parameter settings}

{\it 1) Sample buffer size.}
To test the effect of buffer size on reservoir sampling, we compute the average
CLE and VOR for each video sequence using nine different sample buffer sizes.
Figs.~\ref{fig:buffersize_particlenum} (a) and (b) show the quantitative CLE and
VOR performance on five video sequences.
It is clear that the average CLE (VOR) decreases (increases) as the
buffer size increases, and
plateaus with approximately more than 300 samples.

{\it 2) Number of particles.}
In general, more particles
enable visual trackers to locate the object more accurately, but lead
to a higher computational cost.
Thus, it is crucial for visual trackers to keep a good balance between
accuracy and efficiency using a
moderate number of particles.
Fig.~\ref{fig:buffersize_particlenum} (c) shows the average VOC success rates
(i.e., $\frac{\#\mbox{success frames}}{\# \mbox{total frames}}$)
of the proposed tracking algorithm on three video sequences.
From Fig.~\ref{fig:buffersize_particlenum} (c), we can see that the success
rate  rapidly grows with increasing
particle number and then converges at approximately 200-300 particles for each sequence.
Fig.~\ref{fig:buffersize_particlenum} (d) displays the average
CPU time (spent by the proposed tracking algorithm in each frame) with
different particle numbers.
It is observed from Fig.~\ref{fig:buffersize_particlenum} (d) that the average
CPU time slowly increase.

{\it 3) Comparison of different linear representations.}
The objective of this task is to evaluate the performance of four
linear representations: our linear
representation with metric learning,
our linear representation without
metric learning,
compressive sensing linear representation~\cite{Li-Shen-Shi-cvpr2011},
and $\ell_{1}$-regularized linear
representation~\cite{Meo-Ling-ICCV09}.
For a fair comparison, we utilize the raw pixel features as in
\cite{Li-Shen-Shi-cvpr2011,Meo-Ling-ICCV09}.
Tab.~\ref{Tab:different_linear_representation_VOC_CLE}
shows the average performance of these four linear representations
in CLE, VOR, and success rate on four video sequences.
Clearly, our linear
representation with metric learning consistently achieves
better tracking results than the three other linear representations.
Please see the supplementary file for the details of
the frame-by-frame tracking results (i.e., CLE, VOR, success rate).

{\it 4) Evaluation of different sampling methods.}
Here, we examine the performance of two
sampling methods: uniform~\cite{vitter1985random} and time-weighted~\cite{efraimidis2006weighted} reservoir sampling.
Tab.~\ref{Tab:sampling_VOC_CLE} shows the experimental results of the two
sampling methods in CLE, VOR, and success rate on five video sequences.
From Tab.~\ref{Tab:sampling_VOC_CLE}, we can see that weighted reservoir
sampling performs better than
ordinary reservoir sampling. More results of these two
sampling methods can be found in the supplementary file.

\begin{figure*}[t]
 \vspace{-0.1cm}
\centering
\includegraphics[scale=0.4]{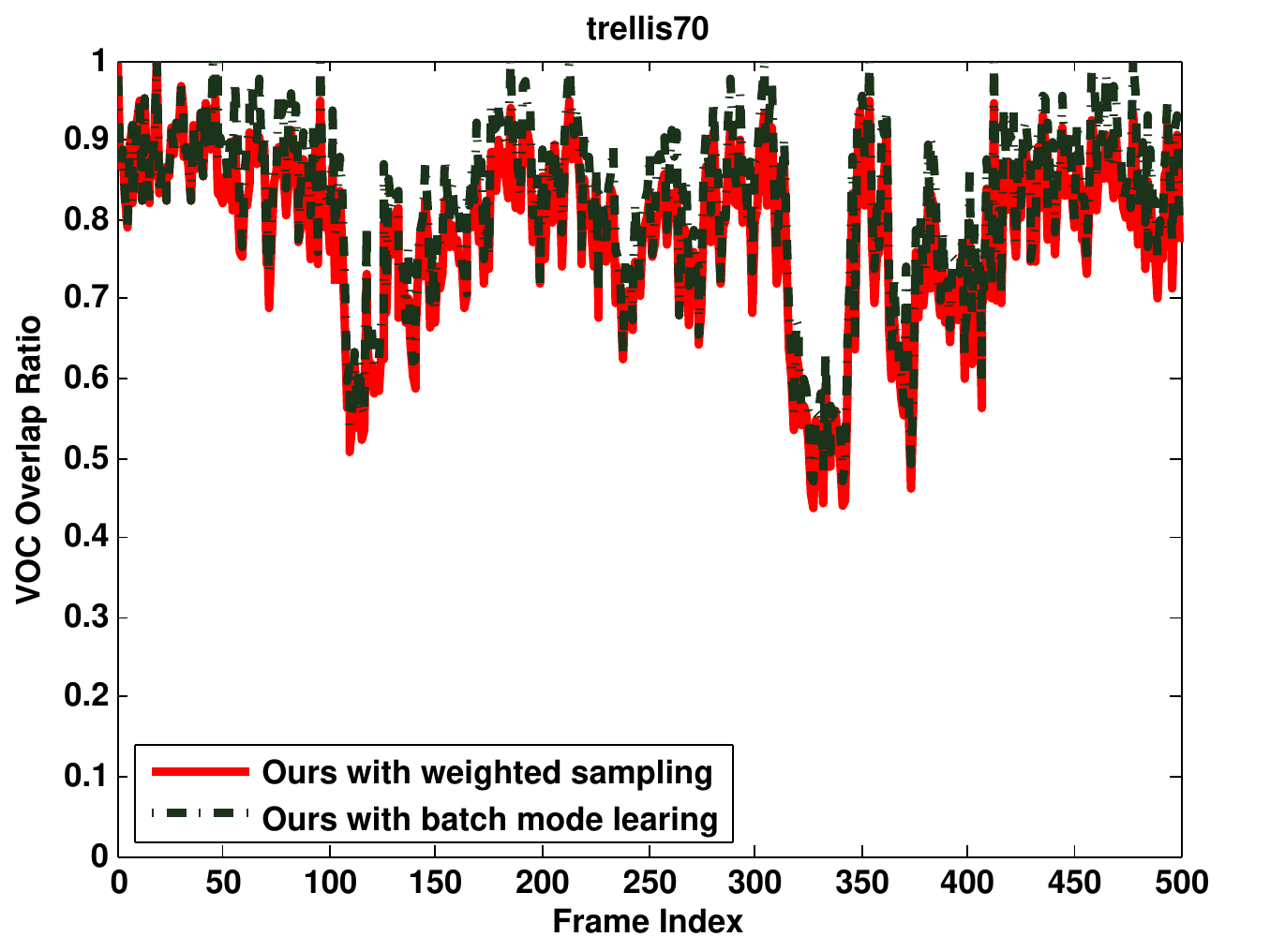}\hspace{-0.21cm}
\includegraphics[scale=0.4]{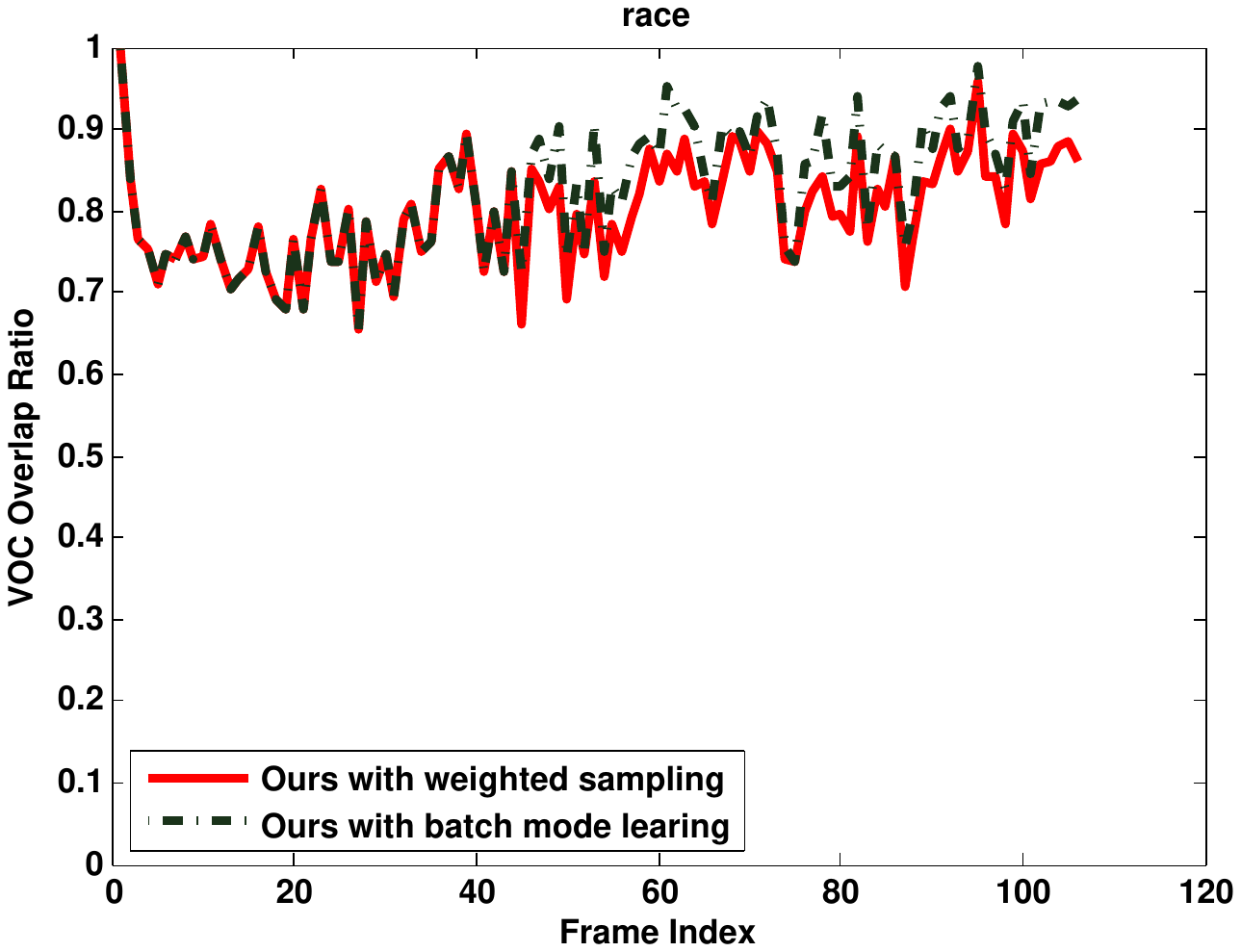} \hspace{-0.21cm}
\includegraphics[scale=0.4]{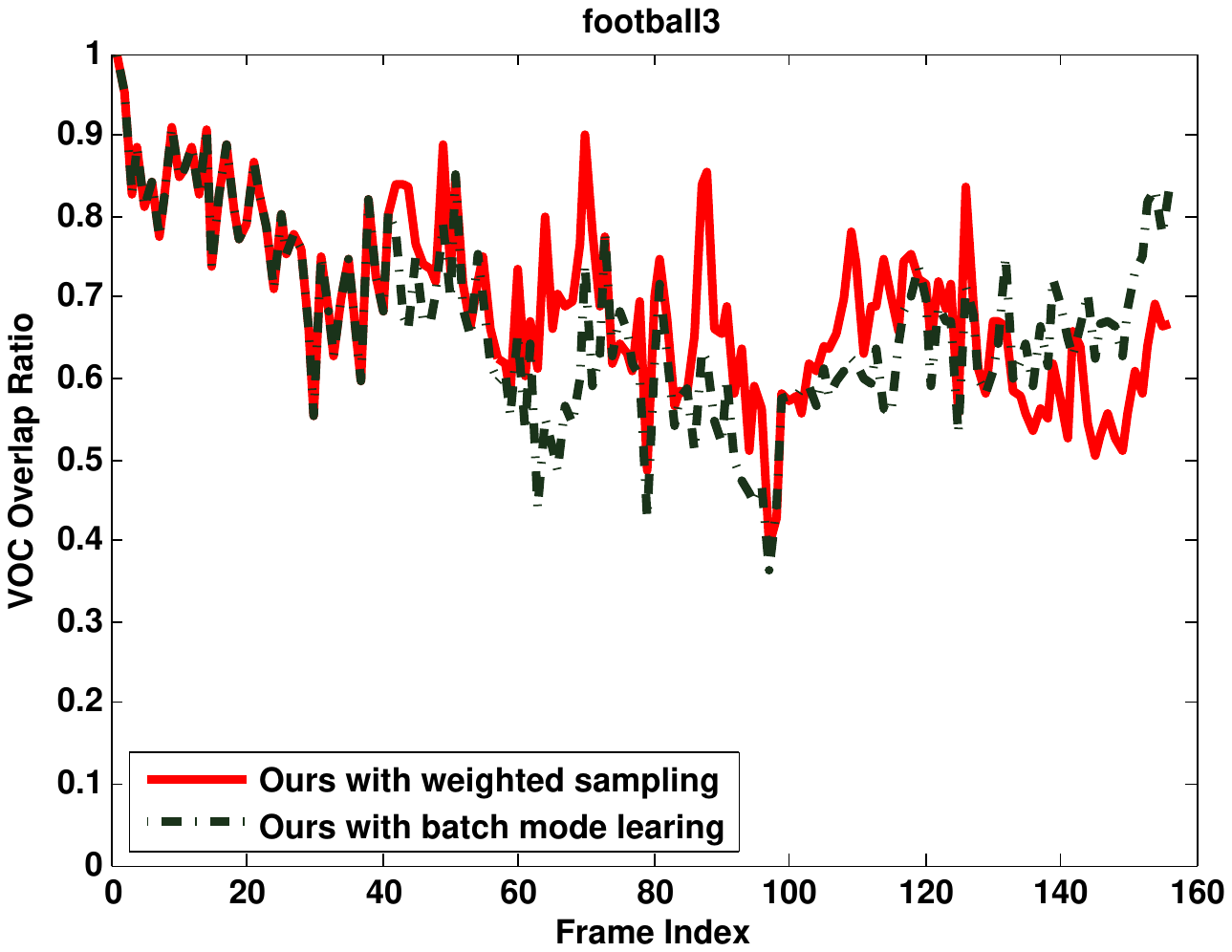}\\
\vspace{-0.3cm}
 \caption{Quantitative comparison of the proposed tracker
with weighted reservoir sampling and batch mode learning in average VOR
on three video sequences.
Clearly, the tracking performance of our weighted reservoir sampling
is very close to that of batch mode learning.
}
\label{fig:batch_mode_learning}
\end{figure*}

\begin{figure*}[t]
 \vspace{-0.1cm}
\centering
\includegraphics[scale=0.38]{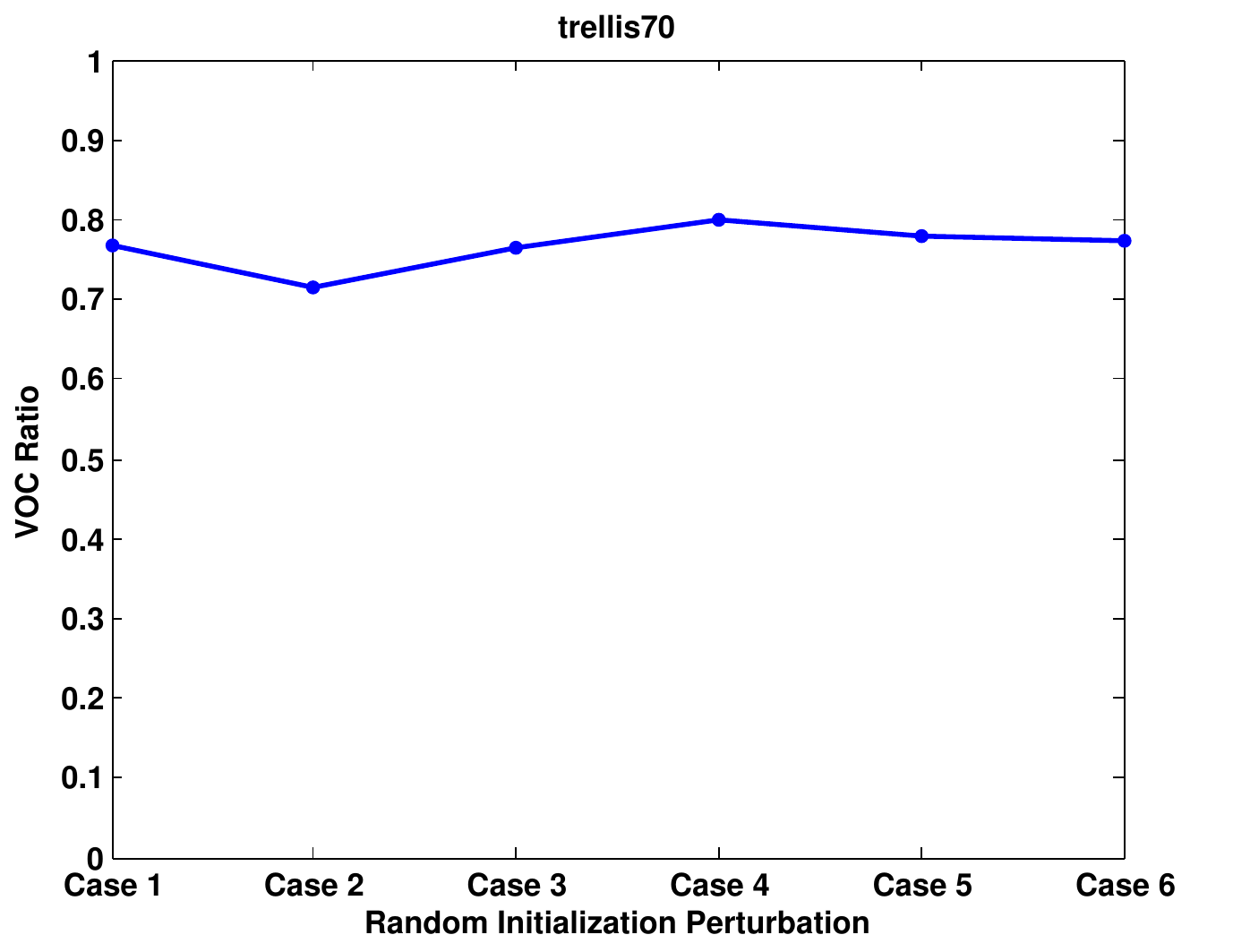}\hspace{-0.21cm}
\includegraphics[scale=0.38]{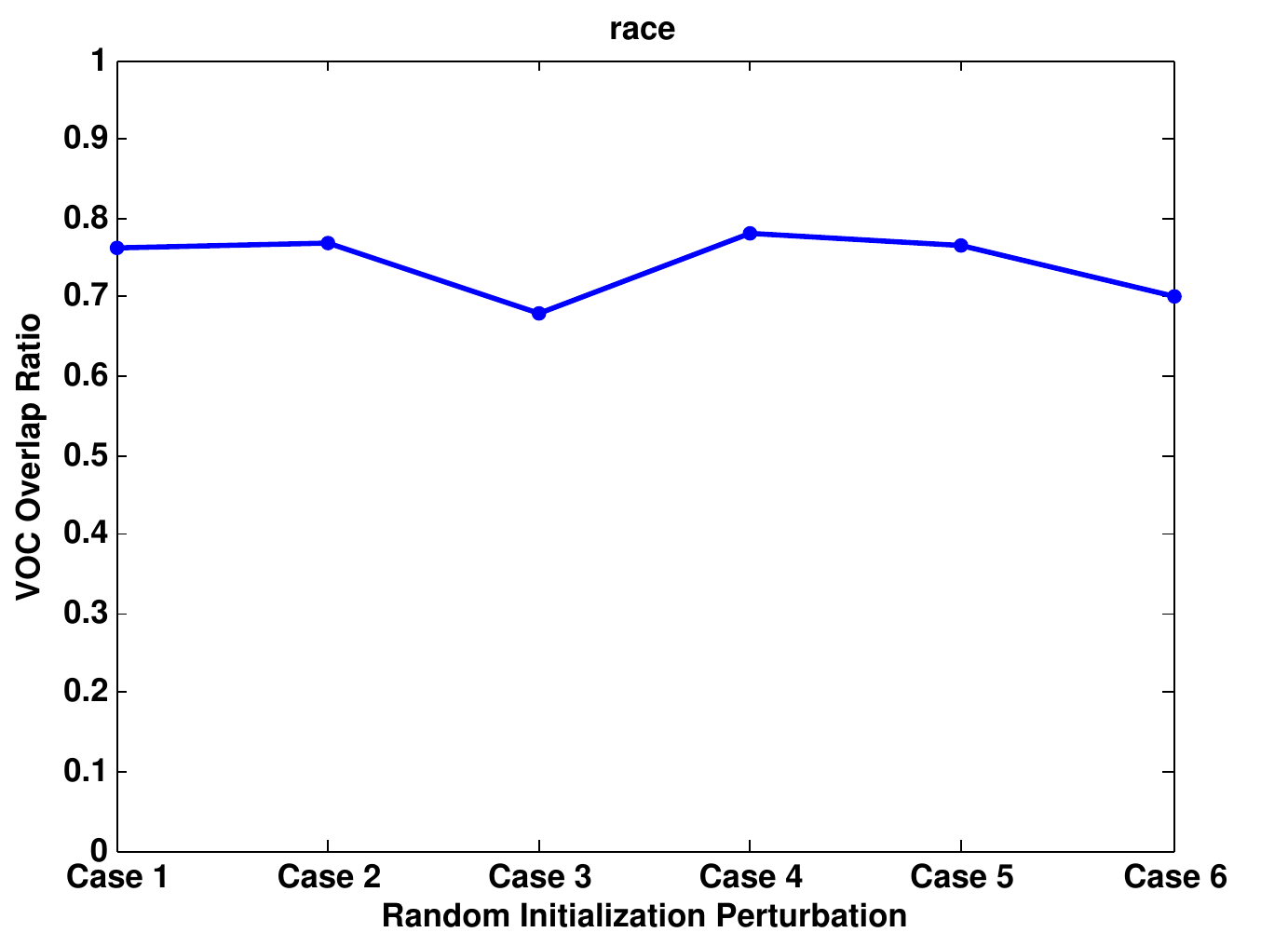} \hspace{-0.21cm}
\includegraphics[scale=0.38]{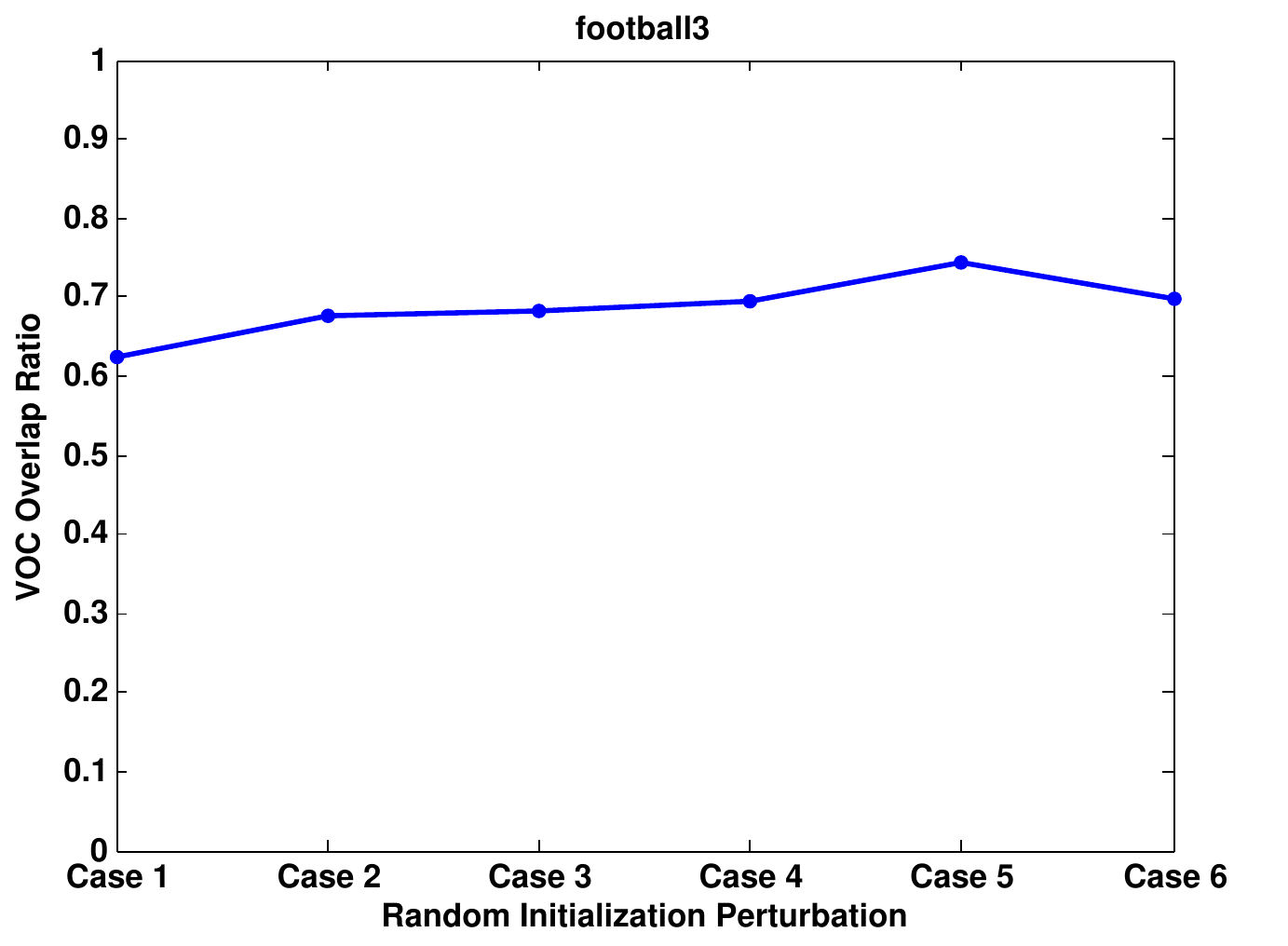}\\
\vspace{-0.3cm}
 \caption{Quantitative evaluation of the proposed tracker
with five different initialization configurations
(obtained by moderate random perturbation on the original initialization setting)
in VOR on three video sequences.
It is clear that the proposed tracker is not very sensitive to
different initialization configurations.
}
\label{fig:initialization}
\end{figure*}

\begin{figure*}[t]
 \vspace{-0.1cm}
\centering
\includegraphics[scale=0.38]{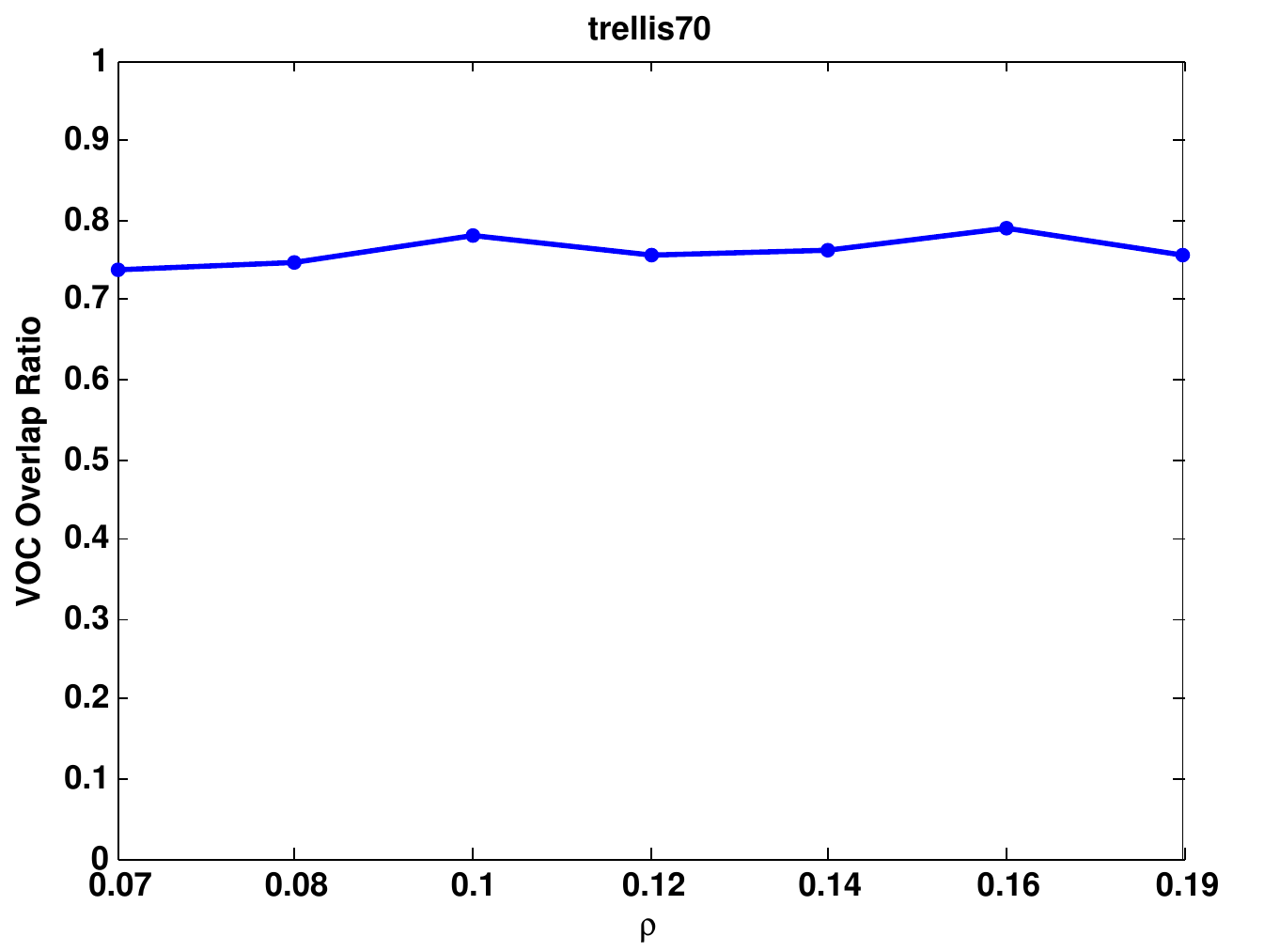}\hspace{-0.21cm}
\includegraphics[scale=0.38]{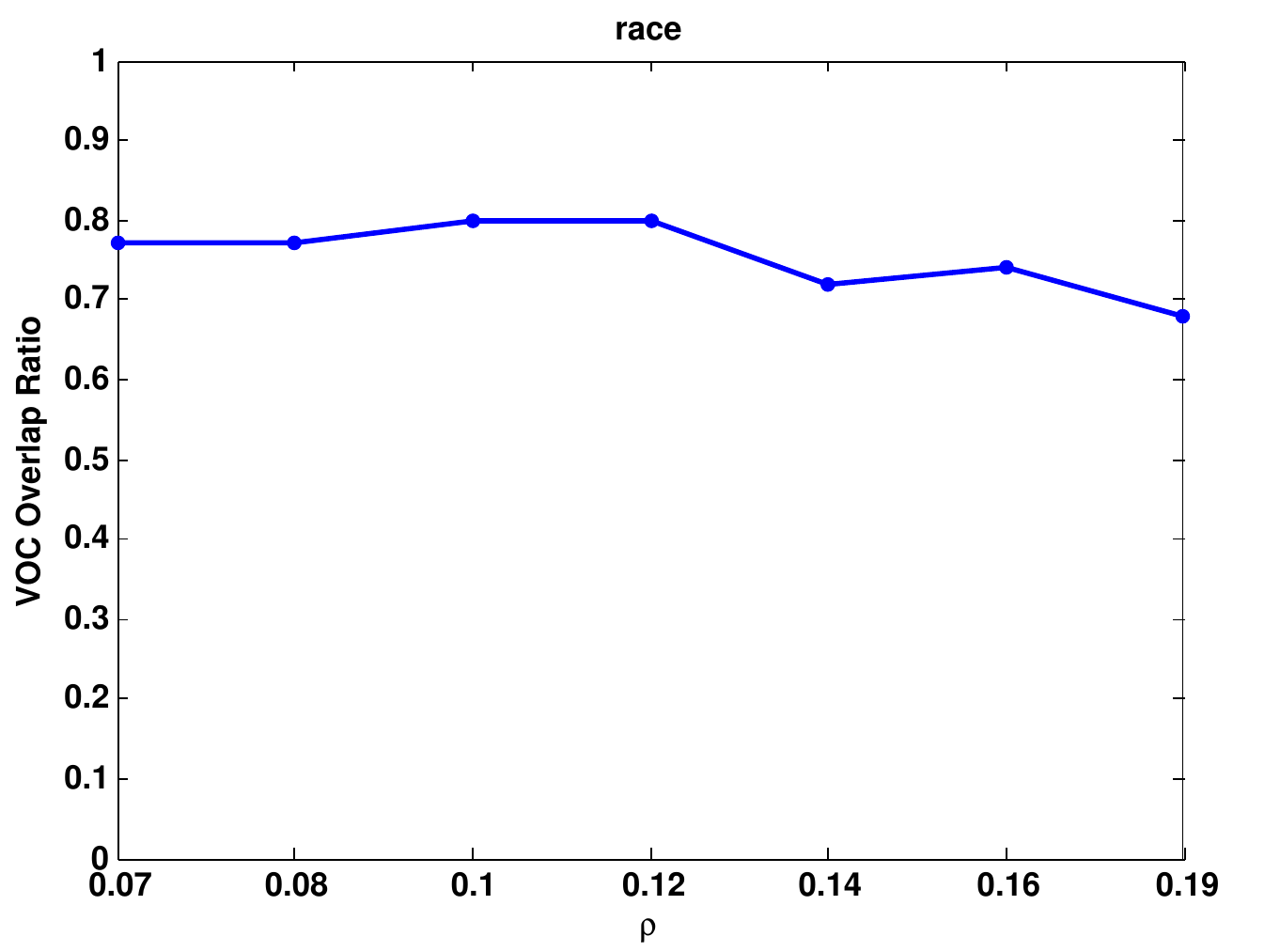} \hspace{-0.21cm}
\includegraphics[scale=0.38]{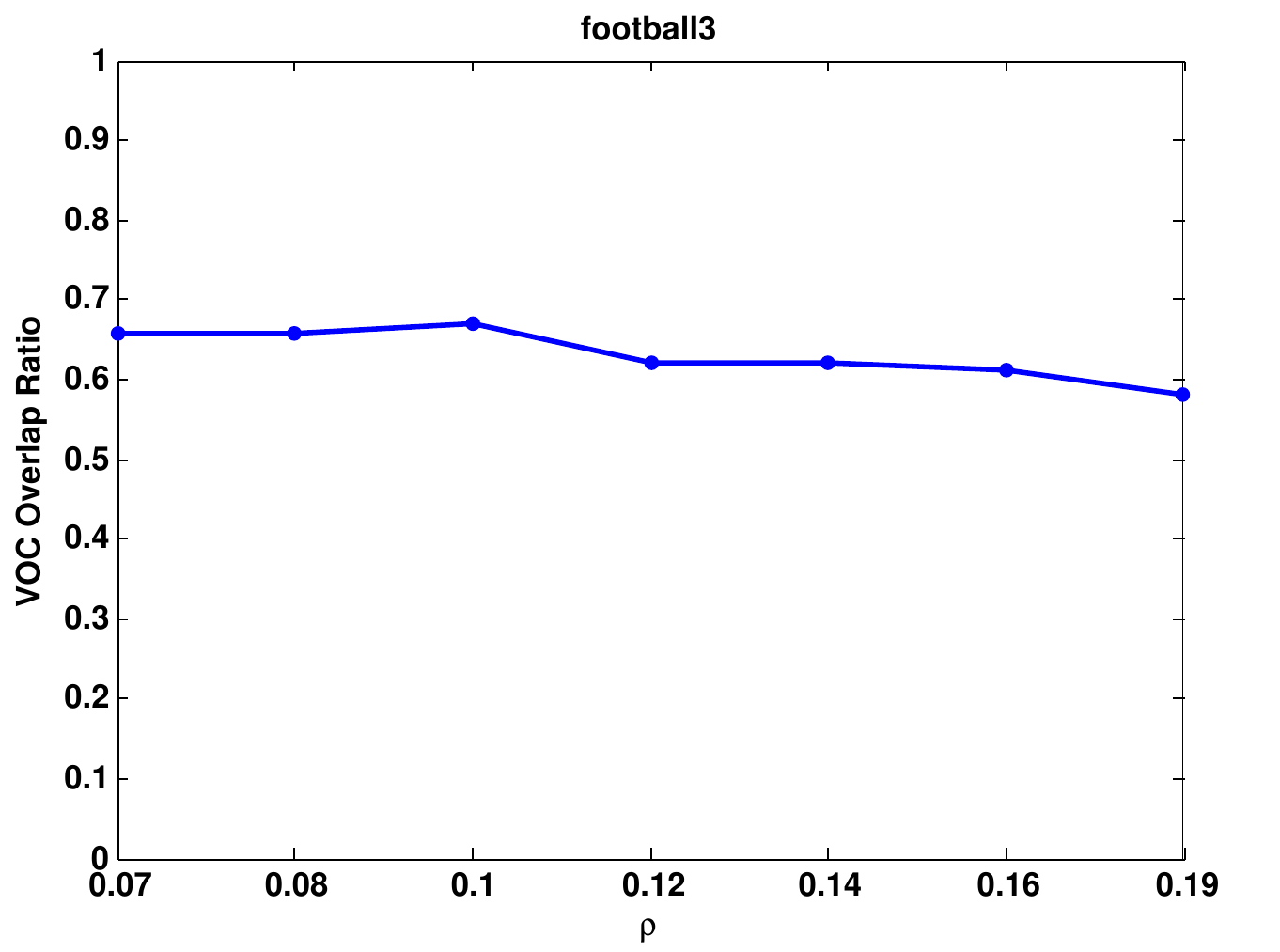}\\
\vspace{-0.3cm}
 \caption{Quantitative evaluation of
 the proposed tracker with different settings for
 the trade-off control factor $\rho$ such that
 $\rho \in \{0.07, 0.08, 0.10, 0.12, 0.14, 0.16, 0.19\}$.
 It is observed that the proposed tracker is not very
 sensitive to the setting of $\rho$.
 }
 \label{fig:tradeoff_factor}
\end{figure*}

\begin{table}[t]
\begin{center}
\scalebox{0.51}
{
\begin{tabular}{c||c|c|c|c|c|c||c|c|c|c|c|c||c|c|c|c|c|c}
\hline \scriptsize
& \multicolumn{6}{|c||}{CLE} & \multicolumn{6}{|c||}{VOR} & \multicolumn{6}{|c}{Success Rate}\\
\hline
& \makebox[1.1cm]{trellis70}   &  \makebox[0.9cm]{race} &
\makebox[0.9cm]{cubicle} & \makebox[1.3cm]{football3} & \makebox[0.9cm]{trace} & \makebox[0.9cm]{seq-jd} &
 \makebox[1.1cm]{trellis70}   &  \makebox[0.9cm]{race} &
\makebox[0.9cm]{cubicle} & \makebox[1.3cm]{football3} & \makebox[0.9cm]{trace} & \makebox[0.9cm]{seq-jd}
& \makebox[1.1cm]{trellis70}   &  \makebox[0.9cm]{race} &
\makebox[0.9cm]{cubicle} & \makebox[1.3cm]{football3} & \makebox[0.9cm]{trace} & \makebox[0.9cm]{seq-jd} \\\hline
Ours+pixels with metric  &      \bf 20.91   & \bf 5.18   & \bf 21.12    & \bf 9.47   & \bf 22.88    & \bf8.67     & \bf 0.49    &\bf0.79    &\bf0.54    &\bf0.55
&\bf0.37    &\bf0.58     &\bf0.53    &\bf0.99    &\bf0.69    &\bf0.55    &\bf0.28    &\bf0.70\\
Ours+pixels without metric  &   66.32   & 7.91   &40.48   &14.83  &139.30   &20.59     &0.27    &0.67    &0.18    &0.46    &0.07    &0.51     &0.30    &0.89    &0.25    &0.28    &0.08    &0.62\\
CS+pixels  &                    68.45   & 5.51   &31.00   &53.56   &78.98   &39.75     &0.19    &0.68    &0.37    &0.22    &0.17    &0.17     &0.19    &0.90    &0.47    &0.24    &0.09    &0.11\\
L1+pixels  &                    27.64  &160.79   &24.26   &64.07  &108.64   &12.76     &0.28    &0.08    &0.41    &0.16    &0.09    &0.52     &0.34    &0.10    &0.49    &0.16    &0.06    &0.61\\
\hline
\end{tabular}
}
\end{center}
\vspace{-0.6cm}
\caption{Quantitative evaluation of the proposed tracker using
different linear representations on four video sequences.
The table shows their average CLEs,
VORs, and success rates.} \vspace{-0.2cm}
\label{Tab:different_linear_representation_VOC_CLE}
\end{table}

\begin{table}[t]
\begin{center}
\scalebox{0.52}
{
\begin{tabular}{c||c|c|c|c|c|c||c|c|c|c|c|c||c|c|c|c|c|c}
\hline \scriptsize
& \multicolumn{6}{|c||}{CLE} & \multicolumn{6}{|c||}{VOR} & \multicolumn{6}{|c}{Success Rate}\\
\hline
& \makebox[1.1cm]{trellis70}   &  \makebox[0.9cm]{race} &
\makebox[0.9cm]{cubicle} & \makebox[1.3cm]{football3} & \makebox[0.9cm]{trace} & \makebox[0.9cm]{seq-jd} &
 \makebox[1.1cm]{trellis70}   &  \makebox[0.9cm]{race} &
\makebox[0.9cm]{cubicle} & \makebox[1.3cm]{football3} & \makebox[0.9cm]{trace} & \makebox[0.9cm]{seq-jd}
& \makebox[1.1cm]{trellis70}   &  \makebox[0.9cm]{race} &
\makebox[0.9cm]{cubicle} & \makebox[1.3cm]{football3} & \makebox[0.9cm]{trace} & \makebox[0.9cm]{seq-jd} \\\hline
Ours with weighted sampling   &    \bf5.62    &\bf8.52    &\bf4.31    &\bf3.92    &\bf8.65    &\bf 4.30    &\bf 0.78    &\bf 0.80    &\bf0.74    &\bf0.69    &\bf0.70    &\bf0.72
&\bf0.98    &\bf0.99    &\bf0.98    &\bf0.97    &\bf0.89    &\bf0.94\\
Ours without weighted sampling  &  8.58   &13.92    &\bf4.31    &4.67   &15.52    &5.21    &0.70    &0.75    &\bf 0.74    &0.67    &0.61    &0.68
&0.90    &\bf0.99    &\bf0.98    &0.93    &0.58    &0.88\\
\hline
\end{tabular}
}
\end{center}
\vspace{-0.6cm}
\caption{Quantitative evaluation of the proposed tracker using
different sampling methods on five video sequences.
The table shows their average CLEs, VORs,
and success rates.}
\label{Tab:sampling_VOC_CLE}
\end{table}

{\it 5) Performance with and without metric learning.}
To justify the effect of different metric learning mechanisms, we
design several experiments on five video sequences.
Tab.~\ref{Tab:metric_success_rate}  shows the corresponding experimental
results of different metric learning mechanisms
in CLE, VOR, and success rate.
From Tab.~\ref{Tab:metric_success_rate}, we can see that the performance of
metric learning is  better than that of no
metric learning. In addition, the performance of metric learning with
no eigendecomposition is close
to that of metric learning with step-by-step eigendecomposition, and
better than that of
metric learning with final eigendecomposition. Therefore, the obtained
results are consistent with those
in~\cite{chechik2010large}. Besides, metric learning with step-by-step
eigendecomposition
is much slower than that with no eigendecomposition which is adopted
by the proposed
tracking algorithm.

{\it 6) Evaluation of weighted reservoir sampling and batch mode learning.}
To balance effectiveness and efficiency, weighted reservoir sampling aims to maintain
limited-sized foreground and background sample buffers used
for learning a metric-weighted linear representation.
In contrast, batch mode learning requires storing all the
foreground and background samples during tracking, which leads
to expensive computation and high memory usage.
Therefore,
we conduct a quantitative comparison experiment
between weighted reservoir sampling and batch mode learning on three video sequences,
as shown in Fig.~\ref{fig:batch_mode_learning}.
From Fig.~\ref{fig:batch_mode_learning}, we observe that
the tracking performance with weighted reservoir sampling
is able to well approximate that of batch mode learning.

{\it 7) Effect of random initialization perturbation.}
Here, we aim to investigate the tracking performance of the proposed
tracker with different initialization configurations,
which are generated by moderate random perturbation (i.e., relatively small center location offset) on
the original bounding box after manual initialization.
Fig.~\ref{fig:initialization} shows the average VOR tracking performance on three video sequences
in different initialization cases. It is clearly seen from
Fig.~\ref{fig:initialization} that the proposed tracker
achieves the mutually close tracking results, and is not sensitive to
different initialization settings.

{\it 8) Investigation of the trade-off control factor $\rho$.}
To evaluate the effect of the discriminative metric-weighted
reconstruction information from foreground and background buffers,
we make a quantitative empirical study of the proposed tracker with
different configurations of $\rho$ (referred to in Equ.~\eqref{eq:particle_liki_model}).
Fig.~\ref{fig:tradeoff_factor} displays the quantitative
average VOR tracking results on three video sequences using different configurations
of $\rho$ such that  $\rho \in \{0.07, 0.08, 0.10, 0.12, 0.14, 0.16, 0.19\}$.
Apparently, the proposed tracker is not very sensitive to
the configuration of $\rho$ within a moderate range.

\begin{table}[t]
\begin{center}
\scalebox{0.52}
{
\begin{tabular}{c||c|c|c|c|c|c||c|c|c|c|c|c||c|c|c|c|c|c}
\hline \scriptsize
& \multicolumn{6}{|c||}{CLE} & \multicolumn{6}{|c||}{VOR} & \multicolumn{6}{|c}{Success Rate}\\
\hline
& \makebox[1.1cm]{trellis70}   &  \makebox[0.9cm]{race} &
\makebox[0.9cm]{cubicle} & \makebox[1.3cm]{football3} & \makebox[0.9cm]{trace} & \makebox[0.9cm]{seq-jd} &
 \makebox[1.1cm]{trellis70}   &  \makebox[0.9cm]{race} &
\makebox[0.9cm]{cubicle} & \makebox[1.3cm]{football3} & \makebox[0.9cm]{trace} & \makebox[0.9cm]{seq-jd}
& \makebox[1.1cm]{trellis70}   &  \makebox[0.9cm]{race} &
\makebox[0.9cm]{cubicle} & \makebox[1.3cm]{football3} & \makebox[0.9cm]{trace} & \makebox[0.9cm]{seq-jd} \\\hline
\hline
ML w/o eigen                   &5.62    &8.52    &4.31    &3.92    &8.65    &4.30   &0.78    &0.80    &0.74    &0.69    &0.70    &0.72   &0.98    &\bf0.99    &\bf0.98    &0.97    &0.89    &0.94\\
ML with final eigen            &8.46   &10.17    &5.95    &7.80   &18.38    &7.06   &0.70    &0.69    &0.67    &0.54    &0.52    &0.63   &0.94    &\bf0.99    &0.94    &0.55    &0.57    &0.82\\
ML with step-by-step eigen     &\bf4.56    &\bf6.23    &\bf2.16    &\bf3.84   &\bf7.46    &\bf3.12   &\bf0.83    &\bf0.75    &\bf0.78    &\bf0.74    &\bf0.75    &\bf0.76   &\bf0.99    &\bf0.99    &\bf0.98    &\bf0.98    &\bf0.98    &\bf0.97\\
No metric learning             &8.82   &16.80    &5.29    &5.07   &16.94    &6.09   &0.68    &0.65    &0.67    &0.63    &0.51    &0.64   &0.92    &0.97    &0.88    &0.91    &0.40    &0.85\\
\hline
\end{tabular}
}
\end{center}
\vspace{-0.5cm}
\caption{
Quantitative evaluation of the proposed tracker with different
metric learning configurations on five video sequences.
The table reports their average tracking results in
CLE, VOR, and success rate.
 \vspace{-0.1cm}}
\label{Tab:metric_success_rate}
\end{table}

\vspace{-0.25cm}
\subsection{Comparison with the state-of-the-art trackers}

To demonstrate the effectiveness of the proposed tracking algorithm,
we make a qualitative and quantitative comparison with several
state-of-the-art trackers, referred to as
FragT (Fragment-based tracker \cite{Adam-Fragment-2006}),
MILT (multiple instance boosting-based
tracker~\cite{Babenko-Yang-Belongie-cvpr2009}), VTD (visual tracking
decomposition~\cite{Kwon-Lee-CVPR2010}), OAB (online
AdaBoost \cite{Grabner-Grabner-Bischof-BMVC2006}), IPCA (incremental
PCA~\cite{Limy-Ross17}), L1T ($\ell_{1}$ minimization
tracker~\cite{Meo-Ling-ICCV09}),
CT (compressive tracker~\cite{zhang2012real}), Struck (structured
learning tracker~\cite{harestruck_iccv2011}),
DML (discriminative metric learning
tracker~\cite{wang2010discriminative}),
TLD (tracking-learning-detection~\cite{TLD}),
ASLA (adaptive structural local sparse model~\cite{ASLSAM}),
and SCM (sparsity-based collaborative model~\cite{SCM}).
Moreover, the proposed tracking algorithm has two versions that are
unstructured and structured (respectively referred to as Ours and Ours+S).

\begin{figure*}[t]
\vspace{-0.1cm}
\centering
\includegraphics[width=0.8\linewidth]{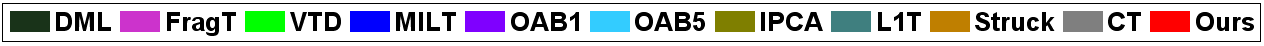}\\
\includegraphics[width=0.85\linewidth]{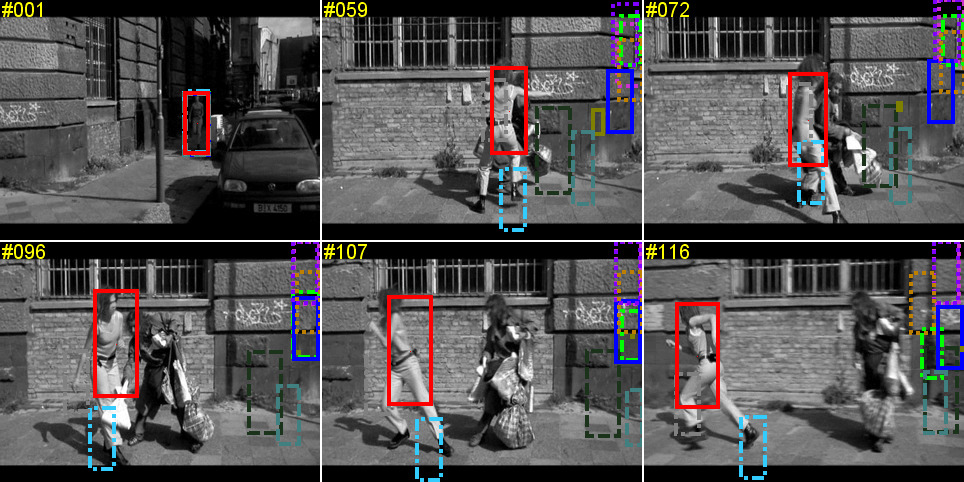}\\
 \caption{Tracking results of different trackers over some
representative frames from the ``Lola'' video sequence
in the scenarios with drastic scale changes and body pose variations.
}
 \label{fig:tracking_Lola} \vspace{-0.1cm}
\end{figure*}

\begin{figure*}[t]
\centering
\includegraphics[width=0.8\linewidth]{./Suppelmentary_File/NewFigs/NewLegendBar.PNG}\\
\includegraphics[width=0.85\linewidth]{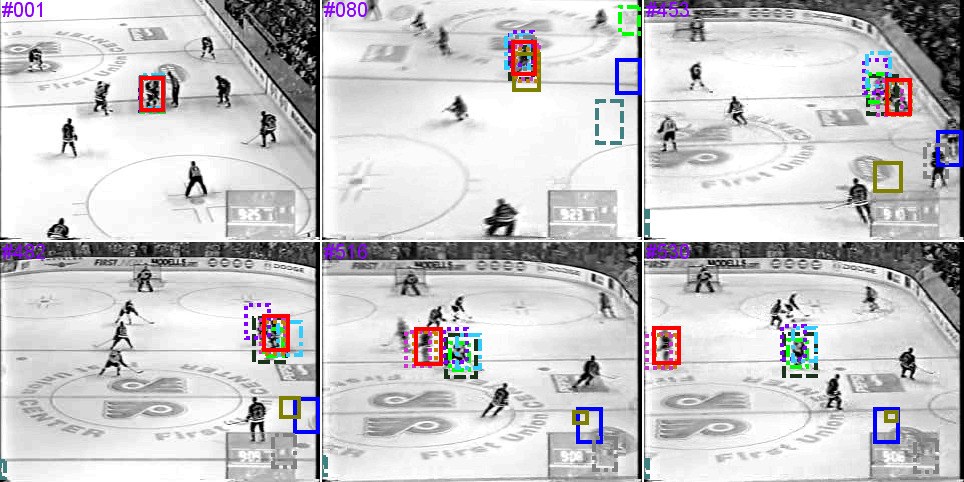}\\
 \caption{Tracking results of different trackers over some
representative frames from the ``iceball'' video sequence
in the scenarios with partial occlusions,  out-of-plane rotations, body pose variations, and abrupt motion.
}
 \label{fig:tracking_iceball} \vspace{-0.1cm}
\end{figure*}

\begin{figure*}[t]
\centering
\includegraphics[width=0.8\linewidth]{./Suppelmentary_File/NewFigs/NewLegendBar.PNG}\\
\includegraphics[width=0.85\linewidth]{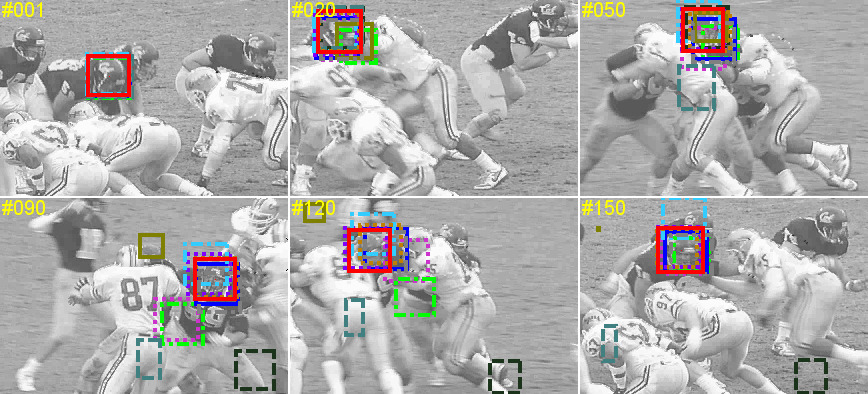}\\
 \caption{Tracking results of different trackers over some
representative frames from the ``football3'' video sequence
in the scenarios with motion blurring, partial occlusions,
head pose variations, and background clutters.
}
 \label{fig:tracking_football3} \vspace{-0.1cm}
\end{figure*}

\begin{figure*}[t]
\centering
\includegraphics[width=0.8\linewidth]{./Suppelmentary_File/NewFigs/NewLegendBar.PNG}\\
\includegraphics[width=0.85\linewidth]{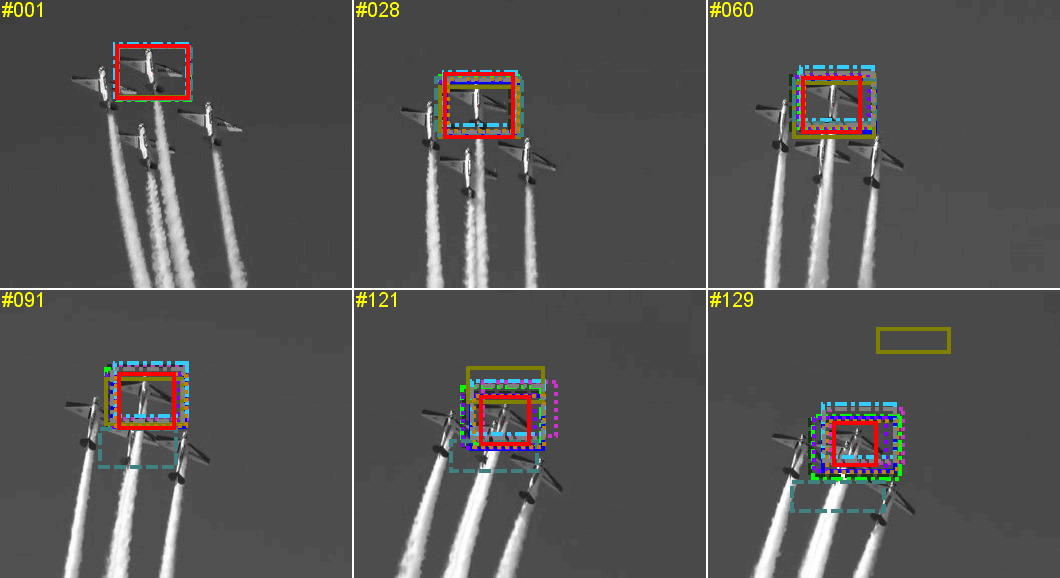}\\
 \caption{Tracking results of different trackers over some
representative frames from the ``planeshow'' video sequence
in the scenarios with shape deformations, out-of-plane rotations, and pose variations.
}
 \label{fig:tracking_planeshow} \vspace{-0.1cm}
\end{figure*}

\begin{figure*}[t]
\centering
\includegraphics[width=0.8\linewidth]{./Suppelmentary_File/NewFigs/NewLegendBar.PNG}\\
\includegraphics[width=0.85\linewidth]{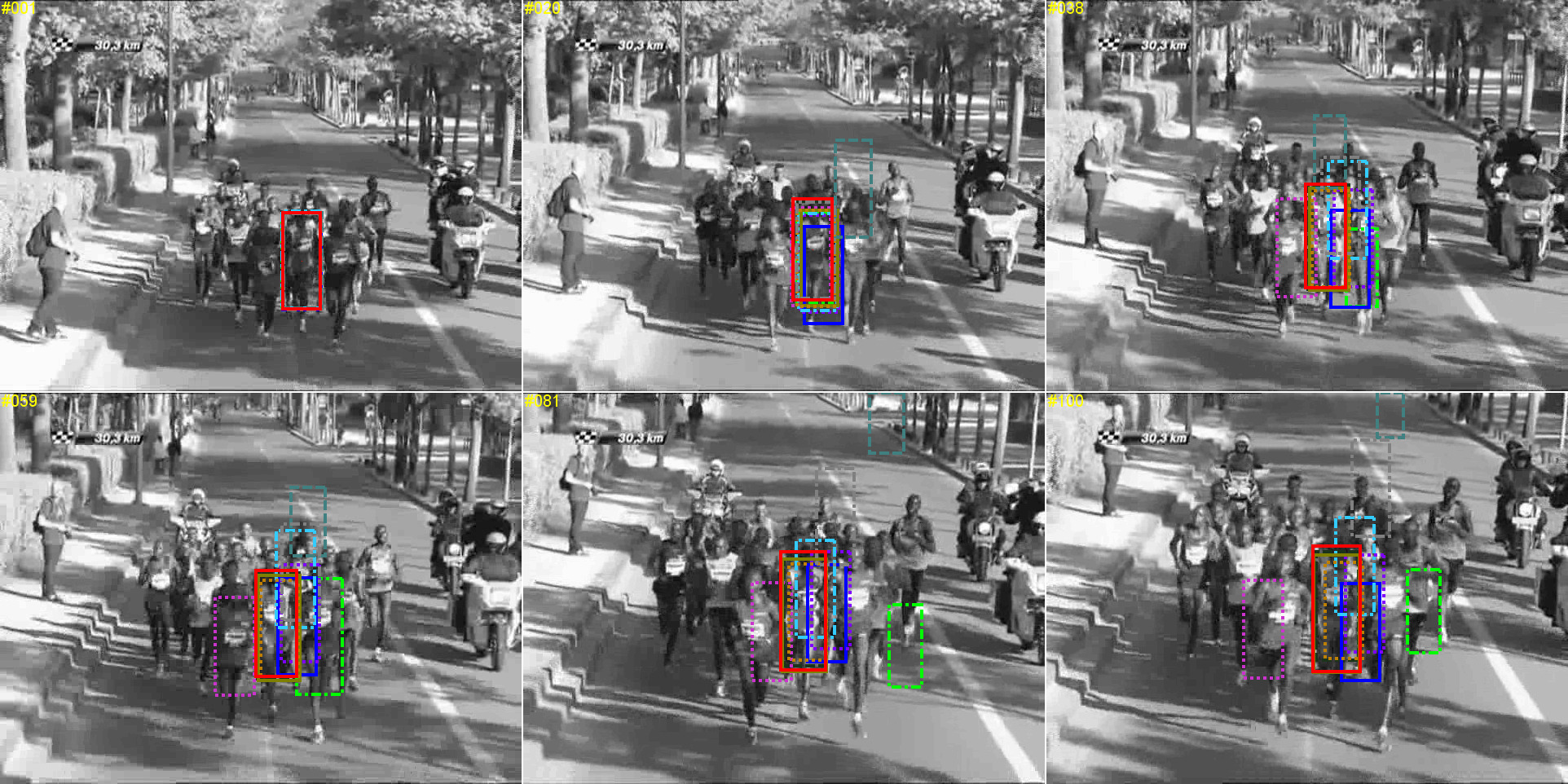}\\
 \caption{Tracking results of different trackers over some
representative frames from the ``race'' video sequence
in the scenarios with background clutters.
}
 \label{fig:tracking_race} \vspace{-0.1cm}
\end{figure*}

\begin{figure*}[t]
\centering
\includegraphics[width=0.8\linewidth]{./Suppelmentary_File/NewFigs/NewLegendBar.PNG}\\
\includegraphics[width=0.85\linewidth]{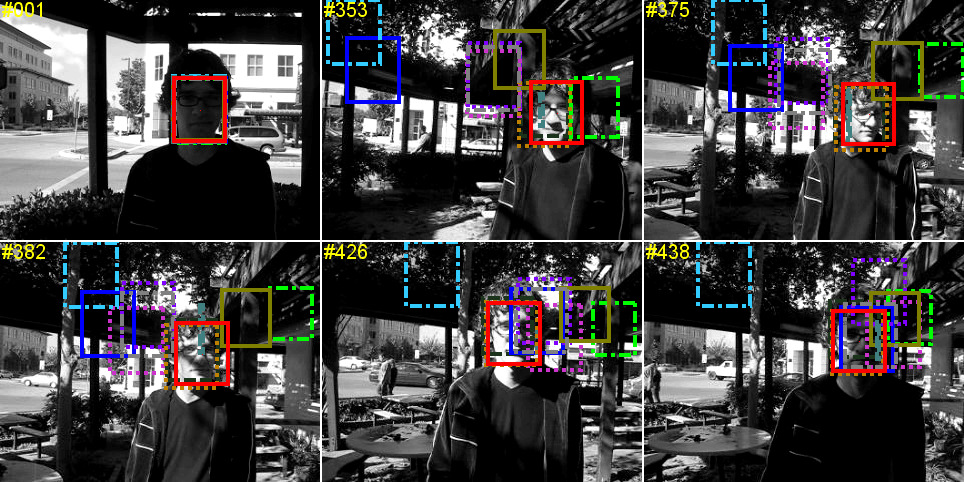}\\
 \caption{Tracking results of different trackers over some
representative frames from the ``trellis70'' video sequence
in the scenarios with drastic illumination changes and  head pose variations.
}
 \label{fig:tracking_trellis70} \vspace{-0.1cm}
\end{figure*}

\begin{figure*}[t]
\centering
 \vspace{-0.6cm}
\includegraphics[scale=0.3]{./Suppelmentary_File/NewFigs/NewLegendBar.PNG}\\
\includegraphics[scale=0.41]{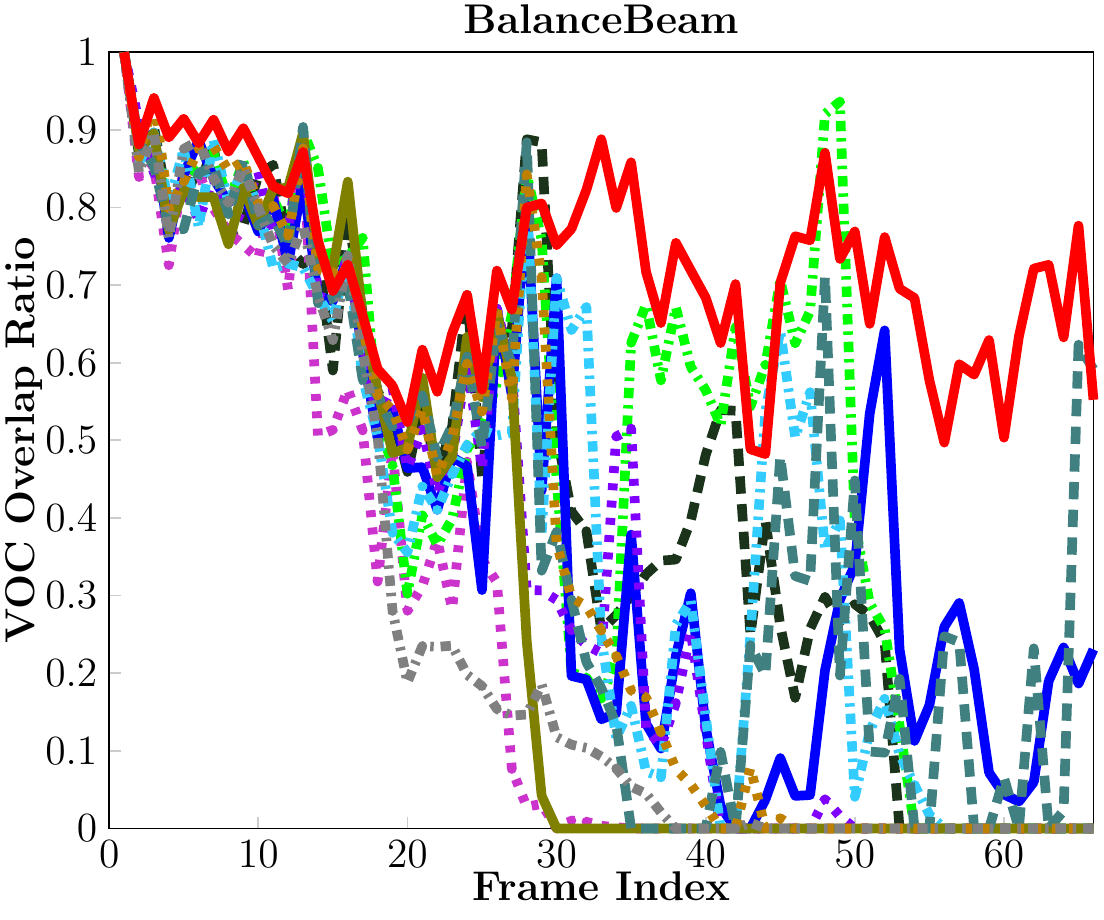}
\includegraphics[scale=0.41]{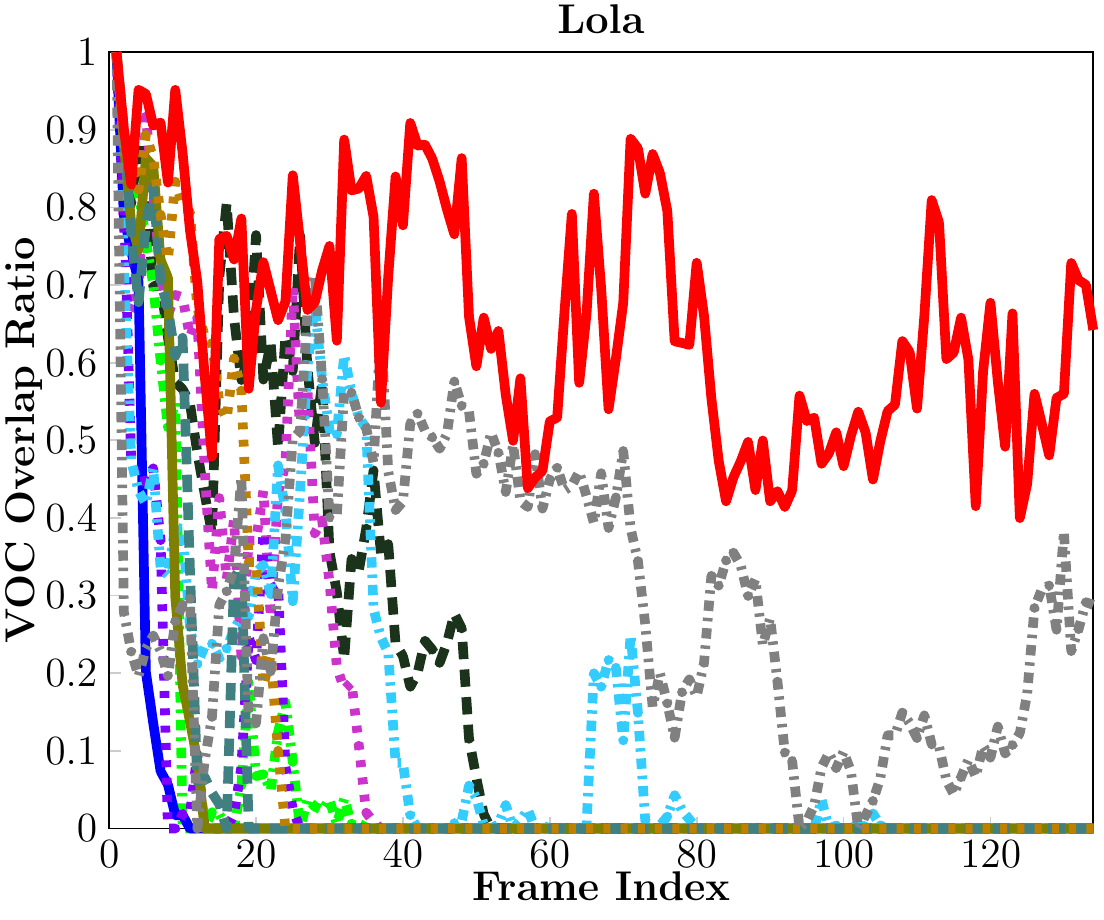}
\includegraphics[scale=0.41]{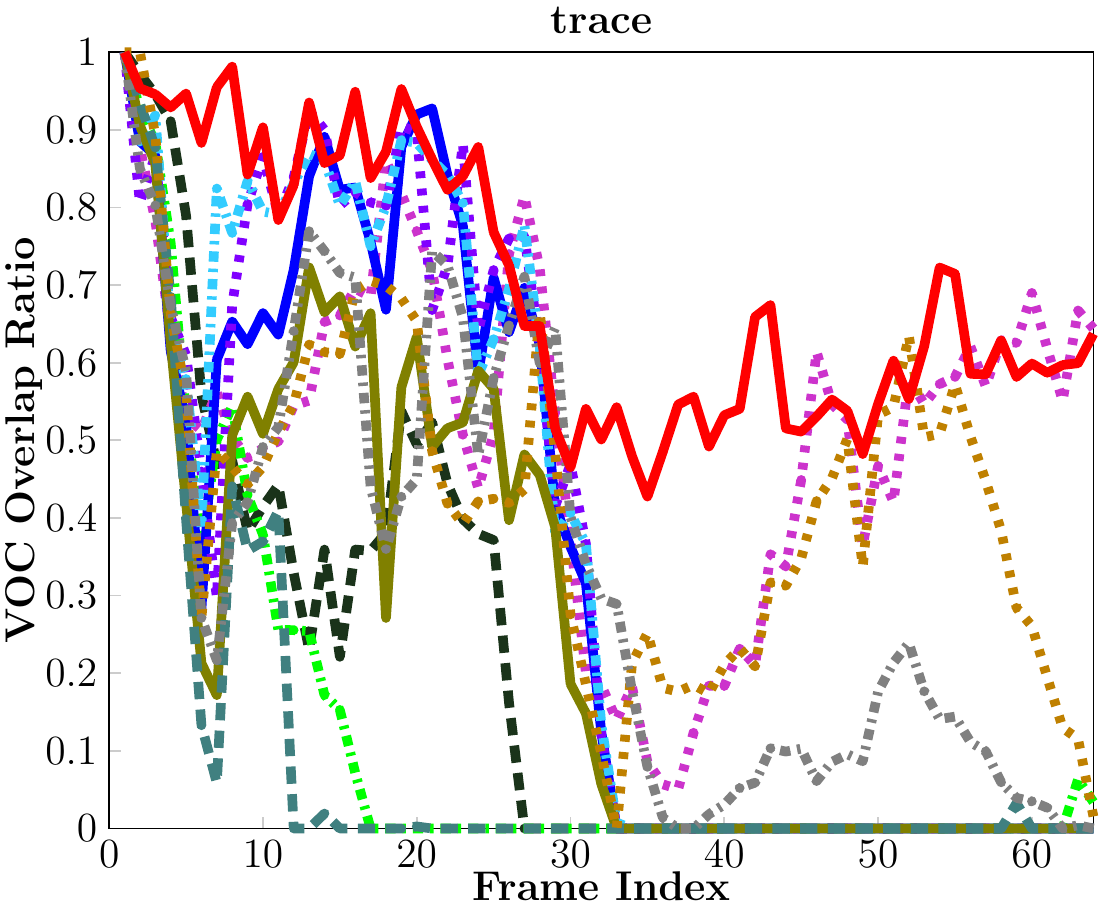}\\
\includegraphics[scale=0.41]{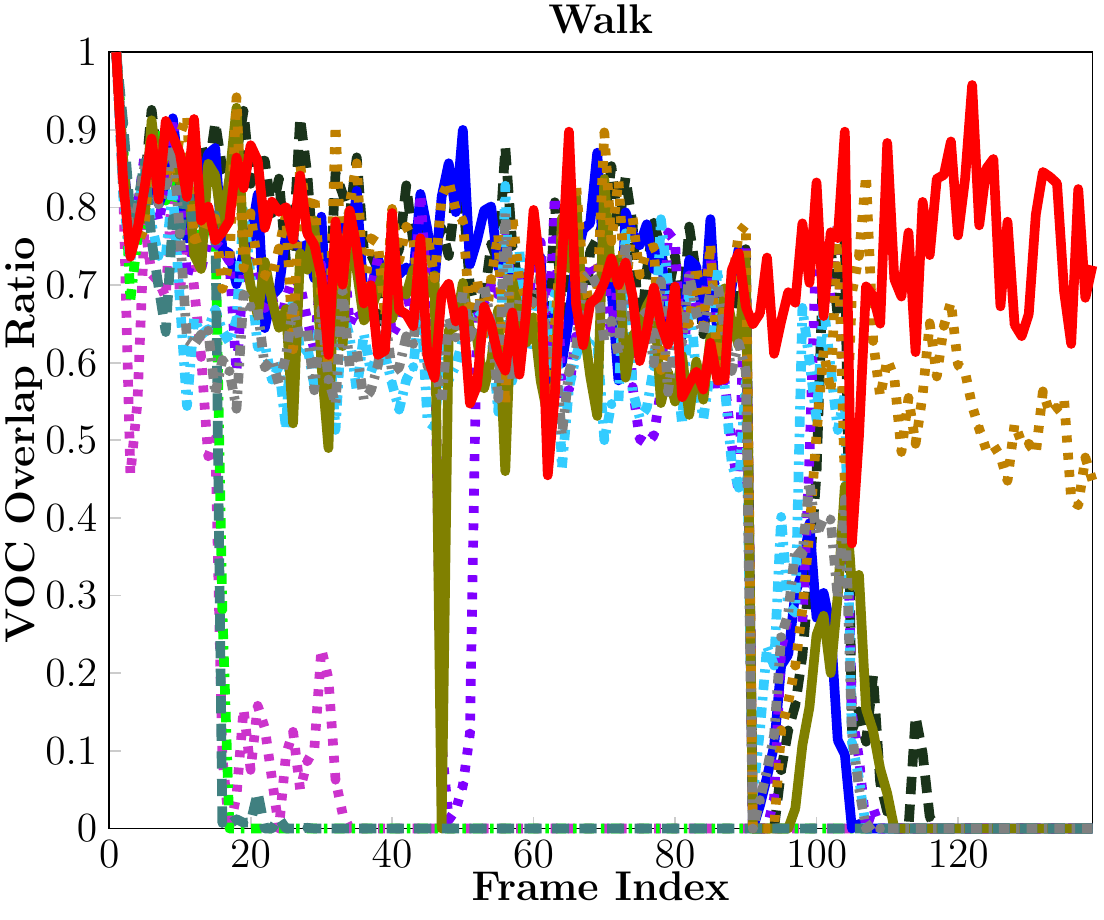}
\includegraphics[scale=0.41]{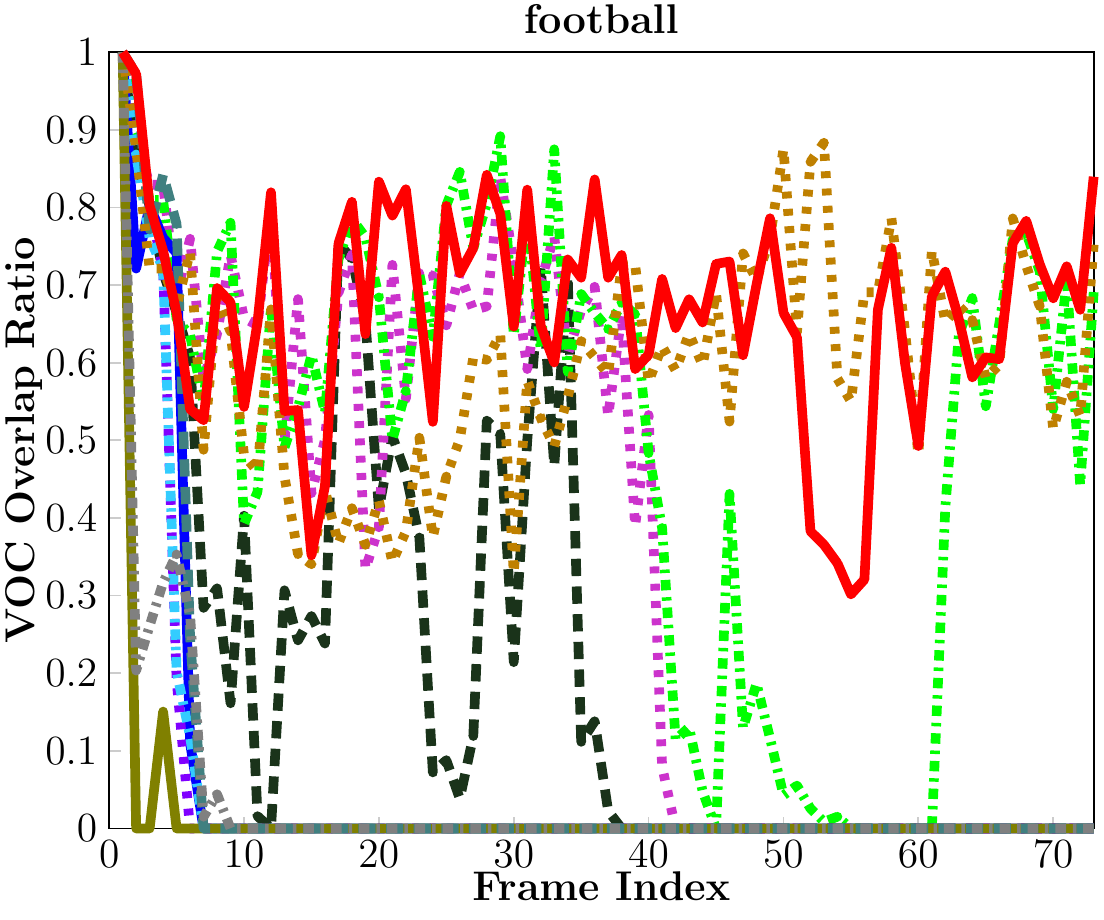}
\includegraphics[scale=0.41]{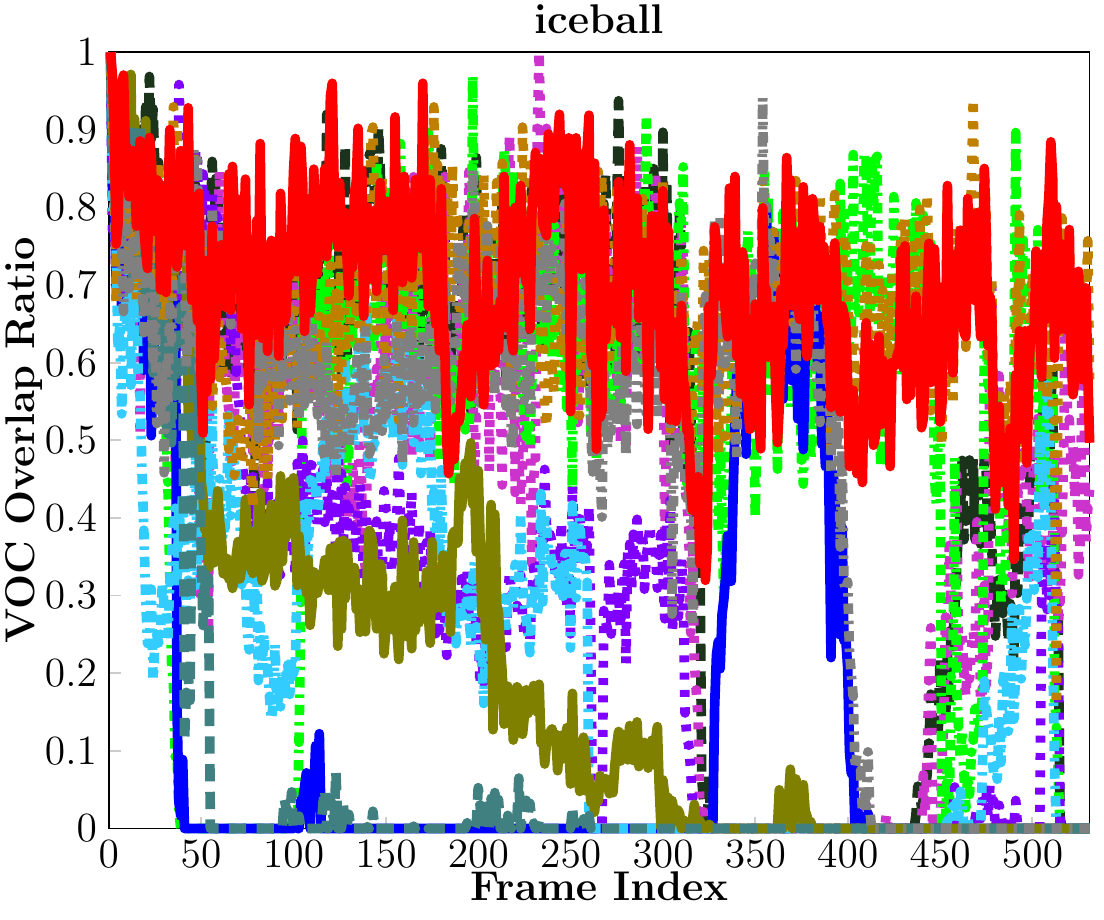}\\
\includegraphics[scale=0.41]{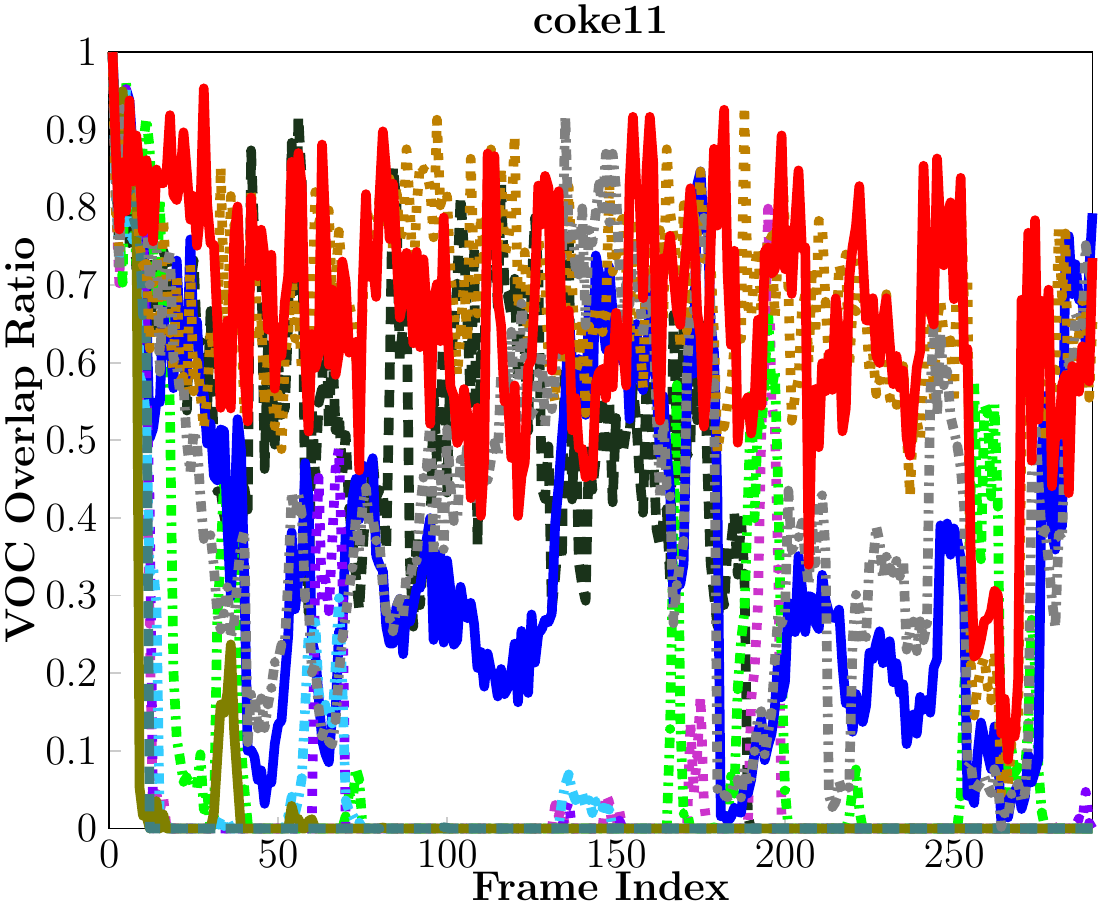}
\includegraphics[scale=0.41]{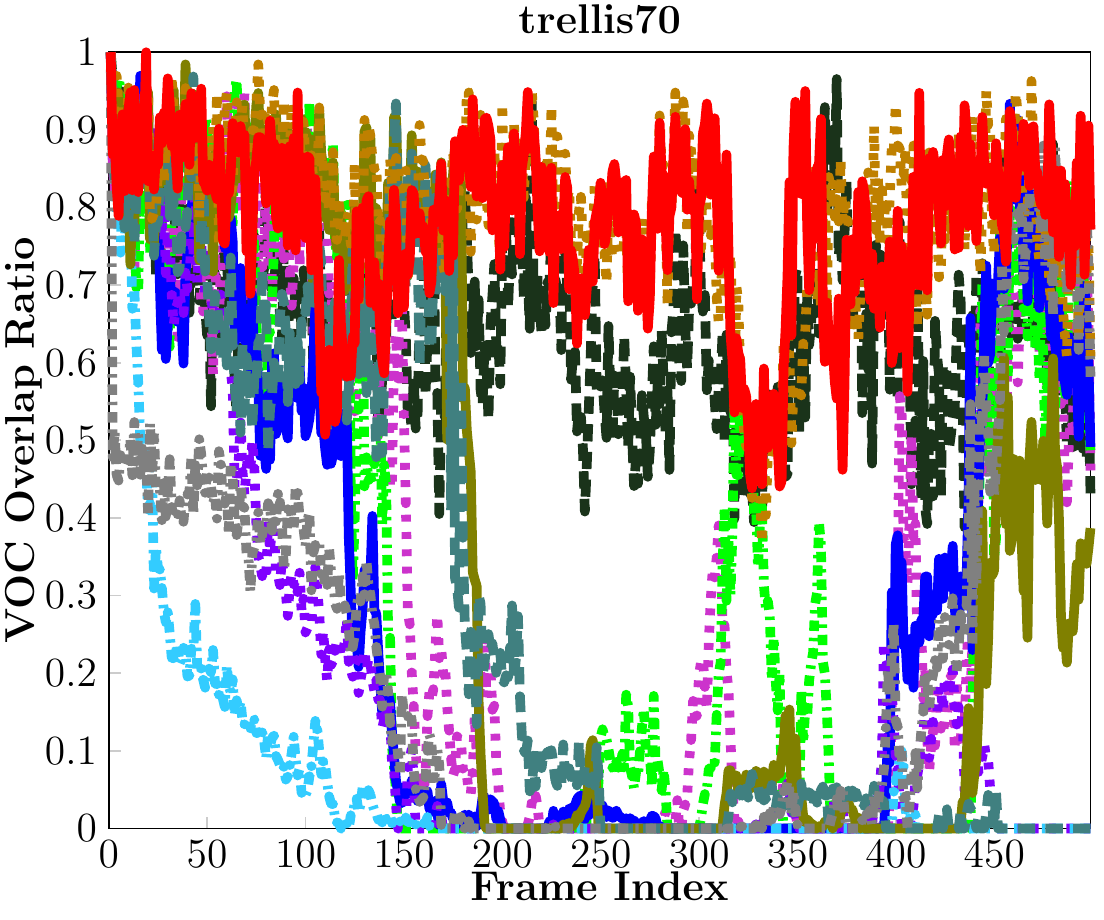}
\includegraphics[scale=0.41]{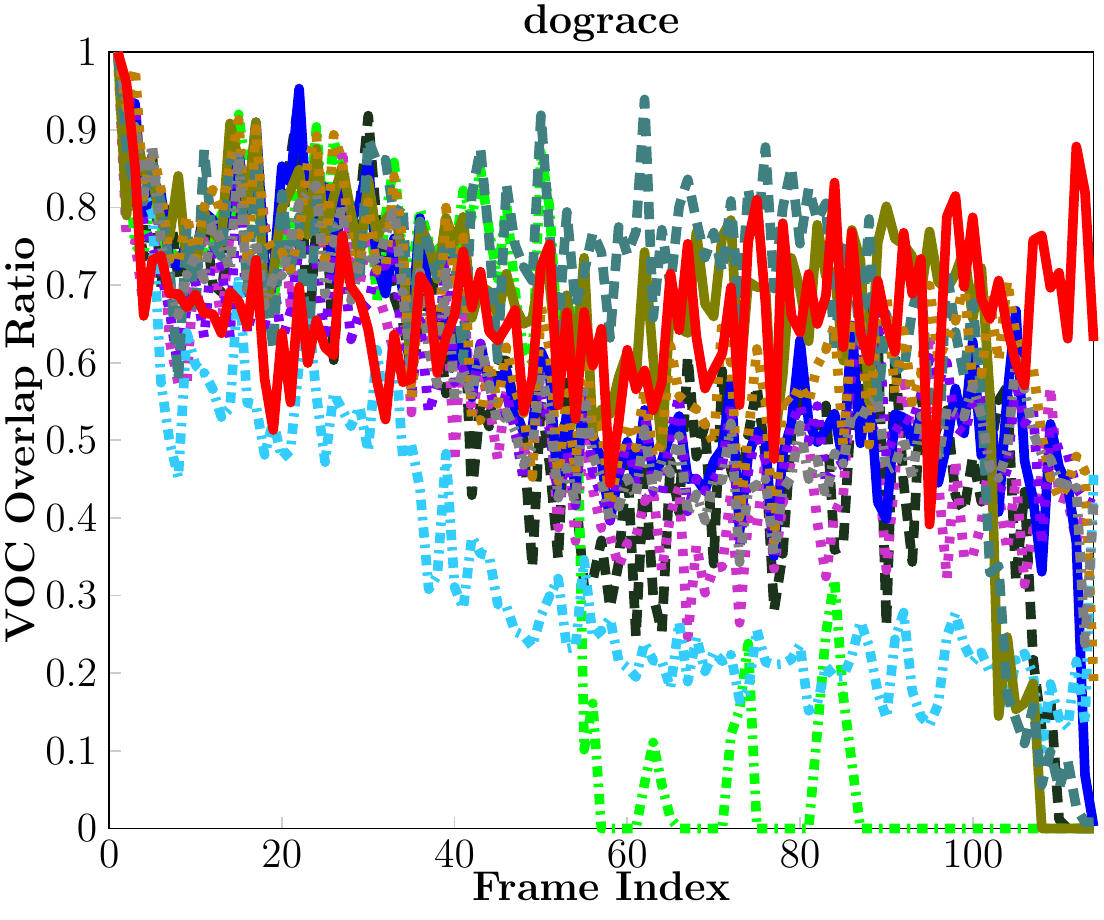}\\
\includegraphics[scale=0.41]{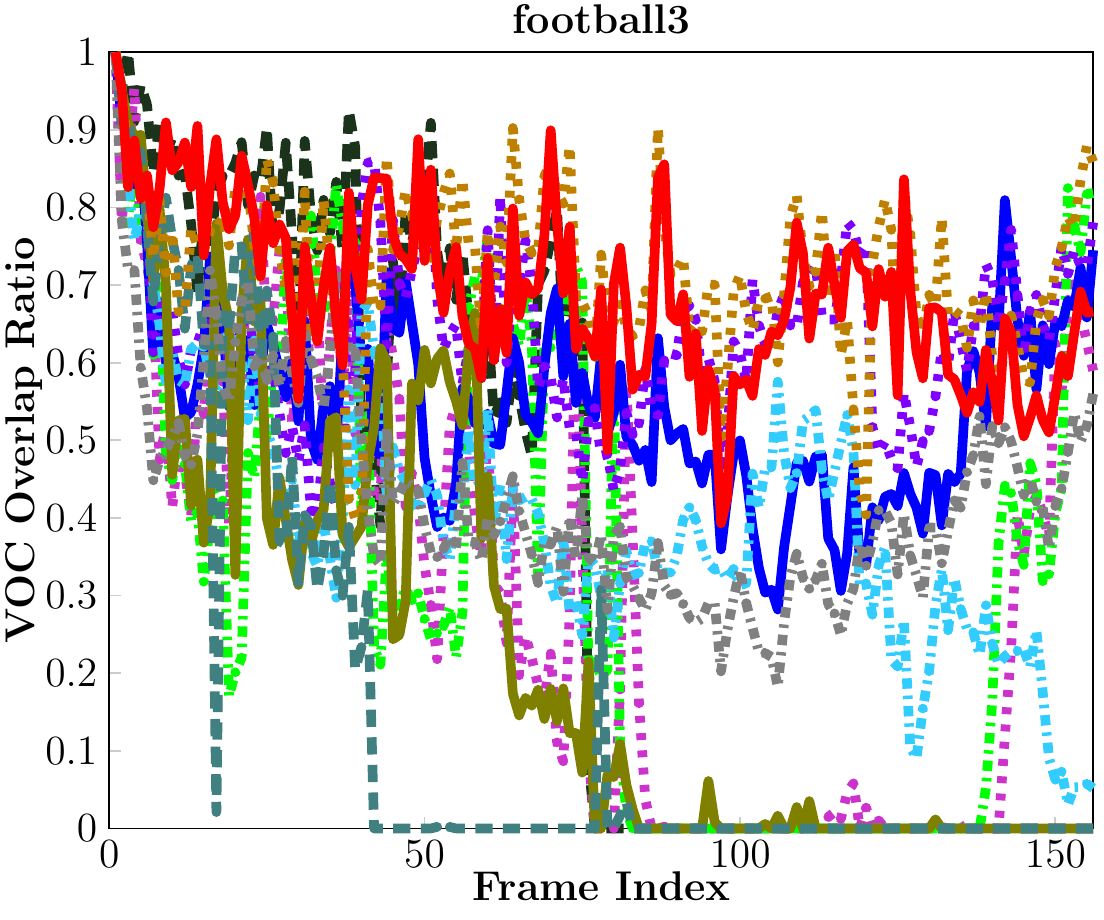}
\includegraphics[scale=0.41]{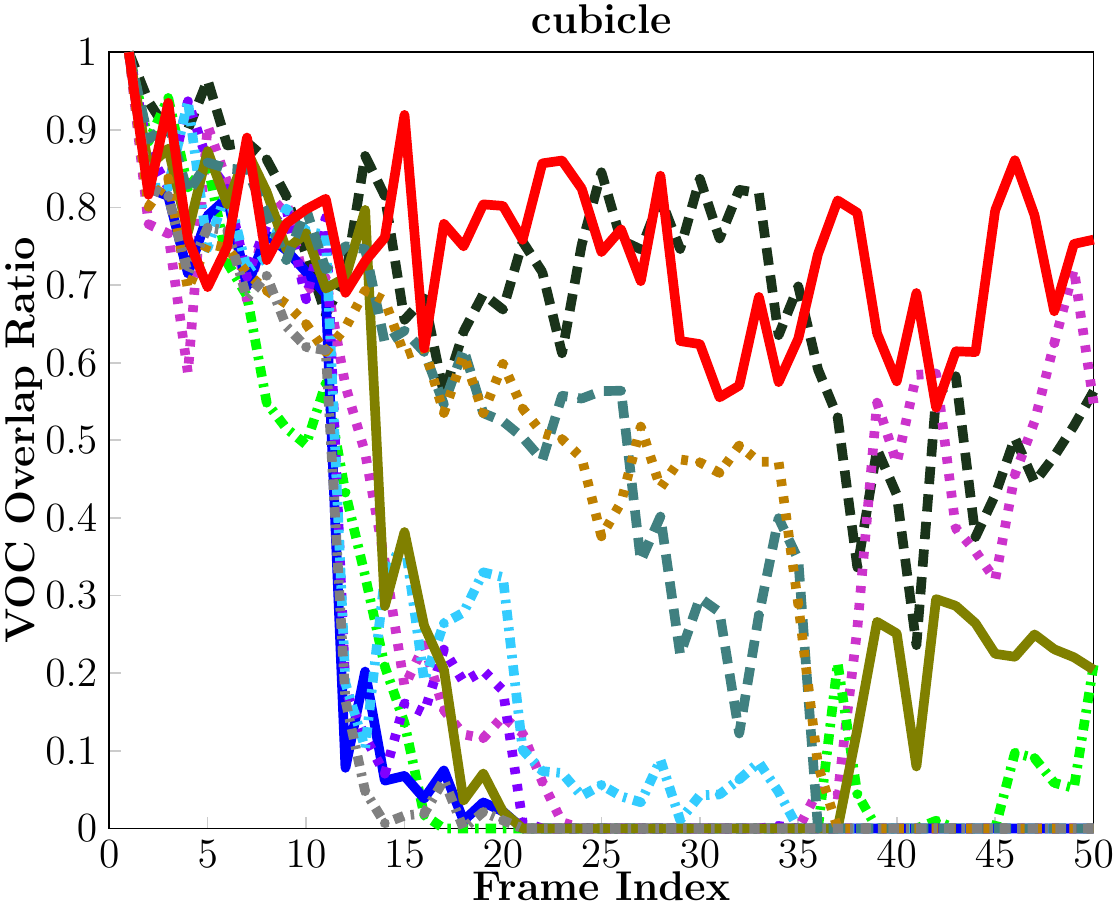}
\includegraphics[scale=0.41]{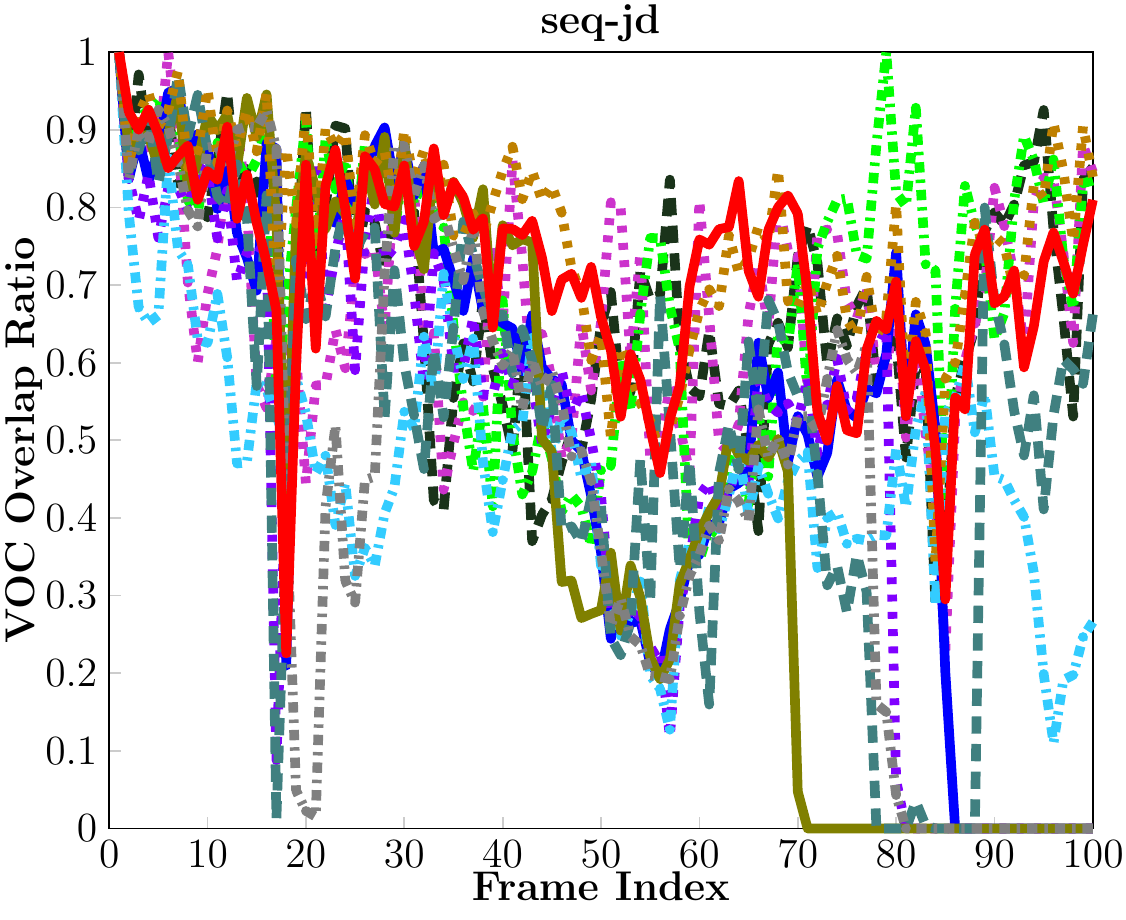}\\
\includegraphics[scale=0.41]{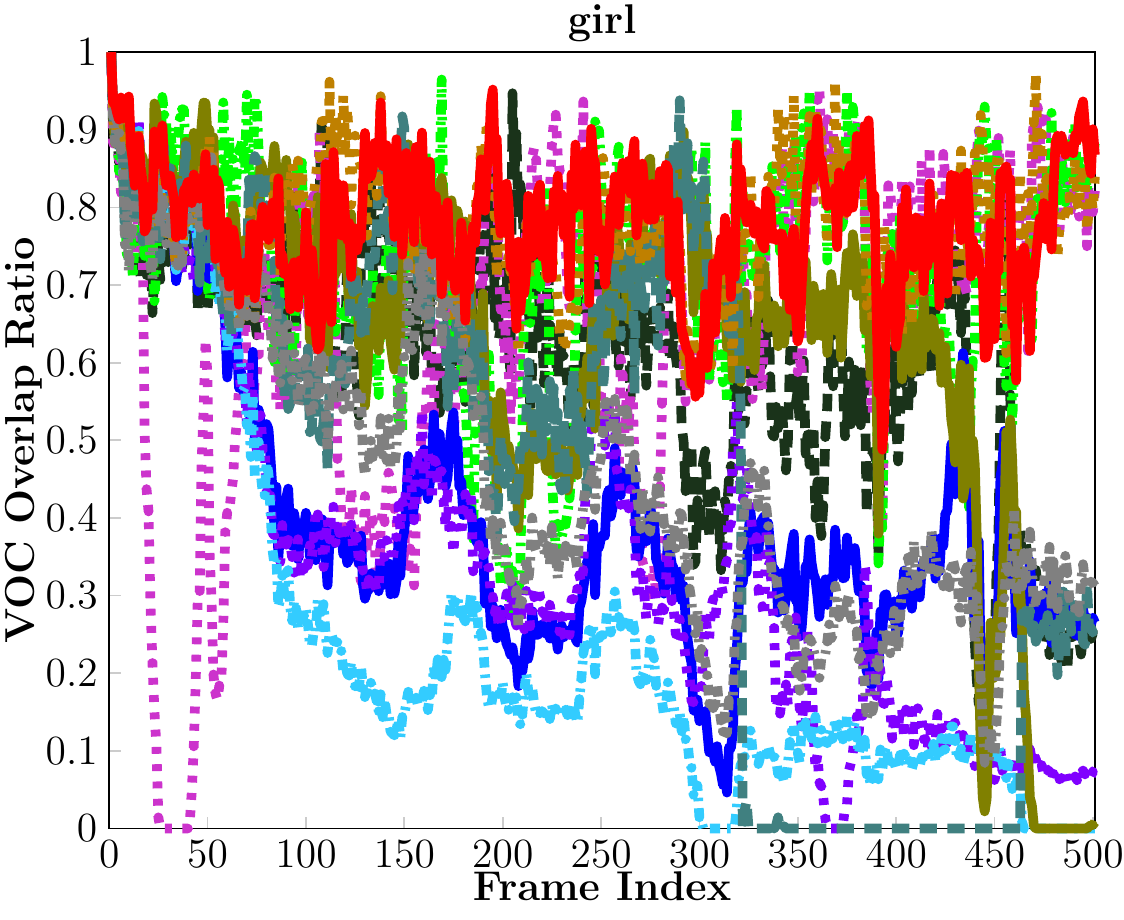}
\includegraphics[scale=0.41]{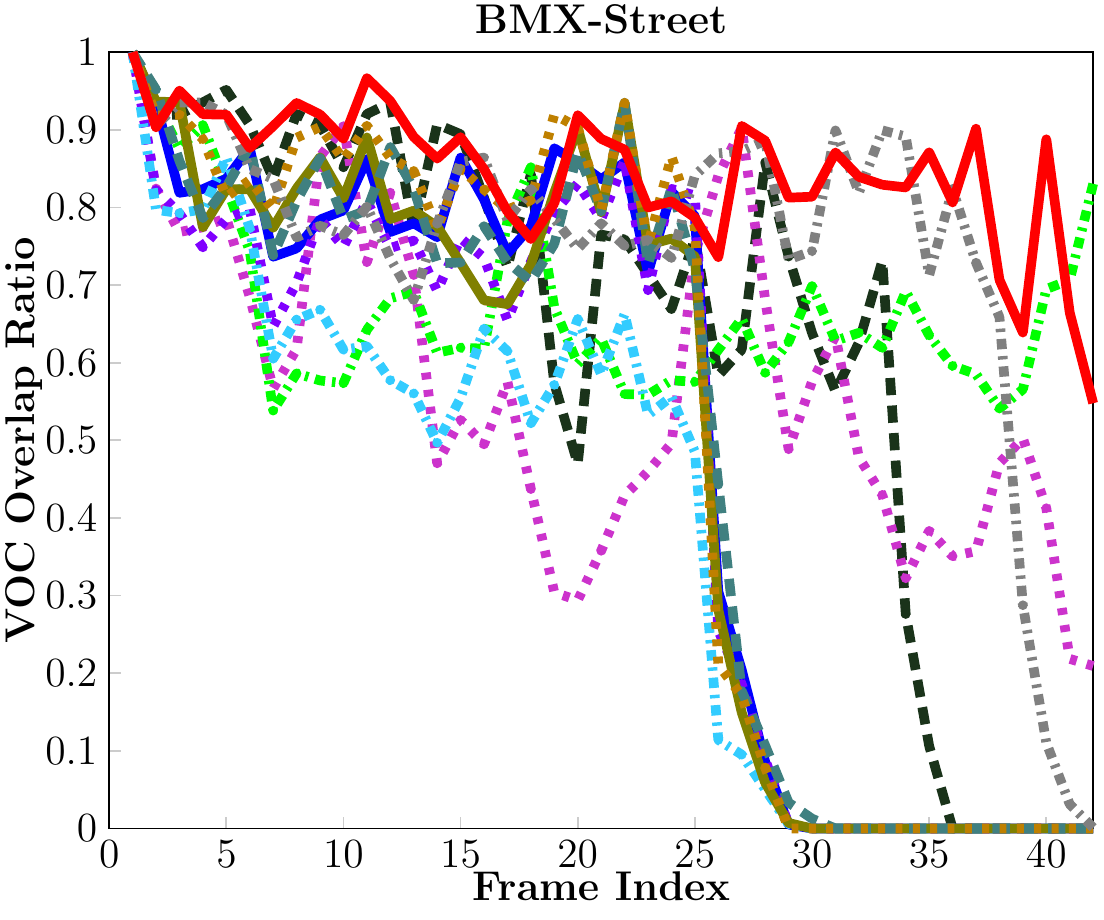}
\includegraphics[scale=0.41]{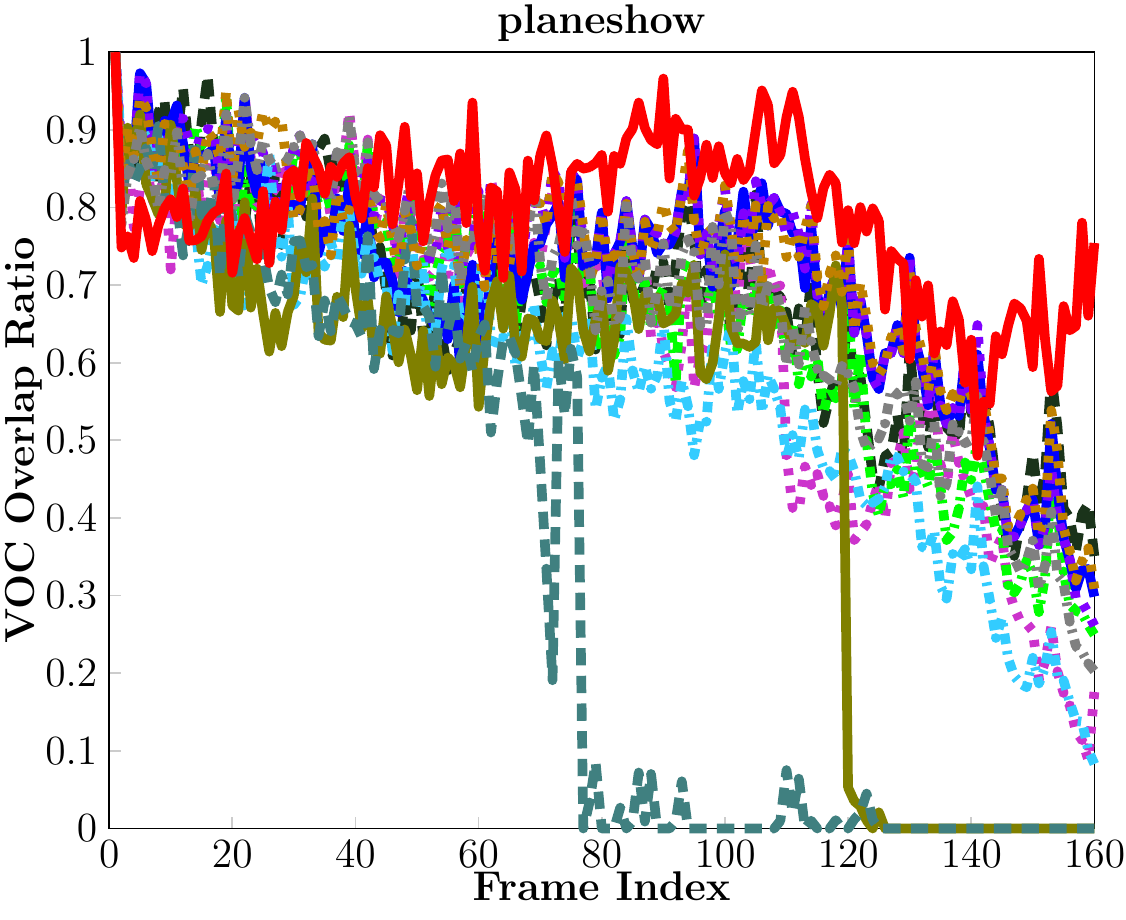}\\
\vspace{-0.13cm}
\caption{Quantitative comparison of different trackers in VOR on the first fifteen
video sequences.}
 \vspace{-0.11cm}
  \label{fig:exp_voc_curve}
\end{figure*}

In the experiments,
the following
trackers are implemented using their publicly available source code: FragT,
MILT,
VTD,
OAB, CT, Struck
IPCA,
L1T, TLD, ASLA,
and SCM.
For OAB, there are two different versions (namely, OAB1 and OAB5), which
are based on two different configurations (i.e., the search scale $r=1$ and $r=5$
as in \cite{Babenko-Yang-Belongie-cvpr2009}).

Figs.~\ref{fig:tracking_Lola}-\ref{fig:tracking_trellis70} show
the qualitative tracking results of the eleven trackers
on several sample frames on six video sequences.
Fig.~\ref{fig:exp_voc_curve} and Tab.~\ref{Tab:quantitative} report the quantitative tracking
results of the eleven trackers (in CLE, VOR, and success rate) over several video sequences.
The complete tracking results and quantitative comparisons for all the
eighteen video sequences can be found in
the supplementary file.
From Fig.~\ref{fig:exp_voc_curve} and Tab.~\ref{Tab:quantitative},
we observe that the proposed tracking algorithm
achieves the best tracking performance by all measures on most video sequences.
In the experiments, the TLD tracker (using the default parameter settings) produces the incomplete tracking results over some video sequences
because of its particular tracking-learning-detection properties (i.e., tracking reliability analysis
by simultaneously performing object detection and optical flow-based verification). Therefore, we only show
the video sequences in which the TLD tracker can always achieve stable tracking performances for all the frames.
The reasons for the incomplete tracking results are briefly analyzed as follows.
In principle, the TLD tracker takes a tracking-by-detection strategy that
needs to simultaneously perform object detection as well as optical flow-based tracking verification across successive frames.
In the case of severe occlusions (or drastic pose changes or tiny objects or strong background clutters),
it adaptively evaluates the tracking reliability by performing optical flow-based tracking verification or object classification
(whose classification score may be very low), and is likely to
remove the unreliable tracking results, leading to the tracking unavailability over several frames.
With the emergence of the visually feasible tracked objects, the detection component of
the TLD tracker is automatically activated to localize the tracked objects.

Subsequently, we briefly analyze the reasons why our tracker works
well in some challenging situations. In essence, the foreground samples
stored in the buffer approximately constitute an object manifold
that contains the intrinsic structural information on object appearance.
After distance metric learning, the object manifold encodes
more discriminative information on object/non-object classification.
If test samples are contaminated by some complicated
factors (e.g., shape deformation, noisy corruption, and illumination variation),
the intrinsic manifold structural properties of object appearance
are very helpful to recover these test samples from contamination by manifold embedding
(i.e., metric-weighted linear regression).
Moreover, time-weighted reservoir sampling is able to ensure
that the sample buffer retains useful old samples with
a long lifespan and meanwhile adapts to recent appearance changes.
Therefore, metric-weighted linear regression on such a sample buffer
can not only alleviate the tracker drift problem but also adapt to
complicated appearance changes. Besides, metric-weighted
linear regression on the background buffer can generate
the discriminative information to help the tracker
reject several false foreground samples (caused by occlusion, out-of-plane rotation, and pose variation).
Finally, the used features during tracking are extracted in a block-division manner.
Therefore, they are capable of encoding the local geometrical information on object appearance,
leading to the robustness to complicated scenarios (e.g., partial occlusion).
Of course, when the appearance distinction between the foreground
and background samples (especially for background clutter) is small,
discriminative distance metric learning cannot improve
the performance of metric-weighted linear regression.
In this case, our tracker is incapable of accurately capturing
the target location (e.g., the ``coke11'' sequence shown in Tab.~\ref{Tab:quantitative}) or even failing to track.

\begin{table}[t]
\centering
\scalebox{0.56}
{
   \centering
    \begin{tabular}{c|c|c|c|c|c|c|c|c|c|c|c|c|c|c|c|c|c|c|c}\hline

          &       & B-Beam & Lola  & trace & Walk  & football & iceball & coke11 & trellis70 & dograce & football3 & cubicle & seq-jd & girl  & B-Street & planeshow & race  & CamSeq01 & car11 \\\hline

    {\multirow{15}[2]{*}{CLE}} & Ours+S & {\textbf{3.52 }} & {\textbf{10.67 }} & {\textbf{6.23 }} & {\textbf{5.91 }} & {\textbf{2.88 }} & {\textbf{2.75 }} & {\textbf{5.13 }} & {\textbf{4.22 }} & {\textbf{3.12 }} & {5.63} & {\textbf{3.10 }} & {\textbf{3.02 }} & {\bf 7.35 } & {\textbf{3.11 }} & {\textbf{2.89 }} & {6.96 } & {\textbf{2.87 }} & {2.62 } \\
    {} & Ours  & {4.68 } & {11.08 } & {8.65 } & {6.08 } & {3.14 } & {3.03 } & {5.76 } & {5.62 } & {3.54 } & {\bf 3.92 } & {4.31 } & {4.30 } & {7.72 } & {3.31 } & {3.00 } & {8.52 } & {3.00 } & {2.74 } \\
    {} & DML   & {22.03 } & {102.62 } & {91.42 } & {30.23 } & {65.34 } & {15.94 } & {30.83 } & {9.10 } & {10.57 } & {89.66 } & {5.02 } & {5.47 } & {20.26 } & {22.10 } & {3.28 } & {7.29 } & {3.74 } & {3.79 } \\
    {} & FragT & {46.61 } & {141.00 } & {19.70 } & {92.46 } & {31.95 } & {13.86 } & {59.54 } & {39.86 } & {8.43 } & {28.52 } & {25.54 } & {5.83 } & {20.92 } & {15.29 } & {10.04 } & {45.28 } & {9.20 } & {65.31 } \\
    {} & VTD   & {15.82 } & {142.69 } & {110.25 } & {103.31 } & {9.81 } & {14.35 } & {45.66 } & {51.68 } & {24.56 } & {32.88 } & {46.11 } & {4.88 } & {10.97 } & {7.27 } & {3.64 } & {65.38 } & {7.09 } & {32.18 } \\
    {} & MILT  & {19.08 } & {138.97 } & {68.00 } & {37.08 } & {103.59 } & {76.09 } & {17.71 } & {60.41 } & {7.97 } & {9.23 } & {43.53 } & {16.24 } & {39.96 } & {41.16 } & {4.99 } & {25.55 } & {13.26 } & {44.79 } \\
    {} & OAB1  & {31.10 } & {146.08 } & {65.14 } & {36.47 } & {76.53 } & {24.64 } & {57.90 } & {68.58 } & {8.51 } & {6.78 } & {41.66 } & {20.42 } & {51.51 } & {43.32 } & {4.18 } & {33.44 } & {8.21 } & {24.73 } \\
    {} & OAB5  & {24.15 } & {67.16 } & {66.31 } & {36.74 } & {97.16 } & {37.00 } & {64.97 } & {126.36 } & {18.56 } & {14.87 } & {37.06 } & {11.71 } & {67.13 } & {47.71 } & {11.30 } & {42.18 } & {5.85 } & {9.85 } \\
    {} & IPCA  & {41.64 } & {127.88 } & {80.60 } & {37.33 } & {133.67 } & {41.60 } & {56.60 } & {46.13 } & {6.70 } & {31.99 } & {33.61 } & {24.16 } & {20.90 } & {41.69 } & {28.02 } & {\bf 5.26 } & {21.22 } & {2.42 } \\
    {} & L1T   & {22.60 } & {139.89 } & {108.64 } & {124.09 } & {103.08 } & {119.69 } & {64.70 } & {27.64 } & {5.28 } & {64.07 } & {24.26 } & {12.76 } & {44.12 } & {46.98 } & {32.27 } & {160.79 } & {37.40 } & {25.46 } \\
    {} & Struck & {28.62 } & {139.90 } & {24.15 } & {10.83 } & {3.24 } & {3.76 } & {5.50 } & {4.82 } & {4.54 } & {4.32 } & {22.61 } & {3.32 } & {7.38 } & {42.06 } & {5.77 } & {7.71 } & {6.54 } & {2.09 } \\
    {} & CT    & {37.22 } & {23.04 } & {41.59 } & {35.94 } & {106.72 } & {29.97 } & {15.83 } & {50.97 } & {8.98 } & {13.50 } & {45.78 } & {22.06 } & {34.10 } & {9.33 } & {6.61 } & {76.20 } & {11.34 } & {29.43 } \\
    {} & ALSA  & {15.73 } & {30.50 } & {23.11 } & {28.14 } & {3.99 } & {3.17 } & {12.98 } & {4.99 } & {11.20 } & {31.23 } & {25.34 } & {20.19 } & {36.87 } & {6.44 } & {3.19 } & {35.41 } & {10.38 } & {2.05 } \\
    {} & SCM   & {16.98 } & {22.94 } & {20.35 } & {31.81 } & {26.60 } & {3.56 } & {94.74 } & {12.30 } & {8.68 } & {36.85 } & {23.04 } & {11.66 } & {7.47 } & {8.25 } & {5.18 } & {33.12 } & {7.17 } & {\textbf{1.92 }} \\
    {} & TLD   & {-} & {-} & {-} & {-} & {-} & {-} & {9.73 } & {13.11 } & {-} & {-} & {19.58 } & {-} & {19.15 } & {-} & {-} & {-} & {9.10 } & {-} \\\hline
    {\multirow{15}[2]{*}{VOR}} & Ours+S & {\textbf{0.78 }} & {\textbf{0.67 }} & {\textbf{0.74 }} & {\textbf{0.73 }} & {\textbf{0.69 }} & {\textbf{0.71 }} & {0.65 } & {\textbf{0.83 }} & {\textbf{0.67 }} & {\textbf{0.72 }} & {\textbf{0.78 }} & {\textbf{0.76 }} & {\textbf{0.78 }} & {\textbf{0.86 }} & {\textbf{0.80 }} & {\textbf{0.82 }} & {\textbf{0.83 }} & {0.77 } \\
    {} & Ours  & {0.72 } & {0.66 } & {0.70 } & {0.72 } & {0.67 } & {0.69 } & {0.65 } & {0.78 } & {0.66 } & {\textbf{0.72 }} & {0.74 } & {0.72 } & {\textbf{0.78 }} & {0.85 } & {0.79 } & {0.80 } & {0.82 } & {0.76 } \\
    {} & DML   & {0.43 } & {0.18 } & {0.20 } & {0.54 } & {0.21 } & {0.47 } & {0.36 } & {0.65 } & {0.54 } & {0.36 } & {0.68 } & {0.68 } & {0.60 } & {0.63 } & {0.68 } & {0.78 } & {0.79 } & {0.69 } \\
    {} & FragT & {0.24 } & {0.13 } & {0.51 } & {0.09 } & {0.36 } & {0.44 } & {0.05 } & {0.34 } & {0.54 } & {0.30 } & {0.35 } & {0.67 } & {0.64 } & {0.57 } & {0.64 } & {0.23 } & {0.65 } & {0.08 } \\
    {} & VTD   & {0.49 } & {0.06 } & {0.12 } & {0.09 } & {0.50 } & {0.52 } & {0.11 } & {0.33 } & {0.38 } & {0.32 } & {0.20 } & {0.70 } & {0.73 } & {0.67 } & {0.66 } & {0.16 } & {0.74 } & {0.37 } \\
    {} & MILT  & {0.41 } & {0.03 } & {0.34 } & {0.50 } & {0.06 } & {0.13 } & {0.35 } & {0.29 } & {0.59 } & {0.55 } & {0.18 } & {0.53 } & {0.39 } & {0.50 } & {0.71 } & {0.34 } & {0.63 } & {0.15 } \\
    {} & OAB1  & {0.34 } & {0.04 } & {0.35 } & {0.46 } & {0.05 } & {0.26 } & {0.04 } & {0.16 } & {0.57 } & {0.63 } & {0.21 } & {0.47 } & {0.34 } & {0.47 } & {0.73 } & {0.28 } & {0.64 } & {0.37 } \\
    {} & OAB5  & {0.40 } & {0.13 } & {0.36 } & {0.45 } & {0.05 } & {0.21 } & {0.04 } & {0.06 } & {0.35 } & {0.40 } & {0.24 } & {0.46 } & {0.24 } & {0.39 } & {0.58 } & {0.36 } & {0.74 } & {0.43 } \\
    {} & IPCA  & {0.30 } & {0.05 } & {0.26 } & {0.46 } & {0.02 } & {0.19 } & {0.03 } & {0.36 } & {0.65 } & {0.23 } & {0.30 } & {0.45 } & {0.62 } & {0.50 } & {0.51 } & {0.79 } & {0.49 } & {0.78 } \\
    {} & L1T   & {0.41 } & {0.07 } & {0.09 } & {0.09 } & {0.06 } & {0.07 } & {0.03 } & {0.28 } & {0.67 } & {0.16 } & {0.41 } & {0.52 } & {0.46 } & {0.50 } & {0.33 } & {0.08 } & {0.27 } & {0.43 } \\
    {} & Struck & {0.35 } & {0.11 } & {0.43 } & {0.65 } & {0.60 } & {0.68 } & {\textbf{0.66 }} & {0.80 } & {0.64 } & {\textbf{0.72 }} & {0.41 } & {\textbf{0.76 }} & {\textbf{0.78 }} & {0.52 } & {0.73 } & {0.66 } & {0.68 } & {\textbf{0.79 }} \\
    {} & CT    & {0.25 } & {0.30 } & {0.33 } & {0.45 } & {0.03 } & {0.47 } & {0.41 } & {0.22 } & {0.57 } & {0.43 } & {0.17 } & {0.45 } & {0.45 } & {0.75 } & {0.69 } & {0.26 } & {0.63 } & {0.29 } \\
    {} & ALSA  & {0.47 } & {0.39 } & {0.40 } & {0.51 } & {0.57 } & {0.60 } & {0.42 } & {0.78 } & {0.55 } & {0.30 } & {0.24 } & {0.46 } & {0.39 } & {0.76 } & {0.60 } & {0.18 } & {0.66 } & {\textbf{0.79 }} \\
    {} & SCM   & {0.36 } & {0.42 } & {0.43 } & {0.50 } & {0.27 } & {0.61 } & {0.11 } & {0.70 } & {0.57 } & {0.29 } & {0.45 } & {0.48 } & {0.76 } & {0.71 } & {0.65 } & {0.23 } & {0.71 } & {0.77 } \\
    {} & TLD   & {-} & {-} & {-} & {-} & {-} & {-} & {0.56 } & {0.69 } & {-} & {-} & {0.48 } & {-} & {0.66 } & {-} & {-} & {-} & {0.70 } & {-} \\\hline
    {\multirow{15}[2]{*}{\begin{tabular}[c]{@{}c@{}}Success\\Rate\end{tabular}}} & Ours+S & {\textbf{0.95 }} & {\textbf{0.81 }} & {\textbf{0.98 }} & {\textbf{0.99 }} & {\textbf{0.90 }} & {0.94 } & {0.88 } & {\textbf{0.99 }} & {\textbf{0.98 }} & {\textbf{0.97 }} & {\textbf{0.98 }} & {0.96 } & {\textbf{0.99 }} & {\textbf{0.98 }} & {\textbf{0.99 }} & {\textbf{0.99 }} & {\textbf{0.99 }} & {\textbf{1.00 }} \\
    {} & Ours  & {0.94 } & {0.80 } & {0.89 } & {0.98 } & {0.88 } & {0.93 } & {0.87 } & {0.98 } & {0.97 } & {\textbf{0.97 }} & {\textbf{0.98 }} & {0.95 } & {\textbf{0.99 }} & {0.97 } & {0.98 } & {\textbf{0.99 }} & {\textbf{0.99 }} & {\textbf{1.00 }} \\
    {} & DML   & {0.43 } & {0.18 } & {0.12 } & {0.67 } & {0.20 } & {0.59 } & {0.37 } & {0.90 } & {0.60 } & {0.46 } & {0.82 } & {0.85 } & {0.79 } & {0.74 } & {0.86 } & {\textbf{0.99 }} & {\textbf{0.99 }} & {0.92 } \\
    {} & FragT & {0.25 } & {0.11 } & {0.63 } & {0.09 } & {0.47 } & {0.52 } & {0.05 } & {0.40 } & {0.49 } & {0.32 } & {0.37 } & {0.88 } & {0.79 } & {0.51 } & {0.71 } & {0.22 } & {0.90 } & {0.08 } \\
    {} & VTD   & {0.57 } & {0.07 } & {0.11 } & {0.11 } & {0.62 } & {0.70 } & {0.10 } & {0.37 } & {0.47 } & {0.31 } & {0.20 } & {0.79 } & {0.92 } & {\bf 0.98 } & {0.78 } & {0.20 } & {\textbf{0.99 }} & {0.43 } \\
    {} & MILT  & {0.37 } & {0.03 } & {0.42 } & {0.64 } & {0.07 } & {0.16 } & {0.28 } & {0.34 } & {0.67 } & {0.61 } & {0.22 } & {0.61 } & {0.19 } & {0.58 } & {0.89 } & {0.18 } & {0.87 } & {0.08 } \\
    {} & OAB1  & {0.37 } & {0.01 } & {0.40 } & {0.62 } & {0.05 } & {0.14 } & {0.04 } & {0.13 } & {0.67 } & {0.87 } & {0.22 } & {0.58 } & {0.18 } & {0.58 } & {0.88 } & {0.19 } & {0.80 } & {0.39 } \\
    {} & OAB5  & {0.43 } & {0.08 } & {0.42 } & {0.67 } & {0.05 } & {0.12 } & {0.04 } & {0.03 } & {0.23 } & {0.24 } & {0.22 } & {0.38 } & {0.14 } & {0.53 } & {0.69 } & {0.25 } & {0.91 } & {0.33 } \\
    {} & IPCA  & {0.34 } & {0.06 } & {0.31 } & {0.62 } & {0.01 } & {0.09 } & {0.03 } & {0.38 } & {0.87 } & {0.22 } & {0.25 } & {0.45 } & {0.80 } & {0.58 } & {0.74 } & {\textbf{0.99 }} & {0.60 } & {\bf 1.00 } \\
    {} & L1T   & {0.43 } & {0.07 } & {0.06 } & {0.11 } & {0.07 } & {0.08 } & {0.04 } & {0.34 } & {0.87 } & {0.16 } & {0.49 } & {0.61 } & {0.57 } & {0.58 } & {0.45 } & {0.10 } & {0.28 } & {0.59 } \\
    {} & Struck & {0.40 } & {0.13 } & {0.37 } & {0.83 } & {0.74 } & {\textbf{0.95 }} & {\textbf{0.93 }} & {0.96 } & {0.83 } & {0.96 } & {0.47 } & {\textbf{0.97 }} & {\textbf{0.99 }} & {0.58 } & {0.89 } & {0.88 } & {0.93 } & {\bf 1.00 } \\
    {} & CT    & {0.27 } & {0.15 } & {0.29 } & {0.64 } & {0.01 } & {0.67 } & {0.35 } & {0.12 } & {0.60 } & {0.27 } & {0.22 } & {0.47 } & {0.37 } & {0.88 } & {0.83 } & {0.20 } & {0.89 } & {0.17 } \\
    {} & ALSA  & {0.46 } & {0.31 } & {0.41 } & {0.68 } & {0.65 } & {0.73 } & {0.40 } & {0.95 } & {0.58 } & {0.31 } & {0.20 } & {0.48 } & {0.37 } & {0.93 } & {0.70 } & {0.20 } & {0.96 } & {\bf 1.00 } \\
    {} & SCM   & {0.30 } & {0.53 } & {0.42 } & {0.66 } & {0.32 } & {0.78 } & {0.12 } & {0.84 } & {0.61 } & {0.30 } & {0.63 } & {0.49 } & {\textbf{0.99 }} & {0.91 } & {0.83 } & {0.23 } & {0.94 } & {\bf 1.00 } \\
    {} & TLD   & {-} & {-} & {-} & {-} & {-} & {-} & {0.70 } & {0.81 } & {-} & {-} & {0.63 } & {-} & {0.90 } & {-} & {-} & {-} & {0.96 } & {-} \\\hline
    \end{tabular}%
}
\vspace{-0.25cm}
\caption{Quantitative comparison results of the fifteen trackers over
all the video sequences.
The table reports their average CLEs, VORs, and
success rates over each video sequence. Clearly, our tracker achieves the best tracking performance in most cases.
In the experiments, the TLD tracker produces the incomplete tracking results over some video sequences
because of its particular tracking-learning-detection properties (i.e., tracking reliability analysis
by simultaneously performing object detection and optical flow-based verification). Therefore, we only show
the video sequences in which the TLD tracker can always achieve stable tracking performances for all the frames.
 }
\label{Tab:quantitative}
\end{table}%

Moreover, time weighted reservoir sampling can
alleviate error accumulation during tracking.
Although some false foreground/background
samples may be added to the buffers because of tracking errors,
the old samples with a long lifespan can effectively reduce the
influence of the false foreground/background
samples on metric-weighted least square regression, leading
to robust tracking results. In other words, the
metric-weighted least square regression problem has two types
of reconstruction costs. One is based on
the old samples with a long lifespan, and the other
relies on the recent samples. Actually,
the regression cost for the old samples with a long lifespan
works as a regularizer that can resist tracker drift.
In addition, metric-weighted
linear regression on the background buffer can generate
the discriminative information to help the tracker
reject false foreground samples (caused by occlusion, out-of-plane rotation, and pose variation).
For instance, some false foreground training samples are
included into the foreground sample buffer at
the 18th frame of the ``seq-jd'' video sequence because of partial occlusions,
as shown in the supplementary demo video files.
After the 19th frame (without occlusions), our tracker
is still able to accurately localize the head target
with the help of discriminative metric-weighted
linear regression on the foreground and background
sample suffers.

\begin{figure}
\centering
\vspace{-0.5cm}
\includegraphics[scale=0.38]{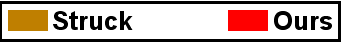}\\
\includegraphics[scale=0.6]{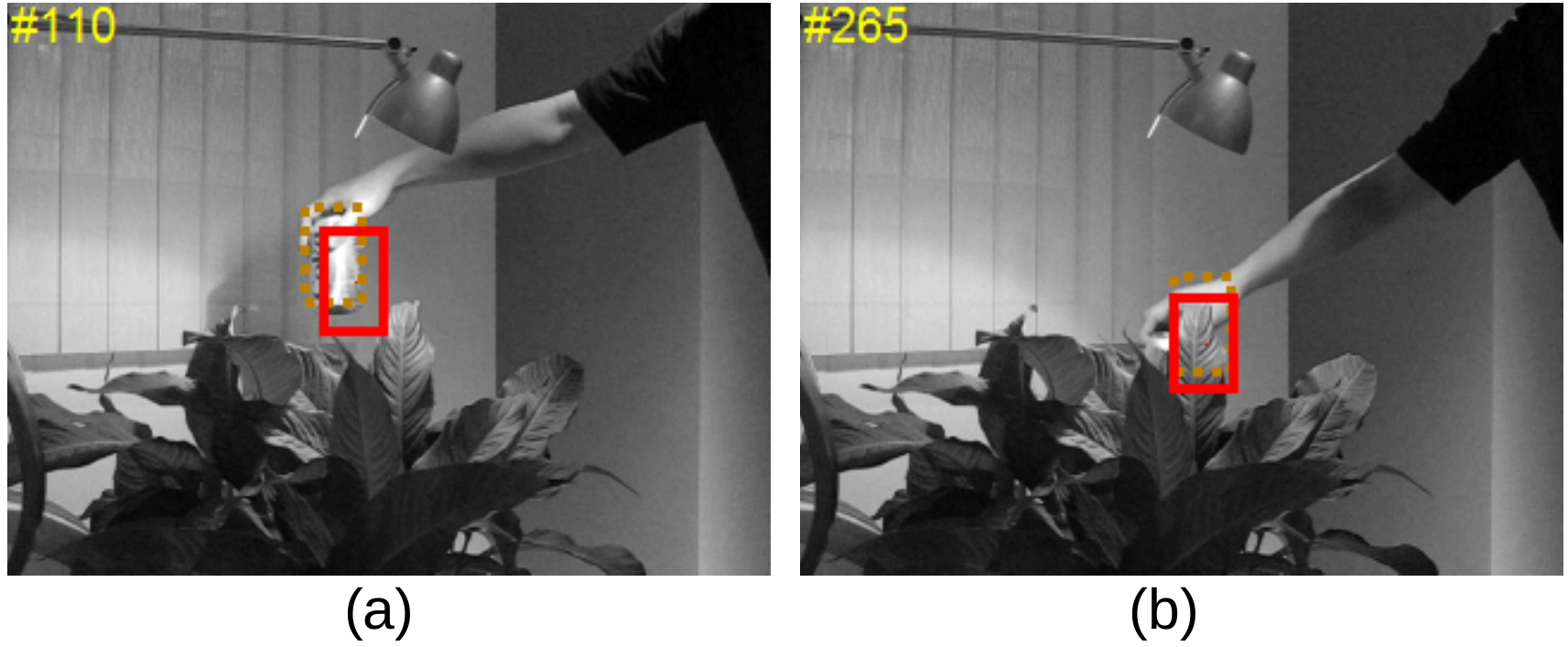}\\ \vspace{-0.41cm}
 \caption{Comparison between Struck and our tracker (highlighted in different colors). Specifically,
 (a) shows the case that our tracker almost fails while Struck succeeds in localizing the object. (b) displays
 the case that both of the trackers lose the object.
 }
 \label{fig:illustration} \vspace{-0.3cm}
\end{figure}

Fig.~\ref{fig:illustration} shows a failure case for our method  during
the ``coke11'' video sequence.
As shown in Fig.~\ref{fig:illustration}(a), the appearance difference between the tracked object and its surrounding background
is relatively small (i.e., they appear to be visually edgeless or textureless).
As a result, the foreground and background metric-weighted linear reconstruction costs for these regions
with respect to a set of foreground or background basis samples
are mutually close, resulting in a low confidence score for discriminative object/non-object classification with several false detection hypotheses.
Based on structured SVM learning for optimizing the structural localization
measure, Struck is capable of encoding the joint feature-location contextual information on object appearance,
leading to a relatively stable tracking result.
As shown in Fig.~\ref{fig:illustration}(b), a severe occlusion event leads to the almost complete disappearance of the tracked object.
Consequently, both of the trackers fail to track in this scenario.

Potentially, the further performance improvement can be made in the following two respects: 1) the current versions of the trackers are still
based on several hand-crafted visual features (e.g., HOG and Haar), which are weak in adaptively capturing the intrinsic discriminative appearance properties
of the tracked objects in various scenarios. Therefore, adaptive online feature learning is a potential solution to this issue.
2) the integration of tracking reliability analysis could handle some abnormal events like severe occlusions. If and only if
the tracking results are very reliable, then the detectors or classifiers can be updated.

The most competitive method to ours is Struck~\cite{harestruck_iccv2011},
which achieves a comparable or better tracking performance  on five of the eighteen sequences. Struck is
based on structured learning, and directly optimizes the
VOC overlap criterion using a structured SVM formulation.
In the following section, we extend our framework to optimize
this same criterion and re-evaluate against Struck.

\vspace{-0.26cm}
\subsection{Empirical evaluation of structured metric learning}

\begin{table}[t]
\vspace{-0.45cm}
\begin{center}
\scalebox{0.52}
{
\begin{tabular}{c||c|c|c|c|c|c||c|c|c|c|c|c||c|c|c|c|c|c}
\hline \scriptsize
& \multicolumn{6}{|c||}{CLE} & \multicolumn{6}{|c||}{VOR} & \multicolumn{6}{|c}{Success Rate}\\
\hline
& \makebox[1.1cm]{trellis70}   &  \makebox[0.9cm]{race} &
\makebox[0.9cm]{cubicle} & \makebox[1.3cm]{football3} & \makebox[0.9cm]{trace} & \makebox[0.9cm]{seq-jd} &
 \makebox[1.1cm]{trellis70}   &  \makebox[0.9cm]{race} &
\makebox[0.9cm]{cubicle} & \makebox[1.3cm]{football3} & \makebox[0.9cm]{trace} & \makebox[0.9cm]{seq-jd}
& \makebox[1.1cm]{trellis70}   &  \makebox[0.9cm]{race} &
\makebox[0.9cm]{cubicle} & \makebox[1.3cm]{football3} & \makebox[0.9cm]{trace} & \makebox[0.9cm]{seq-jd} \\\hline
\hline
Struck                          &\bf 4.82    &    7.71    &22.61   &4.32    &24.15   &3.32   &0.80          &0.66    &0.41          &\bf0.72    &0.43    &\bf0.76    &0.96    &0.88       &0.47        &0.96    &0.37    &\bf0.97\\
Non-structured metric learning  &    5.62    &    8.52    &4.31    &\bf3.92    &8.65    &4.30   &0.78       &0.80    &0.74          &\bf0.72    &0.70    &0.72        &0.98    &\bf0.99     &\bf0.98    &\bf0.97    &0.89    &0.95\\
Structured metric learning      &    4.22    &\bf 6.96    &\bf3.10    &5.63    &\bf6.23    &\bf3.02   &\bf0.83    &0.82    &0.78    &\bf0.72    &\bf0.74    &\bf0.76  &\bf0.99    &\bf0.99    &\bf0.98    &\bf0.97    &\bf0.98   &0.96\\
No metric learning              &    7.80    &    15.86    &4.43    &4.01   &15.75    &5.16   &0.73    &0.69    &0.70               &0.68       &0.56     &0.68        &0.93    &\bf0.99     &0.96       &0.95    &0.62    &0.94\\

\hline
\end{tabular}
}
\end{center}
\vspace{-0.6cm}
\caption{Quantitative evaluation of the proposed tracker using
different learning strategies on eight video sequences.
The table reports their average CLEs,
VORs, and success rates across frames.}
\label{Tab:structured_metric_VOC_CLE}
\end{table}

\begin{figure}[h!]
\vspace{-0.66cm}
\centering
\includegraphics[scale=0.55]{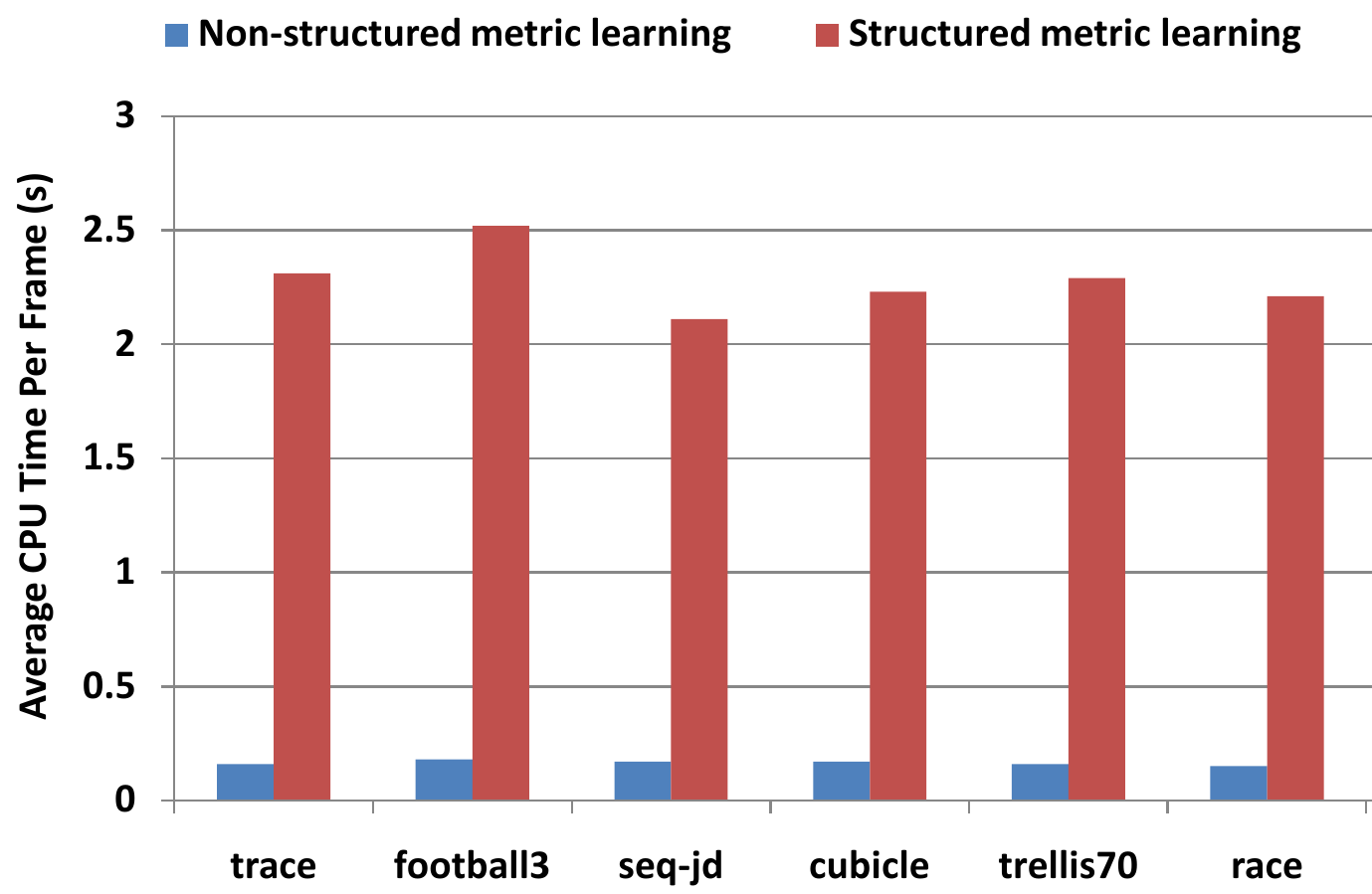}
 \vspace{-0.2cm}
 \caption{Runtime performance of
the proposed tracker using
different metric learning strategies on eight video sequences.
The table reports their average running time of
performing metric learning across frames.}
 \label{fig:structured_metric_runtime} \vspace{-0.86cm}
\end{figure}
To evaluate the effect of structured
metric learning, we
compare its tracking performance to our previous method on eight video sequences.
For computational efficiency, the structured metric learning method
takes a uniform sampling strategy to
randomly generate a collection of bounding boxes around the current tracker bounding box.
Using these bounding boxes, we construct a set of triplet-based
structural constraints (referred to in Equ.~\eqref{eq:online_optimization_function_structure})
for online structured metric learning.

Tab.~\ref{Tab:structured_metric_VOC_CLE} reports
their average frame-by-frame VORs, CLEs, and success rates on the
four video sequences.
Clearly, it is seen from Tab.~\ref{Tab:structured_metric_VOC_CLE}
that the structured metric learning method
outperforms
both the non-structured
metric learning method and the method without metric learning.
In addition, Fig.~\ref{fig:structured_metric_runtime} shows
the average runtime performance of the tracking approach using non-structured or
structured metric learning on the eight video sequences.
It is clearly seen from Fig.~\ref{fig:structured_metric_runtime}
that structured metric learning is about 20 times slower than
non-structured metric learning.
From Tab.~\ref{Tab:structured_metric_VOC_CLE}, we see that
non-structured metric learning
achieves a reasonably close tracking performance to
structured metric learning. However, the computational efficiency
of non-structured metric learning is much better than that of
structured metric learning. Therefore, in the applications presented in the next Section,
we use non-structured metric learning.

The Struck tracker constructs an
object localization scoring function based on
structured SVM learning, which learns a linear
SVM scoring function in a max-margin structured
output optimization framework. Therefore,
the tracking accuracy of the Struck tracker
solely depends on the learned SVM scoring function.
In the case of tracking errors,
the learned SVM scoring function
is
contaminated,
causing the error accumulations across frames,
which often leads to tracker drift in
complicated scenarios. In contrast, our tracker
takes a data-driven strategy for object localization.
Namely, our tracker takes advantage of reservoir sampling
to effectively maintain the foreground/background
buffers, which store the recently
included samples while keeping the old samples
with a long lifespan. Besides, our metric-weighted
linear regression cost based on
these old samples essentially works as a regularizer
that reduces the influence of the tracking
error accumulations, resulting in the tracking robustness.
Combined with structured metric learning, our tracker
has the capability of performing robust visual tracking
in a more discriminative metric space, leading to the further
performance improvements.

\subsection{Experimental summary}
\label{sec:exp_discussion}

Based on the obtained experimental results, we observe that
the proposed tracking algorithm has the following properties. First,
after the buffer size exceeds a certain value (around $300$ in our
experiments), the tracking performance is stable with increasing
buffer size, as shown in Fig.~\ref{fig:buffersize_particlenum}. This is desirable since we do not need a large buffer
size to achieve promising performance.
Second, in contrast to many existing particle filtering-based trackers
whose running time is typically linear in the number of particles, our method's
running time is sublinear in the number of
particles, as shown in Fig.~\ref{fig:buffersize_particlenum}.
Moreover, its tracking performance rapidly improves and
finally converge to a certain value, as shown in
Fig.~\ref{fig:buffersize_particlenum}.
Third, based on linear representation with metric learning, it performs
better in tracking accuracy, as shown
in Tab.~\ref{Tab:different_linear_representation_VOC_CLE}.
Fourth, it utilizes weighted reservoir
sampling to effectively maintain and update the foreground
and background sample buffers for metric learning, as shown in Tab.~\ref{Tab:sampling_VOC_CLE}.
Fifth, as shown in Tab.~\ref{Tab:metric_success_rate}, the performance of
our metric learning with
no eigendecomposition is close to that of computationally expensive metric learning with
step-by-step eigendecomposition.
Sixth, compared with other state-of-the-art trackers, it is capable of
effectively
adapting to complicated appearance changes in the tracking process by
constructing
an effective metric-weighted linear representation with weighed
reservoir sampling, as shown in
Tab.~\ref{Tab:quantitative}.
Last, using the structured metric learning is capable of
improving the tracking performance in CLE and VOR, as shown in Tab.~\ref{Tab:structured_metric_VOC_CLE}.
That is because the structured metric learning encodes the underlying
the structural interaction information on data samples, which plays an important role in robust visual tracking.

\section{Pedestrian tracking and identification}

Recent studies have demonstrated the effectiveness
of combining object identification and
tracking together.
To achieve this goal, we need to first localize the object of interest
and  then assign it to one of the predefined object classes using
the object tracking information.
Without loss of generality, we suppose
there are totally $K$ object classes that correspond to
$K$ static template sample sets
$\{\mathbf{P}_{\mathcal{T}}^{k}\}_{k=1}^{K}$
(collected before object tracking).
After performing object tracking on a video sequence ranging from frame 1 to frame $t$,
we obtain the consecutive object observations $\mathbf{y}_{1:t}$
whose object classification scores are denoted as $(\mathcal{S}(\mathbf{y}_{1}), \ldots, \mathcal{S}(\mathbf{y}_{t}))$
(as defined in Equ.~\eqref{eq:particle_liki_model}).
In essence, these classification scores reflect the likelihood
of the observations to be generated from the object of interest.
With respect to $\mathbf{P}_{\mathcal{T}}^{k}$, we compute
a set of reconstruction errors
$(g(\mathbf{x}_{k}^{1}; \mathbf{M}, \mathbf{P}_{\mathcal{T}}^{k}, \mathbf{y}_{1}), \ldots,
g(\mathbf{x}_{k}^{t}; \mathbf{M}, \mathbf{P}_{\mathcal{T}}^{k}, \mathbf{y}_{t}))$
such that
$\mathbf{x}_{k}^{t} = \arg \min_{\mathbf{x}}  g(\mathbf{x};
\mathbf{M}, \mathbf{P}_{\mathcal{T}}^{k}, \mathbf{y}_{t})$.
Based on these reconstruction errors, the cumulative distance of $\mathbf{y}_{t}$
with respect to the $k$-th object class is calculated as: \vspace{-0.26cm}
\begin{equation}
\mathcal{H}_{k}(\mathbf{y}_{t}) = \sum_{i=1}^{t} \omega(\mathbf{y}_{i}) g(\mathbf{x}_{k}^{i}; \mathbf{M}, \mathbf{P}_{\mathcal{T}}^{k}, \mathbf{y}_{i}),
\label{eq:identification_score} \vspace{-0.26cm}
\end{equation}
where $\omega(\mathbf{y}_{i})$ is a weighting factor that measures
the prior weight of $\mathbf{y}_{i}$ generated from the object of interest.
Here, we just use the object classification score $\mathcal{S}(\mathbf{y}_{i})$ to approximate
the prior weight $\omega(\mathbf{y}_{i})$ in the process of object identification
(such that $\omega(\mathbf{y}_{i}) \propto \mathcal{S}(\mathbf{y}_{i})$).
As a result, the object class membership for $\mathbf{y}_{t}$ is determined by:
$k^{\ast} = \underset{1\leq k \leq K}{\arg \min} \thickspace \mathcal{H}_{k}(\mathbf{y}_{t})$.
In addition, the above-mentioned object identification module has
the capability of automatically detecting the abnormal events
(e.g., occlusion). When
$g(\mathbf{x}_{k}^{i}; \mathbf{M}, \mathbf{P}_{\mathcal{T}}^{k}, \mathbf{y}_{i})$
is very large, the tracked objects often have drastic appearance changes (caused by
occlusion, noisy corruption, shape deformation, and so on).
Therefore, the abnormal changes in object appearance can be
automatically detected by checking the value of
$g(\mathbf{x}_{k}^{i}; \mathbf{M}, \mathbf{P}_{\mathcal{T}}^{k}, \mathbf{y}_{i})$.

\begin{figure}[t]
\vspace{-0.95cm}
\centering
\includegraphics[scale=0.38]{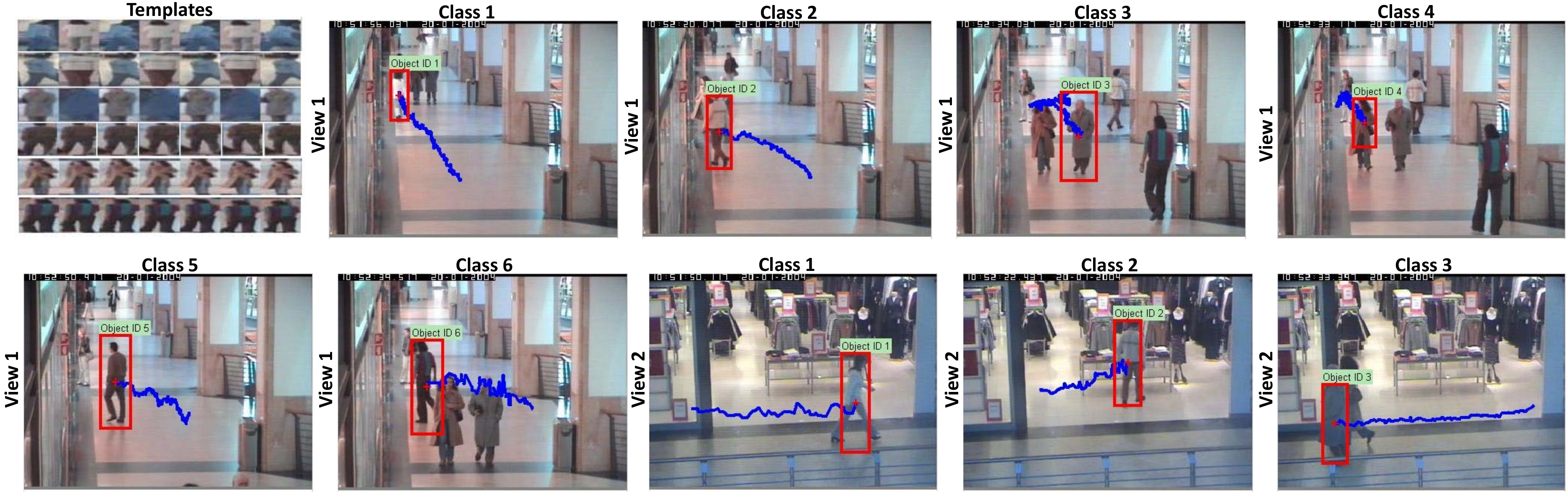}
\vspace{-0.25cm}
 \caption{Two-view pedestrian identification examples.
It is clear that the pedestrians can be accurately
identified.
 \vspace{-0.15cm}}
 \label{fig:more_example}
\end{figure}

\begin{figure*}[t]
\vspace{-0.5cm}
\centering
\includegraphics[scale=0.43]{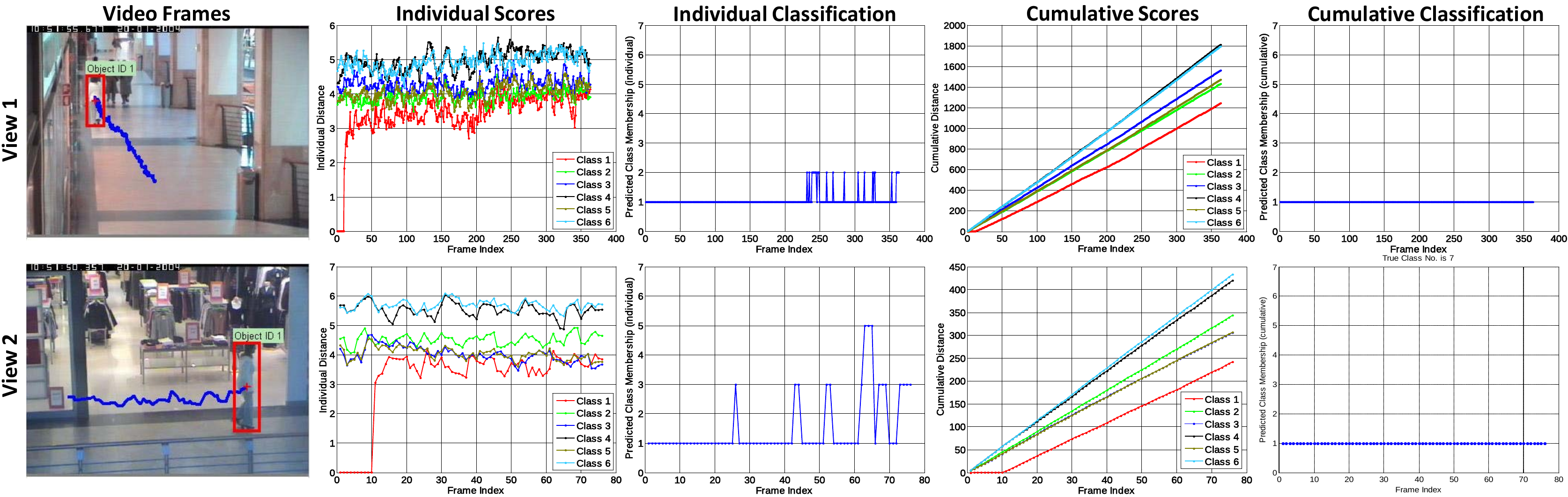}
\vspace{-0.3cm}
 \caption{Illustration of our pedestrian identification method.
The first column shows the tracking results on the video frames at two different viewpoints;
the second column displays the frame-by-frame  reconstruction errors based on
frame-independent metric-weighted linear regression; the third column exhibits
the frame-by-frame identification results associated with the second column;
the fourth column plots the frame-by-frame cumulative reconstruction errors
based on frame-dependent metric-weighted linear regression; and
the last column corresponds to the frame-by-frame identification results associated with the fourth column.
 Clearly, our cumulative classification method is able to correctly
 identify the same pedestrian from two different viewpoints.
 \vspace{-0.0cm}}
 \label{fig:identification_example}
\end{figure*}

Based on Equ.~\eqref{eq:identification_score},
we carry out the pedestrian identification task
on the video sequences\footnote[1]{http://homepages.inf.ed.ac.uk/rbf/CAVIARDATA1/}
with two viewpoints. Prior to pedestrian tracking and identification,
we collect a set of static templates from some training video sequences.
In total, there are six individual pedestrians corresponding to
six object classes. Fig.~\ref{fig:more_example} shows
the pedestrian tracking and identification results as well as
the static templates used in tracking.
For a clear illustration,
we give an intuitive example of showing the
whole pedestrian identification process, as shown in
Fig.~\ref{fig:identification_example}.
From Fig.~\ref{fig:identification_example}, we
observe that our method is able to
accurately recognize the tracked pedestrian's
identity throughout the
entire video sequence.
More pedestrian identification results
can be found in the supplementary.

\begin{figure}
 \vspace{-0.9cm}
\centering
\includegraphics[scale=0.41]{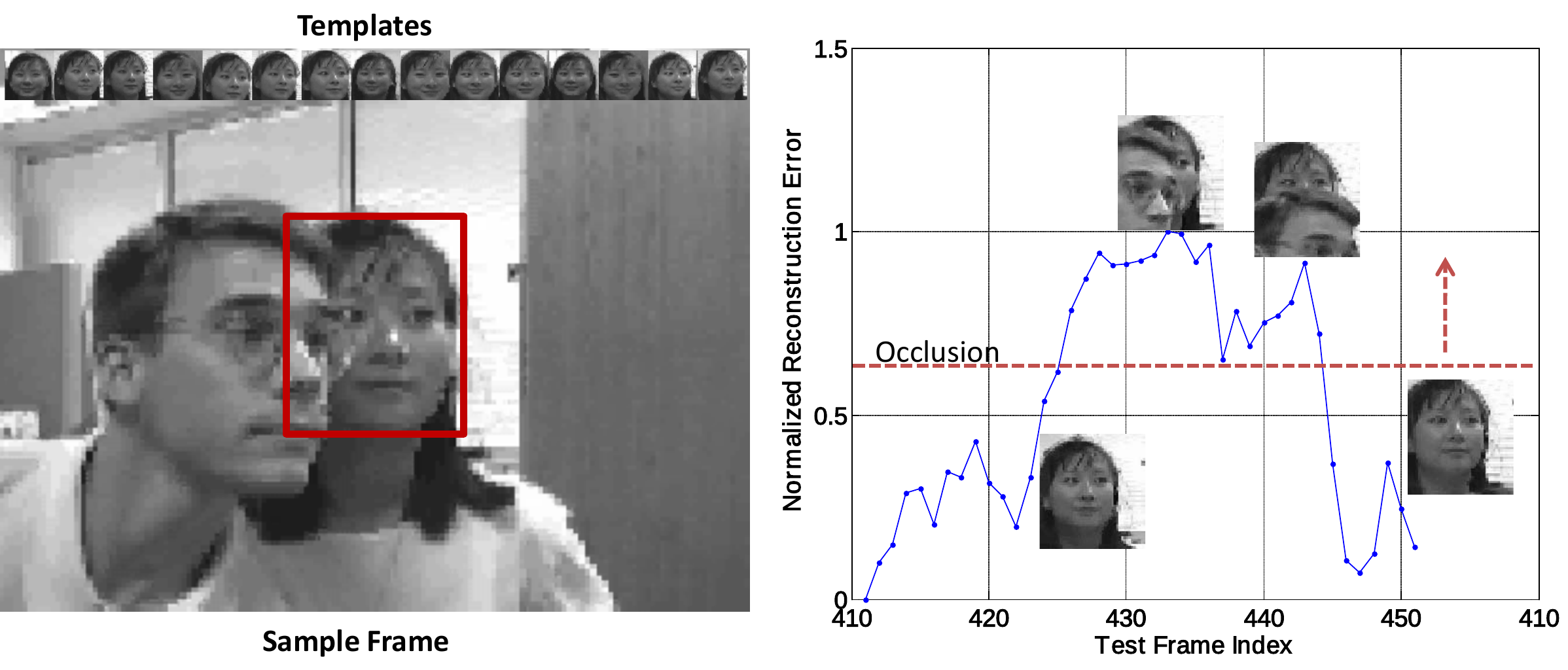}
\vspace{-1.0cm}
 \caption{Example of detecting the occlusion events.
 The left part shows a sample frame and the static templates,
 and the right part plots the frame-by-frame occlusion detection results
 during tracking.
 \vspace{-0.8cm}}
 \label{fig:occlusion}
\end{figure}
Moreover, we apply our method to detect whether the abnormal
events (e.g., occlusion) take place in the tracking
process.
Namely, if the target appearance
is weakly correlated with the static templates
(i.e., high reconstruction errors),
there is likely to be some abnormal events occurring
during tracking.
Fig.~\ref{fig:occlusion}
shows an example of detecting the frame-by-frame
occlusion events on the ``girl'' video sequence.
It is seen from Fig.~\ref{fig:occlusion} that
our method succeeds in
detecting the occlusion events
in most cases.

\section{Conclusion}

In this work, we have proposed an online metric-weighted
linear representation for robust visual tracking.
With a closed-form analytical solution, the proposed
linear representation is capable of effectively
encoding the discriminative information
on object/non-object classification. We designed
an online Mahalanobis distance
metric learning scheme, including online non-structured and structured metric learning.
The metric learning scheme aims
to distinguish the relative
importance of individual feature dimensions
and capture the correlation
between feature dimensions
in a feasible metric space.
We empirically show that
adding a metric to the linear representation considerably
improves the robustness of the tracker. To make the online
metric learning even more efficient, for the first time,
we design a learning mechanism based on time-weighted
reservoir sampling.
With this mechanism, recently streamed samples in the video are
assigned higher weights.
We have also theoretically proved that metric learning based on
the proposed reservoir sampling with limited-sized sampling
buffers can effectively approximate metric learning using all the
received training samples.
Compared with state-of-the-art trackers on eighteen challenging
sequences, we empirically show that our method is  more robust to
complicated appearance changes,  pose variations, and occlusions.
Furthermore, we also extend our work
to perform pedestrian identification and occlusion event detection
during object tracking.
Experimental results demonstrate the effectiveness of
our work.

To balance efficiency and effectiveness, a mixture of non-structured and structured metric learning
methods can be alternatively applied during tracking. For example,
the non-structured method can produce  metric
learning results with a high update frequency,
and then the structured method further generates
the refined metric learning results with a low update frequency. We plan to investigate this in future.

\end{spacing}

\begin{spacing}{1.2}
{

\footnotesize

\reffont
\bibliography{ref}
\bibliographystyle{ieee}
}
\end{spacing}

\clearpage

\setcounter{page}{1}

\begin{spacing}{1.15}
{

\vspace{0cm}{ \bf Supplementary}

In this supplementary material, we provide more technical details
of online updating the metric-weighted
linear representation, online discriminative distance metric learning,
and a theoretical analysis of time-weighted reservoir sampling.
Furthermore, we show more
experimental results (both qualitatively and quantitatively), including experimental demonstration videos,
more CLE (center location error) and VOR (VOC overlap ratio) curves
for different evaluation tasks, intuitive
frame tracking images,
and more frame-by-frame pedestrian identification results.

\begin{table}[h]
\centering
\scalebox{0.69}
{
\begin{tabular}{l l}
\hline
{Video sequences}   &  {Corresponding video files} \\
\hline
\hline
B-Beam  &  Video\_01\_BalanceBeam.mp4 \\\hline
Lola &  Video\_02\_Lola.mp4 \\\hline
trace &   Video\_03\_trace.mp4 \\\hline
Walk &  Video\_04\_Walk.mp4 \\\hline
football  &  Video\_05\_football.mp4  \\\hline
iceball &  Video\_06\_iceball.mp4 \\\hline
coke11 &  Video\_07\_coke11.mp4 \\\hline
trellis70 &  Video\_08\_trellis70.mp4 \\\hline
dograce &  Video\_09\_dograce.mp4 \\\hline
football3   & Video\_10\_football3.mp4  \\\hline
cubicle  & Video\_11\_cubicle.mp4 \\\hline
seq-jd & Video\_12\_seq-jd.mp4\\\hline
girl & Video\_13\_girl.mp4\\\hline
BMX-Street & Video\_14\_BMX-Street.mp4\\\hline
planeshow & Video\_15\_planeshow.mp4\\\hline
race  & Video\_16\_race.mp4\\\hline
CamSeq01  & Video\_17\_CamSeq01.mp4\\\hline
car11 & Video\_18\_car11.mp4\\
\hline
\end{tabular}
}
\vspace{-0.1cm}
\caption{The configurations of the eighteen experimental demonstration videos. These
demonstration videos can be downloaded at the following link:
{\color{red}http://cs.adelaide.edu.au/users/xi/pamimetricdemo.zip}
 }
\label{Tab:list_video}
\end{table}

\clearpage

\clearpage

\clearpage

\vspace{-0.25cm}
\section{Online proximity based metric learning}
\label{sec:oml_supp}

Having introduced the metric-weighted linear representation  in Sec.~\ref{sec:malso}, we now address the key issue of calculating the metric matrix $\mathbf{M}$.
$\bf M$ should ideally be learned from the visual data, and should be dynamically updated as conditions change throughout a video sequence.

\subsubsection{\bf Triplet-based ranking losses}
   Suppose that we have a set of sample triplets
   $\{(\mathbf{p}, \mathbf{p}^{+},\mathbf{p}^{-})\}$
    with
    $ \mathbf{p},
    \mathbf{p}^{+},
    \mathbf{p}^{-} \in \mathcal{R}^{d}
    $.  These triplets encode the proximity comparison information.
   In each triplet, the distance between $  \bf p $
   and $ {\bf p}^+  $ should be smaller than the distance between
   $ \bf p $ and $ {\bf p}^-  $.

The Mahalanobis distance under  metric $ \bf M$ is defined as:
\vspace{-0.153cm}
\begin{equation}
D_{\mathbf{M}}(\mathbf{p}, \mathbf{q}) = (\mathbf{p}
-\mathbf{q})^{T}\mathbf{M}(\mathbf{p}-\mathbf{q}). \vspace{-0.126cm}
\end{equation}
Clearly, $ \bf M $ must be a symmetric and positive semidefinite
matrix. It is equivalent to learn a projection matrix $ \bf L $ such
that $ {\bf M} = {\bf L} {\bf L}^T$.
In practice,
   we generate the triplets set as:
   $ \bp $ and $ \bp^+$ belong to the same class
   and $ \bp $ and $ \bp^-$ belong to different classes.
So we want the constraints
$ D_\bM ( \bp, \bp^+ ) < D_\bM ( \bp, \bp^- ) $ to be satisfied as
well as possible.
By putting it into a large-margin learning framework, and using the
soft-margin hinge loss,
the loss function for each triplet is: \vspace{-0.016cm}
\begin{equation}
l_{\mathbf{M}}(\mathbf{p}, \mathbf{p}^{+}, \mathbf{p}^{-}) =
       \max\{0, 1 + D_{\mathbf{M}}(\mathbf{p}, \mathbf{p}^{+}) -
D_{\mathbf{M}}(\mathbf{p}, \mathbf{p}^{-})\}. \vspace{-0.01cm}
\label{eq:local_hinge_loss_supp}
\end{equation}

\subsubsection{\bf Large-margin metric learning}
To obtain the optimal distance metric matrix $\mathbf{M}$, we need to
minimize the global loss $L_{\mathbf{M}}$
that takes the sum of hinge losses~\eqref{eq:local_hinge_loss_supp} over
all possible triplets from the training
set: \vspace{-0.21cm}
\begin{equation}
L_{\mathbf{M}} = \underset{(\mathbf{p}, \mathbf{p}^{+},
\mathbf{p}^{-})\in \mathcal{Q}}{\sum} l_{\mathbf{M}}(\mathbf{p},
\mathbf{p}^{+}, \mathbf{p}^{-}), \vspace{-0.21cm}
\label{eq:total_loss_supp}
\end{equation}
where $\mathcal{Q}$ is the triplet set.
   To sequentially optimize the above objective function $L_{\mathbf{M}}$
   in an online fashion,
   we design an iterative algorithm to solve the following
   convex problem:
\begin{equation}
\begin{array}{l}
\mathbf{M}^{k+1} = \underset{\mathbf{M}}{\arg\min}
\frac{1}{2}\|\mathbf{M}-\mathbf{M}^{k}\|^{2}_{F} + C\xi,\\
\mbox{s.t.} \thickspace  D_{\mathbf{M}}(\mathbf{p}, \mathbf{p}^{-})
- D_{\mathbf{M}}(\mathbf{p}, \mathbf{p}^{+})  \geq 1 -  \xi,
\thickspace \xi \geq  0,
\end{array}
\label{eq:online_optimization_function_supp}
\end{equation}
where $\|\cdot\|_{F}$ denotes the Frobenius norm, $\xi$ is a slack
variable, and $C$ is a positive factor controlling
the trade-off between the smoothness term
$\frac{1}{2}\|\mathbf{M}-\mathbf{M}^{k}\|^{2}_{F}$ and the loss  term
$\xi$.
Following the passive-aggressive mechanism used
in~\cite{chechik2010large, crammer2006online}, we only update the
metric matrix $\mathbf{M}$ when $l_{\mathbf{M}}(\mathbf{p},
\mathbf{p}^{+}, \mathbf{p}^{-})>0$.

\subsubsection{\bf Optimization of ${\mathbf M}$}
We optimize the function in Equation~\ref{eq:online_optimization_function_supp}
with Lagrangian regularization:\vspace{-0.12cm}
\begin{equation}
\begin{array}{l}
\mathcal{L}(\mathbf{M}, \eta, \xi, \beta) =
\frac{1}{2}\|\mathbf{M}-\mathbf{M}^{k}\|^{2}_{F} + C\xi - \beta \xi
+ \eta(1-\xi
+
D_{\mathbf{M}}(\mathbf{p},
\mathbf{p}^{+})-D_{\mathbf{M}}(\mathbf{p}, \mathbf{p}^{-})),
\end{array}
\label{eq:Lagrange_loss_supp} \vspace{-0.12cm}
\end{equation}
where $\eta\geq 0$ and $\beta \geq 0$ are Lagrange multipliers.
The optimization procedure is carried out in the following two alternating
steps.

\begin{itemize}
\item {\it Update $\mathbf{M}$}.
By setting
$ \frac{\partial\mathcal{L}(\mathbf{M}, \eta, \xi, \beta)}{\partial
\mathbf{M}} = 0$, we arrive at the update rule
\begin{equation}
\mathbf{M}^{k+1} = \mathbf{M}^{k} + \eta \mathbf{U}
\label{eq:eq:metric_update_supp} \vspace{-0.12cm}
\end{equation}
where
$\mathbf{U} = \mathbf{a}_{-}\mathbf{a}_{-}^{T} - \mathbf{a}_{+}\mathbf{a}_{+}^{T}$ and $\mathbf{a}_{+} = \mathbf{p} - \mathbf{p}^{+}$,
$\mathbf{a}_{-} = \mathbf{p} - \mathbf{p}^{-}$.

 By taking the
 derivative of $\mathcal{L}(\mathbf{M}, \eta, \xi, \beta)$
 w.r.t. $\mathbf{M}$, we have the following: \vspace{-0.12cm}
 \begin{equation}
 \begin{array}{ll}
 \frac{\partial\mathcal{L}(\mathbf{M}, \eta, \xi, \beta)}{\partial
 \mathbf{M}} & = \mathbf{M}-\mathbf{M}^{k} +
 \eta\frac{\partial
 [D_{\mathbf{M}}(\mathbf{p},
 \mathbf{p}^{+})-D_{\mathbf{M}}(\mathbf{p}, \mathbf{p}^{-})]
 }
 {\partial
 \mathbf{M}}.\\
 \end{array} \vspace{-0.12cm}
 \end{equation}
 Mathematically, $\frac{
 \partial
 [D_{\mathbf{M}}(\mathbf{p},
 \mathbf{p}^{+})-D_{\mathbf{M}}(\mathbf{p}, \mathbf{p}^{-})]
 }
 {\partial \mathbf{M}}$ can be
 formulated as: \vspace{-0.12cm}
 \begin{equation}
 \frac{\partial
 [D_{\mathbf{M}}(\mathbf{p},
 \mathbf{p}^{+})-D_{\mathbf{M}}(\mathbf{p}, \mathbf{p}^{-})]
 }
 {\partial \mathbf{M}}
 =
 \mathbf{a}_{+}\mathbf{a}_{+}^{T} - \mathbf{a}_{-}\mathbf{a}_{-}^{T},
 \label{eq:triplet_update_supp} \vspace{-0.12cm}
 \end{equation}
 where $\mathbf{a}_{+} = \mathbf{p} - \mathbf{p}^{+}$ and
 $\mathbf{a}_{-} = \mathbf{p} - \mathbf{p}^{-}$. The optimal
 $\mathbf{M}^{k+1}$ is obtained
 by setting $\frac{\partial\mathcal{L}(\mathbf{M}, \eta, \xi,
 \beta)}{\partial \mathbf{M}}$ to zero. As a result,
 the following relation holds: \vspace{-0.12cm}
 \begin{equation}
 \mathbf{M}^{k+1} = \mathbf{M}^{k} + \eta
 (\mathbf{a}_{-}\mathbf{a}_{-}^{T} - \mathbf{a}_{+}\mathbf{a}_{+}^{T}).
 \label{eq:eq:metric_update_supp} \vspace{-0.12cm}
 \end{equation}

\item {\it Update $\eta$}. Subsequently, we take the derivative of the
Lagrangian~\eqref{eq:Lagrange_loss_supp} w.r.t. $\xi$ and set it
to zero, leading to the update rule:
 \vspace{-0.2cm}
 \begin{equation}
 \frac{\partial\mathcal{L}(\mathbf{M}, \eta, \xi, \beta)}{\partial
 \xi} = C - \beta - \eta = 0.
 \label{eq:lagrange_multiplier_supp} \vspace{-0.2cm}
 \end{equation}
 Clearly, $\beta\geq 0$ leads to the fact that $\eta \leq C$.
 For notational simplicity, $\mathbf{a}_{-}\mathbf{a}_{-}^{T} -
 \mathbf{a}_{+}\mathbf{a}_{+}^{T}$ is abbreviated as $\mathbf{U}$
 hereinafter.
 By substituting Equs.~\eqref{eq:eq:metric_update_supp}
 and~\eqref{eq:lagrange_multiplier_supp} into Equ.~\eqref{eq:Lagrange_loss_supp} with $\mathbf{M} = \mathbf{M}^{k+1}$,  we have: \vspace{-0.2cm}
 \begin{equation}
 \mathcal{L}(\eta) = \frac{1}{2}\eta^{2}\|\mathbf{U}\|_{F}^{2} +
 \eta(1+D_{\mathbf{M}^{k+1}}(\mathbf{p},
 \mathbf{p}^{+})-D_{\mathbf{M}^{k+1}}(\mathbf{p}, \mathbf{p}^{-})),
 \label{eq:Lagrange_eta_supp} \vspace{-0.12cm}
 \end{equation}
 where $D_{\mathbf{M}^{k+1}}(\mathbf{p}, $ $ \mathbf{p}^{+}) =
 \mathbf{a}_{+}^{T } $ $ (\mathbf{M}^{k}+\eta \mathbf{U})\mathbf{a}_{+}$ and
 $D_{\mathbf{M}^{k+1}}(\mathbf{p}, $ $ \mathbf{p}^{-}) =
 \mathbf{a}_{-}^{T}(\mathbf{M}^{k}+\eta \mathbf{U})\mathbf{a}_{-}$.
 As a result, $\mathcal{L}(\eta)$ can be reformulated as: \vspace{-0.12cm}
 \begin{equation}
 \mathcal{L}(\eta) = \lambda_{2}\eta^{2} + \lambda_{1}\eta +
 \lambda_{0}, \vspace{-0.016cm}
 \end{equation}
 where $\lambda_{2}= \frac{1}{2}\|\mathbf{U}\|_{F}^{2} +
 \mathbf{a}_{+}^{T}\mathbf{U}\mathbf{a}_{+} -
 \mathbf{a}_{-}^{T}\mathbf{U}\mathbf{a}_{-}$, $\lambda_{1} = 1 +
 \mathbf{a}_{+}^{T}\mathbf{M}^{k}\mathbf{a}_{+} -
 \mathbf{a}_{-}^{T}\mathbf{M}^{k}\mathbf{a}_{-}$,
 and $\lambda_{0} = 0$. To obtain the optimal $\eta$, we need to
 differentiate $\mathcal{L}(\eta)$ w.r.t. $\eta$ and set it to
 zero: \vspace{-0.12cm}
 \begin{equation}
 \hspace{-0.05cm}
 \begin{array}{l}
 \frac{\partial \mathcal{L}(\eta)}{\partial \eta} =
 \eta(\|\mathbf{U}\|_{F}^{2} +
 2\mathbf{a}_{+}^{T}\mathbf{U}\mathbf{a}_{+} -
 2\mathbf{a}_{-}^{T}\mathbf{U}\mathbf{a}_{-}) \\
       \hspace{2.1cm} + (1 +
 \mathbf{a}_{+}^{T}\mathbf{M}^{k}\mathbf{a}_{+} -
 \mathbf{a}_{-}^{T}\mathbf{M}^{k}\mathbf{a}_{-}) = 0.\\
 \end{array} \vspace{-0.2cm} \hspace{-0.21cm}
 \end{equation}
 As a result, the following relation holds: \vspace{-0.2cm}
 \begin{equation}
 \eta = -\frac{1 + \mathbf{a}_{+}^{T}\mathbf{M}^{k}\mathbf{a}_{+} -
 \mathbf{a}_{-}^{T}\mathbf{M}^{k}\mathbf{a}_{-}}{\|\mathbf{U}\|_{F}^{2} +
 2\mathbf{a}_{+}^{T}\mathbf{U}\mathbf{a}_{+} -
 2\mathbf{a}_{-}^{T}\mathbf{U}\mathbf{a}_{-}}.
 \label{eq:eta_update_supp} \vspace{-0.12cm}
 \end{equation}
 Due to the constraint of $0 \leq \eta\leq C$, $\eta$ should take the following
 value: \vspace{-0.12cm}
 \begin{equation}
 \hspace{-0.0cm}
 \eta \hspace{-0.08cm}= \hspace{-0.08cm}\min\left\{C, \max \left\{0,\frac{1 + \mathbf{a}_{+}^{T}\mathbf{M}^{k}\mathbf{a}_{+} -
 \mathbf{a}_{-}^{T}\mathbf{M}^{k}\mathbf{a}_{-}}{2\mathbf{a}_{-}^{T}\mathbf{U}\mathbf{a}_{-} \hspace{-0.08cm}- \hspace{-0.08cm}
 2\mathbf{a}_{+}^{T}\mathbf{U}\mathbf{a}_{+} \hspace{-0.08cm}- \hspace{-0.08cm}\|\mathbf{U}\|_{F}^{2}}\right\}\hspace{-0.1cm}\right\}
 \hspace{-0.9cm}
 \label{eq:eta_final_supp} \vspace{-0.2cm}
 \end{equation}
\end{itemize}
The full derivation of each step can be found in the supplementary file.
The complete procedure of online distance metric learning is
summarized in Algorithm~\ref{alg:online_metric_learning}.

\subsubsection{\bf  Online update}
When updated according to
Algorithm~\ref{alg:online_metric_learning}, $\mathbf{M}$ is modified
by rank-one
additions
such that $\mathbf{M}\longleftarrow \mathbf{M} + \eta
(\mathbf{a}_{-}\mathbf{a}_{-}^{T} -
\mathbf{a}_{+}\mathbf{a}_{+}^{T})$ where $\mathbf{a}_{+} = \mathbf{p} - \mathbf{p}^{+}$ and
$\mathbf{a}_{-} = \mathbf{p} - \mathbf{p}^{-}$ are
two vectors (defined in Equ.~\eqref{eq:eq:metric_update_supp}) for triplet construction, and
$\eta$ is a step-size factor (defined in Equ.~\eqref{eq:eta_final_supp}).
As a result, the original $\mathbf{P}^{T}\mathbf{M}\mathbf{P}$
becomes $\mathbf{P}^{T}\mathbf{M}\mathbf{P}
+ (\eta\mathbf{P}^{T}\mathbf{a}_{-})(\mathbf{P}^{T}\mathbf{a}_{-})^{T}
+  (-\eta\mathbf{P}^{T}\mathbf{a}_{+})(\mathbf{P}^{T}\mathbf{a}_{+})^{T}$.
When $\mathbf{M}$ is modified by a rank-one addition, the inverse of
$\mathbf{P}^{T}\mathbf{M}\mathbf{P}$ can be updated
according to the theory of~\cite{householder1964theory,powell1969theorem}:
\begin{equation}
(\mathbf{J}+\mathbf{u}\mathbf{v}^{T})^{-1} = \mathbf{J}^{-1} -
\frac{\mathbf{J}^{-1}\mathbf{u}\mathbf{v}^{T}\mathbf{J}^{-1}}{1+\mathbf{v}^{T}\mathbf{J}^{-1}\mathbf{u}}.
\label{eq:rank_one_update_supp} \vspace{-0.02cm}
\end{equation}
Here, $\mathbf{J}=\mathbf{P}^{T}\mathbf{M}\mathbf{P}$, $\mathbf{u}=\eta\mathbf{P}^{T}\mathbf{a}_{-}$ (or $\mathbf{u}=-\eta\mathbf{P}^{T}\mathbf{a}_{+}$),
and $\mathbf{v}=\mathbf{P}^{T}\mathbf{a}_{-}$ (or $\mathbf{v}=\mathbf{P}^{T}\mathbf{a}_{+}$).

\clearpage

\section{Online structured metric learning}
\label{sec:structured_metric_learning_supp}

Metric learning based on sample proximity comparisons leads to an efficient online learning algorithm, but requires pre-defined sets of positive and negative samples. In tracking, these usually correspond to target/non-target image patches. The boundary between these classes typically occurs where sample overlap with the target drops below a threshold, but this can be difficult to evaluate exactly and thus introduces some noise into the algorithm.

In this Section, we replace the proximity based metric learning module with an online structured metric learning method for learning  $\bf M$. The main advantage of this method is that it directly learns the metric from measured sample overlap, and therefore does not require the separation of samples into positive and negative classes.

{\bf Structured ranking}
Let $\mathbf{p}_{t}$ and $\mathbf{p}_{t}^{i}$
 denote two feature vectors extracted from two image patches,
 which are respectively associated
 with two bounding boxes
 $\mathbf{R}_{t}$ and $\mathbf{R}_{t}^{i}$ from frame $t$.
 Without loss of generality, let us assume that
 $\mathbf{R}_{t}$  corresponds to the
 bounding box obtained by the current tracker
 while $\mathbf{R}_{t}^{i}$ is associated with
 a bounding box from the area surrounding $\mathbf{R}_{t}$.
 As in~\cite{harestruck_iccv2011}, the structural affinity relationship
 between $\mathbf{p}_{t}$ and $\mathbf{p}_{t}^{i}$
 is captured by the following overlap function: $s^{o}_{t}(\mathbf{p}_{t}\circ \mathbf{R}_{t}, \mathbf{p}_{t}^{i}\circ \mathbf{R}_{t}^{i}) = \frac{\mathbf{R}_{t}\bigcap\mathbf{R}_{t}^{i}}{\mathbf{R}_{t}\bigcup\mathbf{R}_{t}^{i}}.$
 As a result, we define the following optimization problem for structured metric learning:  \vspace{-0.18cm}
 \begin{equation}
 \begin{array}{l}
 \mathbf{M}^{k+1} = \underset{\mathbf{M}}{\arg\min}\thinspace
 \frac{1}{2}\|\mathbf{M}-\mathbf{M}^{k}\|^{2}_{F} + C\xi,\\
 \mbox{s.t.}
 \thickspace D_{\mathbf{M}}(\mathbf{p}_{t}, \mathbf{p}_{t}^{j}) - D_{\mathbf{M}}(\mathbf{p}_{t}, \mathbf{p}_{t}^{i})
  \geq
 \Delta_{ij}  -  \xi, \forall i,j\\
 \end{array}
 \label{eq:online_optimization_function_structure_supp} \vspace{-0.18cm}
 \end{equation}
 where $\xi \geq  0$
 and  $\Delta_{ij} = s^{o}_{t}(\mathbf{p}_{t}\circ \mathbf{R}_{t}, \mathbf{p}_{t}^{i} \circ \mathbf{R}_{t}^{i}) - s^{o}_{t}(\mathbf{p}_{t} \circ \mathbf{R}_{t}, \mathbf{p}_{t}^{j} \circ \mathbf{R}_{t}^{j})$.
 Clearly, the number of constraints in the optimization problem~\eqref{eq:online_optimization_function_structure_supp}
 is
 exponentially large or even
 infinite, making it difficult to optimize.
 Our approach to this optimization problem differs
 from~\cite{harestruck_iccv2011} in four main aspects:
 i) our approach aims to learn a distance metric while
 \cite{harestruck_iccv2011} seeks  a SVM classifier;
 ii)
 we optimize an online max-margin objective function
 while \cite{harestruck_iccv2011} solves a batch-mode
 optimization problem; iii) our optimization problem involves nonlinear constraints on
 triplet-based Mahalanobis distance differences, while the optimization problem in~\cite{harestruck_iccv2011}
 comprises linear constraints on doublet-based SVM classification score differences;
 and iv) our approach directly solves the primal optimization problem while \cite{harestruck_iccv2011}
 optimizes the dual problem.

{\bf Structured optimization}
 Inspired by the cutting-plane method,
 we iteratively construct a constraint set (denoted as $\mathcal{P}$)
 containing the most violated constraints
 for the optimization problem~\eqref{eq:online_optimization_function_structure_supp}.
 In our case,
 the most violated constraint is selected according to
 the following criterion:  \vspace{-0.18cm}
 \begin{equation}
 (\mu, \nu)
 =
 \underset{
 (i, j)}{\arg\max} \hspace{0.2cm}
 \Delta_{ij} + D_{\mathbf{M}}(\mathbf{p}_{t}, \mathbf{p}_{t}^{i}) - D_{\mathbf{M}}(\mathbf{p}_{t}, \mathbf{p}_{t}^{j}),
 \label{eq:most_violated_constraints_supp}  \vspace{-0.18cm}
 \end{equation}
 For notational simplicity, let $l_{\mathbf{M}}(\mathbf{p}_{t}\circ \mathbf{R}_{t}, \mathbf{p}_{t}^{j} \circ \mathbf{R}_{t}^{j},\mathbf{p}_{t}^{i} \circ \mathbf{R}_{t}^{i} )$ denote the loss term
 $\Delta_{ij} + D_{\mathbf{M}}(\mathbf{p}_{t}, \mathbf{p}_{t}^{i}) - D_{\mathbf{M}}(\mathbf{p}_{t}, \mathbf{p}_{t}^{j})$.
 Note that the violated constraints generated from~\eqref{eq:most_violated_constraints_supp}
 are used if and only if $l_{\mathbf{M}}(\mathbf{p}_{t}\circ \mathbf{R}_{t}, \mathbf{p}_{t}^{j} \circ \mathbf{R}_{t}^{j},\mathbf{p}_{t}^{i} \circ \mathbf{R}_{t}^{i} )$ is greater than zero.
 Subsequently, we add the most violated constraint
 to the optimization problem~\eqref{eq:online_optimization_function_structure_supp} in an iterative manner,
 that is, $\mathcal{P} \leftarrow \mathcal{P} \bigcup \{(\mathbf{p}_{t}^{\mu}\circ \mathbf{R}_{t}^{\mu}, \mathbf{p}_{t}^{\nu}\circ \mathbf{R}_{t}^{\nu})\}$.
 The corresponding Lagrangian is formulated as:  \vspace{-0.18cm}
 \begin{equation}
 \begin{array}{l}
 \mathcal{L} =
 \frac{1}{2}\|\mathbf{M}-\mathbf{M}^{k}\|^{2}_{F} +  (C - \beta)\xi
 + \sum_{\ell=1}^{|\mathcal{P}|}\eta_{\ell}[\Delta_{\mu_{\ell}\nu_{\ell}} - \xi
 +D_{\mathbf{M}}(\mathbf{p}_{t}, \mathbf{p}_{t}^{\mu_{\ell}}) - D_{\mathbf{M}}(\mathbf{p}_{t}, \mathbf{p}_{t}^{\nu_{\ell}})],
 \end{array}
 \label{eq:Lagrange_loss_structure_supp}  \vspace{-0.18cm}
 \end{equation}
 where $\beta \geq 0$ and $\eta_{\ell}\geq 0$ are Lagrange multipliers.
 The optimization procedure is once again carried out in two alternating
steps:

 \begin{itemize}
 \item {\it Update $\mathbf{M}$.}
 By setting $\frac{\partial\mathcal{L}}{\partial
 \mathbf{M}}$ to zero, we obtain an updated $\mathbf{M}$ defined as:
\begin{equation}
 \mathbf{M}^{k+1} = \mathbf{M}^{k} + \sum_{\ell=1}^{|\mathcal{P}|}\eta_{\ell}\mathbf{U}_{\ell}
\end{equation}
where
$\mathbf{U}_{\ell} = \mathbf{a}_{t}^{\nu_{\ell}}(\mathbf{a}_{t}^{\nu_{\ell}})^{\T} - \mathbf{a}_{t}^{\mu_{\ell}}(\mathbf{a}_{t}^{\mu_{\ell}})^{\T}$, and $\mathbf{a}_{t}^{n}$ denotes $\mathbf{p}_{t} - \mathbf{p}_{t}^{n}$.

 The first-order derivative of $\mathcal{L}$ w.r.t. $\mathbf{M}$ is expressed as:
  \begin{equation}
  \frac{\partial\mathcal{L}}{\partial
  \mathbf{M}}  = \mathbf{M}-\mathbf{M}^{k} -
  \sum_{\ell=1}^{|\mathcal{P}|}\eta_{\ell}\frac{\partial
  [D_{\mathbf{M}}(\mathbf{p}_{t}, \mathbf{p}_{t}^{\nu_{\ell}}) - D_{\mathbf{M}}(\mathbf{p}_{t}, \mathbf{p}_{t}^{\mu_{\ell}})]
  }
  {\partial
  \mathbf{M}}.
  \label{eq:dldm}
  \end{equation}
  Clearly, $\frac{\partial
  D_{\mathbf{M}}(\mathbf{p}_{t}, \mathbf{p}_{t}^{n})
  }
  {\partial
  \mathbf{M}}$ is equal to $(\mathbf{p}_{t} - \mathbf{p}_{t}^{n})(\mathbf{p}_{t} - \mathbf{p}_{t}^{n})^{\T}$.
  Letting $\mathbf{a}_{t}^{n}$ denote $\mathbf{p}_{t} - \mathbf{p}_{t}^{n}$, we rewrite (\ref{eq:dldm}) as:
  $\frac{\partial\mathcal{L}}{\partial
  \mathbf{M}}  = \mathbf{M}-\mathbf{M}^{k} -
  \sum_{\ell=1}^{|\mathcal{P}|}\eta_{\ell}[\mathbf{a}_{t}^{\nu_{\ell}}(\mathbf{a}_{t}^{\nu_{\ell}})^{\T} - \mathbf{a}_{t}^{\mu_{\ell}}(\mathbf{a}_{t}^{\mu_{\ell}})^{\T}]$.
  By setting $\frac{\partial\mathcal{L}}{\partial
  \mathbf{M}}$ to zero, we obtain the optimal $\mathbf{M}$ defined as:
  $\mathbf{M} = \mathbf{M}^{k} + \sum_{\ell=1}^{|\mathcal{P}|}\eta_{\ell}[\mathbf{a}_{t}^{\nu_{\ell}}(\mathbf{a}_{t}^{\nu_{\ell}})^{\T} - \mathbf{a}_{t}^{\mu_{\ell}}(\mathbf{a}_{t}^{\mu_{\ell}})^{\T}]$.

 \item {\it Update $\eta_{\ell}$.}
 To obtain the optimal solution for all Lagrange multipliers $\eta_{\ell}$,
we take the first-order derivative of $\mathcal{L}$ w.r.t. $\eta_{\ell}$ and set it to zero: \vspace{-0.2cm}
  \begin{equation}
  \hspace{-0.26cm}
  \begin{array}{c}
  \frac{\partial \mathcal{L}}{\partial \eta_{\ell}} =
  \eta_{\ell}[\mathbf{1}^{\T}(\mathbf{U}_{\ell}\circ \mathbf{U}_{\ell})\mathbf{1} + 2(\mathbf{a}_{t}^{\mu_{\ell}})^{\T}\mathbf{U}_{\ell}\mathbf{a}_{t}^{\mu_{\ell}}
  - 2(\mathbf{a}_{t}^{\nu_{\ell}})^{\T}\mathbf{U}_{\ell}\mathbf{a}_{t}^{\nu_{\ell}}]
  + [\Delta_{\mu_{\ell}\nu_{\ell}} + (\mathbf{a}_{t}^{\mu_{\ell}})^{\T}\mathbf{M}^{k}\mathbf{a}_{t}^{\mu_{\ell}} - (\mathbf{a}_{t}^{\nu_{\ell}})^{\T}\mathbf{M}^{k}\mathbf{a}_{t}^{\nu_{\ell}}]\\
  \thickspace + \underset{m\neq \ell}{\sum}\eta_{m}[\mathbf{1}^{\T}(\mathbf{U}_{\ell} \circ \mathbf{U}_{m})\mathbf{1} + (\mathbf{a}_{t}^{\mu_{m}})^{\T}\mathbf{U}_{m}\mathbf{a}_{t}^{\mu_{m}}
  - (\mathbf{a}_{t}^{\nu_{m}})^{\T}\mathbf{U}_{m}\mathbf{a}_{t}^{\nu_{m}}
  + (\mathbf{a}_{t}^{\mu_{m}})^{\T}\mathbf{U}_{\ell}\mathbf{a}_{t}^{\mu_{m}}
  - (\mathbf{a}_{t}^{\nu_{m}})^{\T}\mathbf{U}_{\ell}\mathbf{a}_{t}^{\nu_{m}}]=0,
  \end{array}  \hspace{-0.71cm} \vspace{-0.2cm}
  \end{equation}
  where $\mathbf{1}$ is the all-one column vector and $\circ$ is the
  elementwise product operator.
  Hence, we have a linear equation
  $\mathbf{B}\bm{\eta} = \mathbf{f}$, where $\bm{\eta} = (\eta_{1}, \eta_{2}, \ldots, \eta_{|\mathcal{P}|})^{\T}$,
  $\mathbf{f} = (f_{1}, f_{2}, \ldots, f_{|\mathcal{P}|})$ with
  $f_{\ell}$ being $-[\Delta_{\mu_{\ell}\nu_{\ell}} + (\mathbf{a}_{t}^{\mu_{\ell}})^{\T}\mathbf{M}^{k}\mathbf{a}_{t}^{\mu_{\ell}} - (\mathbf{a}_{t}^{\nu_{\ell}})^{\T}\mathbf{M}^{k}\mathbf{a}_{t}^{\nu_{\ell}}]$,
  and $\mathbf{B} = (b_{\ell m})_{|\mathcal{P}|\times |\mathcal{P}|}$ with
  $b_{\ell m}$ being
  $\mathbf{1}^{\T}(\mathbf{U}_{\ell} \circ \mathbf{U}_{m})\mathbf{1} + (\mathbf{a}_{t}^{\mu_{m}})^{\T}\mathbf{U}_{m}\mathbf{a}_{t}^{\mu_{m}}
  - (\mathbf{a}_{t}^{\nu_{m}})^{\T}\mathbf{U}_{m}\mathbf{a}_{t}^{\nu_{m}}
  + (\mathbf{a}_{t}^{\mu_{m}})^{\T}\mathbf{U}_{\ell}\mathbf{a}_{t}^{\mu_{m}}
  - (\mathbf{a}_{t}^{\nu_{m}})^{\T}\mathbf{U}_{\ell}\mathbf{a}_{t}^{\nu_{m}}$.

  Differentiating $\mathcal{L}$ w.r.t. $\xi$ and setting it to zero, we have
  $C - \beta - \sum_{\ell=1}^{|\mathcal{P}|}\eta_{\ell} = 0$. Since $\beta\geq 0$, the relation $0 \leq \sum_{\ell=1}^{|\mathcal{P}|}\eta_{\ell} \leq C$ holds. Therefore, the optimal $\bm{\eta}^{\ast}$ is efficiently obtained
  by solving the following optimization problem:  \vspace{-0.25cm}
  \begin{equation}
  \begin{array}{l}
  \bm{\eta}^{\ast} = \underset{\bm{\eta}}{\arg \min} \thinspace \|\mathbf{B}\bm{\eta} - \mathbf{f}\|_{1}, \hspace{0.15cm}
  \mbox{s.t.} \thinspace \bm{\eta} \succeq 0; \mathbf{1}^{\T}\bm{\eta} \leq C.
  \end{array}
  \label{eq:eta_linear_solve_supp}  \vspace{-0.25cm}
  \end{equation}
  \end{itemize}

 As before, the optimal $\mathbf{M}$ is updated
 as a sequence of rank-one
 additions: $\mathbf{M}\longleftarrow \mathbf{M} + \eta_{\ell}[\mathbf{a}_{t}^{\nu_{\ell}}(\mathbf{a}_{t}^{\nu_{\ell}})^{\T} - \mathbf{a}_{t}^{\mu_{\ell}}(\mathbf{a}_{t}^{\mu_{\ell}})^{\T}]$.
 As a result, the original $\mathbf{P}^{\T}\mathbf{M}\mathbf{P}$
 becomes $\mathbf{P}^{\T}\mathbf{M}\mathbf{P}
 + (\eta_{\ell}\mathbf{P}^{\T}\mathbf{a}_{t}^{\nu_{\ell}})(\mathbf{P}^{\T}\mathbf{a}_{t}^{\nu_{\ell}})^{\T}
 +  (-\eta_{\ell}\mathbf{P}^{\T}\mathbf{a}_{t}^{\mu_{\ell}})(\mathbf{P}^{\T}\mathbf{a}_{t}^{\mu_{\ell}})^{\T}$.
 When $\mathbf{M}$ is modified by a rank-one addition, the inverse of
 $\mathbf{P}^{\T}\mathbf{M}\mathbf{P}$ can be easily updated
 according to the theory of~\cite{householder1964theory,powell1969theorem}.
 Namely, $(\mathbf{J}+\mathbf{u}\mathbf{v}^{\T})^{-1} = \mathbf{J}^{-1} -
 \frac{\mathbf{J}^{-1}\mathbf{u}\mathbf{v}^{\T}\mathbf{J}^{-1}}{1+\mathbf{v}^{\T}\mathbf{J}^{-1}\mathbf{u}}$.
 Here, $\mathbf{J}=\mathbf{P}^{\T}\mathbf{M}\mathbf{P}$, $\mathbf{u}=\eta_{\ell}\mathbf{P}^{\T}\mathbf{a}_{t}^{\nu_{\ell}}$ (or $\mathbf{u}=-\eta_{\ell}\mathbf{P}^{\T}\mathbf{a}_{t}^{\mu_{\ell}}$),
 and $\mathbf{v}=\mathbf{P}^{\T}\mathbf{a}_{t}^{\nu_{\ell}}$ (or $\mathbf{v}=\mathbf{P}^{\T}\mathbf{a}_{t}^{\mu_{\ell}}$).

\clearpage

\section{Theoretical analysis of time-weighted reservoir sampling}
\label{sec:analysis_t_r_s}

\begin{theorem}
\vspace{-0.0cm}
\label{theo:Weighted_reservoir_sampling_supp}
Given a new training sample $\mathbf{p}$, we have the following relation: \vspace{-0.26cm}
\[
\begin{array}{l}
\frac{1}{|\mathcal{B}_{c_{+}}||\mathcal{B}_{c_{-}}|}\emph{E}_{\mathcal{B}_{c_{+}}}\emph{E}_{\mathcal{B}_{c_{-}}}\left(
\underset{\mathbf{p}^{+} \in \mathcal{B}_{c_{+}}}{\sum}\underset{\mathbf{p}^{-} \in \mathcal{B}_{c_{-}}}{\sum}l_{\mathbf{M}}(\mathbf{p},
\mathbf{p}^{+}, \mathbf{p}^{-})\right)\\
=\overset{h_{c_{+}}}{\underset{i=1}{\sum}}\left[
\frac{{w}_{i}^{c_{+}}}
{\overset{h_{c_{+}}}
{\underset{m=1}{\sum}}{w}_{m}^{c_{+}}
}
\left(
\overset{h_{c_{-}}}{\underset{j=1}{\sum}}
\frac{{w}_{j}^{c_{-}}}
{
\overset{h_{c_{-}}}{\underset{n=1}{\sum}}{w}_{n}^{c_{-}}}
\l_{\mathbf{M}}(\mathbf{p}, \mathbf{p}^{c_{+}}_{i}, \mathbf{p}^{c_{-}}_{j})\right)\right]
\end{array} \vspace{-0.2cm}
\]
where $\emph{E}(\cdot)$ is the expectation operator,
$c_{+} \in \{f, b\}$ is a class indicator variable whose class membership is the same as $\mathbf{p}$
(i.e., if $\mathbf{p}\in f$, $c_{+}=f$; otherwise, $c_{+}=b$),
$c_{-} \in \{f, b\}$  is a class indicator variable whose class membership is different from $\mathbf{p}$
(i.e., if $\mathbf{p}\in f$, $c_{-}=b$; otherwise, $c_{-}=f$),
$\{\mathbf{p}^{f}_{i}\}_{i=1}^{h_{f}}$ and
$\{\mathbf{p}^{b}_{j}\}_{j=1}^{h_{b}}$ denote all the received
training sample sets before $\mathbf{p}$,
${w}_{i}^{f}$ and ${w}_{j}^{b}$ are the corresponding weights of $\mathbf{p}^{f}_{i}$ and $\mathbf{p}^{b}_{j}$.
In our case, any sample weight ${w}_{i}^{f}$ (or ${w}_{j}^{b}$) is defined as:
${w}_{i}^{f}=q^{\mathbb{I}_{i}^{f}}$ (or ${w}_{j}^{b}=q^{\mathbb{I}_{j}^{b}}$) where
$\mathbb{I}_{i}^{f}$ (or $\mathbb{I}_{j}^{b}$) is the
corresponding frame index number of $\mathbf{p}^{f}_{i}$ (or $\mathbf{p}^{b}_{j}$) and $q$ is a constant such that $q\geq 1$.
\vspace{-0.1cm}
\end{theorem}
\begin{proof}
\footnotesize
In total, there are two cases for $\mathbf{p}$: i) $\mathbf{p}$ is a foreground sample (i.e., $\mathbf{p} \in f$) with $c_{+}=f$ and $c_{-}=b$; and
ii) $\mathbf{p}$ is a background sample (i.e., $\mathbf{p} \in b$) with $c_{+}=b$ and $c_{-}=f$.
Therefore, when $\mathbf{p}$ is a foreground sample, we need to prove the following relation:
\begin{equation}
\begin{array}{l}
\frac{1}{|\mathcal{B}_{f}||\mathcal{B}_{b}|}\mbox{E}_{\mathcal{B}_{f}}\mbox{E}_{\mathcal{B}_{b}}\left(
\underset{\mathbf{p}^{+} \in \mathcal{B}_{f}}{\sum}\underset{\mathbf{p}^{-} \in \mathcal{B}_{b}}{\sum}l_{\mathbf{M}}(\mathbf{p},
\mathbf{p}^{+}, \mathbf{p}^{-})\right)\\
=
\overset{h_{f}}{\underset{i=1}{\sum}}\left[
\frac{{w}_{i}^{f}}
{\overset{h_{f}}
{\underset{m=1}{\sum}}{w}_{m}^{f}
}
\left(
\overset{h_{b}}{\underset{j=1}{\sum}}
\frac{{w}_{j}^{b}}
{
\overset{h_{b}}{\underset{n=1}{\sum}}{w}_{n}^{b}}
\l_{\mathbf{M}}(\mathbf{p}, \mathbf{p}^{f}_{i}, \mathbf{p}^{b}_{j})\right)\right].
\end{array}
\label{eq:foreground_prove}
\end{equation}
Conversely, we need to  prove the following relation:
\begin{equation}
\begin{array}{l}
\frac{1}{|\mathcal{B}_{b}||\mathcal{B}_{f}|}\mbox{E}_{\mathcal{B}_{b}}\mbox{E}_{\mathcal{B}_{f}}\left(
\underset{\mathbf{p}^{+} \in \mathcal{B}_{b}}{\sum}\underset{\mathbf{p}^{-} \in \mathcal{B}_{f}}{\sum}l_{\mathbf{M}}(\mathbf{p},
\mathbf{p}^{+}, \mathbf{p}^{-})\right)\\
=
\overset{h_{b}}{\underset{i=1}{\sum}}\left[
\frac{{w}_{i}^{b}}
{\overset{h_{b}}
{\underset{m=1}{\sum}}{w}_{m}^{b}
}
\left(
\overset{h_{f}}{\underset{j=1}{\sum}}
\frac{{w}_{j}^{f}}
{
\overset{h_{f}}{\underset{n=1}{\sum}}{w}_{n}^{f}}
\l_{\mathbf{M}}(\mathbf{p}, \mathbf{p}^{b}_{i}, \mathbf{p}^{f}_{j})\right)\right].
\end{array}
\label{eq:background_prove}
\end{equation}
First of all, we cope with the foreground case defined in Equ.~\eqref{eq:foreground_prove}. The expectation
in Equ.~\eqref{eq:foreground_prove} can be computed as:
\begin{equation}
\begin{array}{l}
\frac{1}{|\mathcal{B}_{f}||\mathcal{B}_{b}|}\mbox{E}_{\mathcal{B}_{f}}\mbox{E}_{\mathcal{B}_{b}}\left(
\underset{\mathbf{p}^{+} \in \mathcal{B}_{f}}{\sum}\underset{\mathbf{p}^{-} \in \mathcal{B}_{b}}{\sum}l_{\mathbf{M}}(\mathbf{p}_{},
\mathbf{p}^{+}, \mathbf{p}^{-})\right)\\
= \frac{1}{|\mathcal{B}_{f}|}\mbox{E}_{\mathcal{B}_{f}}
\left\{\underset{\mathbf{p}^{+} \in \mathcal{B}_{f}}{\sum}
\left[\frac{1}{|\mathcal{B}_{b}|}\mbox{E}_{\mathcal{B}_{b}}
\left(\underset{\mathbf{p}^{-} \in \mathcal{B}_{b}}{\sum}
l_{\mathbf{M}}(\mathbf{p}_{},\mathbf{p}^{+}, \mathbf{p}^{-})
\right)\right]
\right\}.
\end{array}
\label{eq:step1}
\end{equation}
According to the property of weighted reservoir sampling with replacement (as shown in Refs [19, 20]), we have:
\begin{equation}
\frac{1}{|\mathcal{B}_{b}|}\mbox{E}_{\mathcal{B}_{b}}
\left(\underset{\mathbf{p}^{-} \in \mathcal{B}_{b}}{\sum}
l_{\mathbf{M}}(\mathbf{p}_{},\mathbf{p}^{+}, \mathbf{p}^{-})
\right)
=
\underset{\mathbf{v}_{-}\sim \mathbb{W}_{b}}{\mbox{E}}
\left[l_{\mathbf{M}}(\mathbf{p}_{},\mathbf{p}^{+}, \mathbf{v}_{-})\right].
\label{eq:reservior_background}
\end{equation}
Here, $\mathbb{W}_{b}$ is the probability distribution associated with
$\{\mathbf{p}^{b}_{j}\}_{j=1}^{h_{b}}$, and its corresponding probability mass function is defined as:
\begin{equation}
 \Pr_{\mathbb{W}_{b}}(\mathbf{v}_{-}=\mathbf{p}^{b}_{j})=\frac{{w}_{j}^{b}}
{\overset{h_{b}}
{\underset{n=1}{\sum}}{w}_{n}^{b}
}.
\label{eq:background_weighted_distribution}
\end{equation}
As a result, Equ.~\eqref{eq:step1} can be rewritten as:
\begin{equation}
\begin{array}{l}
\frac{1}{|\mathcal{B}_{f}||\mathcal{B}_{b}|}\mbox{E}_{\mathcal{B}_{f}}\mbox{E}_{\mathcal{B}_{b}}\left(
\underset{\mathbf{p}^{+} \in \mathcal{B}_{f}}{\sum}\underset{\mathbf{p}^{-} \in \mathcal{B}_{b}}{\sum}l_{\mathbf{M}}(\mathbf{p}_{},
\mathbf{p}^{+}, \mathbf{p}^{-})\right)\\
= \frac{1}{|\mathcal{B}_{f}|}\mbox{E}_{\mathcal{B}_{f}}
\left\{\underset{\mathbf{p}^{+} \in \mathcal{B}_{f}}{\sum}
\left[
\underset{\mathbf{v}_{-}\sim \mathbb{W}_{b}}{\mbox{E}}
\left[l_{\mathbf{M}}(\mathbf{p}_{},\mathbf{p}^{+}, \mathbf{v}_{-})\right]
\right]
\right\}.
\end{array}
\end{equation}
Similar to Equ.~\eqref{eq:reservior_background}, we obtain the following relation:
\begin{equation}
\begin{array}{l}
\frac{1}{|\mathcal{B}_{f}|}\mbox{E}_{\mathcal{B}_{f}}
\left\{\underset{\mathbf{p}^{+} \in \mathcal{B}_{f}}{\sum}
\left[
\underset{\mathbf{v}_{-}\sim \mathbb{W}_{b}}{\mbox{E}}
\left[l_{\mathbf{M}}(\mathbf{p}_{},\mathbf{p}^{+}, \mathbf{v}_{-})\right]
\right]
\right\}\\
=\underset{\mathbf{v}_{+}\sim \mathbb{W}_{f}}{\mbox{E}}
\left\{
\underset{\mathbf{v}_{-}\sim \mathbb{W}_{b}}{\mbox{E}}
\left[l_{\mathbf{M}}(\mathbf{p}_{},\mathbf{v}_{+}, \mathbf{v}_{-})\right]
\right\}.
\end{array}
\label{eq:double_expectation}
\end{equation}
Here, $\mathbb{W}_{f}$ is the probability distributions corresponding to $\{\mathbf{p}^{f}_{i}\}_{i=1}^{h_{f}}$
with the following probability mass function:
\begin{equation}
\Pr_{\mathbb{W}_{f}}(\mathbf{v}_{+}=\mathbf{p}^{f}_{i})=\frac{{w}_{i}^{f}}
{\overset{h_{f}}
{\underset{m=1}{\sum}}{w}_{m}^{f}
}.
\label{eq:foreground_weighted_distribution}
\end{equation}
Therefore,
\begin{equation}
\begin{array}{l}
\frac{1}{|\mathcal{B}_{f}||\mathcal{B}_{b}|}\mbox{E}_{\mathcal{B}_{f}}\mbox{E}_{\mathcal{B}_{b}}\left(
\underset{\mathbf{p}^{+} \in \mathcal{B}_{f}}{\sum}\underset{\mathbf{p}^{-} \in \mathcal{B}_{b}}{\sum}l_{\mathbf{M}}(\mathbf{p}_{},
\mathbf{p}^{+}, \mathbf{p}^{-})\right)\\
= \underset{\mathbf{v}_{+}\sim \mathbb{W}_{f}}{\mbox{E}}
\left\{
\underset{\mathbf{v}_{-}\sim \mathbb{W}_{b}}{\mbox{E}}
\left[l_{\mathbf{M}}(\mathbf{p}_{},\mathbf{v}_{+}, \mathbf{v}_{-})\right]
\right\}.
\end{array}
\end{equation}
Based on Equ.~\eqref{eq:background_weighted_distribution} and Equ.~\eqref{eq:foreground_weighted_distribution},
we reformulate $\underset{\mathbf{v}_{+}\sim \mathbb{W}_{f}}{\mbox{E}}
\left\{
\underset{\mathbf{v}_{-}\sim \mathbb{W}_{b}}{\mbox{E}}
\left[l_{\mathbf{M}}(\mathbf{p}_{},\mathbf{v}_{+}, \mathbf{v}_{-})\right]
\right\}$ as:
\begin{equation}
\begin{array}{l}
\underset{\mathbf{v}_{+}\sim \mathbb{W}_{f}}{\mbox{E}}
\left\{
\underset{\mathbf{v}_{-}\sim \mathbb{W}_{b}}{\mbox{E}}
\left[l_{\mathbf{M}}(\mathbf{p}_{},\mathbf{v}_{+}, \mathbf{v}_{-})\right]
\right\}
\\
=
\underset{\mathbf{v}_{+}\sim \mathbb{W}_{f}}{\mbox{E}}
\left\{
\overset{h_{b}}{\underset{j=1}{\sum}}\left[
\frac{{w}_{j}^{b}}
{
\overset{h_{b}}{\underset{n=1}{\sum}}{w}_{n}^{b}}
\l_{\mathbf{M}}(\mathbf{p}_{}, \mathbf{v}_{+}, \mathbf{p}^{b}_{j})
\right]
\right\}\\
= \overset{h_{f}}{\underset{i=1}{\sum}}\left[
\frac{{w}_{i}^{f}}
{\overset{h_{f}}
{\underset{m=1}{\sum}}{w}_{m}^{f}
}
\left(
\overset{h_{b}}{\underset{j=1}{\sum}}
\frac{{w}_{j}^{b}}
{
\overset{h_{b}}{\underset{n=1}{\sum}}{w}_{n}^{b}}
\l_{\mathbf{M}}(\mathbf{p}_{}, \mathbf{p}^{f}_{i}, \mathbf{p}^{b}_{j})\right)\right].
\end{array}
\label{eq:step4}
\end{equation}
As a result, we have the following relation:
\begin{equation}
\begin{array}{l}
\frac{1}{|\mathcal{B}_{f}||\mathcal{B}_{b}|}\mbox{E}_{\mathcal{B}_{f}}\mbox{E}_{\mathcal{B}_{b}}\left(
\underset{\mathbf{p}^{+} \in \mathcal{B}_{f}}{\sum}\underset{\mathbf{p}^{-} \in \mathcal{B}_{b}}{\sum}l_{\mathbf{M}}(\mathbf{p},
\mathbf{p}^{+}, \mathbf{p}^{-})\right)\\
=
\overset{h_{f}}{\underset{i=1}{\sum}}\left[
\frac{{w}_{i}^{f}}
{\overset{h_{f}}
{\underset{m=1}{\sum}}{w}_{m}^{f}
}
\left(
\overset{h_{b}}{\underset{j=1}{\sum}}
\frac{{w}_{j}^{b}}
{
\overset{h_{b}}{\underset{n=1}{\sum}}{w}_{n}^{b}}
\l_{\mathbf{M}}(\mathbf{p}, \mathbf{p}^{f}_{i}, \mathbf{p}^{b}_{j})\right)\right].
\end{array}
\label{eq:final_provement_foreground}
\end{equation}
Finally, we complete the proof of Equ.~\eqref{eq:foreground_prove}.
Furthermore, we need to prove the background case defined in Equ.~\eqref{eq:background_prove}.
After a similar
process (i.e., from Equ.~\eqref{eq:step1} to Equ.~\eqref{eq:step4}), we can obtain:
\begin{equation}
\begin{array}{l}
\frac{1}{|\mathcal{B}_{b}||\mathcal{B}_{f}|}\mbox{E}_{\mathcal{B}_{b}}\mbox{E}_{\mathcal{B}_{f}}\left(
\underset{\mathbf{p}^{+} \in \mathcal{B}_{b}}{\sum}\underset{\mathbf{p}^{-} \in \mathcal{B}_{f}}{\sum}l_{\mathbf{M}}(\mathbf{p},
\mathbf{p}^{+}, \mathbf{p}^{-})\right)\\
=
\overset{h_{b}}{\underset{i=1}{\sum}}\left[
\frac{{w}_{i}^{b}}
{\overset{h_{b}}
{\underset{m=1}{\sum}}{w}_{m}^{b}
}
\left(
\overset{h_{f}}{\underset{j=1}{\sum}}
\frac{{w}_{j}^{f}}
{
\overset{h_{f}}{\underset{n=1}{\sum}}{w}_{n}^{f}}
\l_{\mathbf{M}}(\mathbf{p}, \mathbf{p}^{b}_{i}, \mathbf{p}^{f}_{j})\right)\right].
\end{array}
\label{eq:final_provement_background}
\end{equation}
As a result, we complete the proof of Equ.~\eqref{eq:background_prove}.
Based on the conclusions of Equ.~\eqref{eq:final_provement_foreground} and Equ.~\eqref{eq:final_provement_background},
we have:
\begin{equation}
\begin{array}{l}
\frac{1}{|\mathcal{B}_{c_{+}}||\mathcal{B}_{c_{-}}|}\mbox{E}_{\mathcal{B}_{c_{+}}}\mbox{E}_{\mathcal{B}_{c_{-}}}\left(
\underset{\mathbf{p}^{+} \in \mathcal{B}_{c_{+}}}{\sum}\underset{\mathbf{p}^{-} \in \mathcal{B}_{c_{-}}}{\sum}l_{\mathbf{M}}(\mathbf{p},
\mathbf{p}^{+}, \mathbf{p}^{-})\right)\\
=\overset{h_{c_{+}}}{\underset{i=1}{\sum}}\left[
\frac{{w}_{i}^{c_{+}}}
{\overset{h_{c_{+}}}
{\underset{m=1}{\sum}}{w}_{m}^{c_{+}}
}
\left(
\overset{h_{c_{-}}}{\underset{j=1}{\sum}}
\frac{{w}_{j}^{c_{-}}}
{
\overset{h_{c_{-}}}{\underset{n=1}{\sum}}{w}_{n}^{c_{-}}}
\l_{\mathbf{M}}(\mathbf{p}, \mathbf{p}^{c_{+}}_{i}, \mathbf{p}^{c_{-}}_{j})\right)\right].
\end{array} \vspace{-0.0cm}
\end{equation}
Consequently, we complete the proof of Theorem~\ref{theo:Weighted_reservoir_sampling_supp}.
\end{proof}

\break

\section{Performance with and without metric learning}
\label{sec:exp_w_metric_learning}

To justify the effect of different metric learning mechanisms, we
design several experiments on five video sequences.
Fig.~\ref{fig:metric_non_metric} and Tab.~\ref{Tab:metric_success_rate}  show the corresponding experimental
results of different metric learning mechanisms
in CLE, VOR, and success rate.
From Fig.~\ref{fig:metric_non_metric} and Tab.~\ref{Tab:metric_success_rate}, we can see that the performance of
metric learning is  better than that of no
metric learning. In addition, the performance of metric learning with
no eigendecomposition is close
to that of metric learning with step-by-step eigendecomposition, and
better than that of
metric learning with final eigendecomposition. Therefore, the obtained
results are consistent with those
in~\cite{chechik2010large}. Besides, metric learning with step-by-step
eigendecomposition
is much slower than that with no eigendecomposition which is adopted
by the proposed
tracking algorithm.

\begin{table}[t]
\begin{center}
\scalebox{0.55}
{
\begin{tabular}{c||c|c|c|c|c||c|c|c|c|c||c|c|c|c|c}
\hline \scriptsize
& \multicolumn{5}{|c||}{CLE} & \multicolumn{5}{|c||}{VOR} & \multicolumn{5}{|c}{Success Rate}\\
\hline
& \makebox[1.0cm]{cubicle}   &  \makebox[1.0cm]{football} &
\makebox[1.0cm]{iceball} & \makebox[1.1cm]{trellis70} & \makebox[1.0cm]{seq-jd} & \makebox[1.0cm]{cubicle}   &  \makebox[1.0cm]{football} &
\makebox[1.0cm]{iceball} & \makebox[1.1cm]{trellis70} & \makebox[1.0cm]{seq-jd} & \makebox[1.0cm]{cubicle}   &  \makebox[1.0cm]{football} &
\makebox[1.0cm]{iceball} & \makebox[1.1cm]{trellis70} & \makebox[1.0cm]{seq-jd}\\
\hline \hline
ML w/o eigen                &  4.31      &  3.14   &  3.03         & 5.61   & 4.30 &0.74   & 0.67    &0.68   & 0.78   & 0.72 &\bf 0.98    &0.88    &0.93   & 0.98    &0.94\\
ML with final eigen         &  5.95      & 5.69     &  5.10       & 8.49    &7.07 & 0.67    &0.59    &0.63    &0.70    &0.63 &  0.94   & 0.74   & 0.90   & 0.94   & 0.82\\
ML with step-by-step eigen  &  \bf 2.16  &  \bf 1.89    &\bf 1.30    &\bf 4.54    &\bf 3.31 & \bf 0.79   & \bf 0.71   & \bf 0.69    &\bf 0.82  & \bf 0.75 &\bf 0.98    &\bf 0.90   & 0\bf .95    &\bf 0.99    &\bf  0.95\\
No metric learning           & 5.45      &51.73      &4.31           &8.85    &6.29 &0.66    &0.27    &0.64   & 0.68    &0.63 & 0.86   & 0.36   & 0.88   & 0.91   & 0.82\\
\hline
\end{tabular}
}
\end{center}
\vspace{-0.7cm}
\caption{
Quantitative evaluation of the proposed tracker with different
metric learning configurations on five video sequences.
The table reports their average tracking results in
CLE, VOR, and success rate.
 \vspace{-0.5cm}}
\label{Tab:metric_success_rate}
\end{table}

\begin{figure*}[t]
\centering
\includegraphics[scale=0.405]{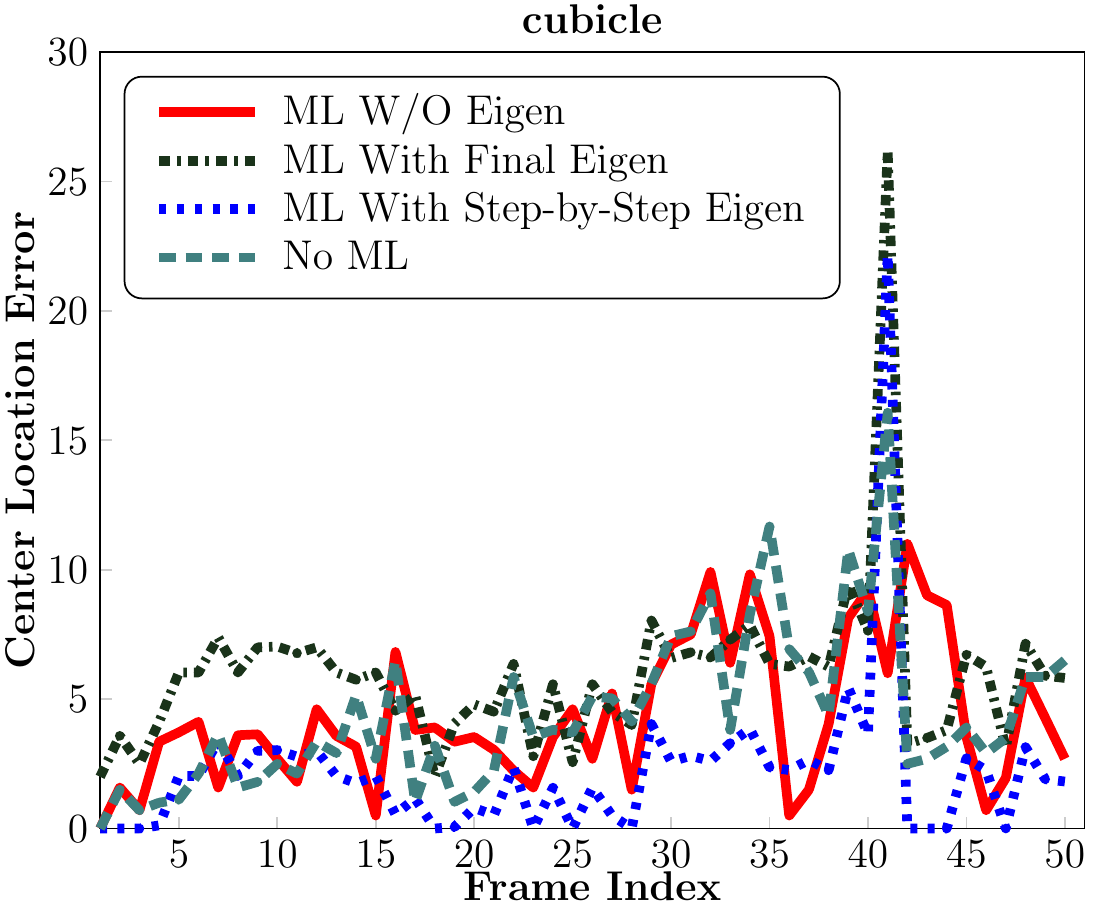}
\includegraphics[scale=0.405]{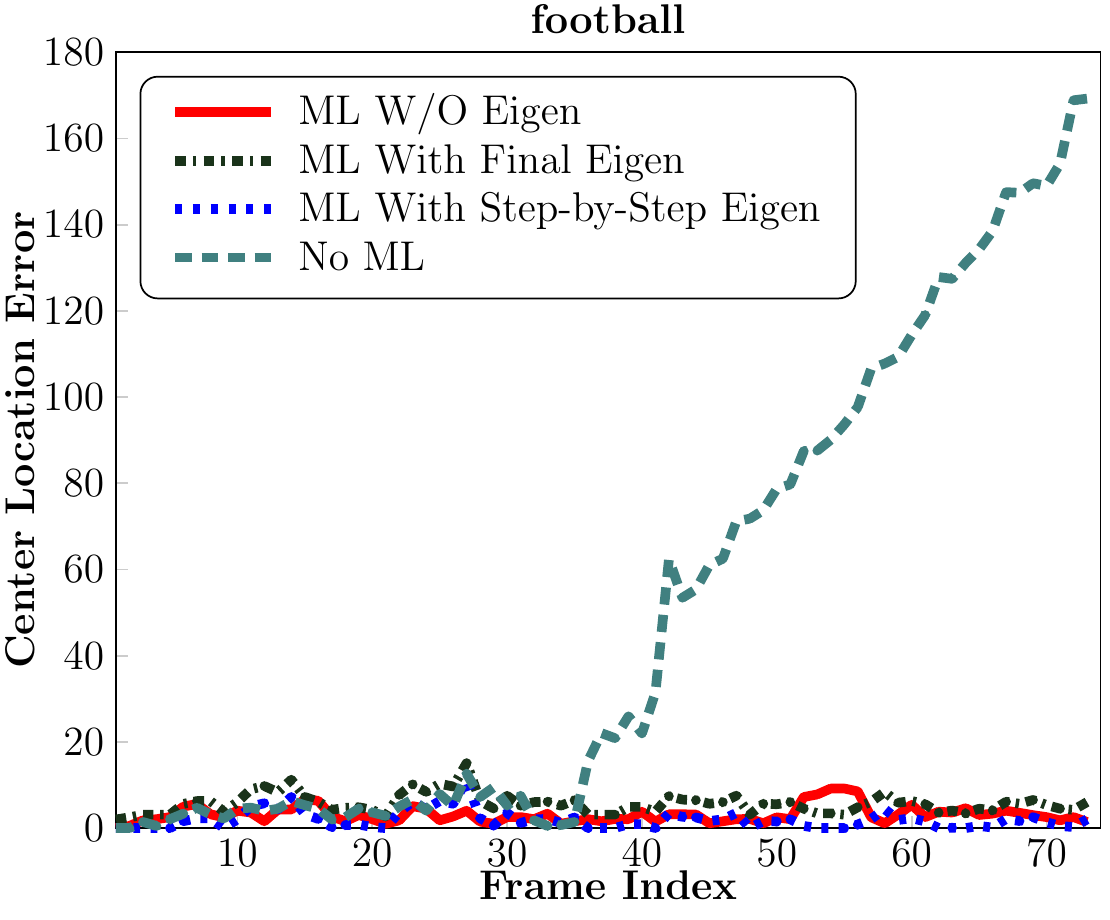}\\
\includegraphics[scale=0.405]{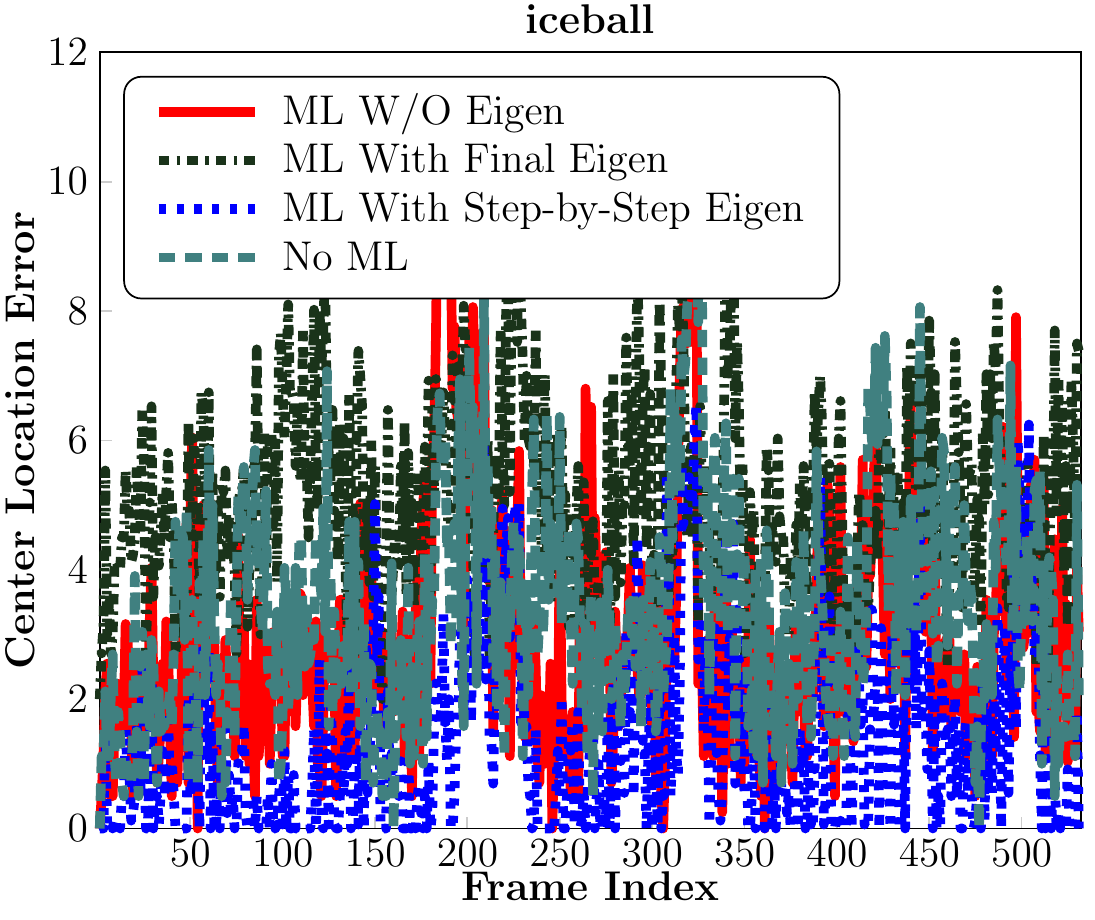}
\includegraphics[scale=0.405]{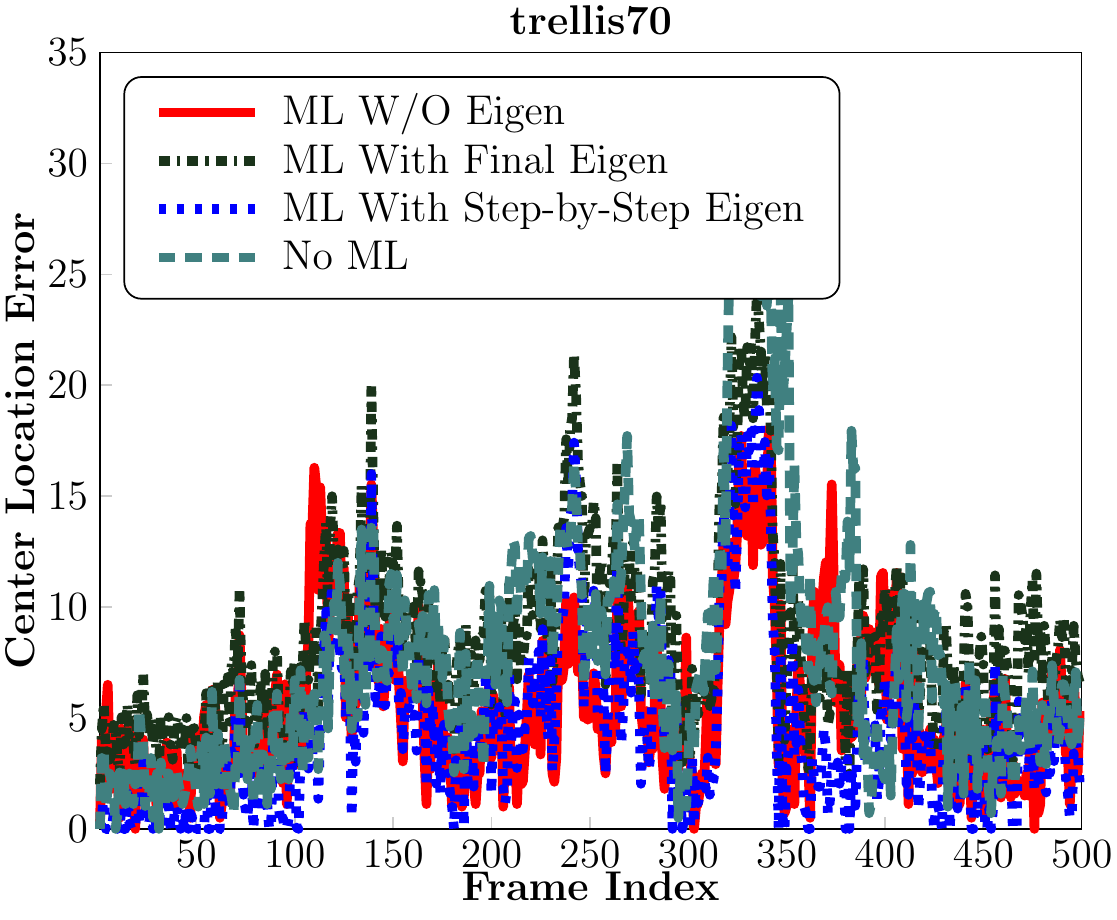}
\includegraphics[scale=0.405]{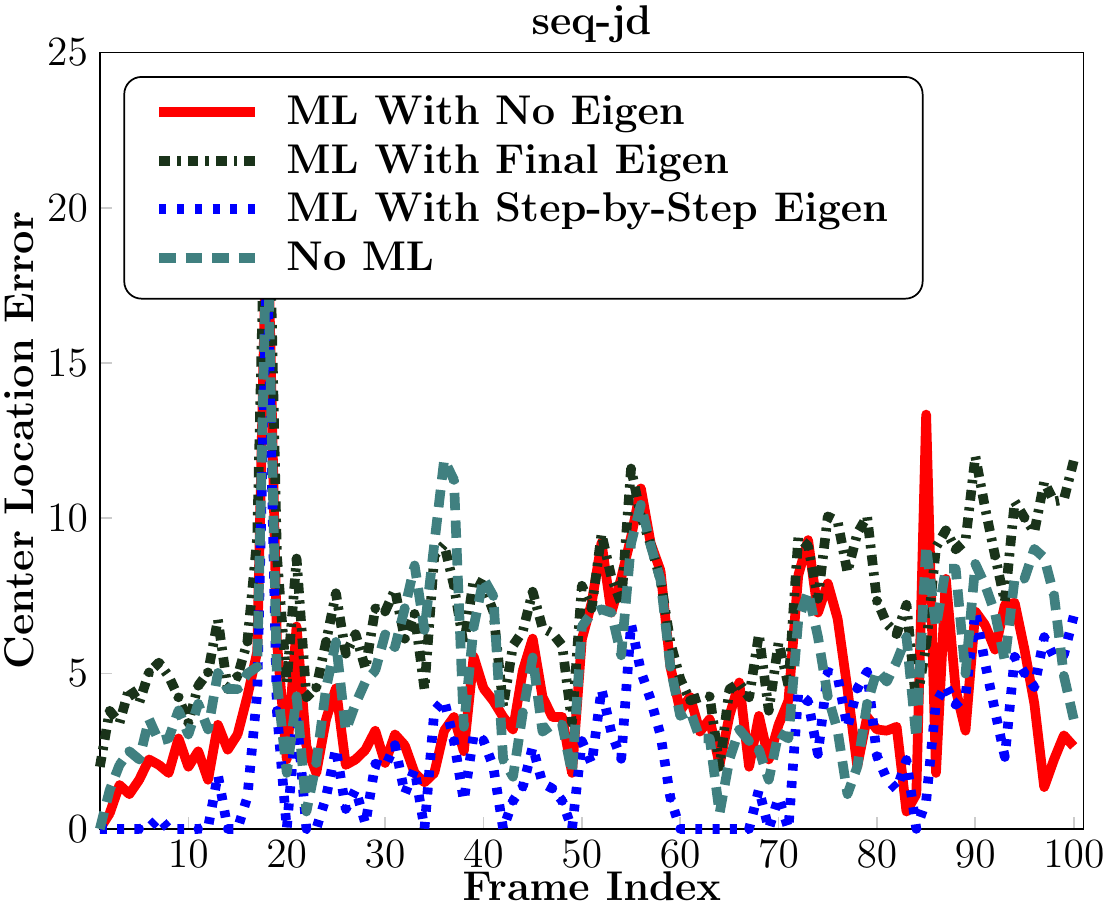}\\
\includegraphics[scale=0.405]{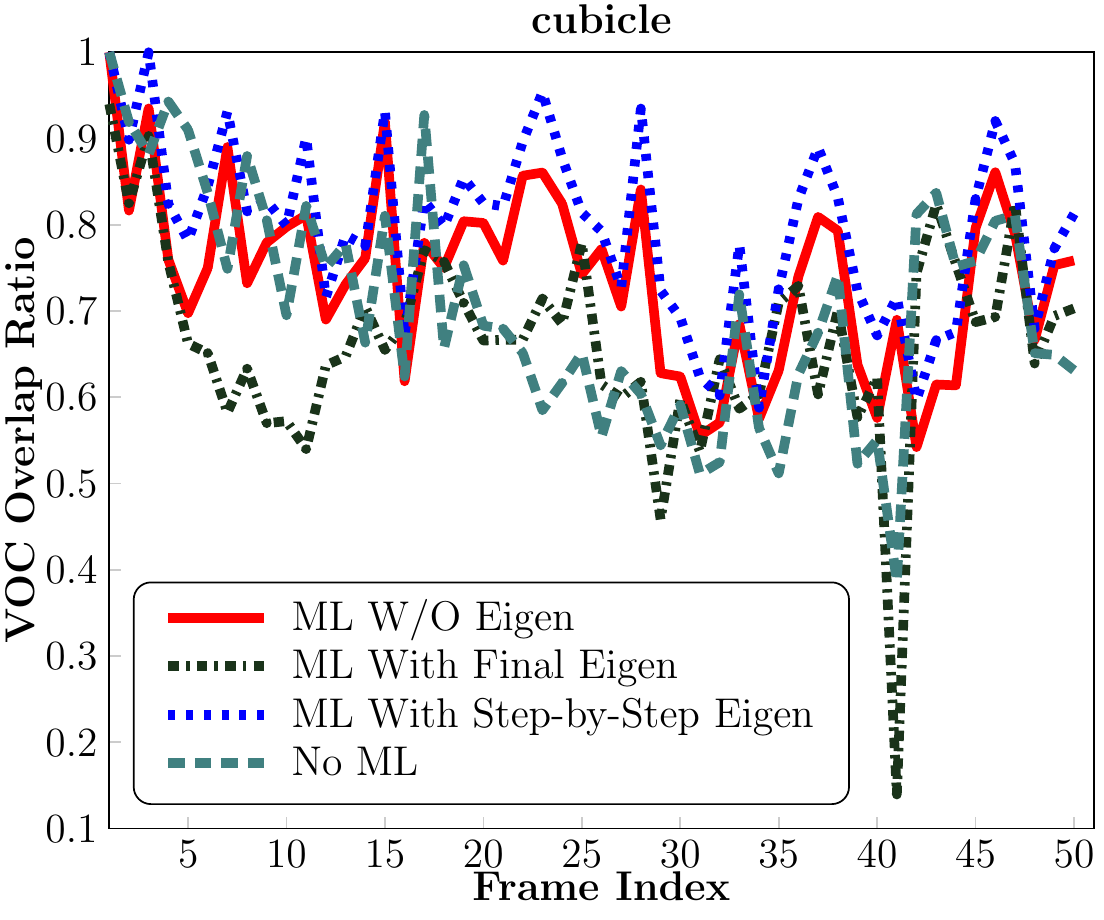}
\includegraphics[scale=0.405]{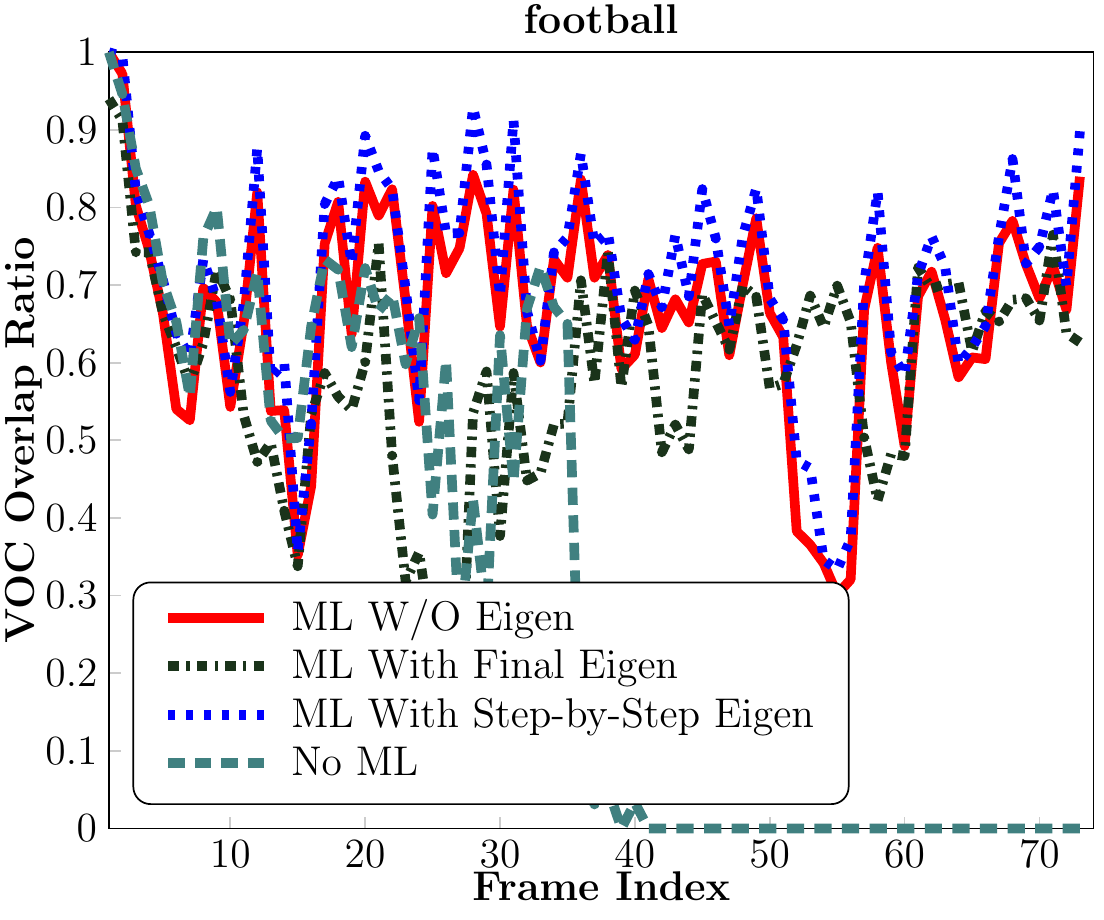}\\
\includegraphics[scale=0.405]{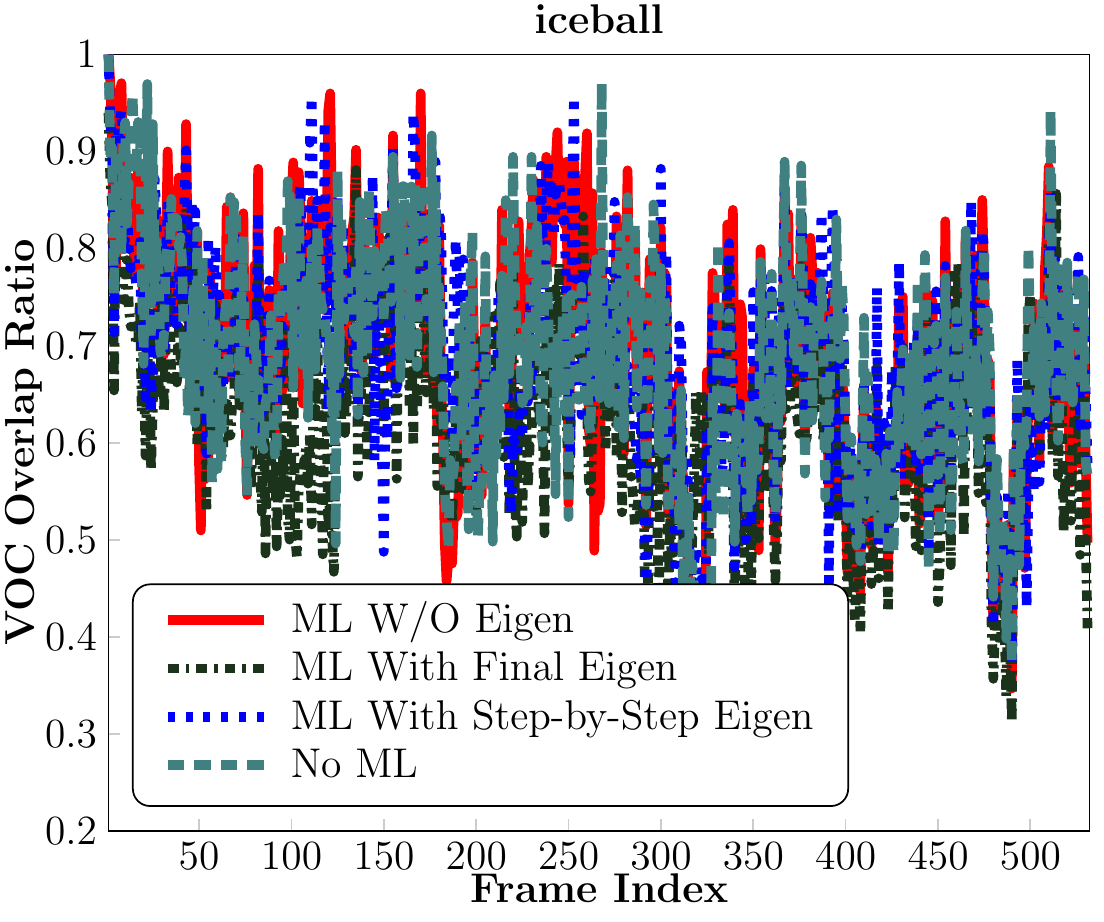}
\includegraphics[scale=0.405]{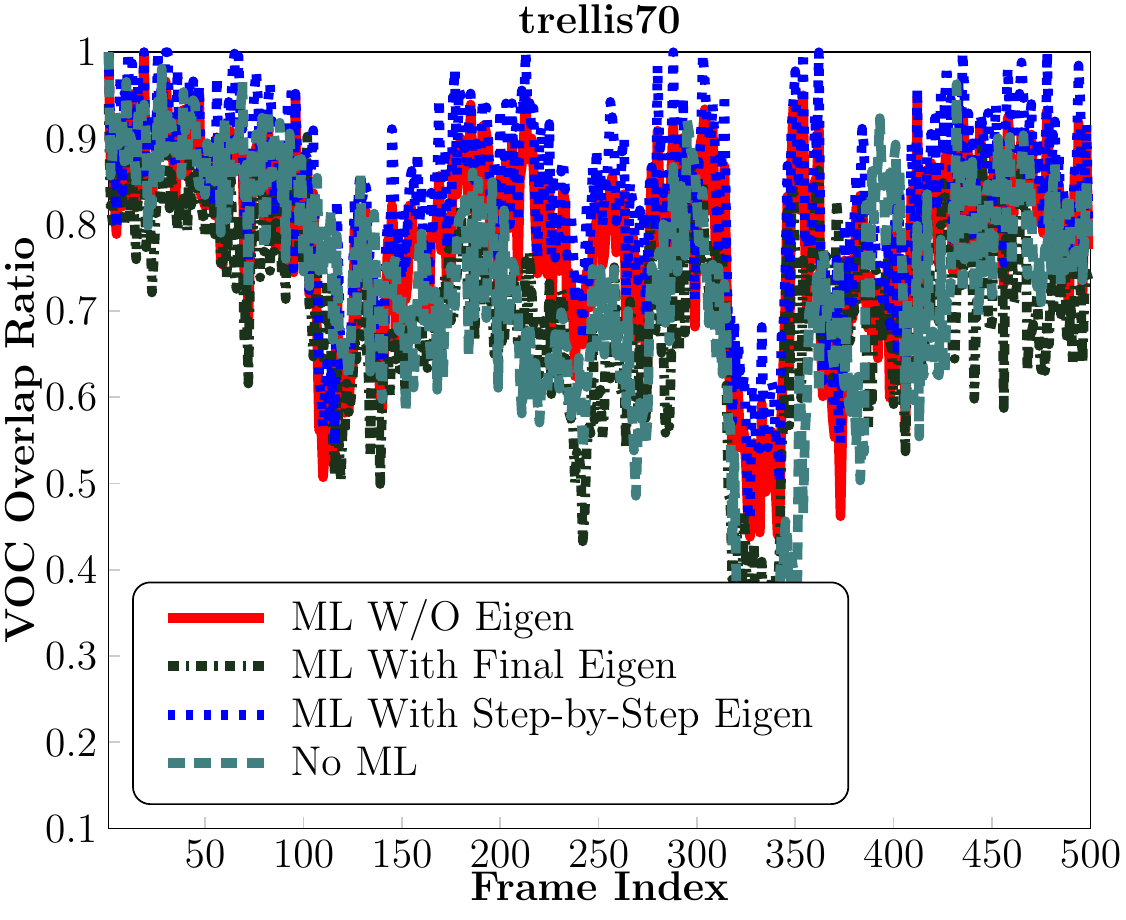}
\includegraphics[scale=0.405]{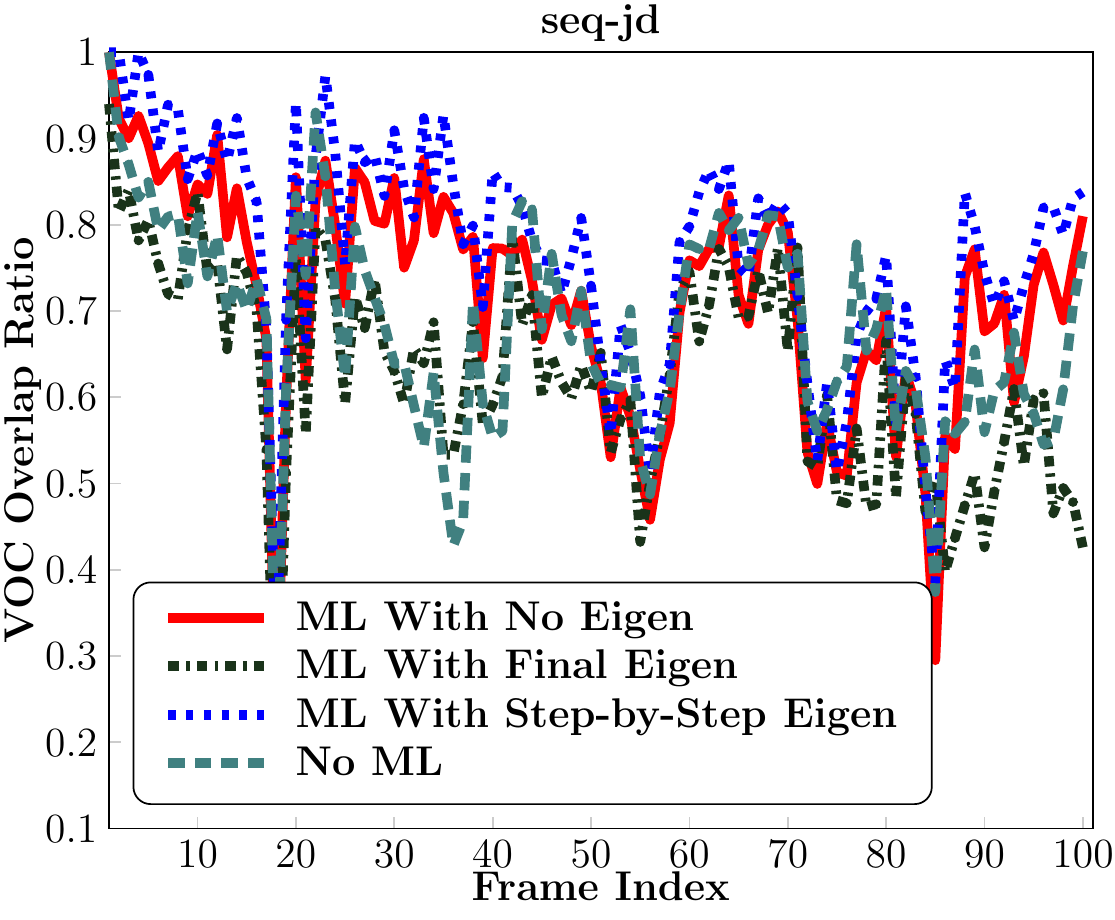}
\vspace{-0.16cm}
 \vspace{-0.0cm}
 \caption{Quantitative evaluation of the proposed tracker with/without
metric learning on five video sequences. The top two rows are
associated with
 the tracking performance in CLE, while the bottom two rows correspond to
 the tracking performance in VOR.
 \vspace{-0.16cm}}
 \label{fig:metric_non_metric}
\end{figure*}

\section{Comparison of different linear representations}
\label{sec:diff_linear_represnetations}

We evaluate the performance of four types of
linear representations including our linear
representation with metric learning,
our linear representation without
metric learning,
compressive sensing linear representation~\cite{Li-Shen-Shi-cvpr2011},
and $\ell_{1}$-regularized linear
representation~\cite{Meo-Ling-ICCV09}.
For a fair comparison, we utilize the raw pixel features which are the same as
\cite{Li-Shen-Shi-cvpr2011,Meo-Ling-ICCV09}.
Fig.~\ref{fig:regression} shows the performance of these four linear
representation methods
in CLE and VPR on four video sequences. Clearly, our linear
representation with metric learning consistently achieves
lower CLE (higher VOR) performance in most frames than the three other linear representations.

\begin{figure*}[t]
 \vspace{-1.8cm}
\centering
\includegraphics[scale=0.36]{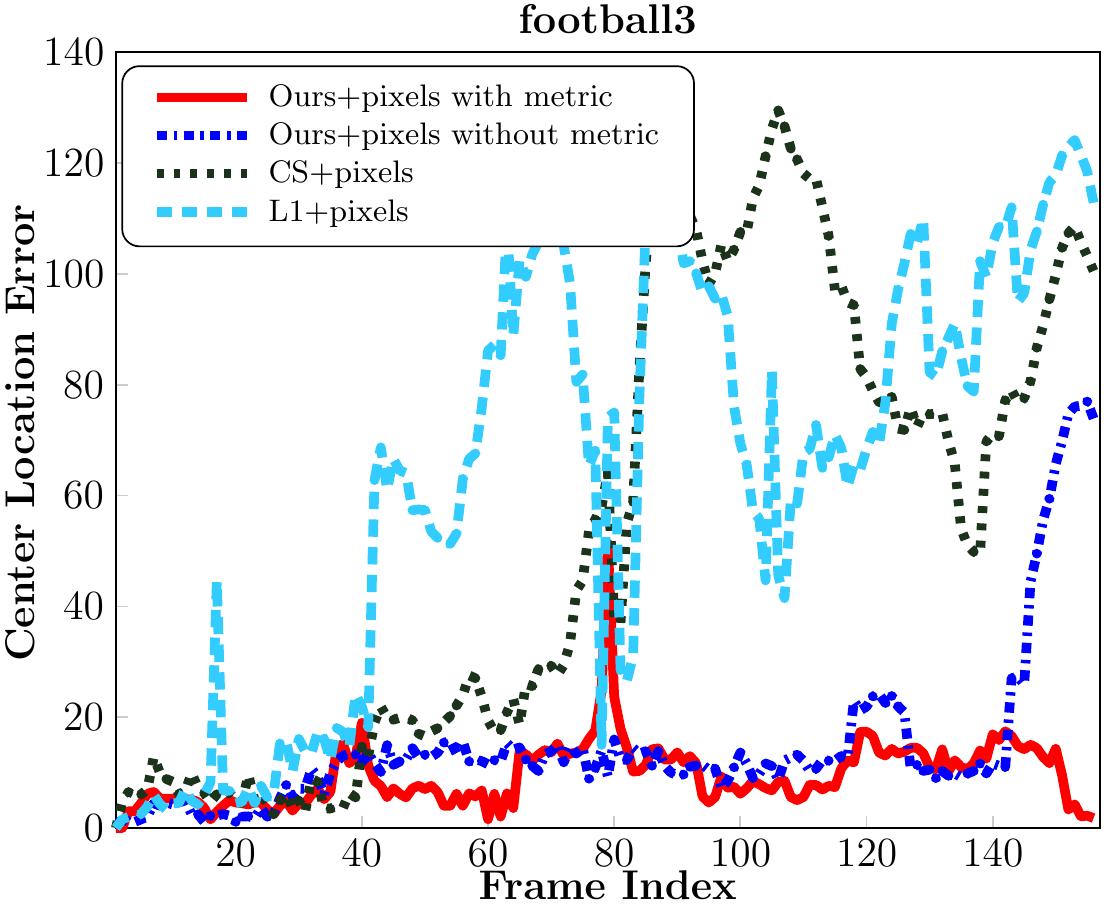}
\hspace{-0.13cm}
\includegraphics[scale=0.36]{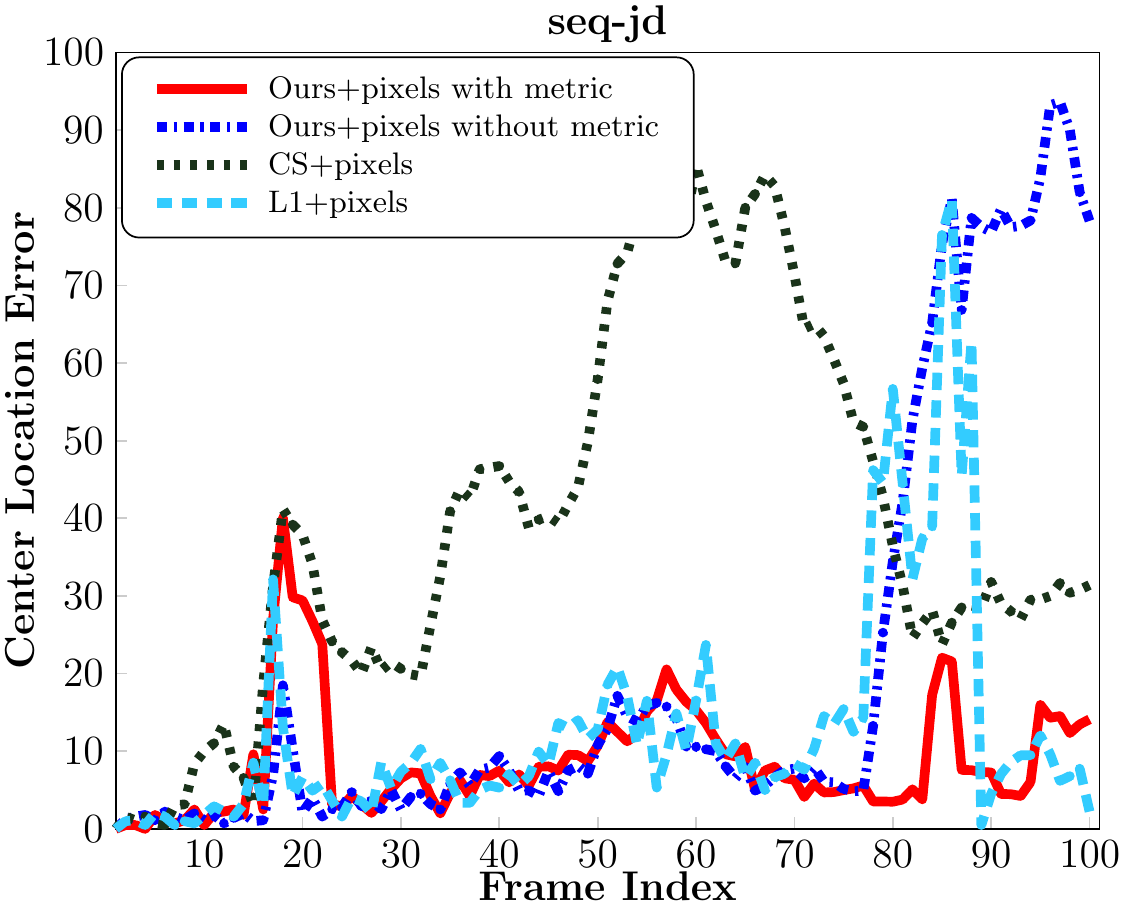}\\
\includegraphics[scale=0.36]{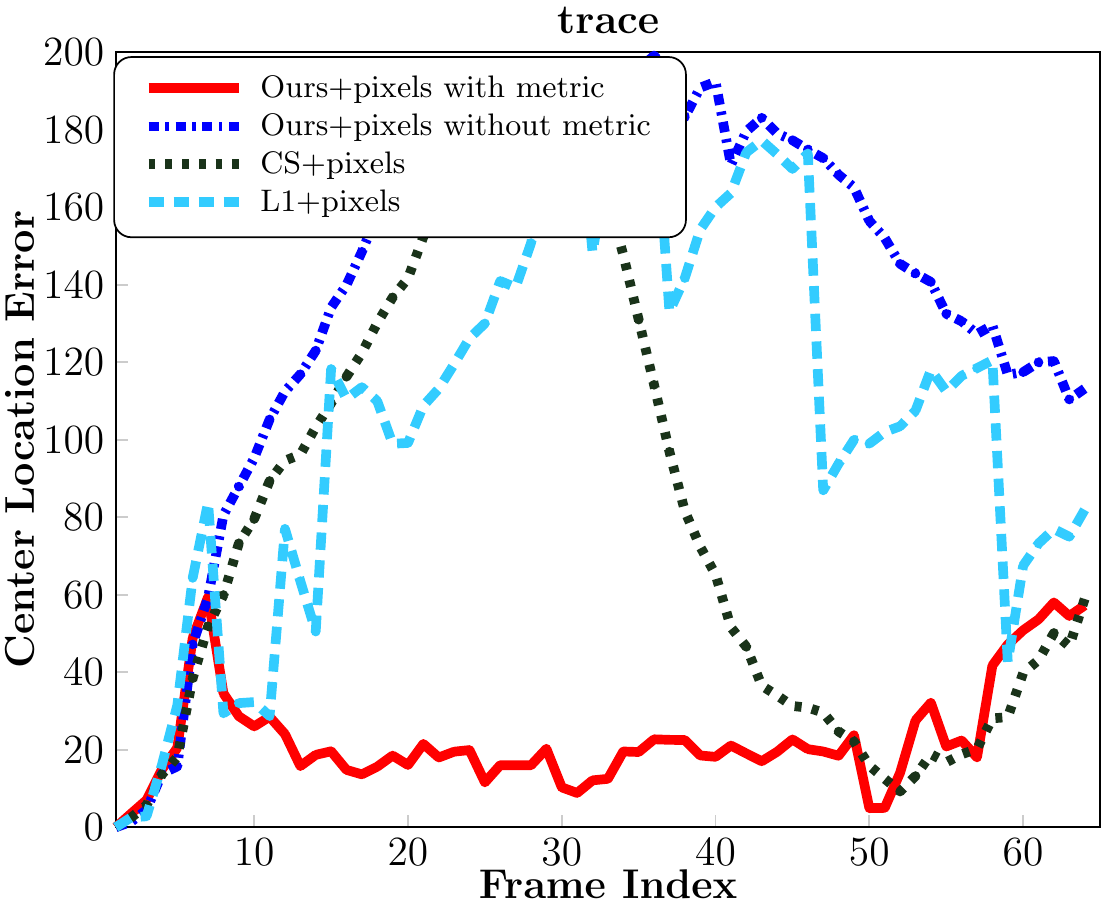}
\hspace{-0.13cm}
\includegraphics[scale=0.36]{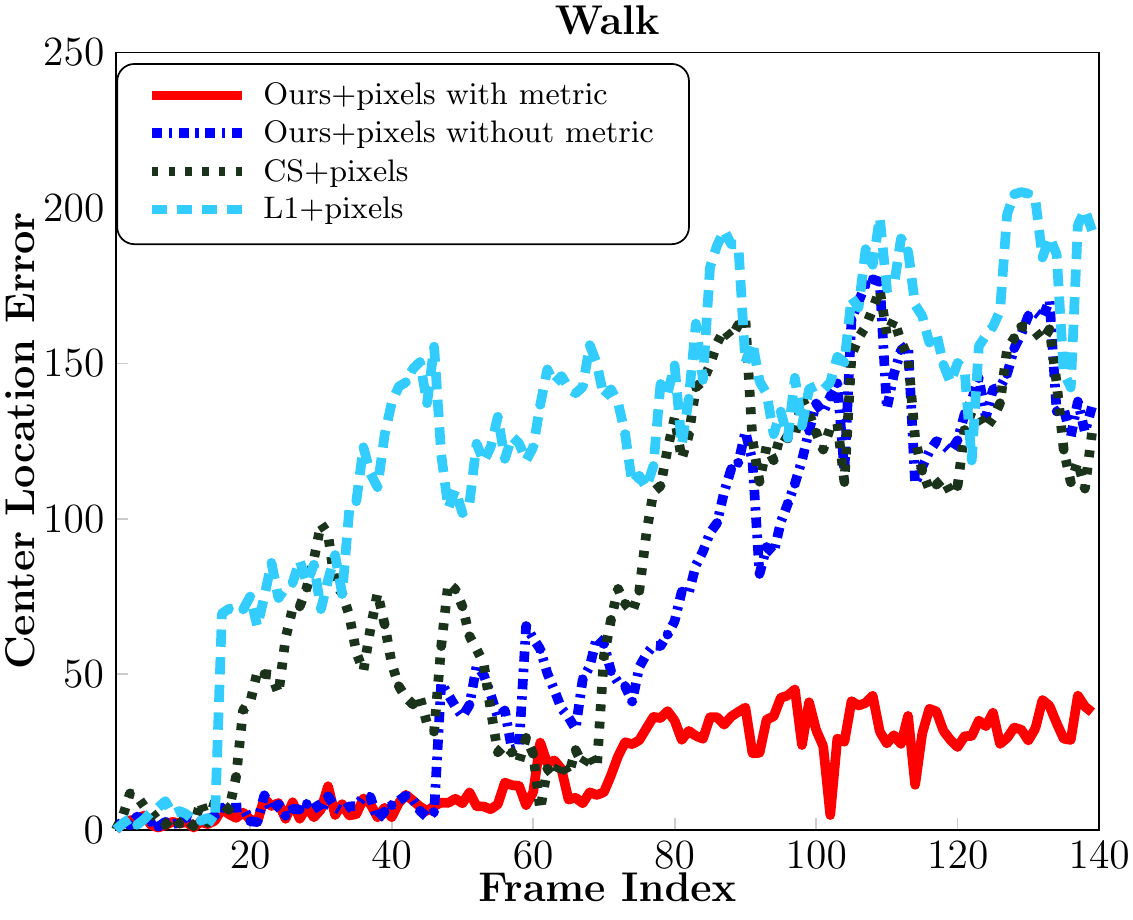}\\
\includegraphics[scale=0.36]{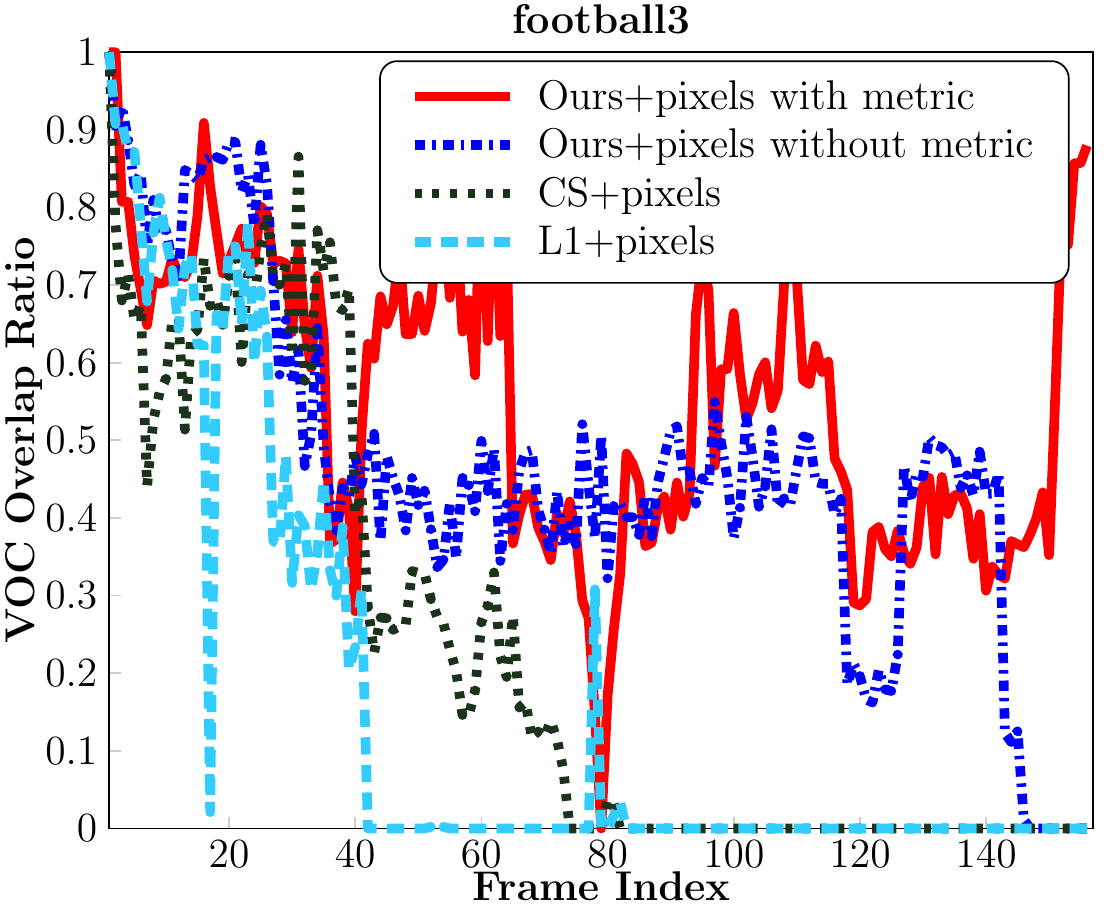}
\includegraphics[scale=0.36]{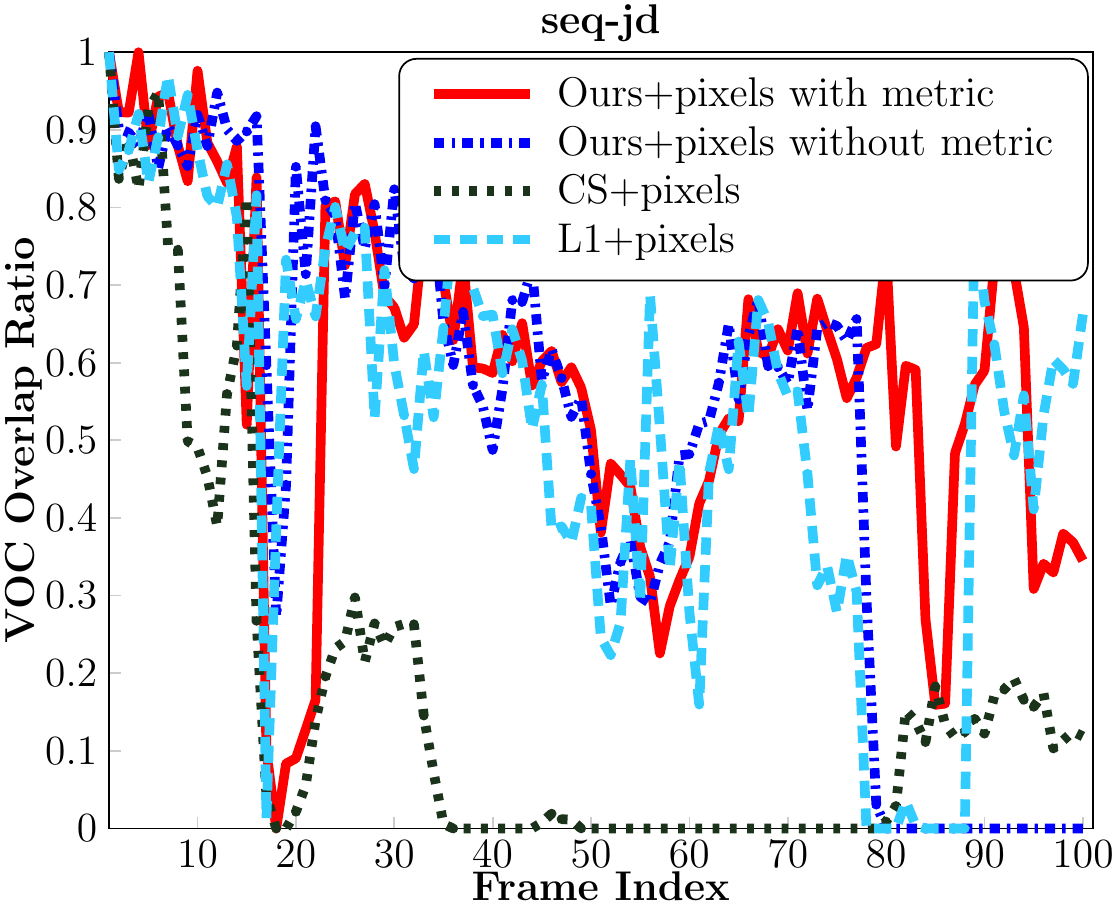}\\
\includegraphics[scale=0.36]{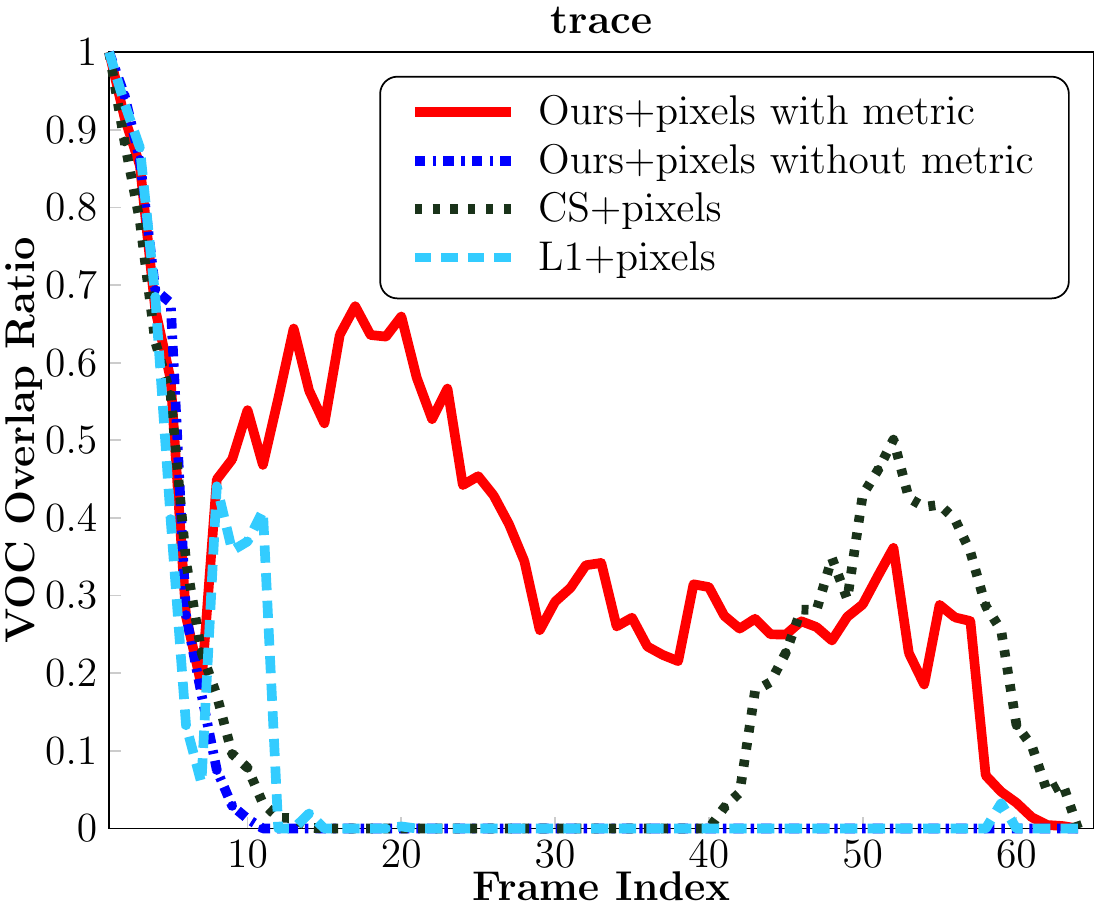}
\includegraphics[scale=0.36]{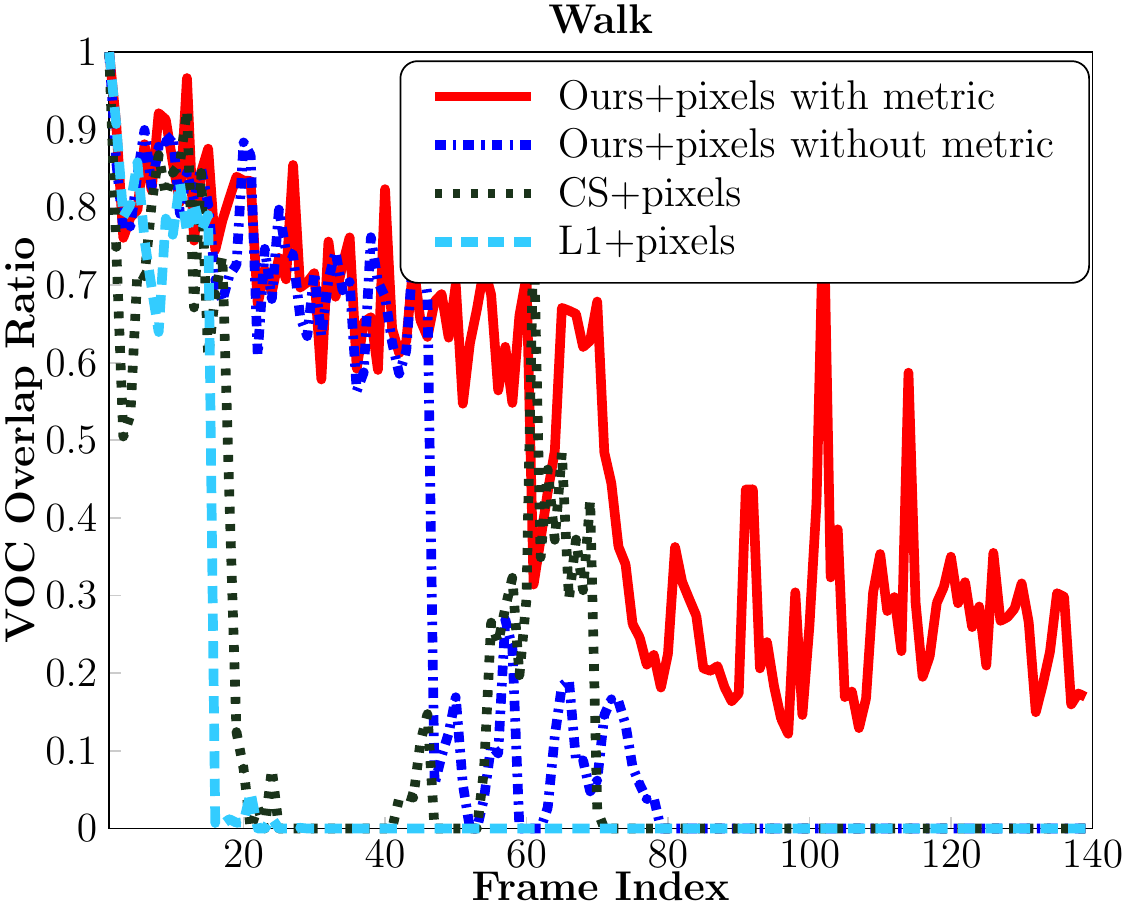}\\
\vspace{-0.16cm}
\vspace{-0.0cm}
 \caption{Quantitative comparison of different linear representation
methods in CLE and VOR on four video sequences.
 \vspace{-0.15cm}}
 \label{fig:regression}
\end{figure*}

\section{Evaluation of different sampling methods}
\label{sec:diff_sampling}

We aim to examine the performance of the two
sampling methods.
Fig.~\ref{fig:sampling} shows the experimental results of the two
sampling methods in CLE and VOR on four video sequences.
From Fig.~\ref{fig:sampling}, we can see that weighted reservoir
sampling performs better than
ordinary reservoir sampling.

\begin{figure*}[t]
 \vspace{-1.8cm}
\centering
\includegraphics[scale=0.405]{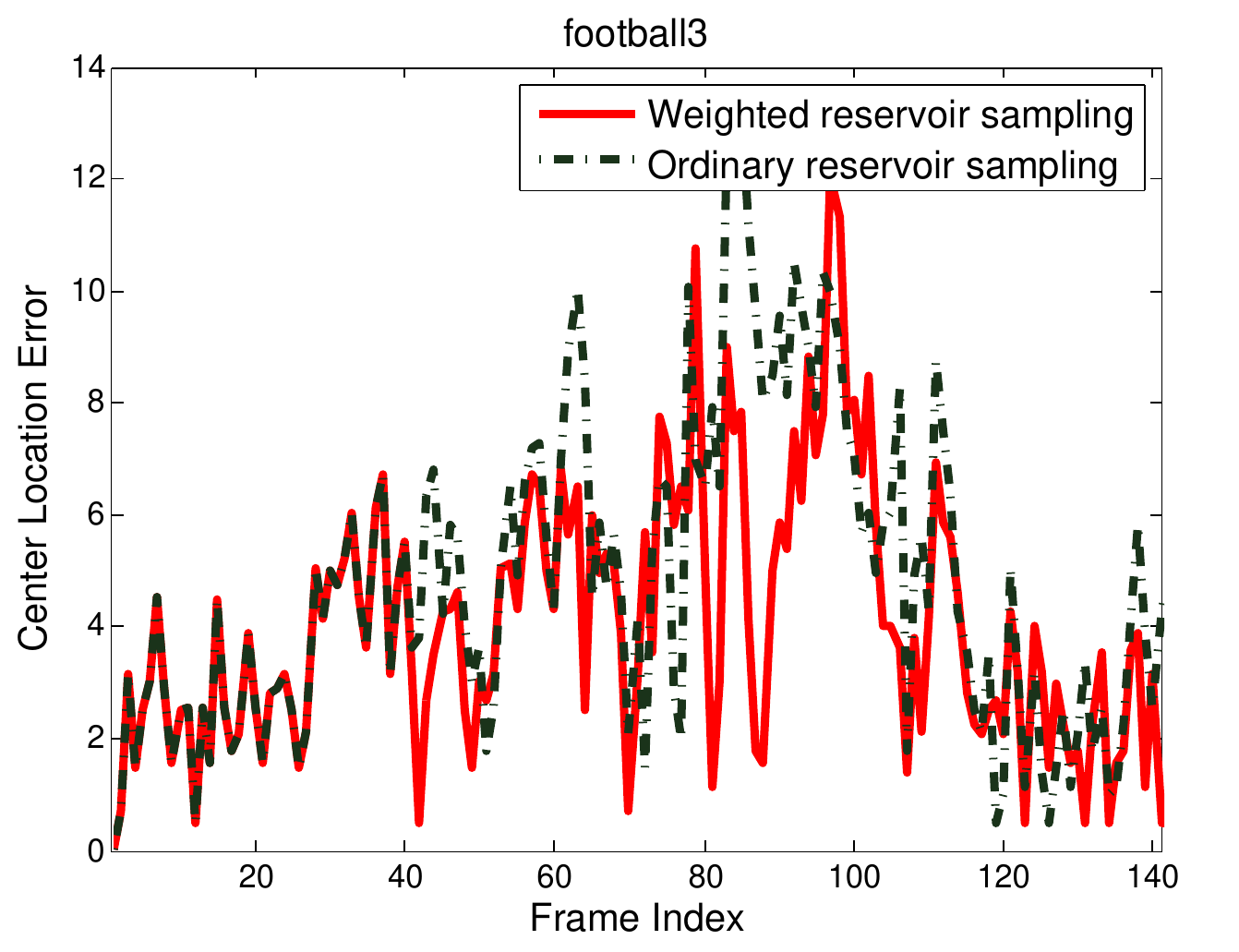}
\includegraphics[scale=0.45]{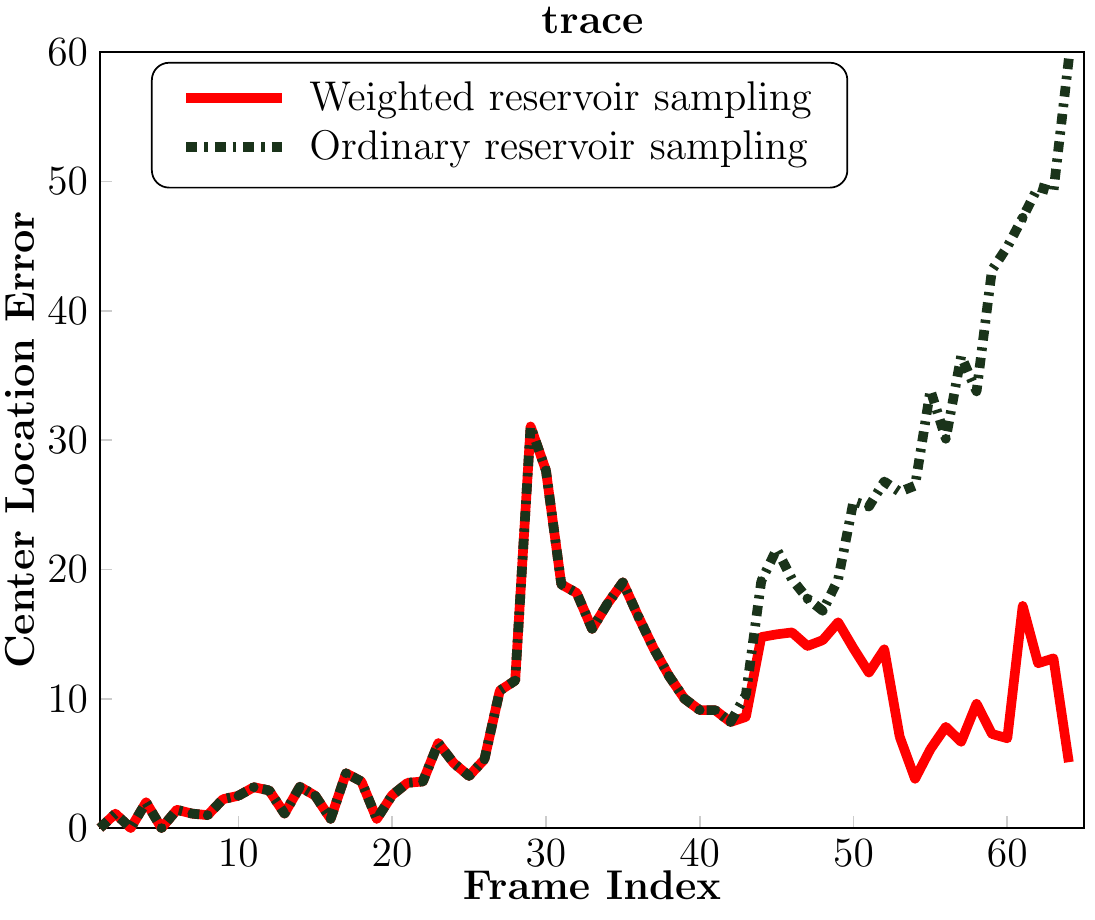}\\
\includegraphics[scale=0.45]{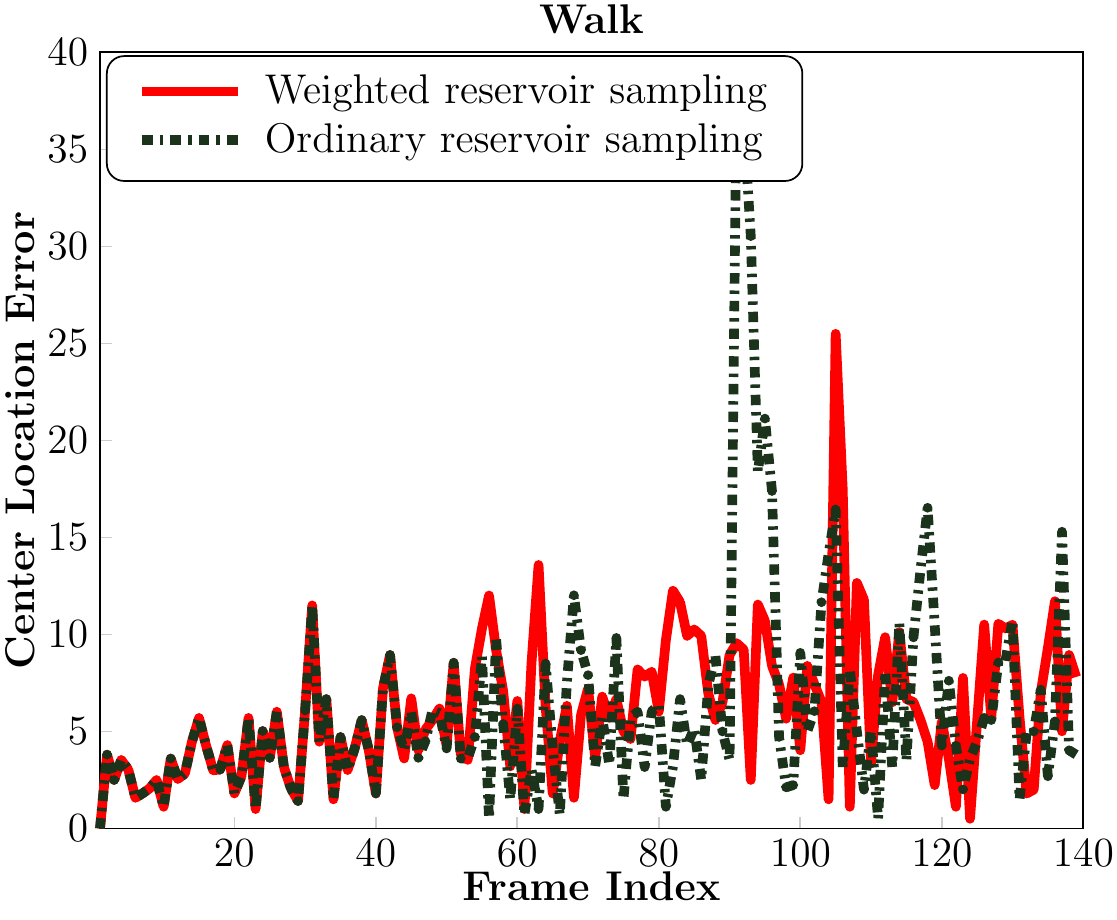}
\includegraphics[scale=0.45]{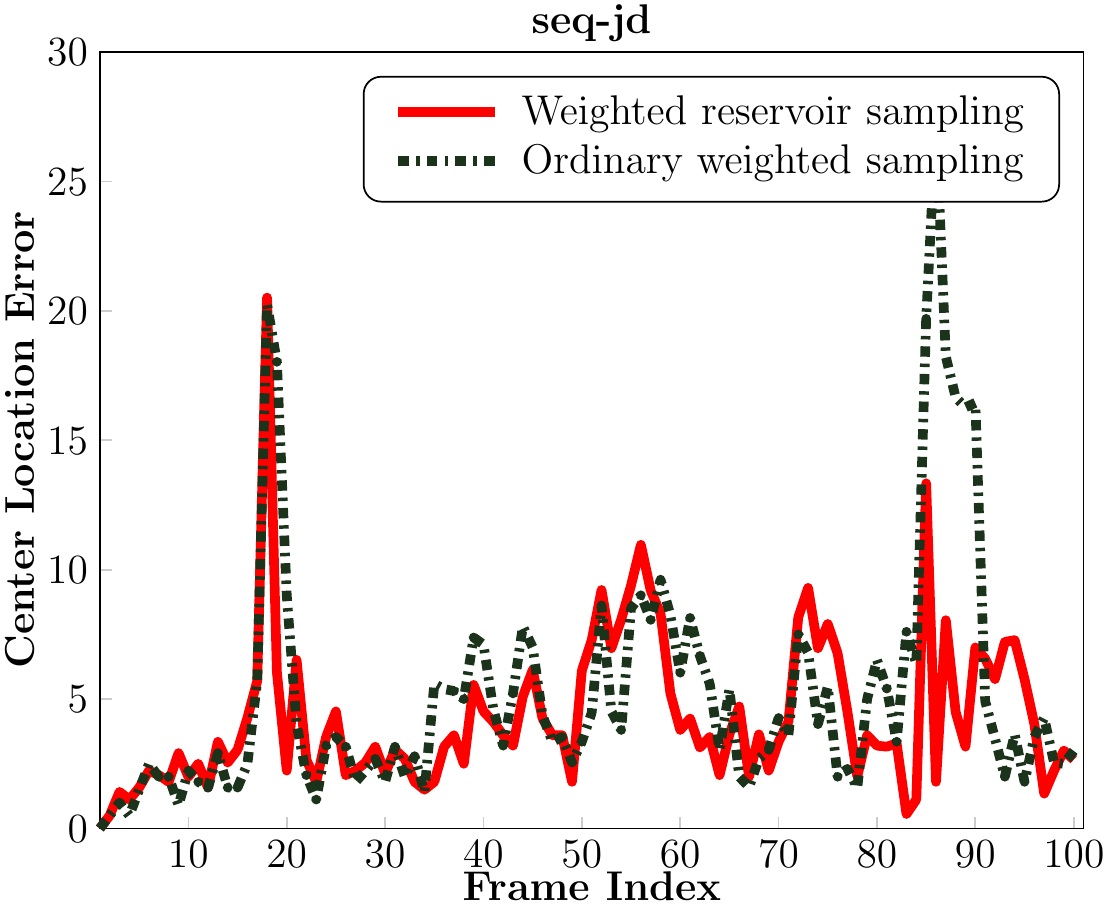}\\
\includegraphics[scale=0.405]{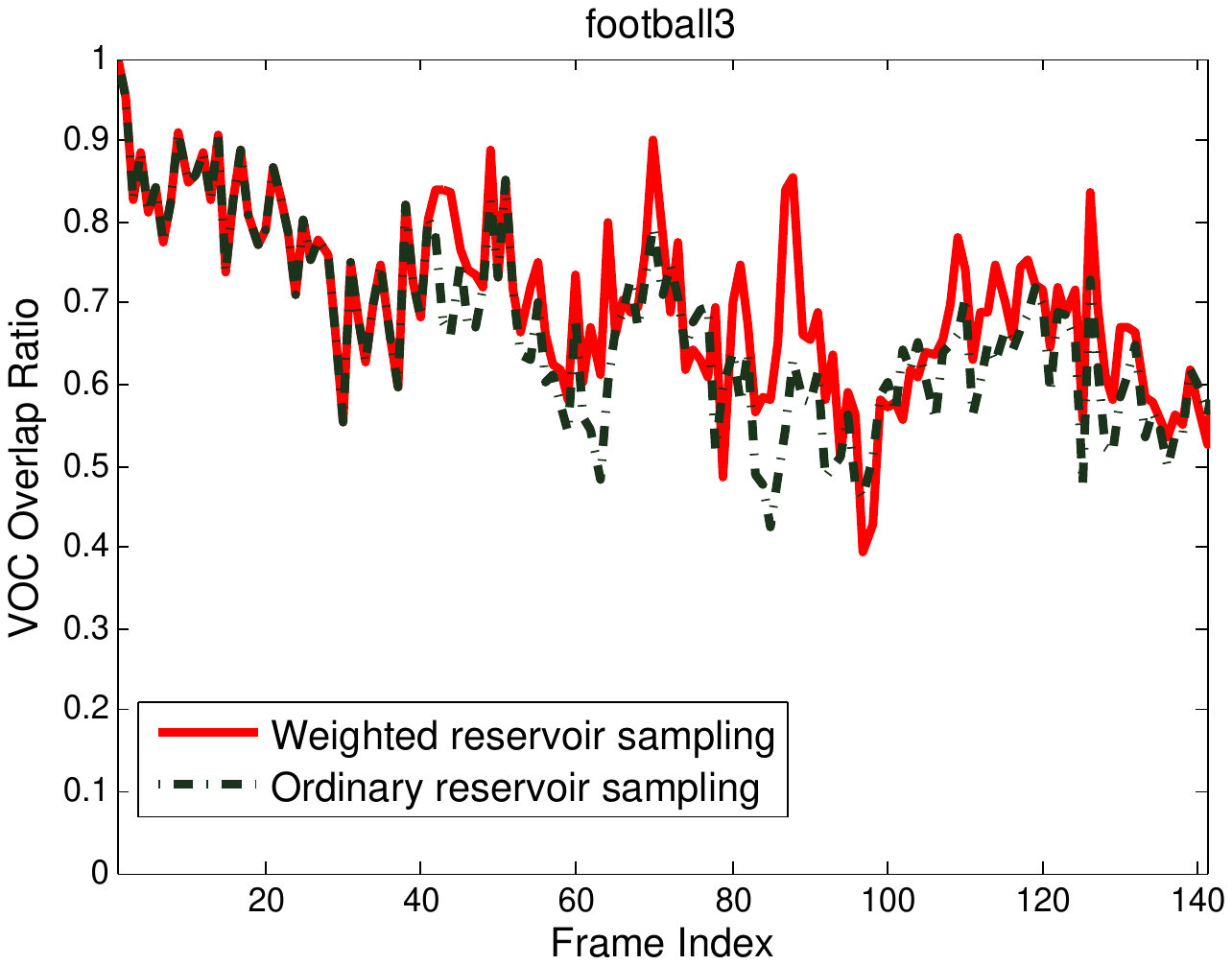}
\includegraphics[scale=0.45]{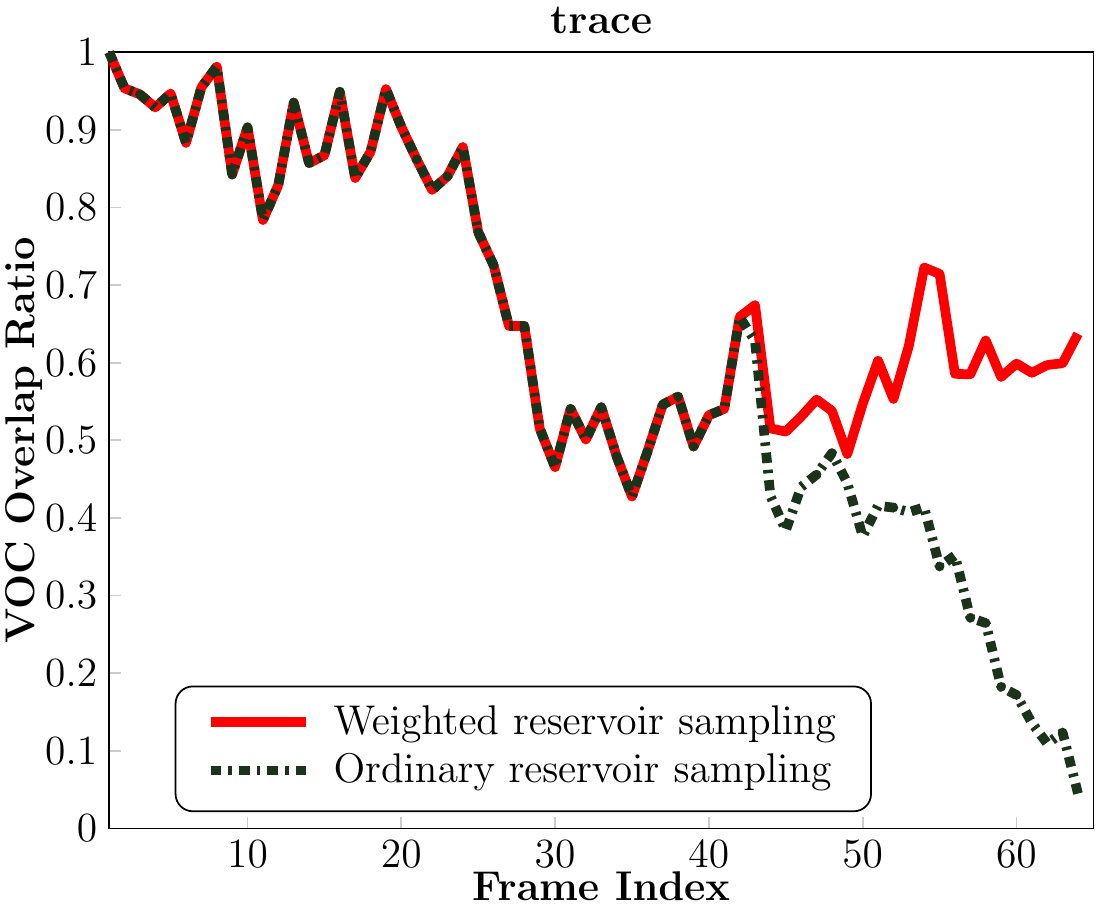}\\
\includegraphics[scale=0.45]{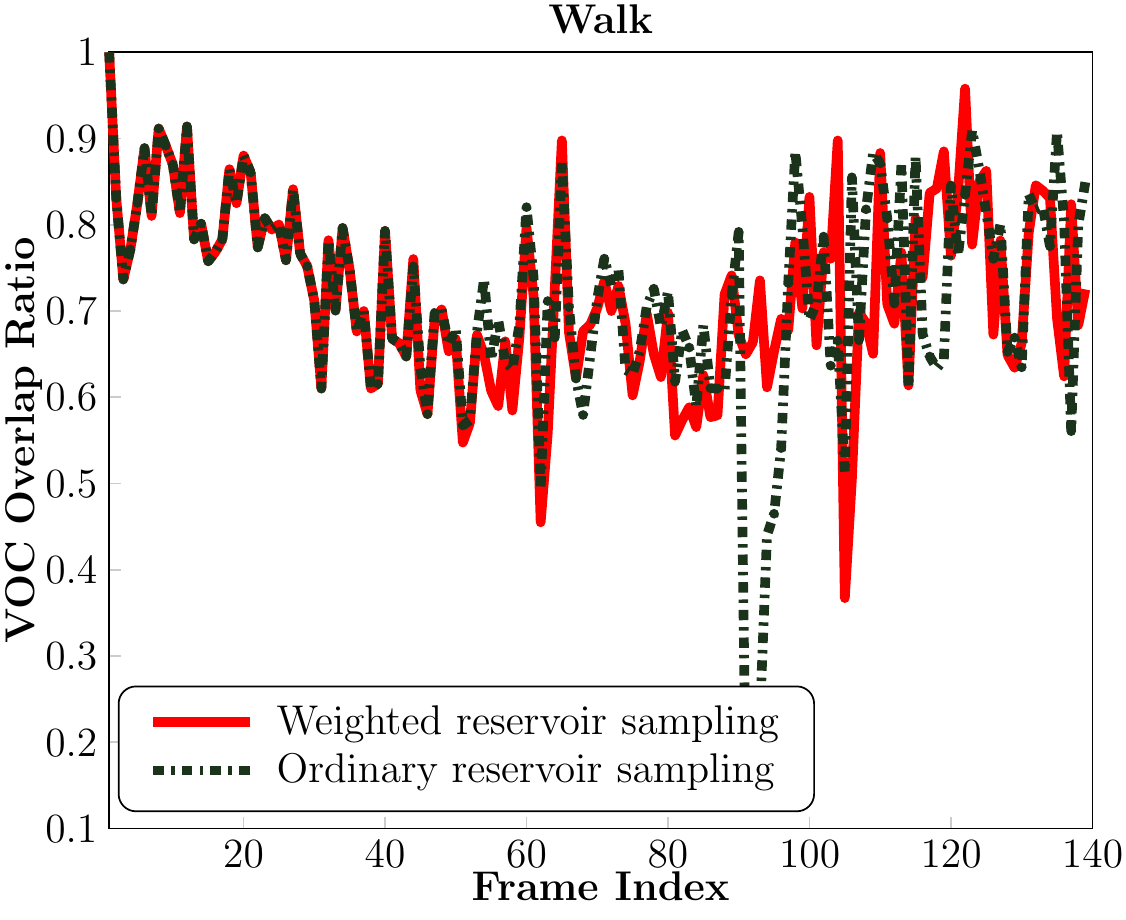}
\includegraphics[scale=0.45]{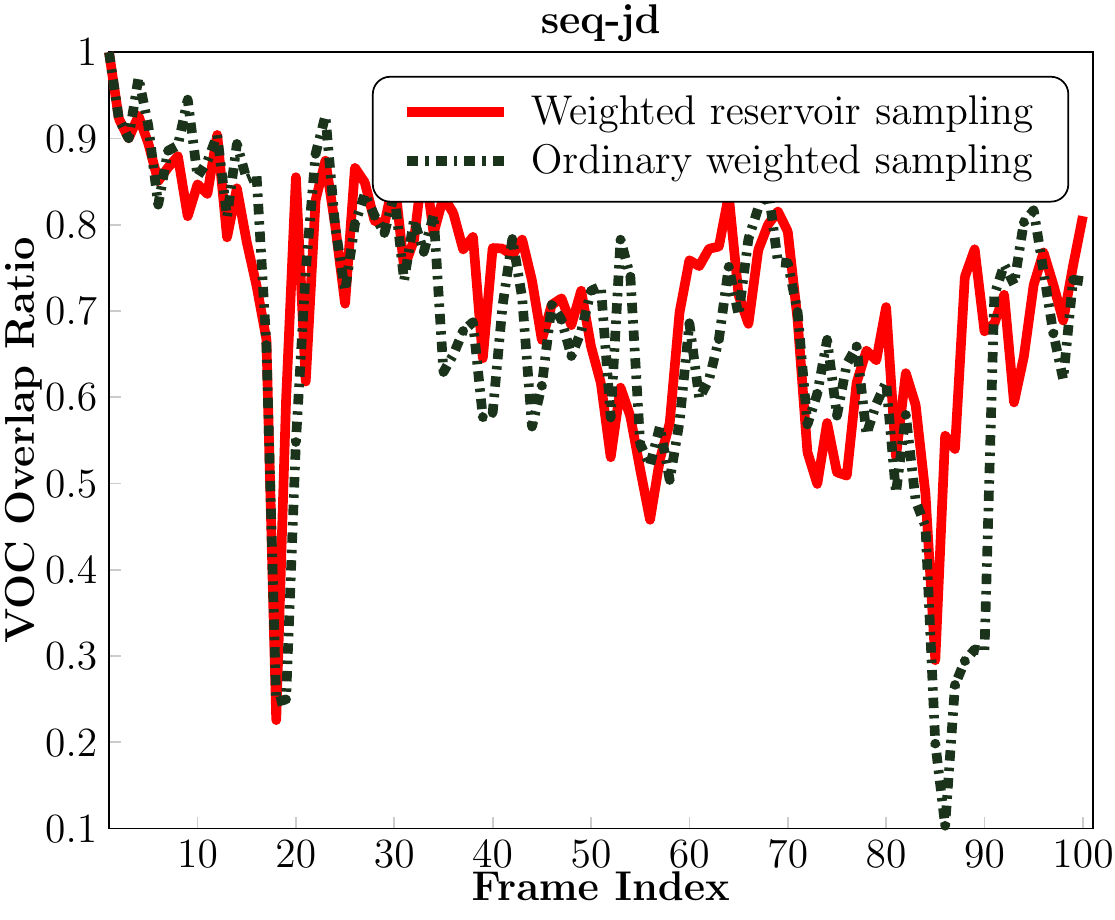}
\vspace{-0.1cm}
\caption{Quantitative comparison of different sampling methods in
CLE on four video sequences
 The top two rows
show the tracking performance in CLE;
 the bottom two rows display the tracking performance in VOR.
Before exceeding the buffer size limit (approximately occurring between frame 40 and
frame 50), the performances
 of different sampling methods are identical.
 \vspace{-0.1cm}}
  \label{fig:sampling}
\end{figure*}

\section{Pedestrian identification}
\label{sec:p_classification}

Based on Equ.~\eqref{eq:identification_score},
we perform the pedestrian identification task
with two viewpoints.
Fig.~\ref{fig:more_example_supp} shows
the quantitative frame-by-frame identification results
for five pedestrians with different viewpoints.
Moreover, Figs.~\ref{fig:Id1_View1}-\ref{fig:Id6_View1}
display the tracking and identification results
for six pedestrians on several representative frames.
Clearly, our method is able to achieve a robust
pedestrian identification performance in different cases.

\begin{figure}[t]
\vspace{-0.26cm}
\centering
\includegraphics[scale=0.5]{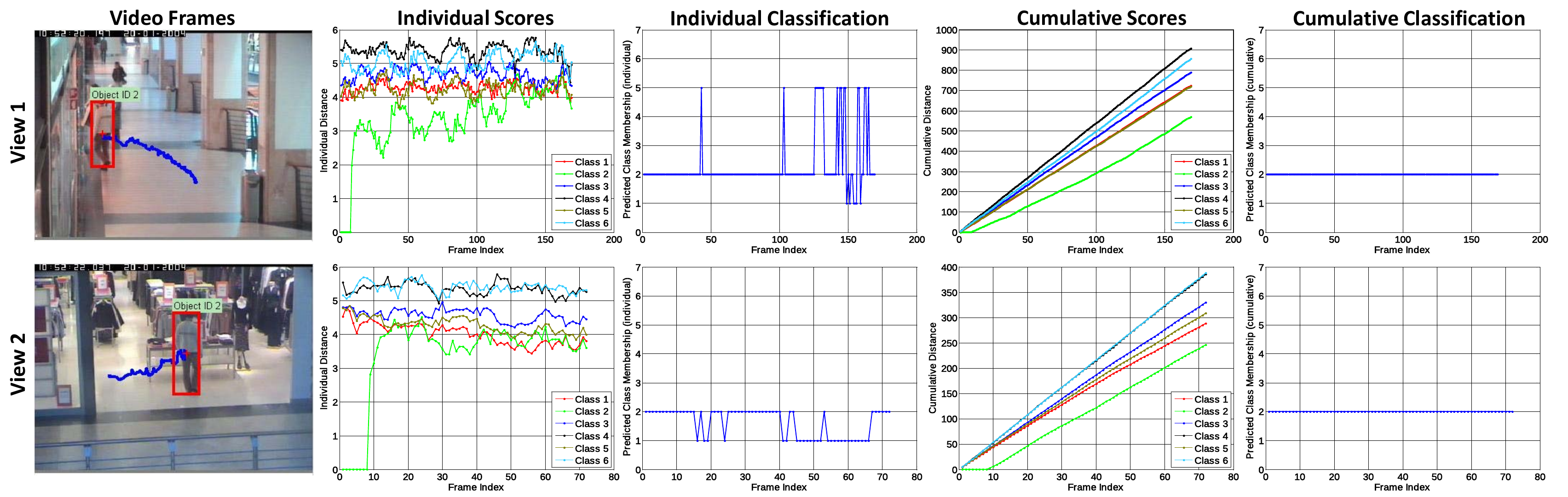}\\
(a) \\
\includegraphics[scale=0.5]{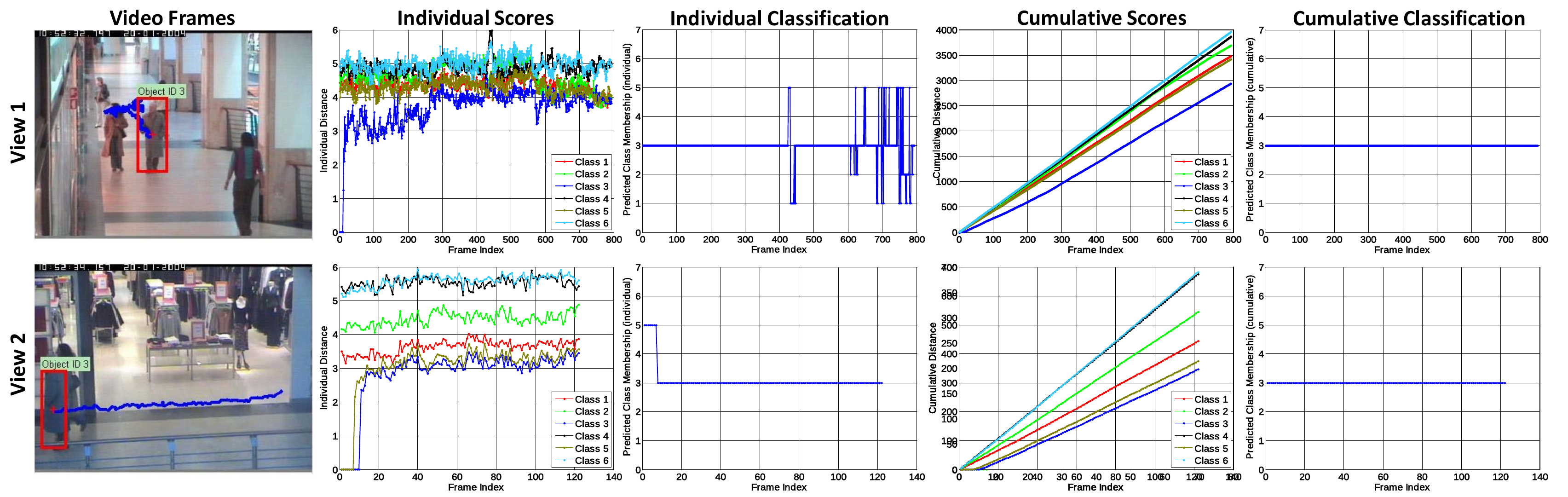}\\
(b)\\
\includegraphics[scale=0.5]{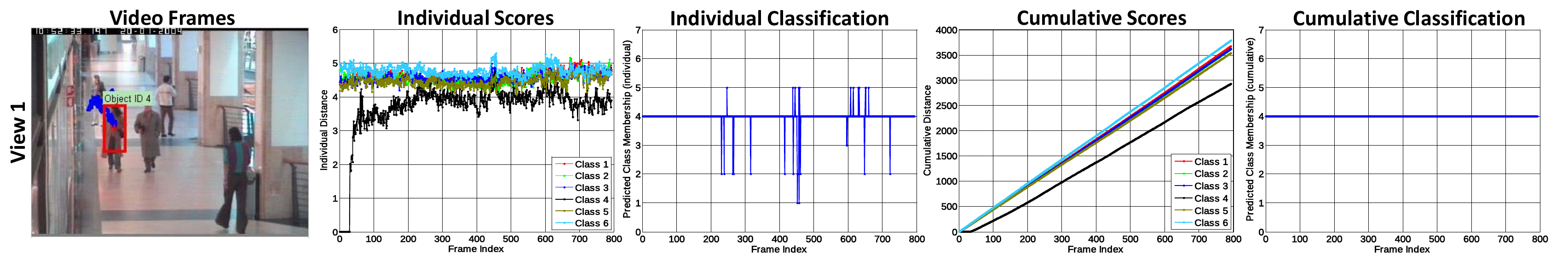}\\
(c)\\
\includegraphics[scale=0.5]{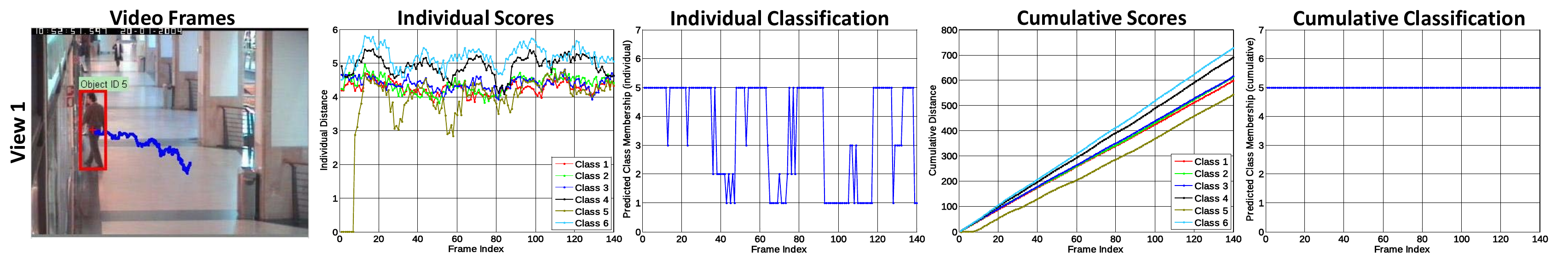}\\
(d)\\
\includegraphics[scale=0.5]{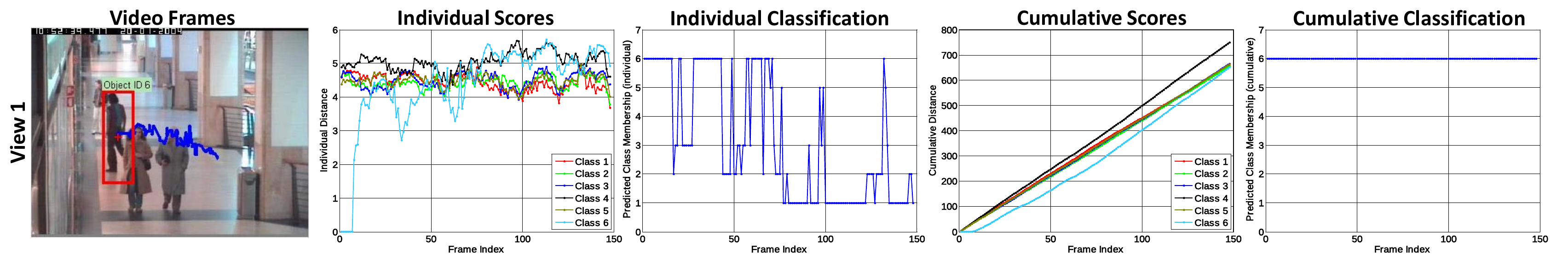}\\
(e)\\
\vspace{-0.5cm}
 \caption{Intuitive illustration of two-view pedestrian identification.
(a)-(e) show the quantitative frame-by-frame identification results
of different pedestrians from two viewpoints.
It is clear that our method is able to assign the pedestrians
to correct classes.
 \vspace{-0.6cm}}
 \label{fig:more_example_supp}
\end{figure}

\begin{figure}[t]
\vspace{-0.16cm}
\centering
\includegraphics[scale=0.3]{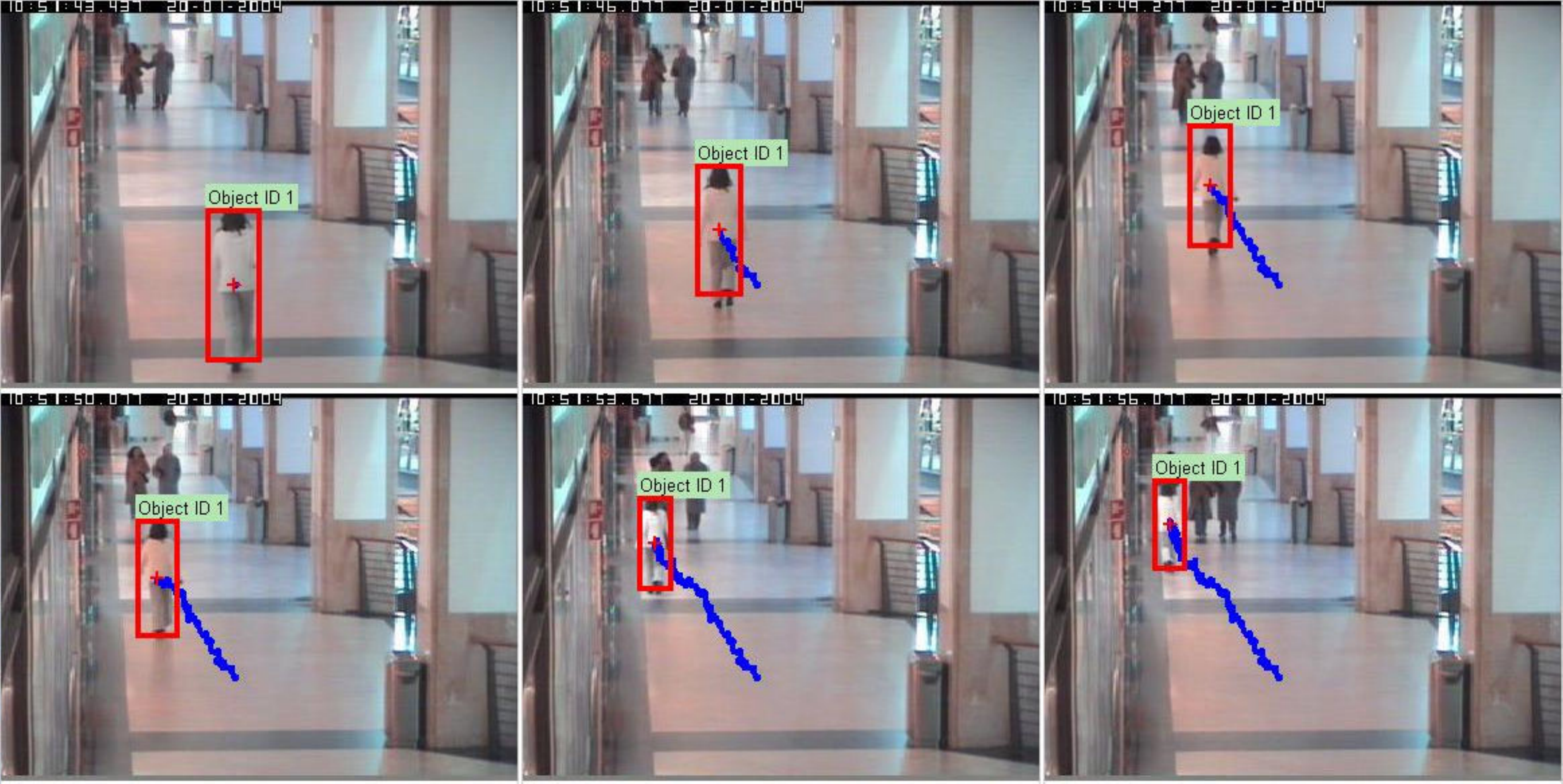}\\
\vspace{-0.6cm}
 \caption{Identification and tracking results for the first pedestrian on several representative frames (from View 1).
 \vspace{-0.6cm}}
 \label{fig:Id1_View1}
\end{figure}

\begin{figure}[t]
\vspace{-0.16cm}
\centering
\includegraphics[scale=0.3]{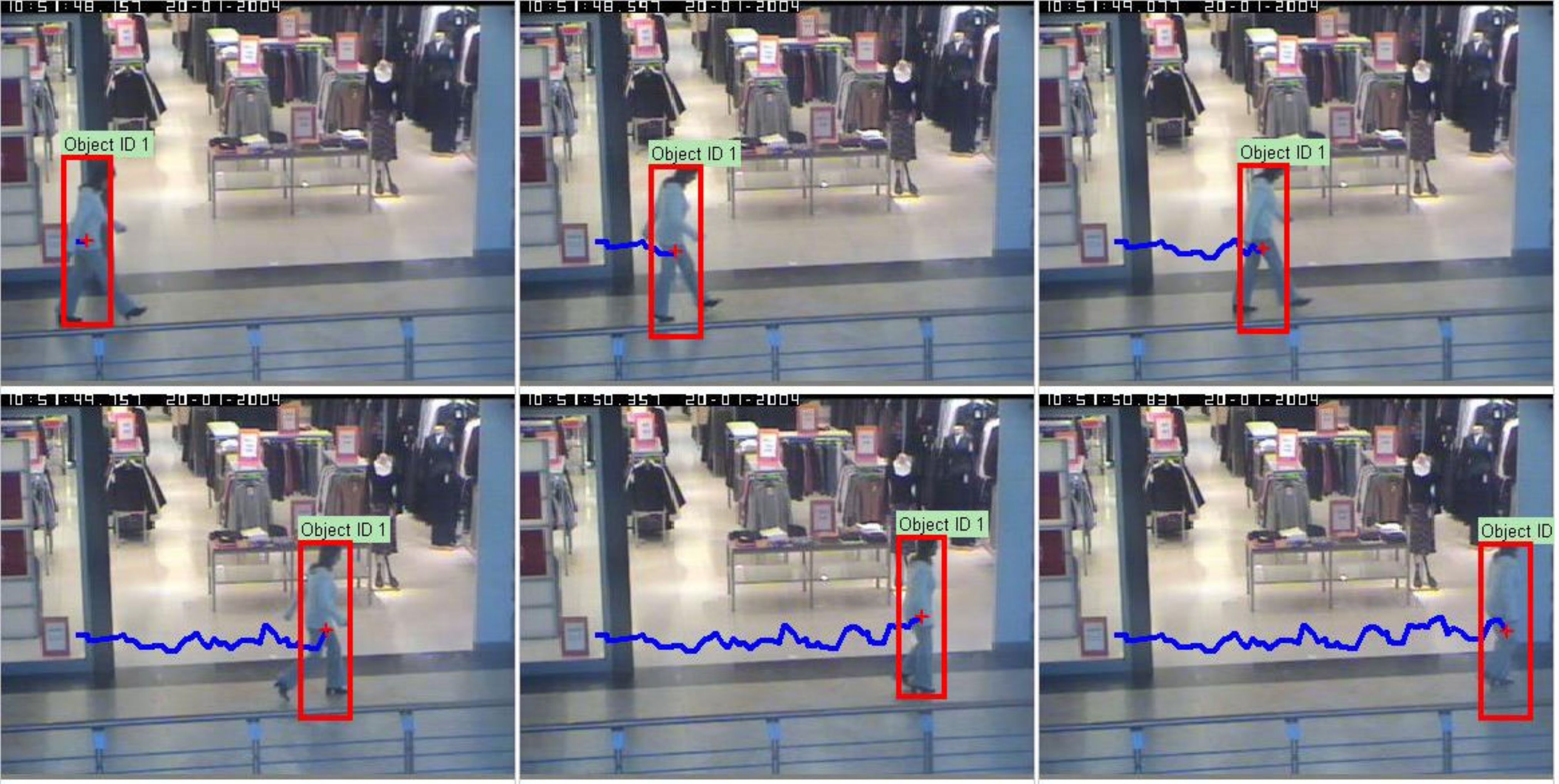}\\
\vspace{-0.6cm}
 \caption{Identification and tracking results for the first pedestrian on several representative frames (from View 2).
 \vspace{-0.6cm}}
 \label{fig:Id1_View2}
\end{figure}

\begin{figure}[t]
\vspace{-0.16cm}
\centering
\includegraphics[scale=0.3]{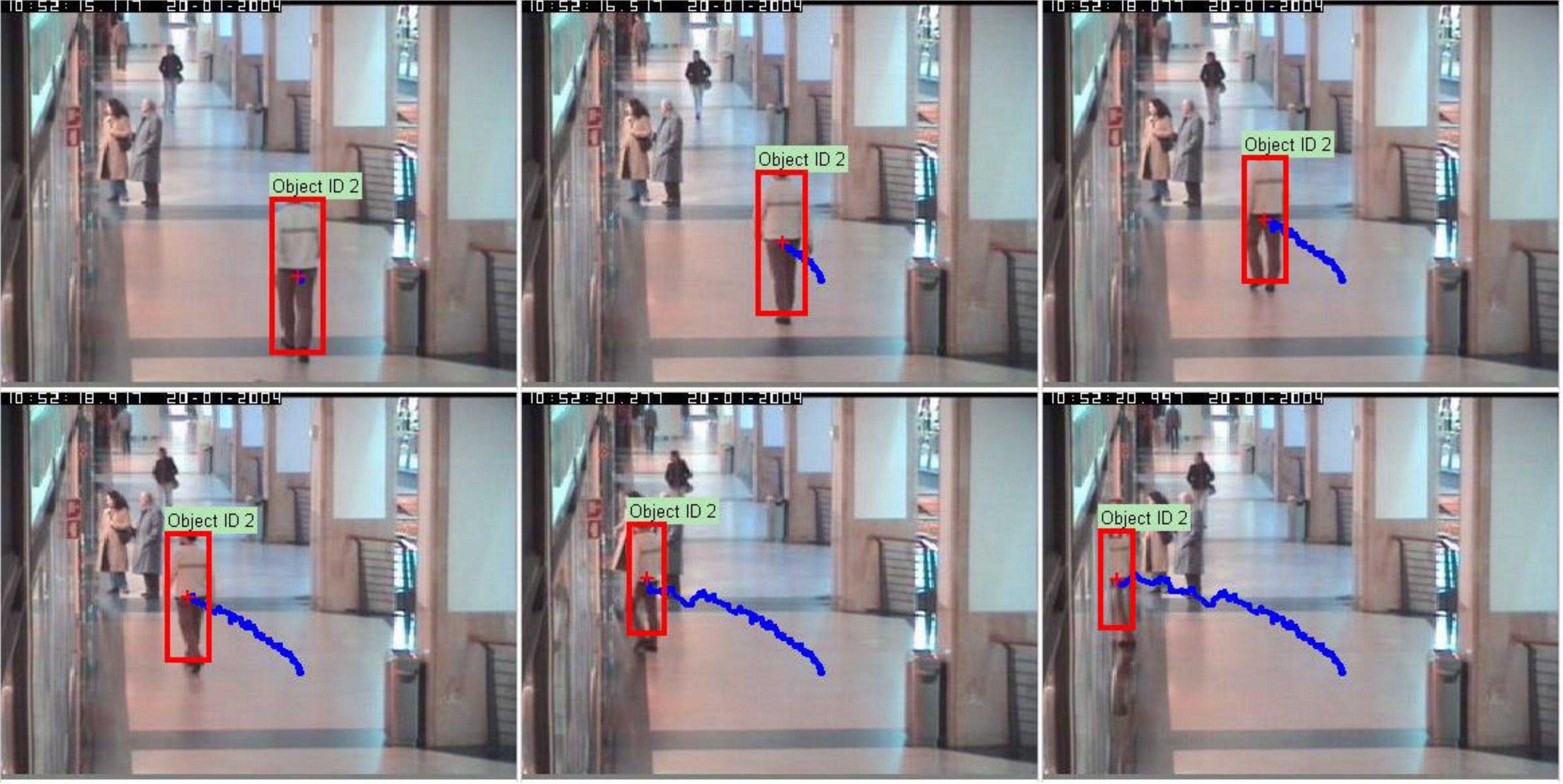}\\
\vspace{-0.6cm}
 \caption{Identification and tracking results for the second pedestrian on several representative frames (from View 1).
 \vspace{-0.6cm}}
 \label{fig:Id2_View1}
\end{figure}

\begin{figure}[t]
\vspace{-0.16cm}
\centering
\includegraphics[scale=0.3]{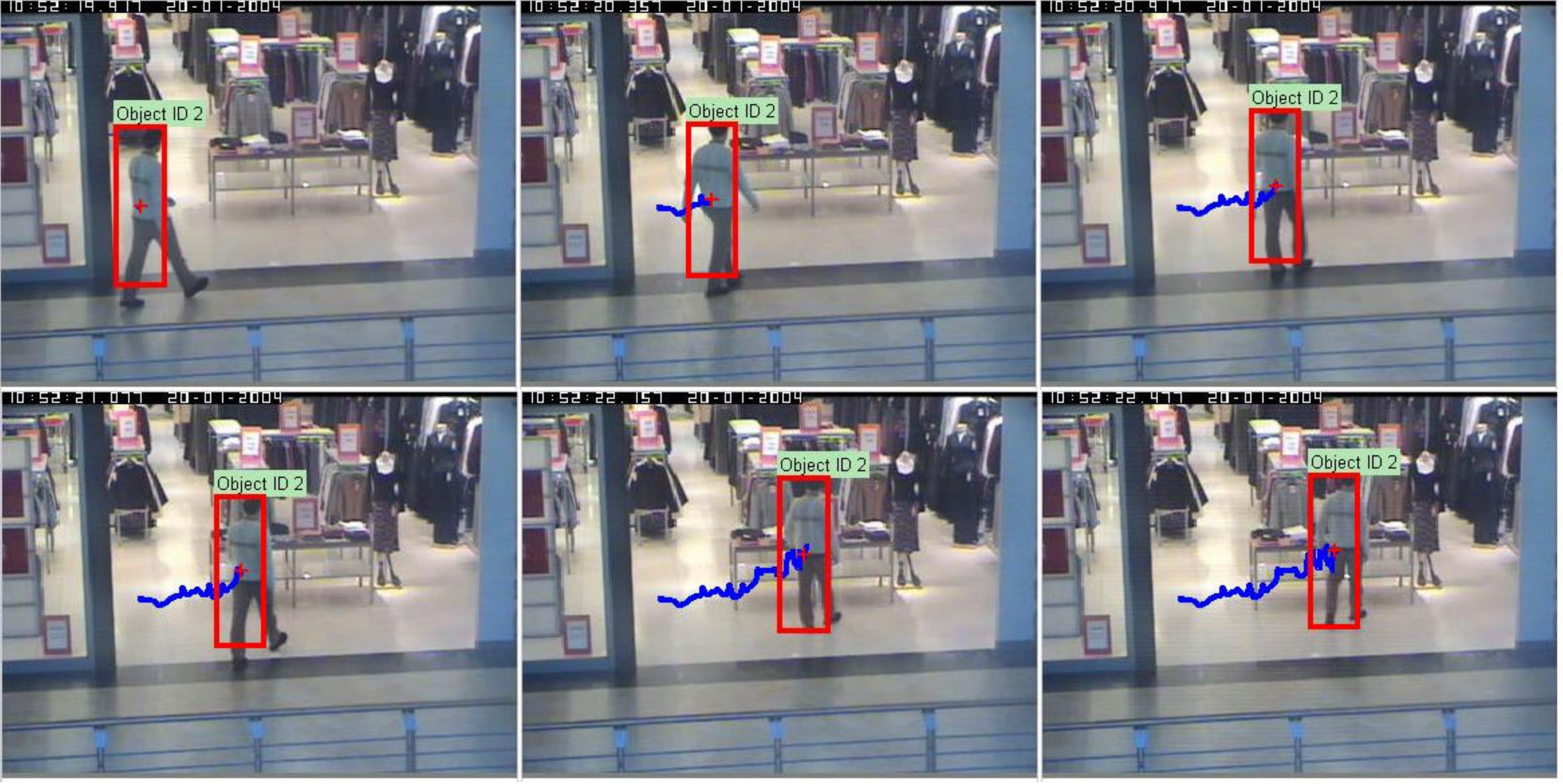}\\
\vspace{-0.6cm}
 \caption{Identification and tracking results for the second pedestrian on several representative frames (from View 2).
 \vspace{-0.6cm}}
 \label{fig:Id2_View2}
\end{figure}

\begin{figure}[t]
\vspace{-0.16cm}
\centering
\includegraphics[scale=0.3]{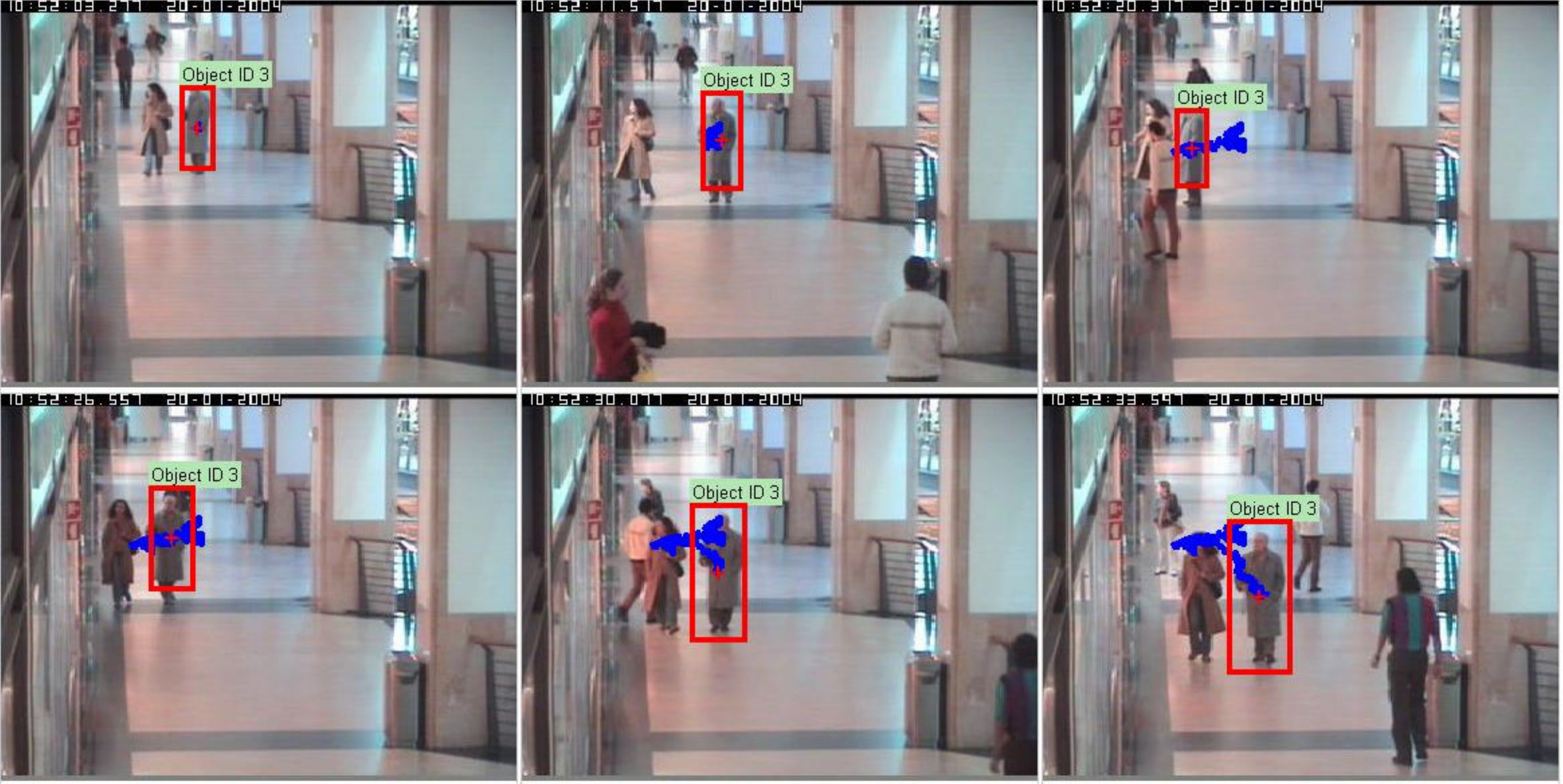}\\
\vspace{-0.6cm}
 \caption{Identification and tracking results for the third pedestrian on several representative frames (from View 1).
 \vspace{-0.6cm}}
 \label{fig:Id3_View1}
\end{figure}

\begin{figure}[t]
\vspace{-0.16cm}
\centering
\includegraphics[scale=0.3]{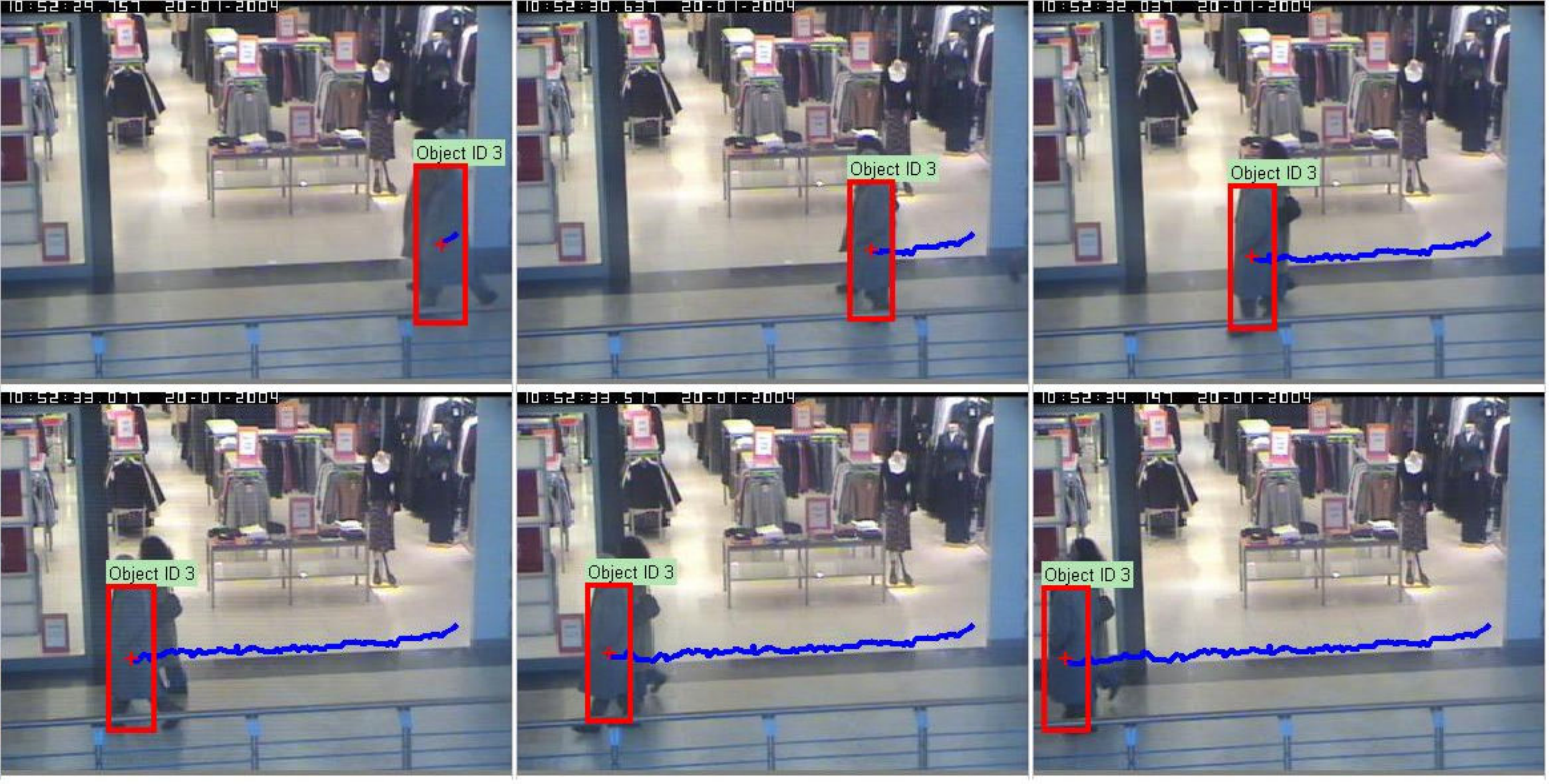}\\
\vspace{-0.6cm}
 \caption{Identification and tracking results for the third pedestrian on several representative frames (from View 2).
 \vspace{-0.6cm}}
 \label{fig:Id3_View2}
\end{figure}

\begin{figure}[t]
\vspace{-0.16cm}
\centering
\includegraphics[scale=0.3]{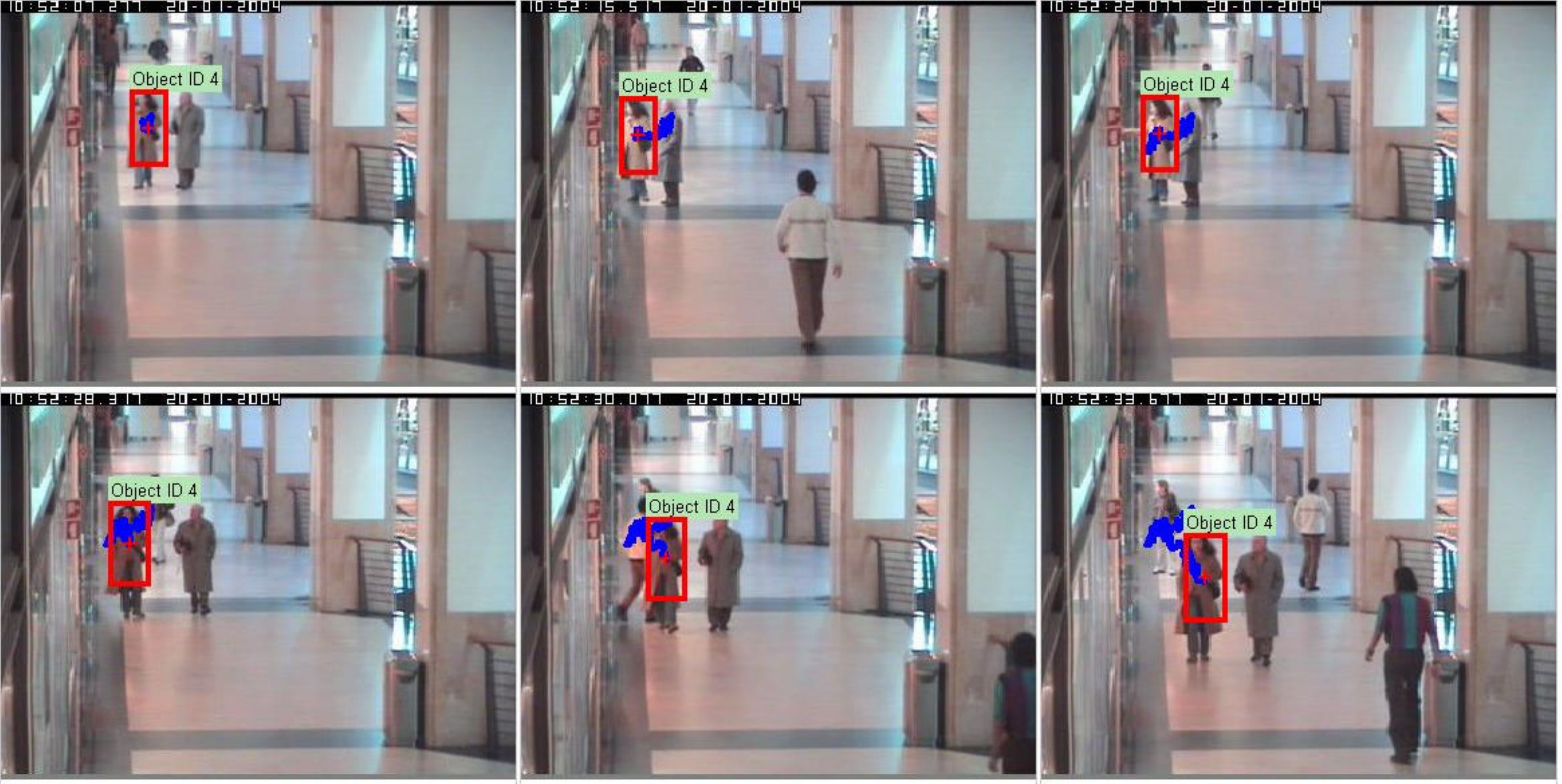}\\
\vspace{-0.6cm}
 \caption{Identification and tracking results for the fourth pedestrian on several representative frames (from View 1).
 \vspace{-0.6cm}}
 \label{fig:Id4_View1}
\end{figure}

\begin{figure}[t]
\vspace{-0.16cm}
\centering
\includegraphics[scale=0.3]{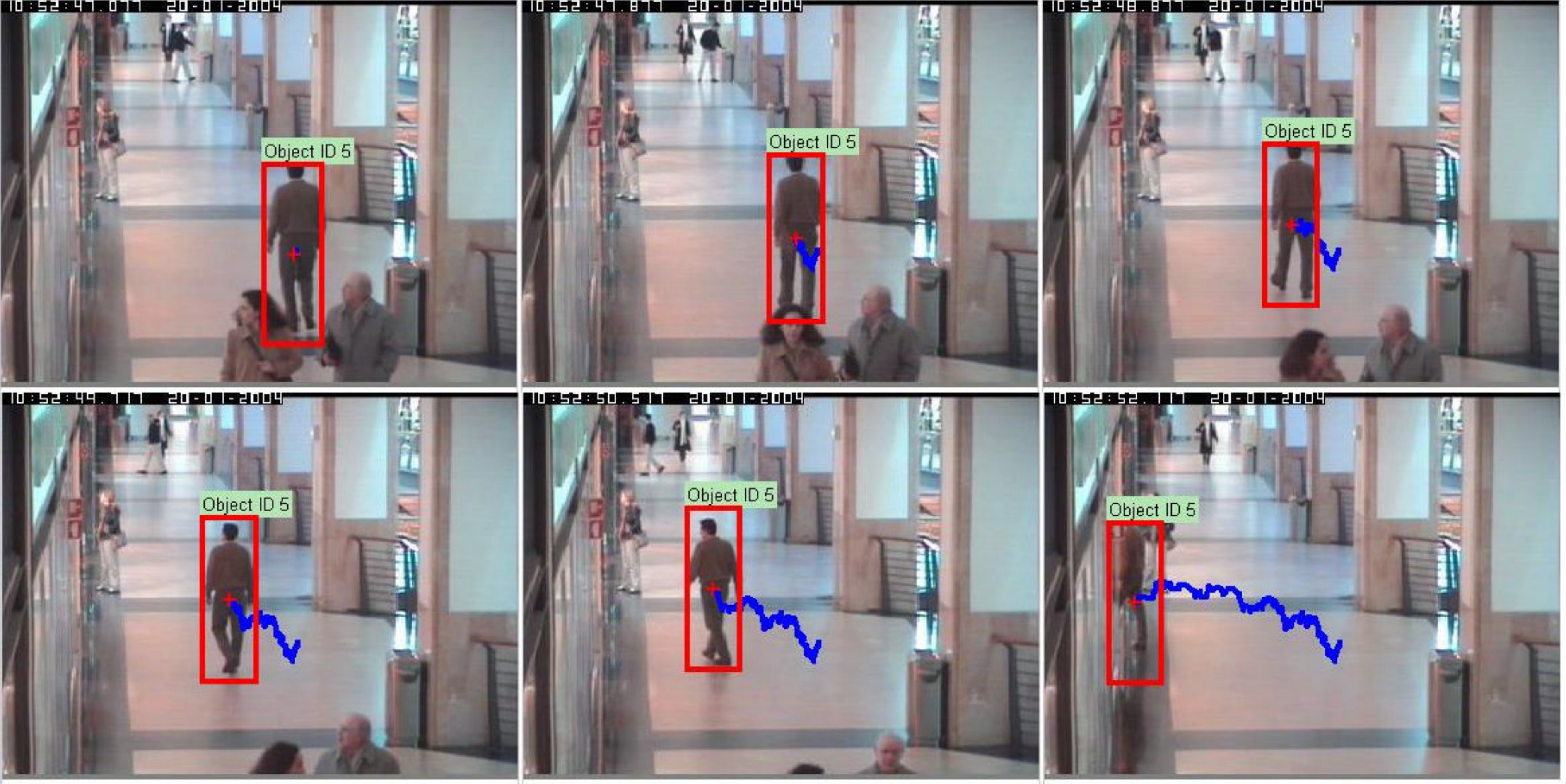}\\
\vspace{-0.6cm}
 \caption{Identification and tracking results for the fifth pedestrian on several representative frames (from View 1).
 \vspace{-0.6cm}}
 \label{fig:Id5_View1}
\end{figure}

\begin{figure}[t]
\vspace{-0.16cm}
\centering
\includegraphics[scale=0.3]{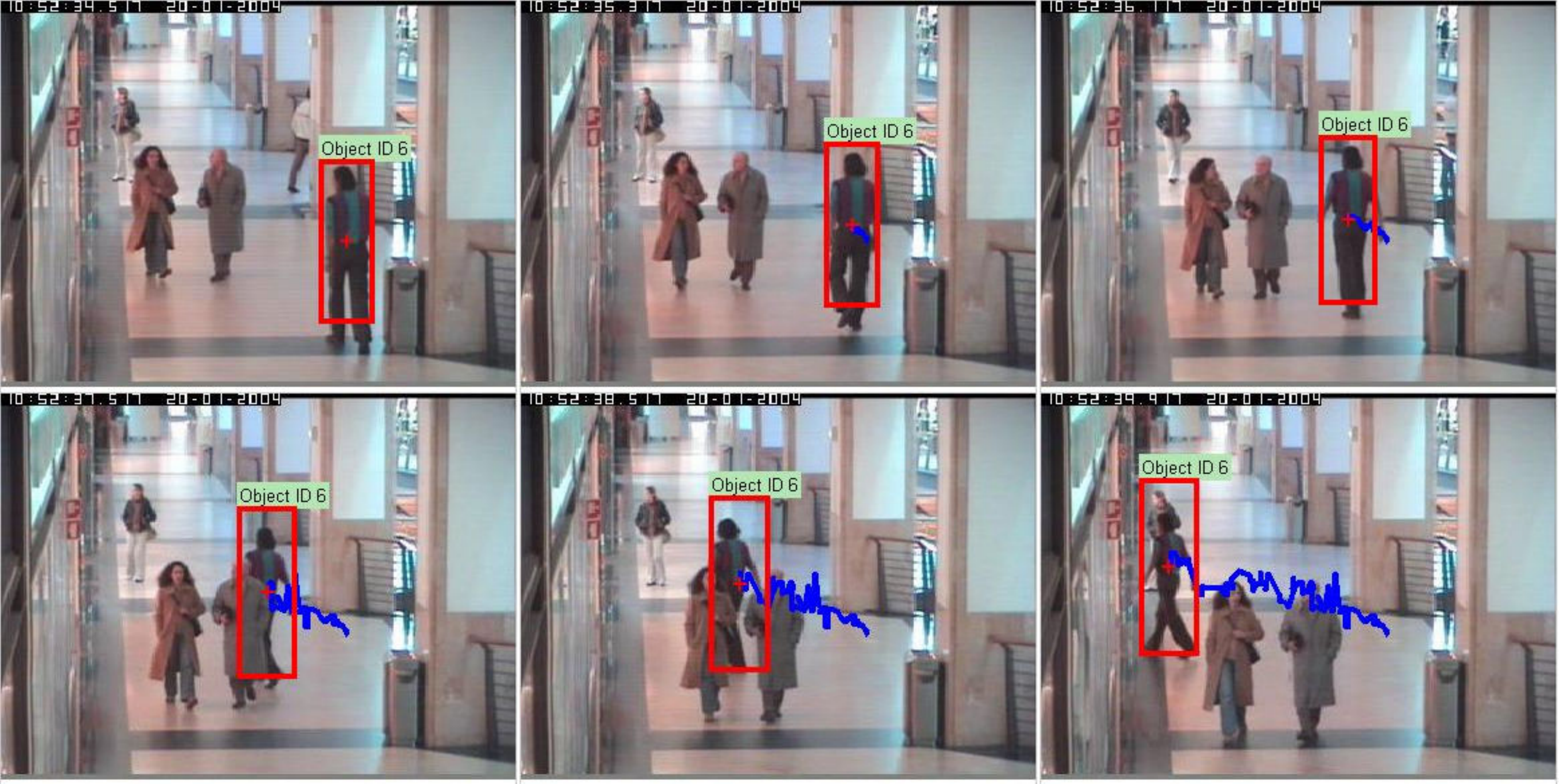}\\
\vspace{-0.6cm}
 \caption{Identification and tracking results for the sixth pedestrian on several representative frames (from View 1).
 \vspace{-0.6cm}}
 \label{fig:Id6_View1}
\end{figure}

\section{Comparison with the state-of-the-art trackers}
\label{sec:comparison_with_competing}

We report the quantitative tracking results of the eleven
trackers in CLE and VOR over the eighteen video sequences.
Figs.~\ref{fig:exp_error_curve} and \ref{fig:exp_voc_curve} plot the frame-by-frame
CLEs and VORs (marked with the curves in different colors) obtained by
the eleven trackers.
From Figs.~\ref{fig:exp_error_curve} and \ref{fig:exp_voc_curve},
we observe that the proposed tracking algorithm
achieves the best tracking performance on most video sequences.

Moreover, we show the corresponding
qualitative tracking results of the eleven trackers (highlighted by the
bounding boxes in different colors)  over the
representative frames of the eighteen video sequences in Figs.~\ref{fig:tracking_balancebeam}--\ref{fig:tracking_car11}.
Clearly, it is seen from Figs.~\ref{fig:tracking_balancebeam}--\ref{fig:tracking_car11}
that the proposed tracking algorithm obtains the most
accurate tracking results in most cases.

\begin{figure*}[t]
 \vspace{-0.8cm}
\centering
\includegraphics[scale=0.3]{./Suppelmentary_File/NewFigs/NewLegendBar.PNG}\\
\includegraphics[scale=0.41]{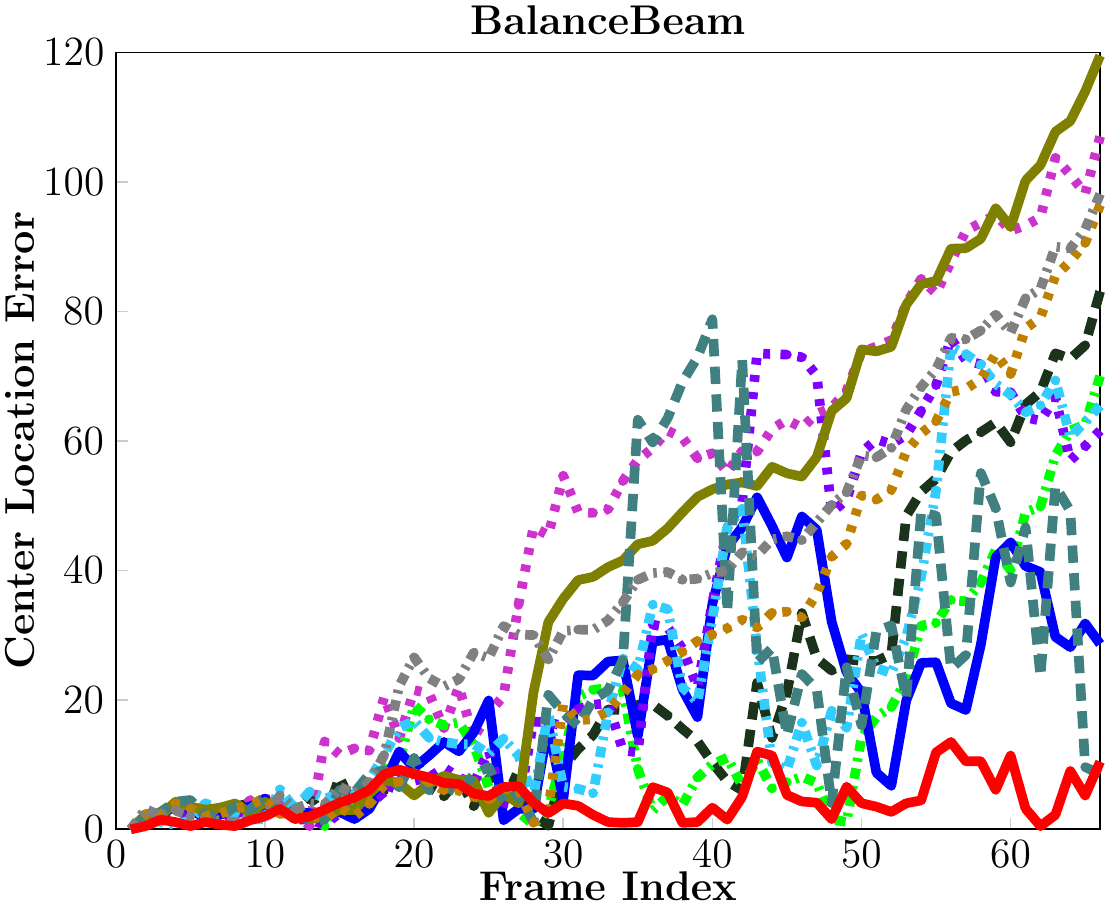}
\includegraphics[scale=0.41]{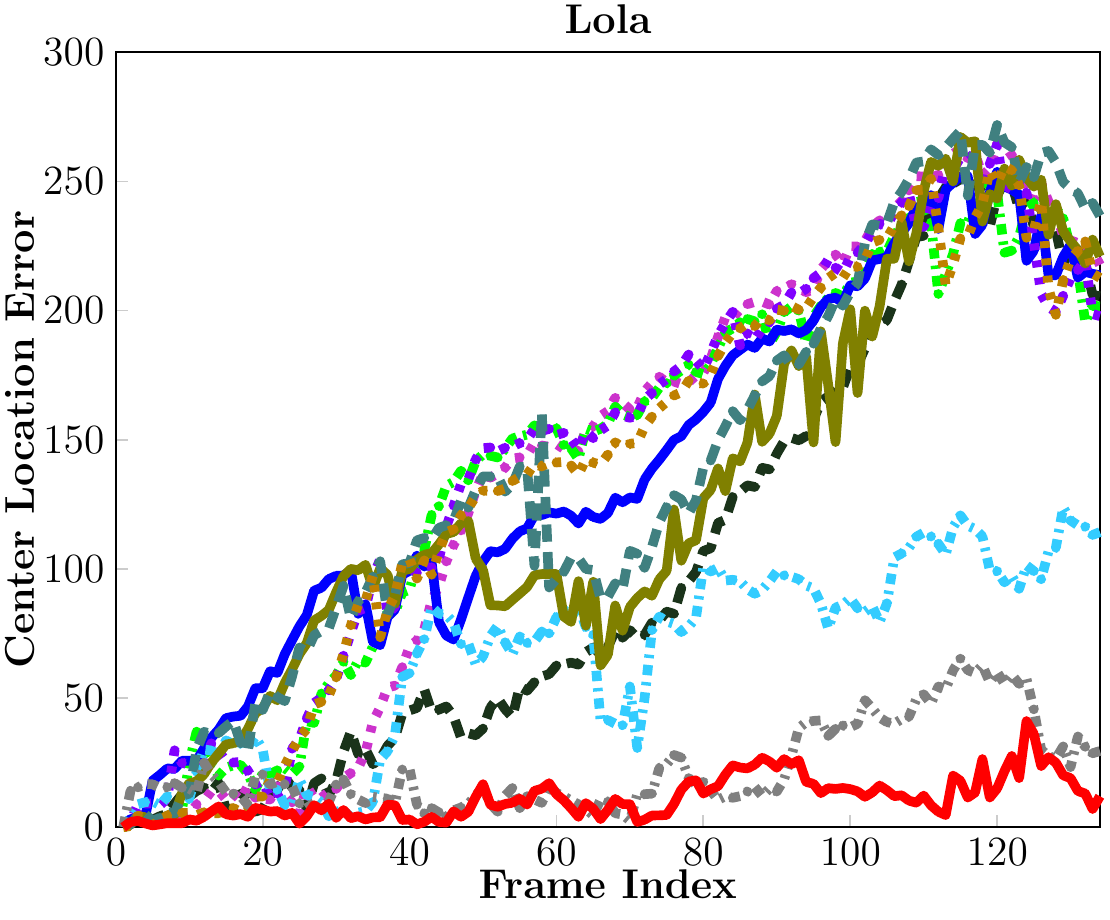}
\includegraphics[scale=0.41]{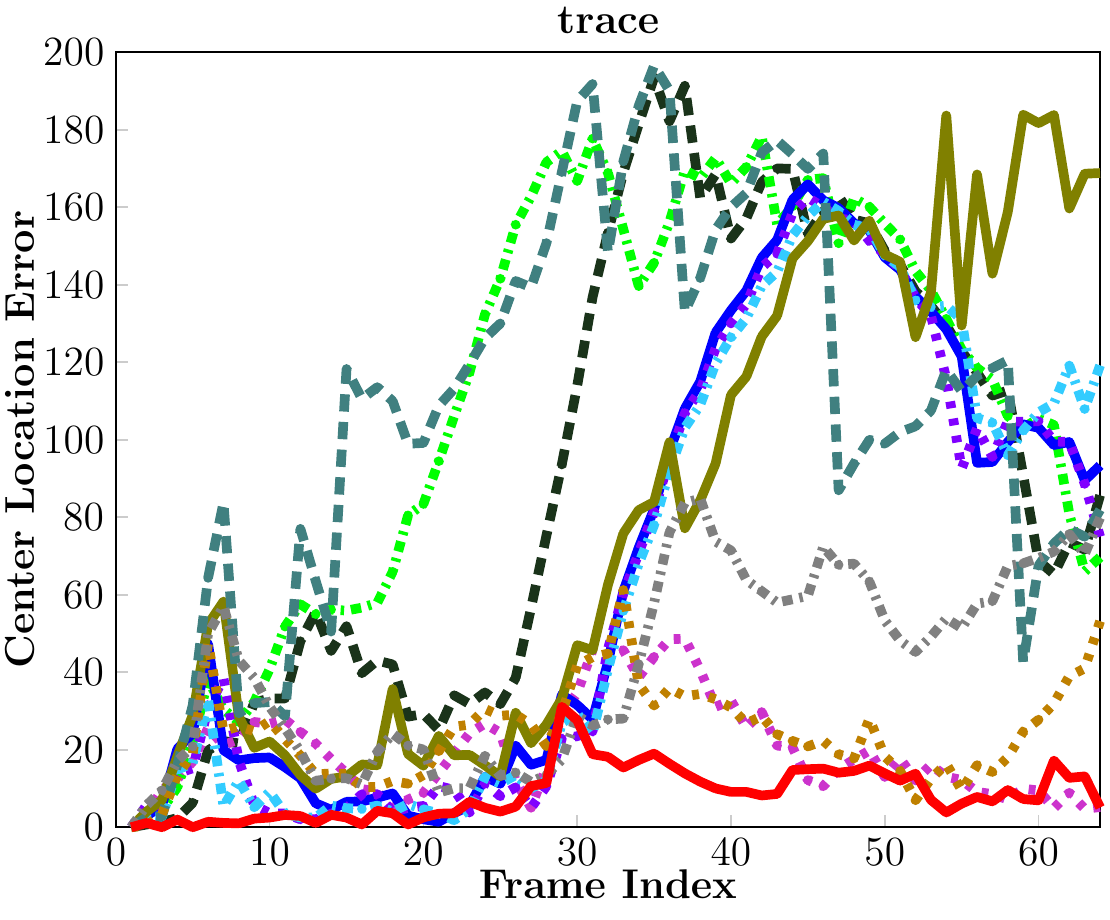}\\
\includegraphics[scale=0.41]{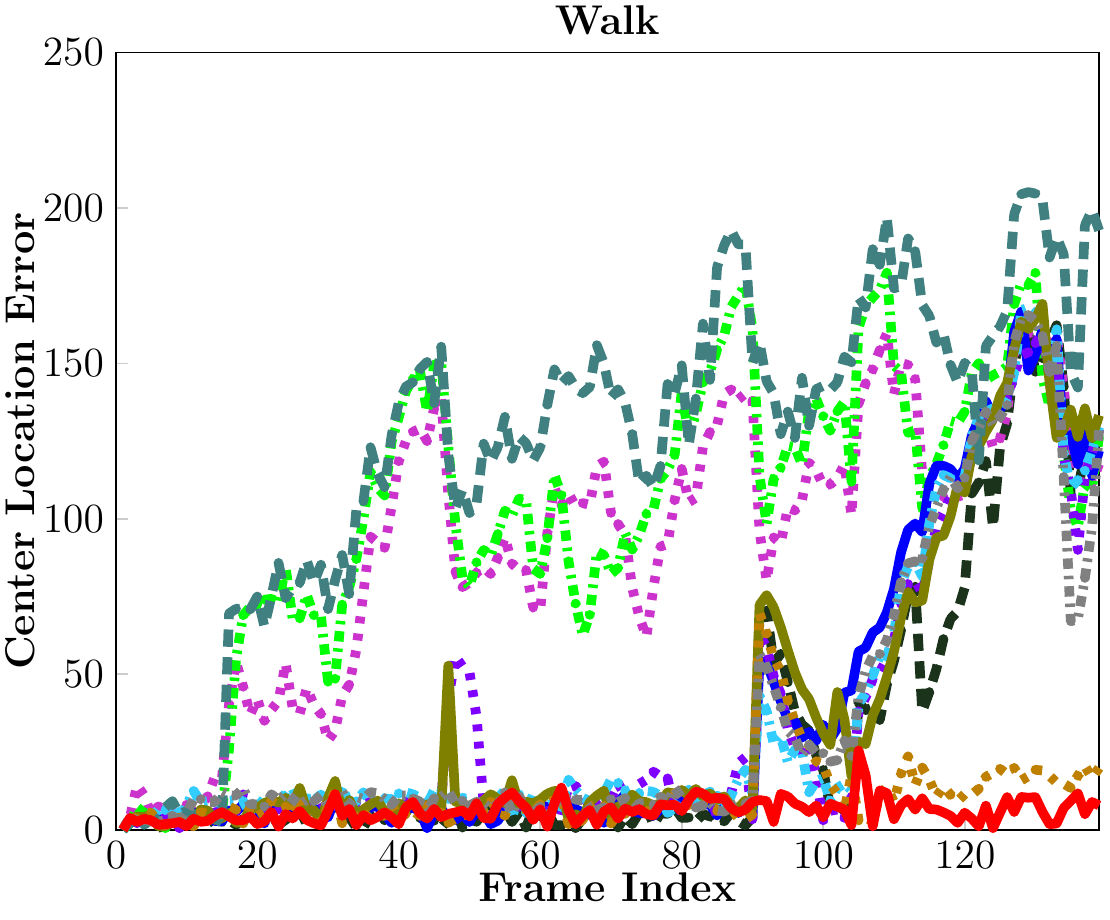}
\includegraphics[scale=0.41]{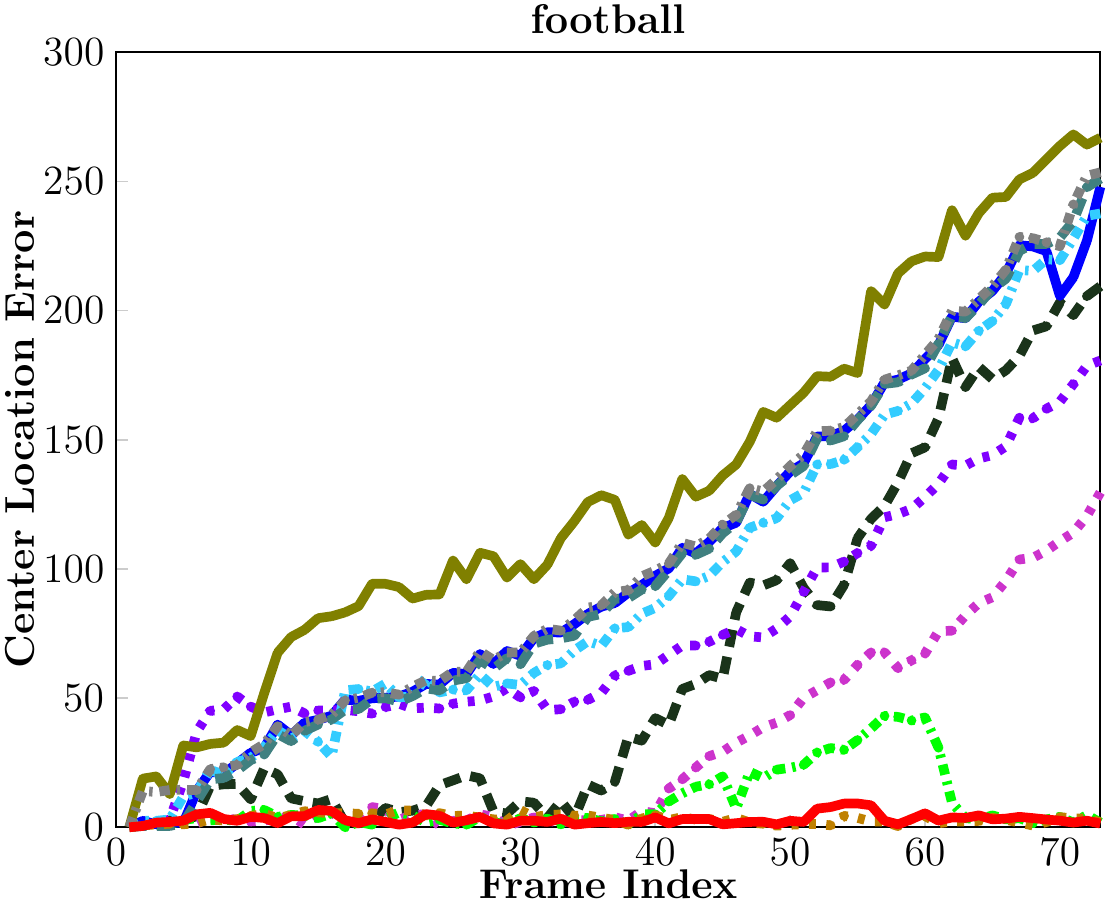}
\includegraphics[scale=0.41]{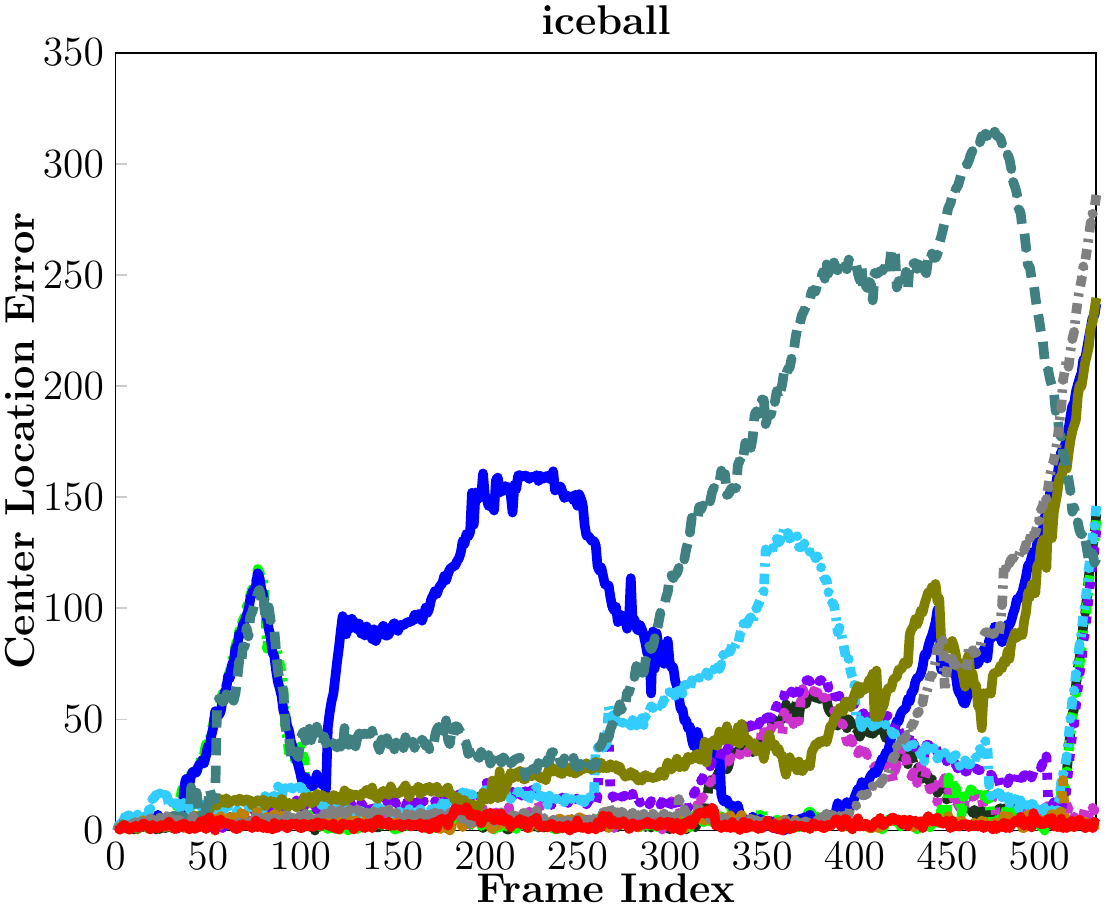}\\
\includegraphics[scale=0.41]{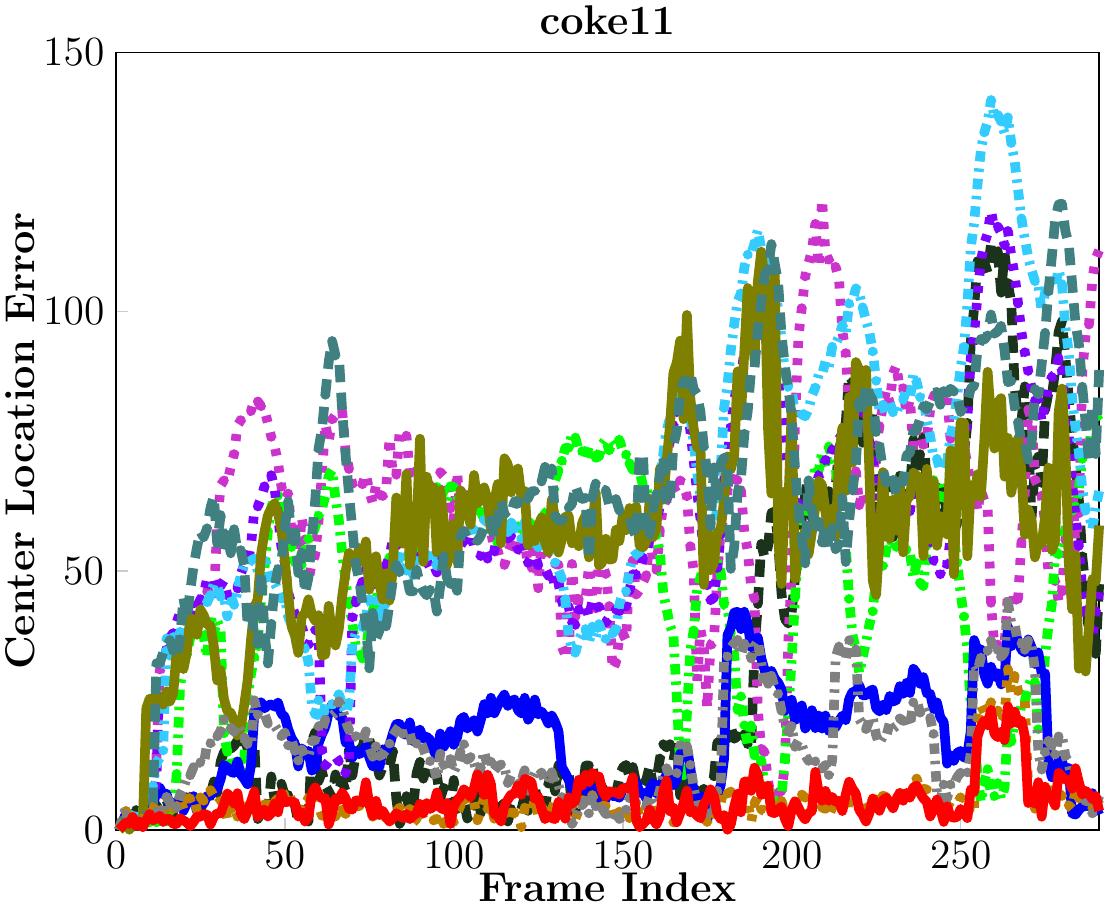}
\includegraphics[scale=0.41]{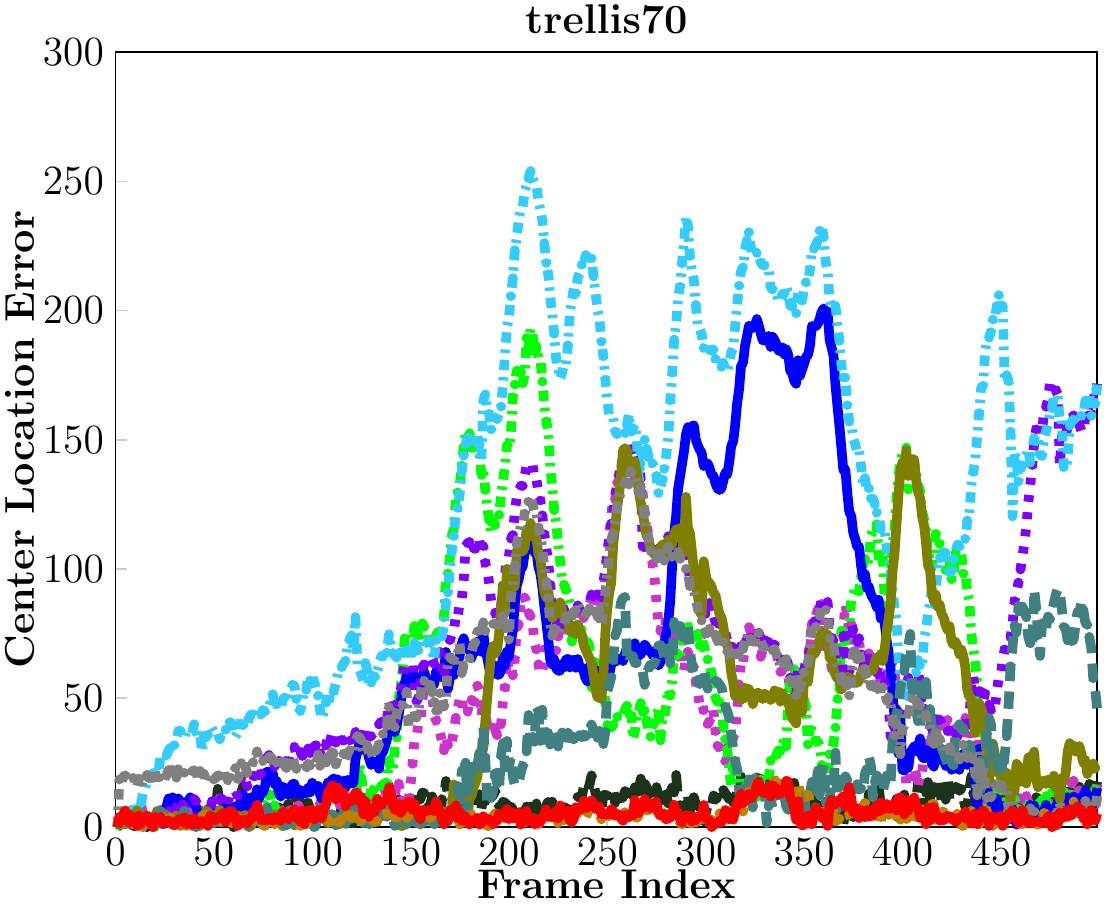}
\includegraphics[scale=0.41]{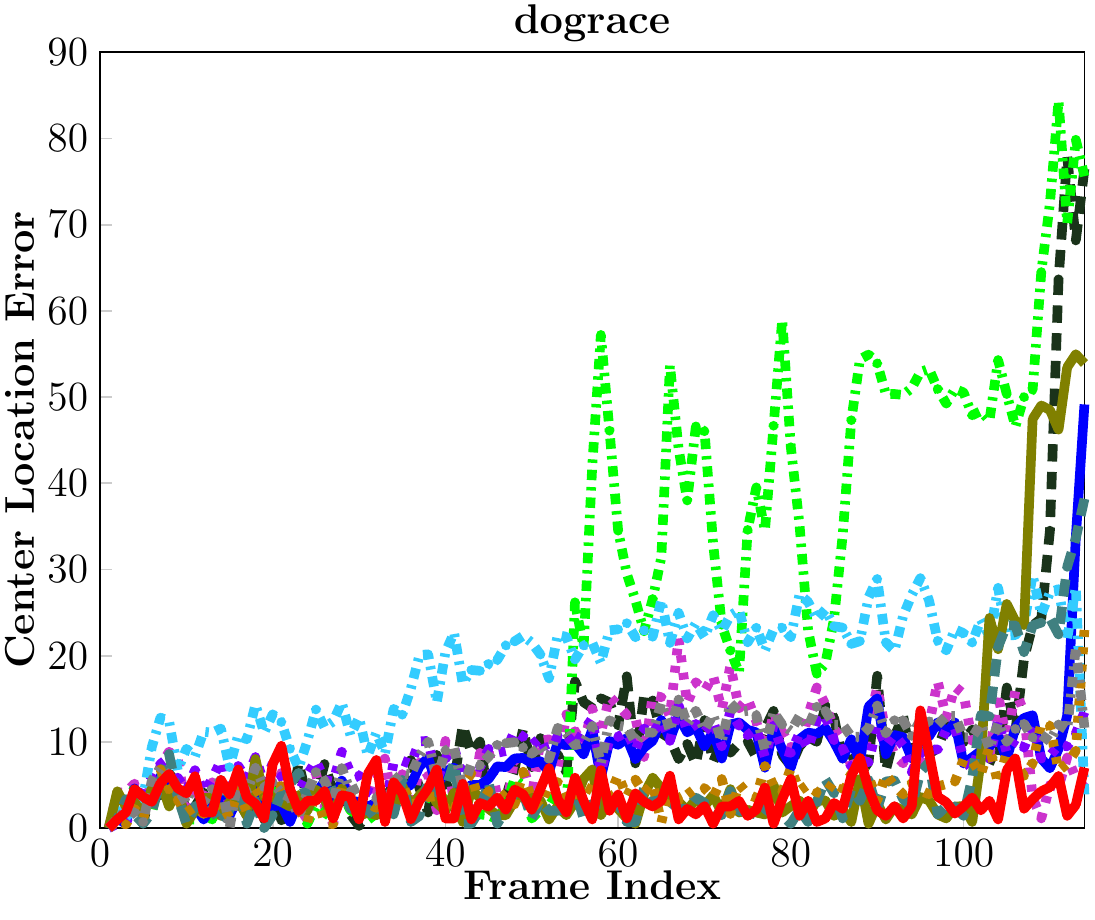}\\
\includegraphics[scale=0.41]{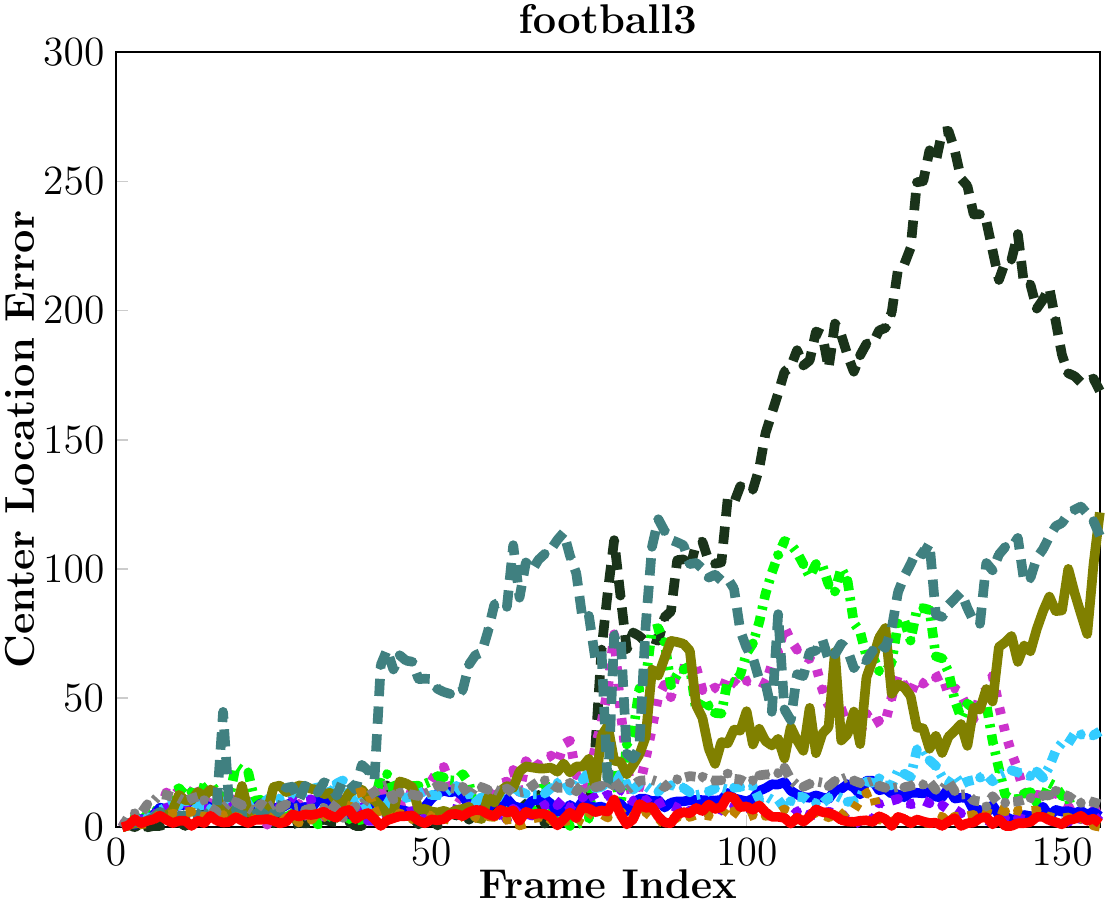}
\includegraphics[scale=0.41]{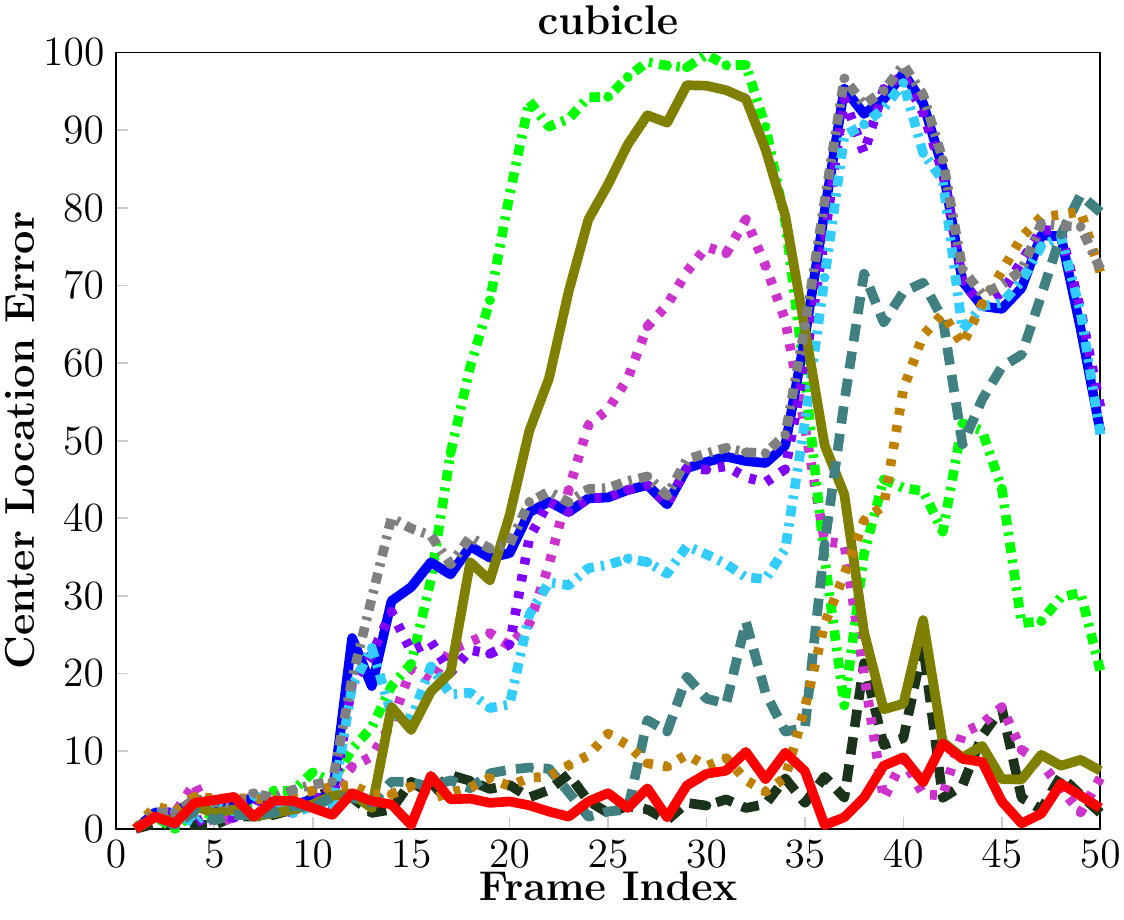}
\includegraphics[scale=0.41]{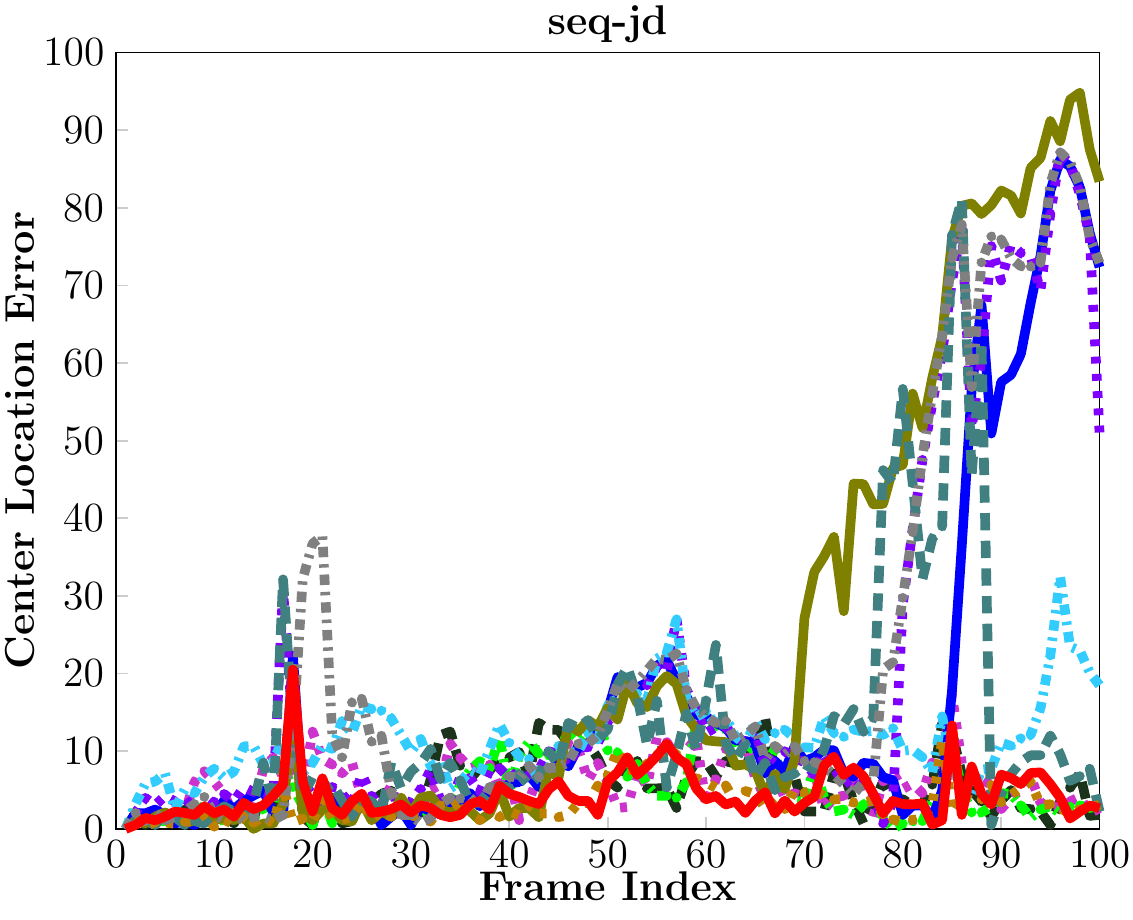}\\
\includegraphics[scale=0.41]{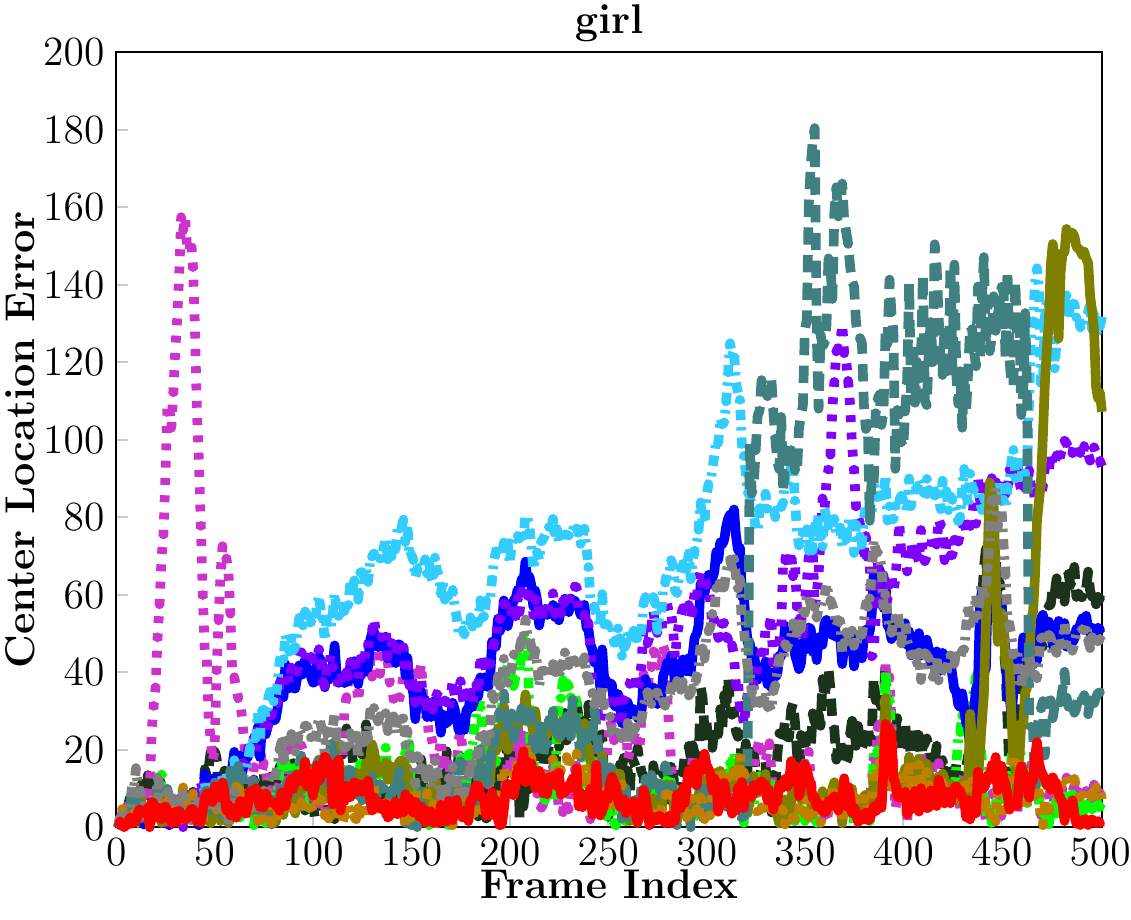}
\includegraphics[scale=0.41]{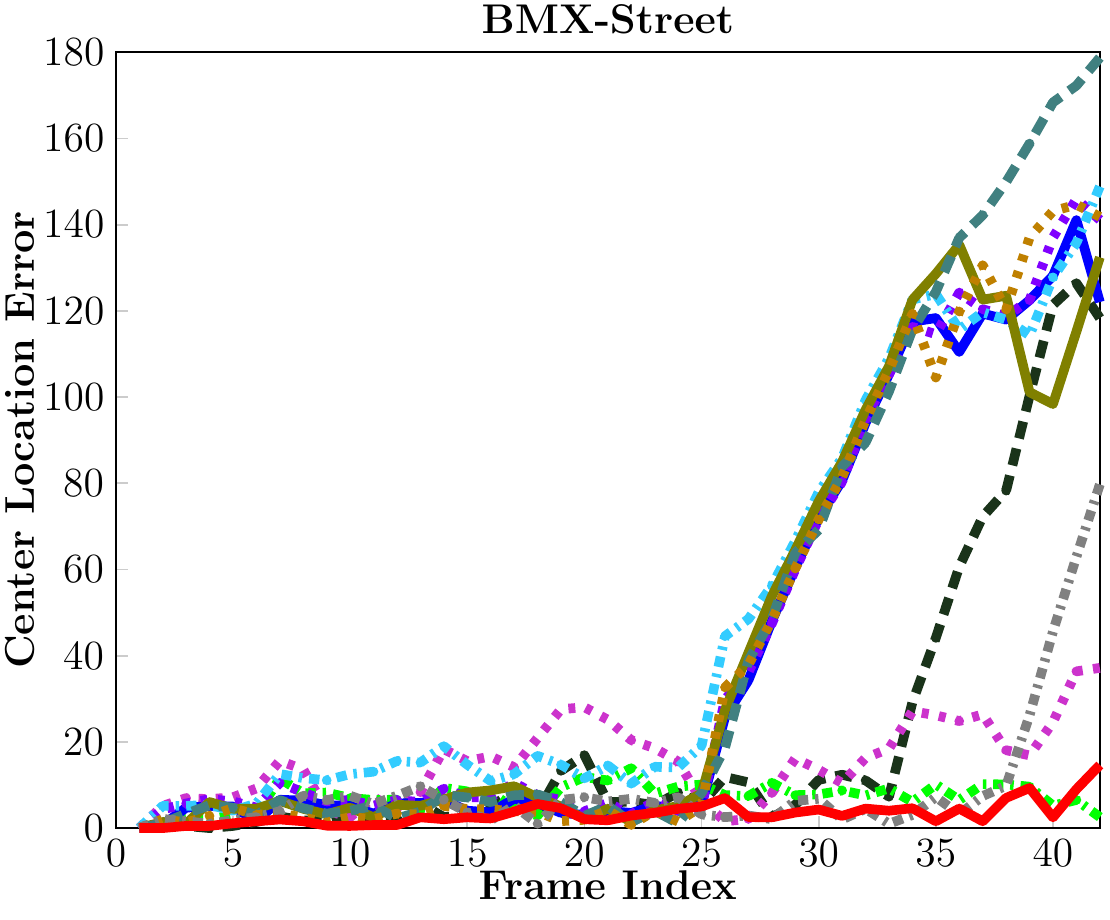}
\includegraphics[scale=0.41]{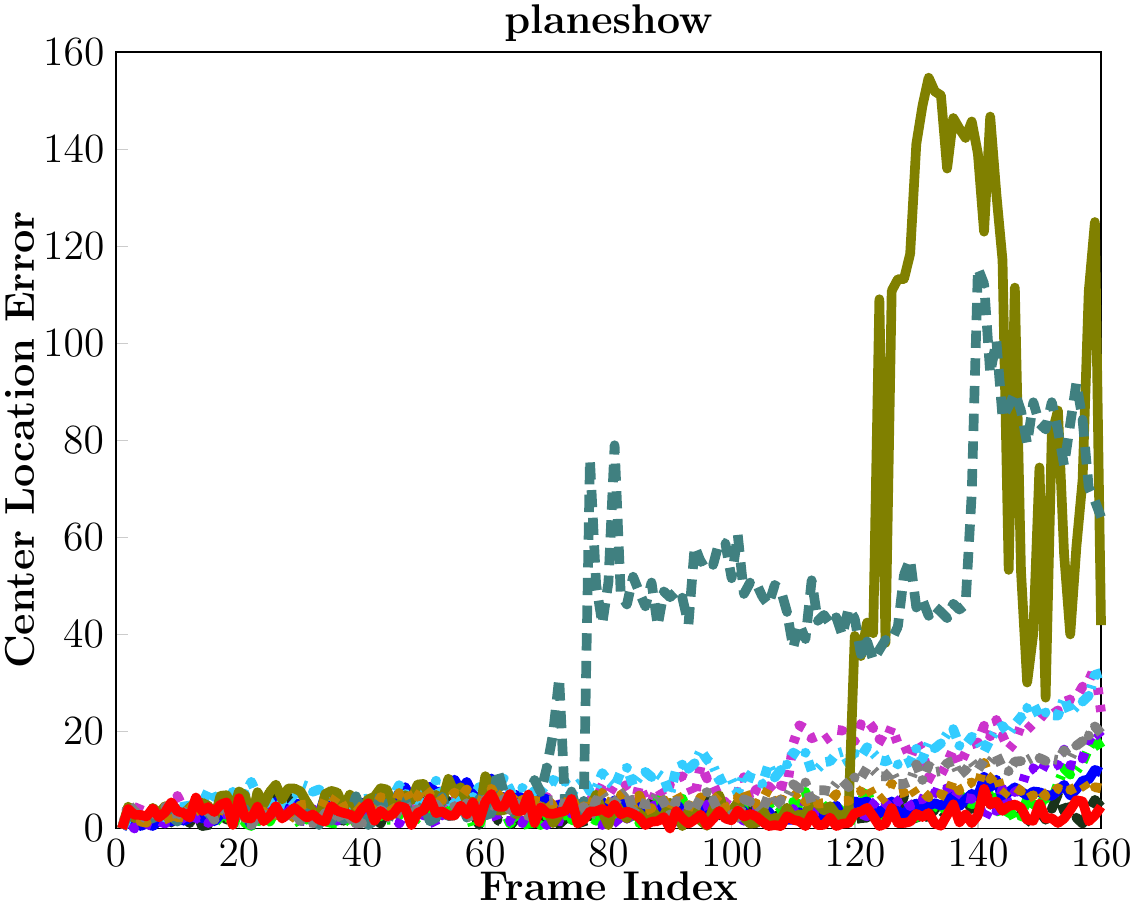}\\
\includegraphics[scale=0.41]{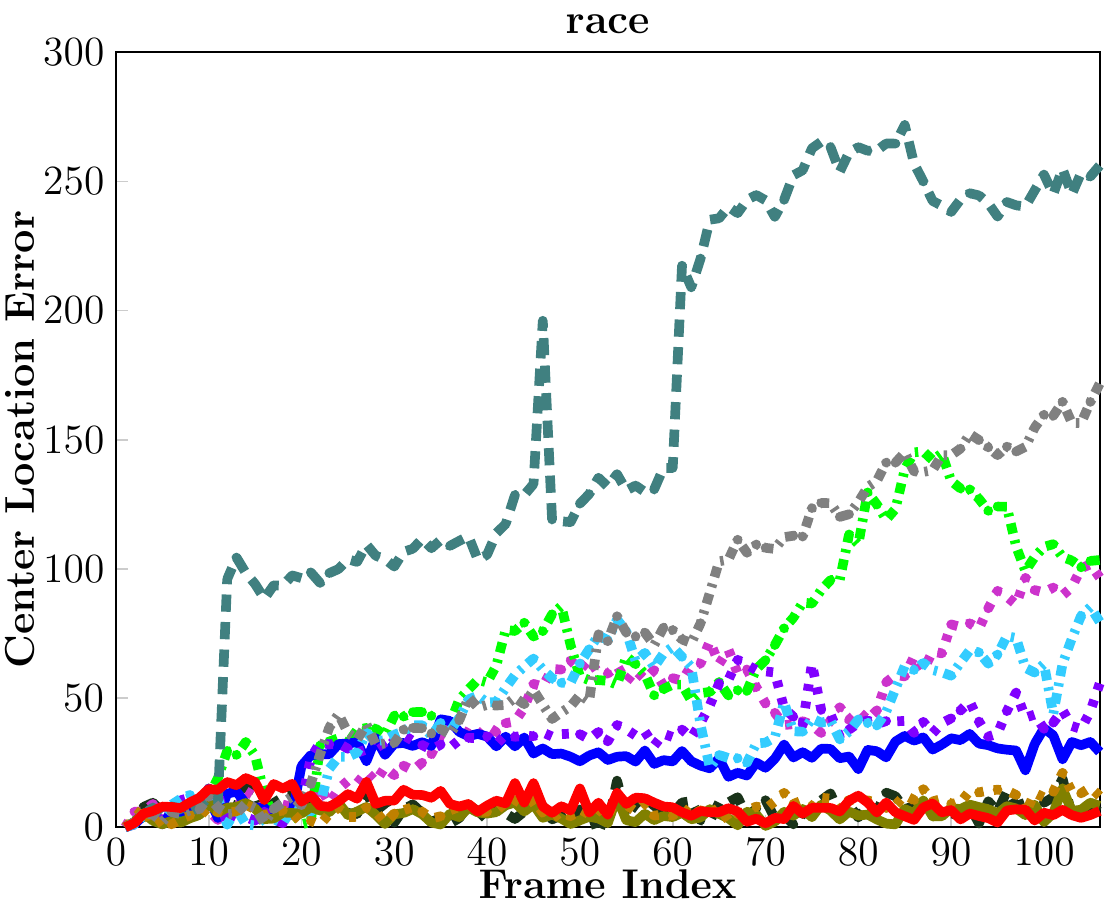}
\includegraphics[scale=0.41]{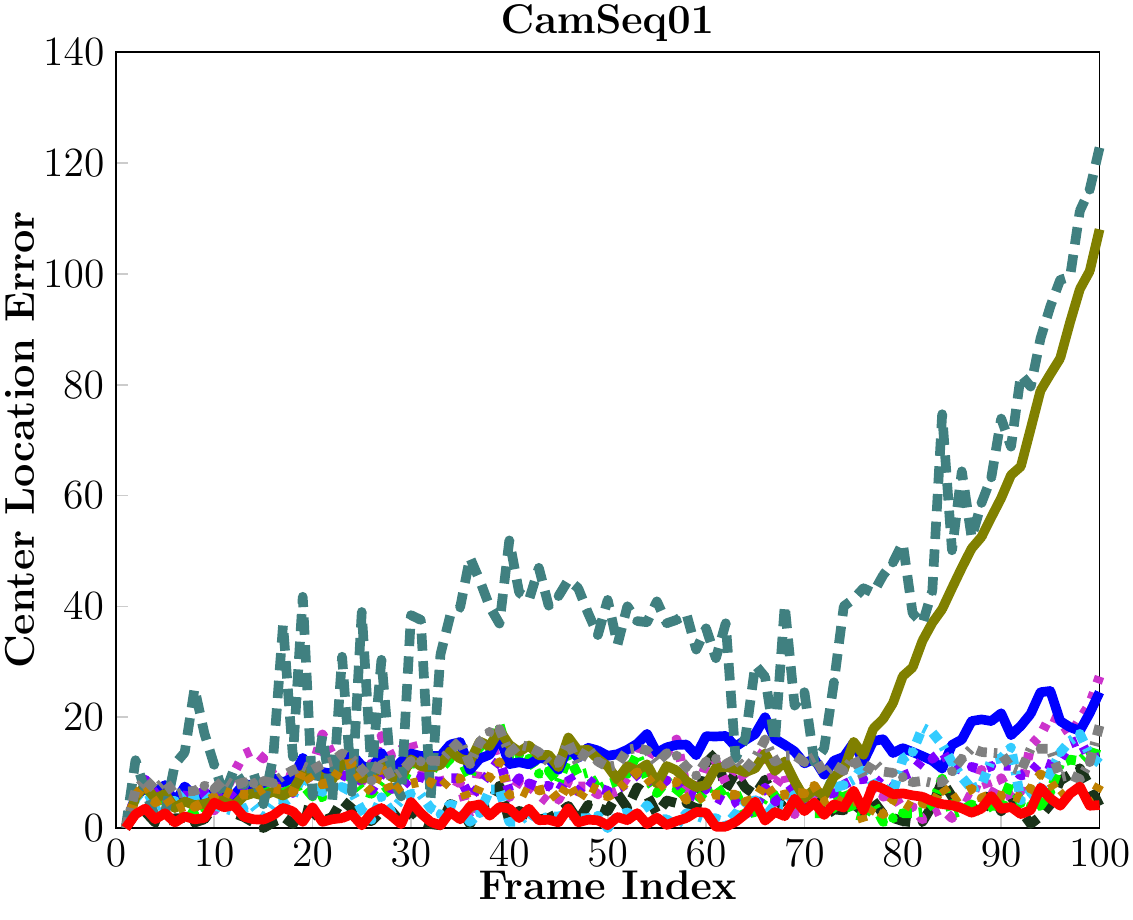}
\includegraphics[scale=0.41]{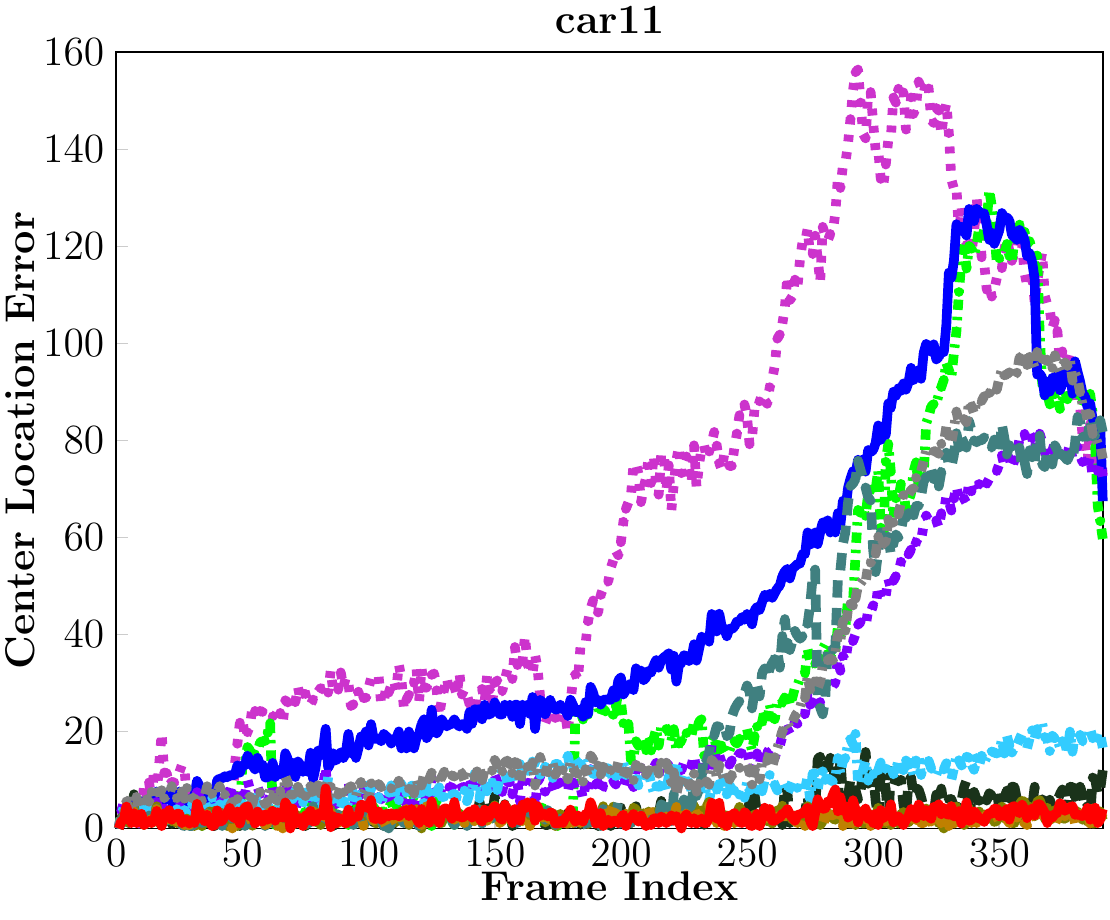}\\
\vspace{-0.13cm}
\caption{Quantitative comparison of different trackers in CLE on the eighteen
video sequences.}
 \vspace{-0.46cm}
  \label{fig:exp_error_curve}
\end{figure*}

\begin{figure*}[t]
\centering
 \vspace{-0.8cm}
\includegraphics[scale=0.3]{./Suppelmentary_File/NewFigs/NewLegendBar.PNG}\\
\includegraphics[scale=0.41]{./Suppelmentary_File/NewFigs/VOC_BalanceBeam_new.pdf}
\includegraphics[scale=0.41]{./Suppelmentary_File/NewFigs/VOC_Lola_new.pdf}
\includegraphics[scale=0.41]{./Suppelmentary_File/NewFigs/VOC_trace_new.pdf}\\
\includegraphics[scale=0.41]{./Suppelmentary_File/NewFigs/VOC_Walk_new.pdf}
\includegraphics[scale=0.41]{./Suppelmentary_File/NewFigs/VOC_football_new.pdf}
\includegraphics[scale=0.41]{./Suppelmentary_File/NewFigs/VOC_iceball_new.pdf}\\
\includegraphics[scale=0.41]{./Suppelmentary_File/NewFigs/VOC_coke11_new.pdf}
\includegraphics[scale=0.41]{./Suppelmentary_File/NewFigs/VOC_trellis70_new.pdf}
\includegraphics[scale=0.41]{./Suppelmentary_File/NewFigs/VOC_dograce_new.pdf}\\
\includegraphics[scale=0.41]{./Suppelmentary_File/NewFigs/VOC_football3_new.pdf}
\includegraphics[scale=0.41]{./Suppelmentary_File/NewFigs/VOC_cubicle_new.pdf}
\includegraphics[scale=0.41]{./Suppelmentary_File/NewFigs/VOC_seq-jd_new.pdf}\\
\includegraphics[scale=0.41]{./Suppelmentary_File/NewFigs/VOC_girl_new.pdf}
\includegraphics[scale=0.41]{./Suppelmentary_File/NewFigs/VOC_BMX-Street_new.pdf}
\includegraphics[scale=0.41]{./Suppelmentary_File/NewFigs/VOC_planeshow_new.pdf}\\
\includegraphics[scale=0.41]{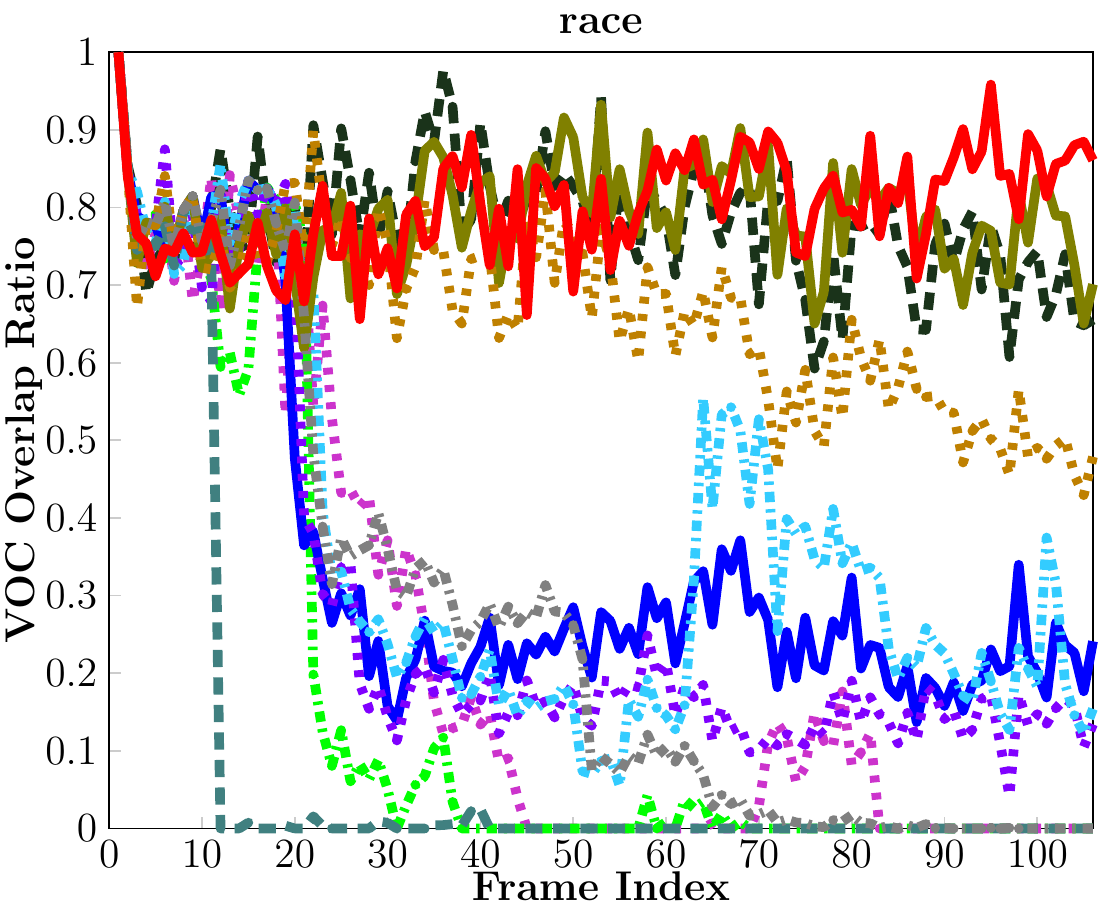}
\includegraphics[scale=0.41]{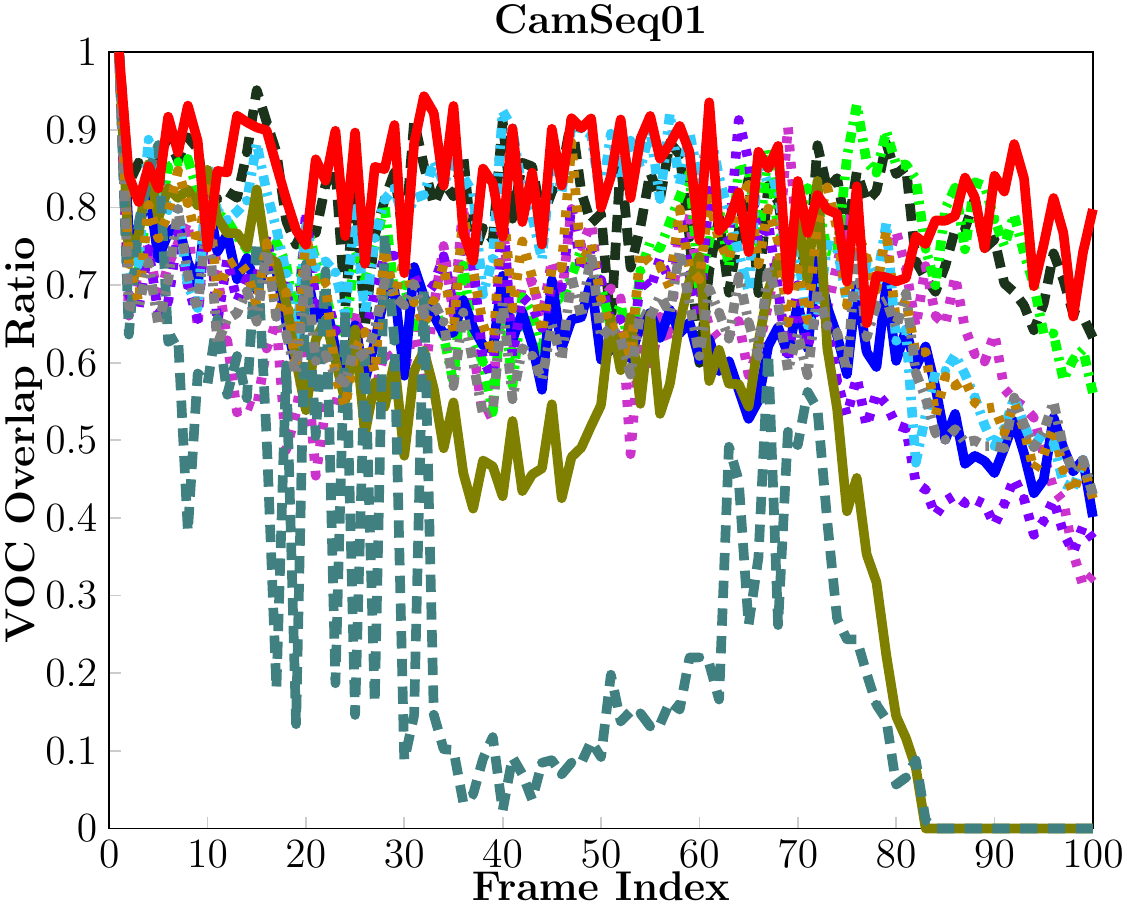}
\includegraphics[scale=0.41]{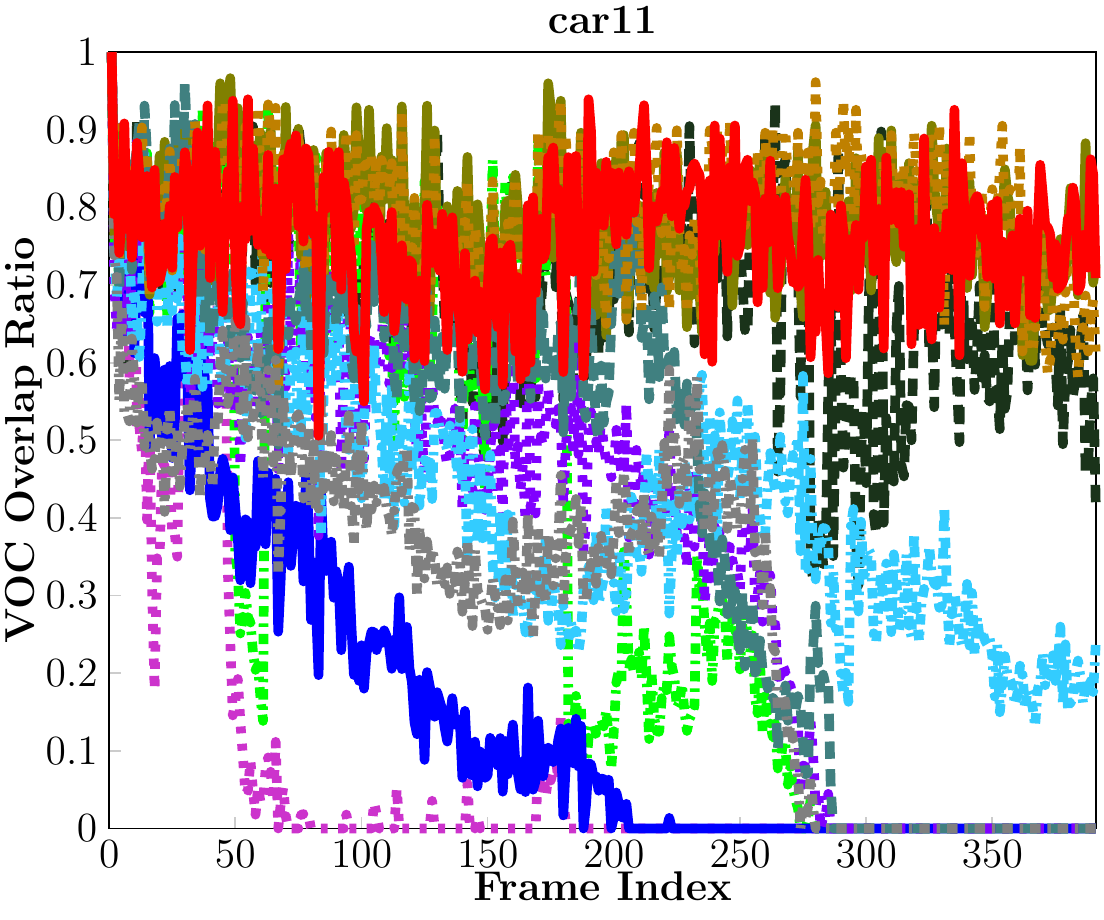}\\
\vspace{-0.13cm}
\caption{Quantitative comparison of different trackers in VOR on all the eighteen
video sequences.}
 \vspace{-0.41cm}
  \label{fig:exp_voc_curve}
\end{figure*}

\begin{figure*}[t]
\centering
\includegraphics[width=0.98\linewidth]{./Suppelmentary_File/NewFigs/NewLegendBar.PNG}\\
\includegraphics[width=0.98\linewidth]{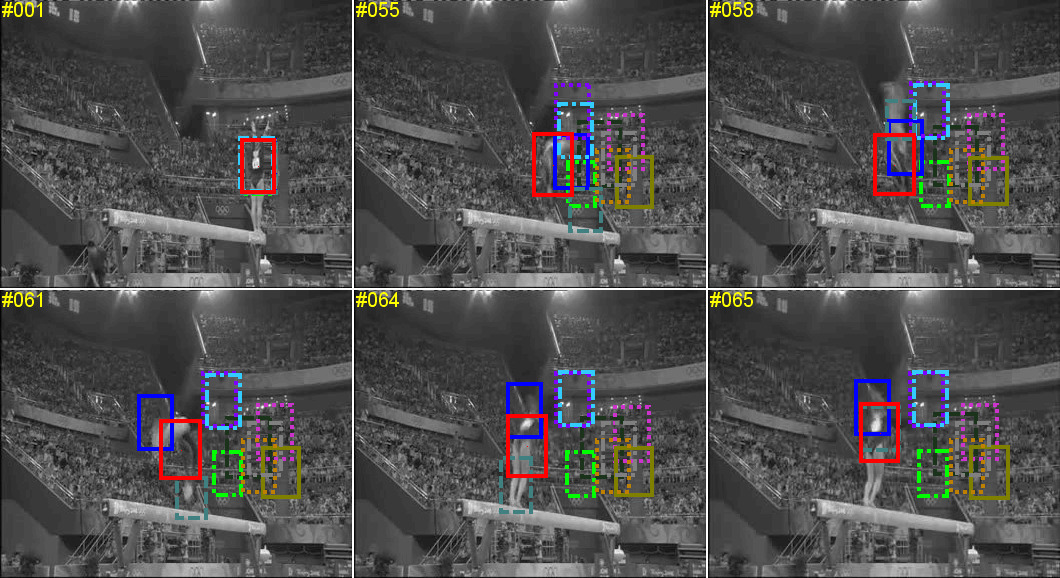}\\
\vspace{-0.0cm}
 \caption{Tracking results of different trackers over some
representative frames from the ``BalanceBeam'' video sequence
in the scenarios with drastic body pose variations and background clutters.
}
 \label{fig:tracking_balancebeam} \vspace{-0.3cm}
\end{figure*}

\begin{figure*}[t]
\centering
\includegraphics[width=0.98\linewidth]{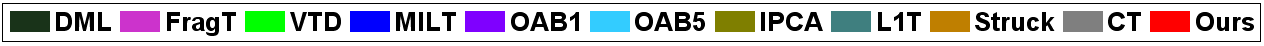}\\
\includegraphics[width=0.98\linewidth]{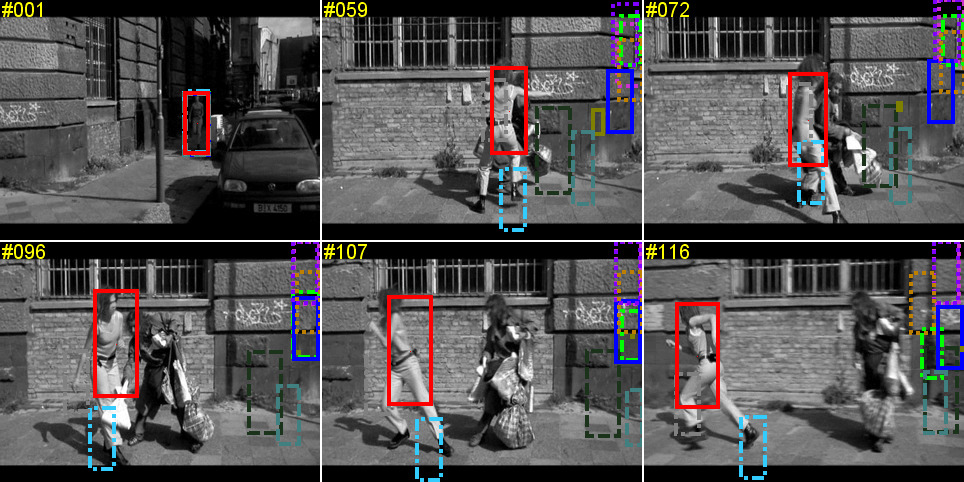}\\
 \caption{Tracking results of different trackers over some
representative frames from the ``Lola'' video sequence
in the scenarios with drastic scale changes and body pose variations.
}
 \label{fig:tracking_Lola} \vspace{-0.3cm}
\end{figure*}

\begin{figure*}[t]
\centering
\includegraphics[width=0.98\linewidth]{./Suppelmentary_File/NewFigs/NewLegendBar.PNG}\\
\includegraphics[width=0.98\linewidth]{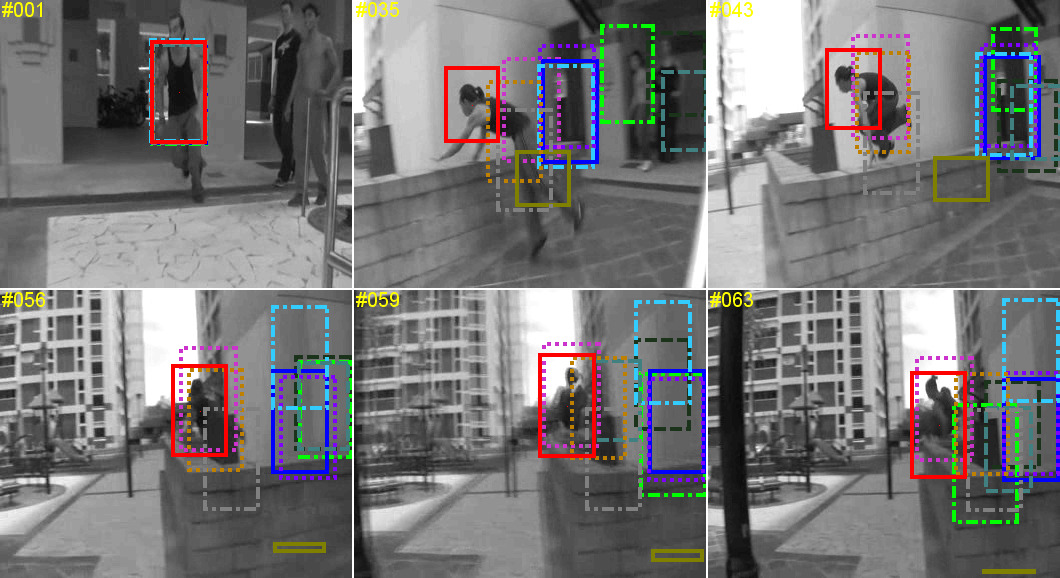}\\
 \caption{Tracking results of different trackers over some
          representative frames from the ``trace'' video sequence
 in the scenarios with drastic body pose variations and shape deformations.
}
 \label{fig:tracking_trace} \vspace{-0.3cm}
\end{figure*}

\begin{figure*}[t]
\centering
\includegraphics[width=0.98\linewidth]{./Suppelmentary_File/NewFigs/NewLegendBar.PNG}\\
\includegraphics[width=0.98\linewidth]{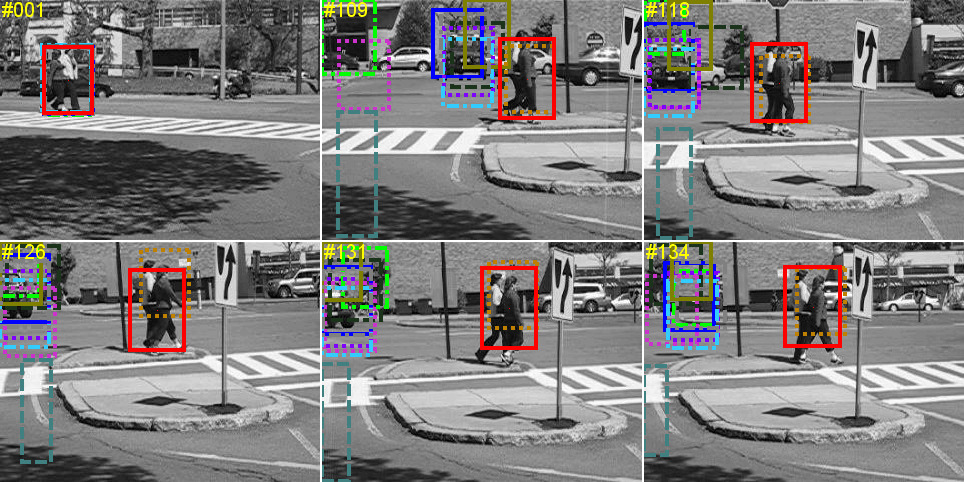}\\
 \caption{Tracking results of different trackers over some
representative frames from the ``Walk'' video sequence
 in the scenarios with drastic camera motion, partial occlusions, and background clutters.
}
 \label{fig:tracking_Walk} \vspace{-0.3cm}
\end{figure*}

\begin{figure*}[t]
\centering
\includegraphics[width=0.98\linewidth]{./Suppelmentary_File/NewFigs/NewLegendBar.PNG}\\
\includegraphics[width=0.98\linewidth]{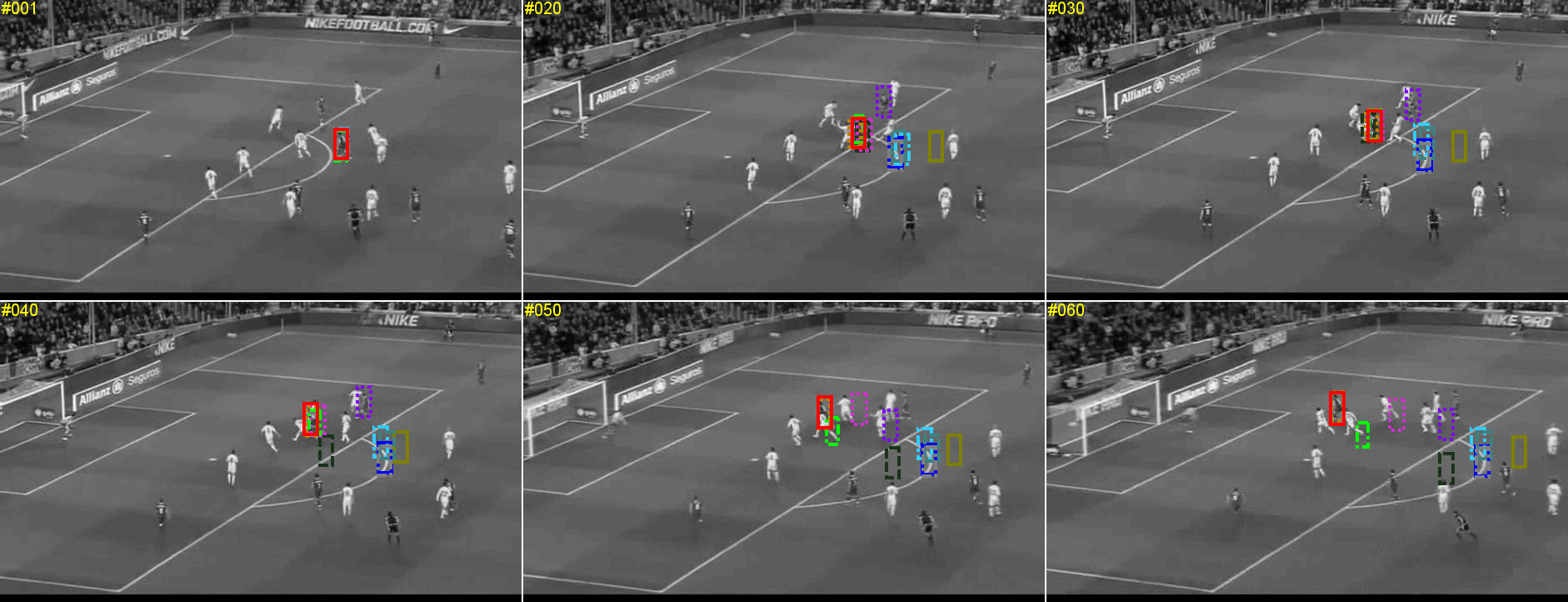}\\
 \caption{Tracking results of different trackers over some
representative frames from the ``football'' video sequence
in the scenarios with small-sized targets and partial occlusions.
}
 \label{fig:tracking_football} \vspace{-0.3cm}
\end{figure*}

\begin{figure*}[t]
\centering
\includegraphics[width=0.98\linewidth]{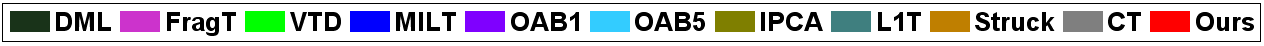}\\
\includegraphics[width=0.98\linewidth]{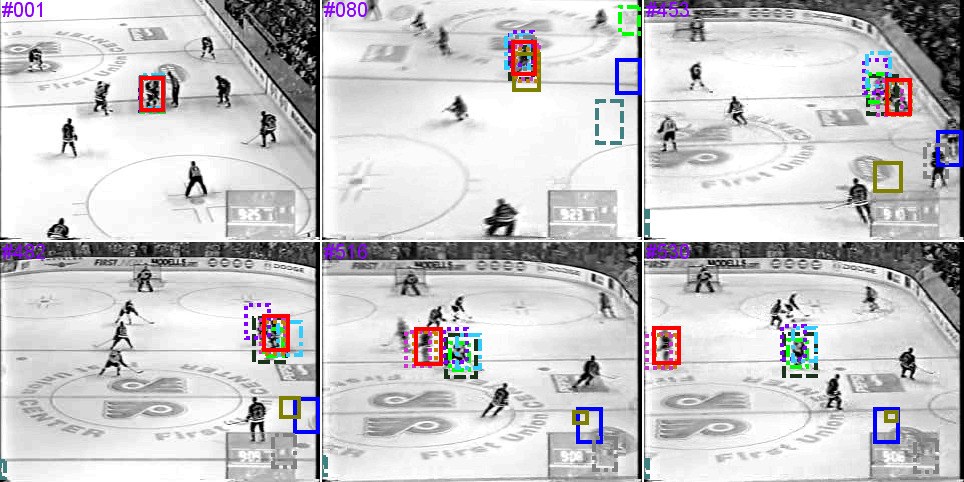}\\
 \caption{Tracking results of different trackers over some
representative frames from the ``iceball'' video sequence
in the scenarios with partial occlusions,  out-of-plane rotations, body pose variations, and abrupt motion.
}
 \label{fig:tracking_iceball} \vspace{-0.3cm}
\end{figure*}

\begin{figure*}[t]
\centering
\includegraphics[width=0.98\linewidth]{./Suppelmentary_File/NewFigs/NewLegendBar.PNG}\\
\includegraphics[width=0.98\linewidth]{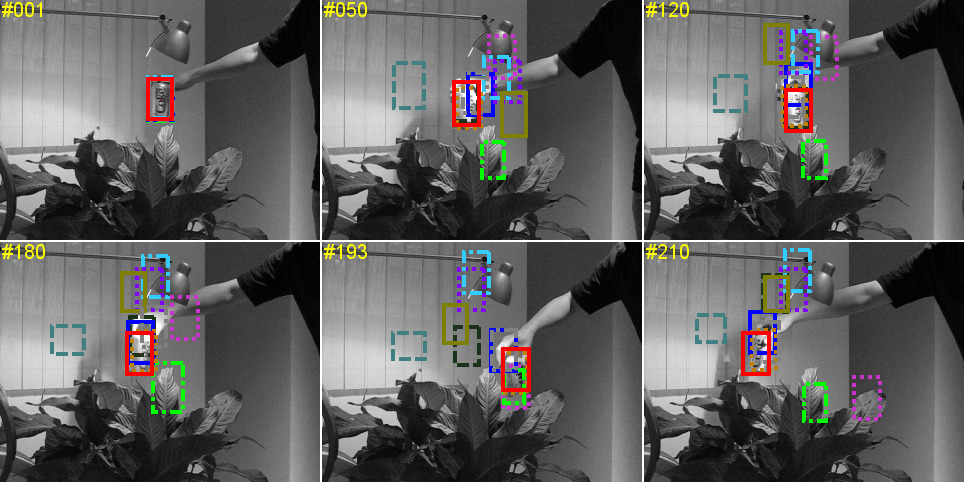}\\
 \caption{Tracking results of different trackers over some
representative frames from the ``coke11'' video sequence
in the scenarios with illumination changes, severe occlusions,  out-of-plane rotations, and background clutters.
}
 \label{fig:tracking_coke11} \vspace{-0.3cm}
\end{figure*}

\begin{figure*}[t]
\centering
\includegraphics[width=0.98\linewidth]{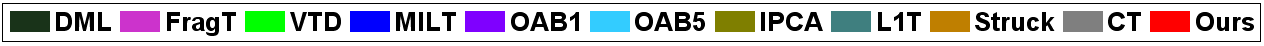}\\
\includegraphics[width=0.98\linewidth]{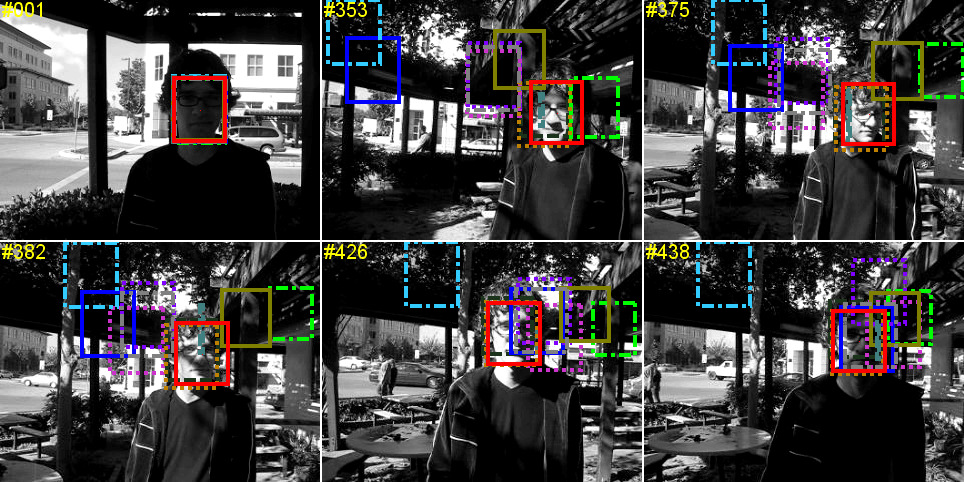}\\
 \caption{Tracking results of different trackers over some
representative frames from the ``trellis70'' video sequence
in the scenarios with drastic illumination changes and  head pose variations.
}
 \label{fig:tracking_trellis70} \vspace{-0.3cm}
\end{figure*}

\begin{figure*}[t]
\centering
\includegraphics[width=0.98\linewidth]{./Suppelmentary_File/NewFigs/NewLegendBar.PNG}\\
\includegraphics[width=0.98\linewidth]{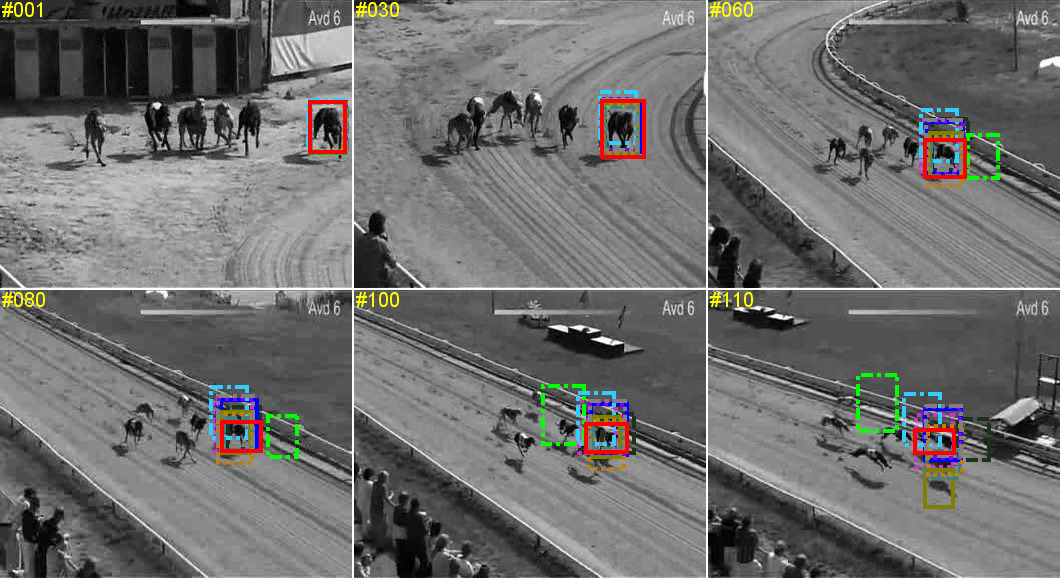}\\
 \caption{Tracking results of different trackers over some
representative frames from the ``dograce'' video sequence
in the scenarios with drastic pose changes and shape deformations.
}
 \label{fig:tracking_dograce} \vspace{-0.3cm}
\end{figure*}

\begin{figure*}[t]
\centering
\includegraphics[width=0.98\linewidth]{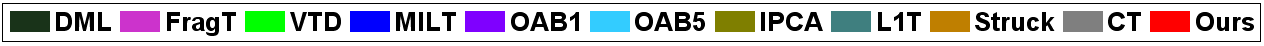}\\
\includegraphics[width=0.98\linewidth]{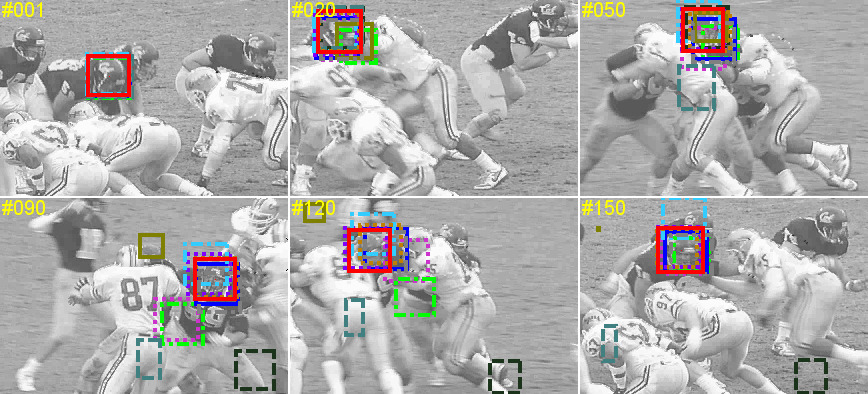}\\
 \caption{Tracking results of different trackers over some
representative frames from the ``football3'' video sequence
in the scenarios with motion blurring, partial occlusions,
head pose variations, and background clutters.
}
 \label{fig:tracking_football3} \vspace{-0.3cm}
\end{figure*}

\begin{figure*}[t]
\centering
\includegraphics[width=0.698\linewidth]{./Suppelmentary_File/NewFigs/NewLegendBar.PNG}\\
\includegraphics[width=0.698\linewidth]{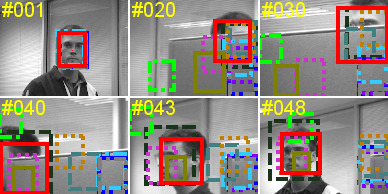}\\
 \caption{Tracking results of different trackers over some
representative frames from the ``cubicle'' video sequence
in the scenarios with severe occlusions, out-of-plane rotations, and head pose changes.
}
 \label{fig:tracking_cubicle} \vspace{-0.3cm}
\end{figure*}

\begin{figure*}[t]
\centering
\includegraphics[width=0.698\linewidth]{./Suppelmentary_File/NewFigs/NewLegendBar.PNG}\\
\includegraphics[width=0.698\linewidth]{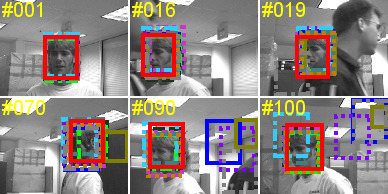}\\
 \caption{Tracking results of different trackers over some
representative frames from the ``seq-jd'' video sequence
in the scenarios with severe occlusions, out-of-plane rotations, and head pose changes.
}
 \label{fig:tracking_seq_jd} \vspace{-0.3cm}
\end{figure*}

\begin{figure*}[t]
\centering
\includegraphics[width=0.698\linewidth]{./Suppelmentary_File/NewFigs/NewLegendBar.PNG}\\
\includegraphics[width=0.698\linewidth]{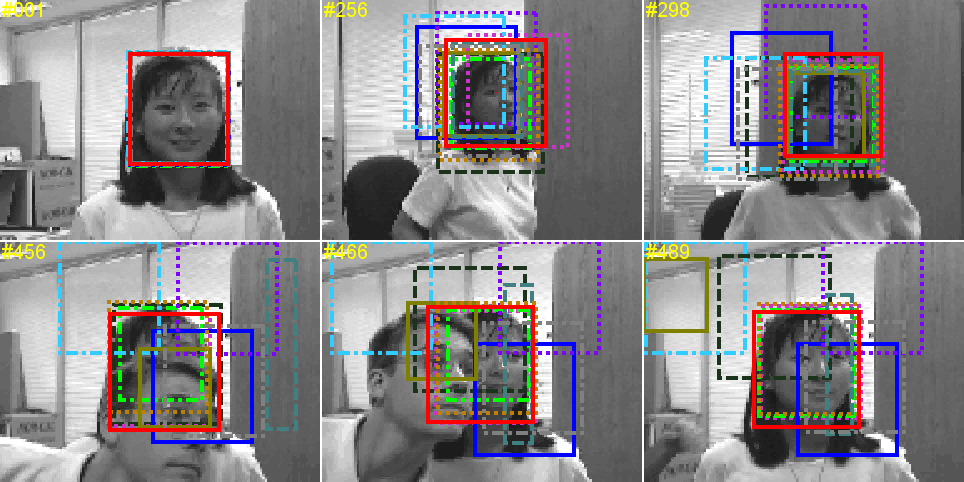}\\
 \caption{Tracking results of different trackers over some
representative frames from the ``girl'' video sequence
in the scenarios with severe occlusions, out-of-plane \& in-plane rotations, and head pose changes.
}
 \label{fig:tracking_girl} \vspace{-0.3cm}
\end{figure*}

\begin{figure*}[t]
\centering
\includegraphics[width=0.698\linewidth]{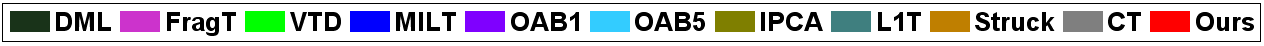}\\
\includegraphics[width=0.698\linewidth]{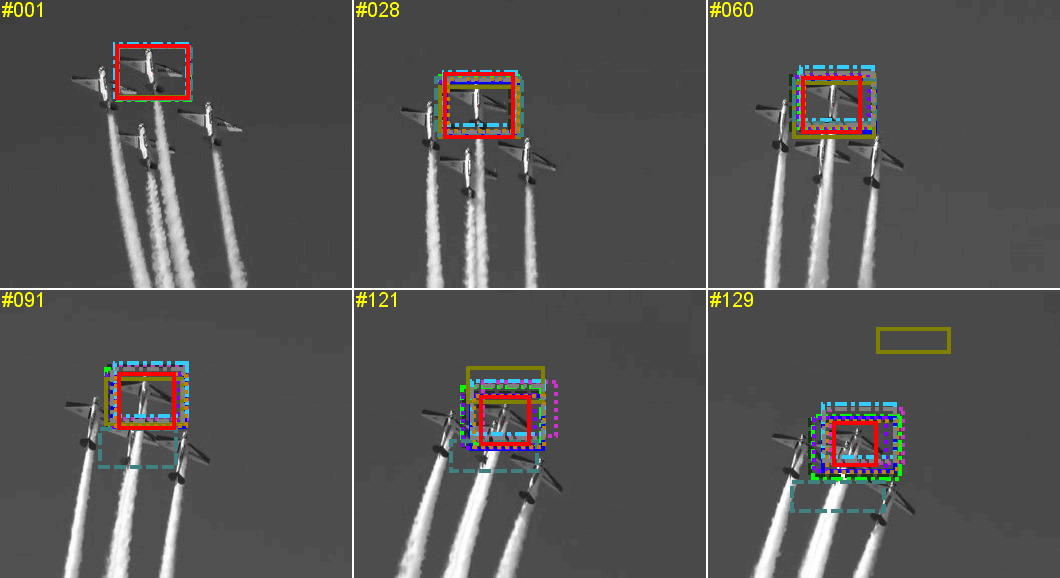}\\
 \caption{Tracking results of different trackers over some
representative frames from the ``planeshow'' video sequence
in the scenarios with shape deformations, out-of-plane rotations, and pose variations.
}
 \label{fig:tracking_planeshow} \vspace{-0.3cm}
\end{figure*}

\begin{figure*}[t]
\centering
\includegraphics[width=0.698\linewidth]{./Suppelmentary_File/NewFigs/NewLegendBar.PNG}\\
\includegraphics[width=0.698\linewidth]{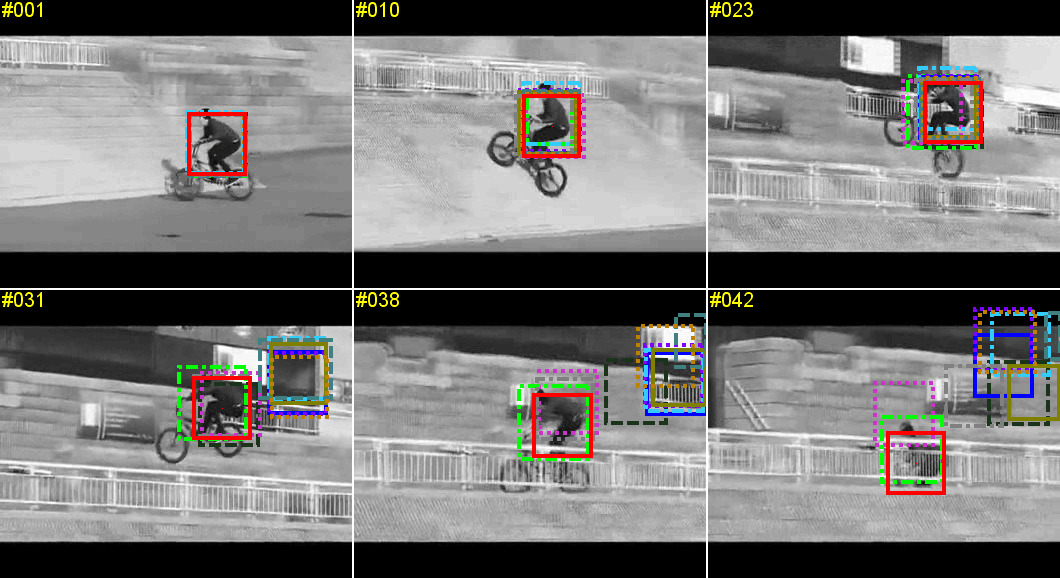}\\
 \caption{Tracking results of different trackers over some
representative frames from the ``BMX-Street'' video sequence
in the scenarios with shape deformations, partial occlusions, and body pose changes.
}
 \label{fig:tracking_BMX_Street} \vspace{-0.3cm}
\end{figure*}

\begin{figure*}[t]
\centering
\includegraphics[width=0.698\linewidth]{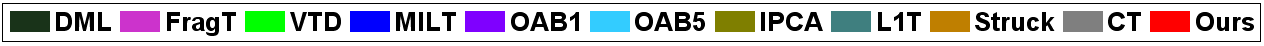}\\
\includegraphics[width=0.698\linewidth]{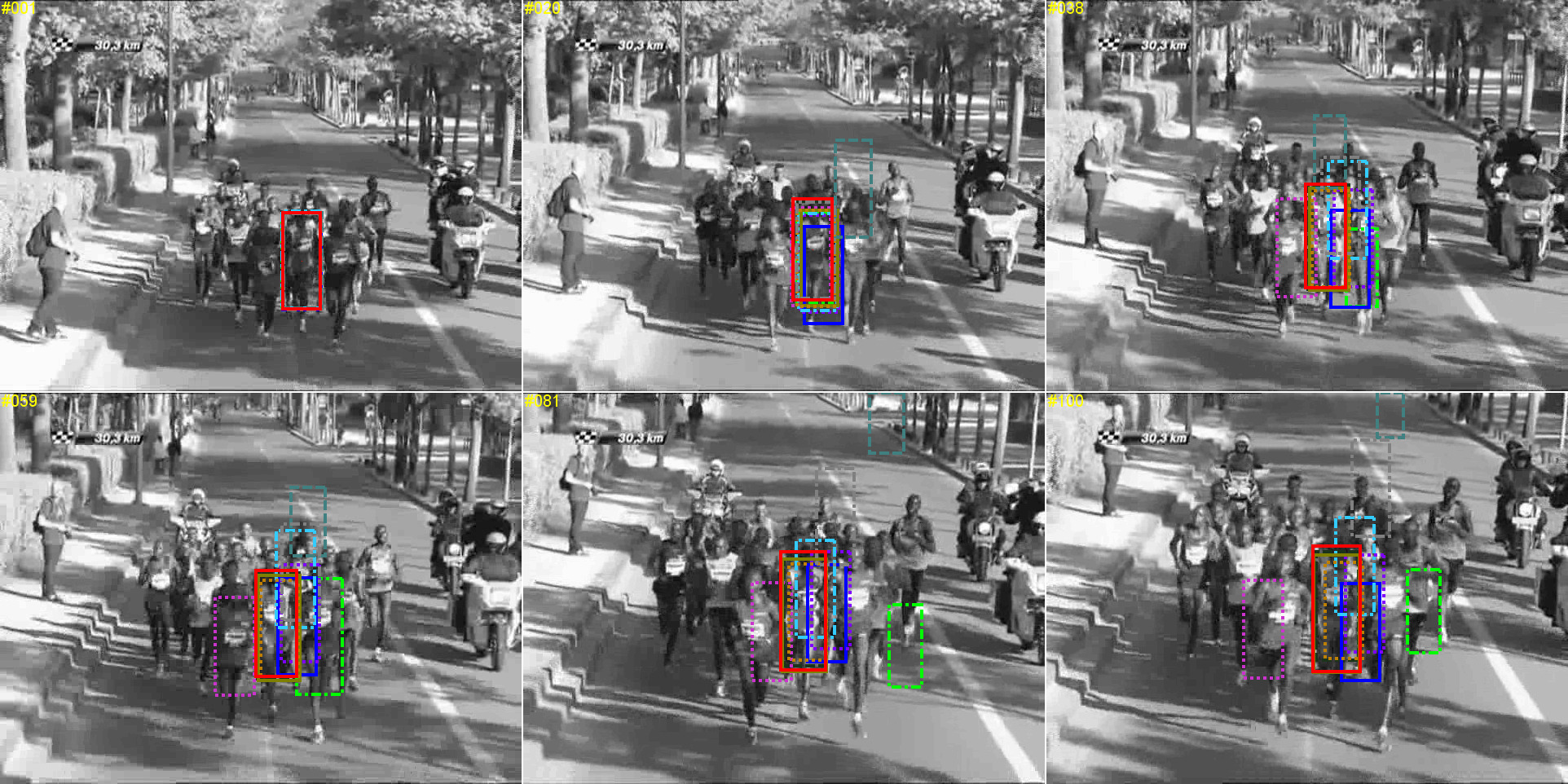}\\
 \caption{Tracking results of different trackers over some
representative frames from the ``race'' video sequence
in the scenarios with background clutters.
}
 \label{fig:tracking_race} \vspace{-0.3cm}
\end{figure*}

\begin{figure*}[t]
\centering
\includegraphics[width=0.698\linewidth]{./Suppelmentary_File/NewFigs/NewLegendBar.PNG}\\
\includegraphics[width=0.698\linewidth]{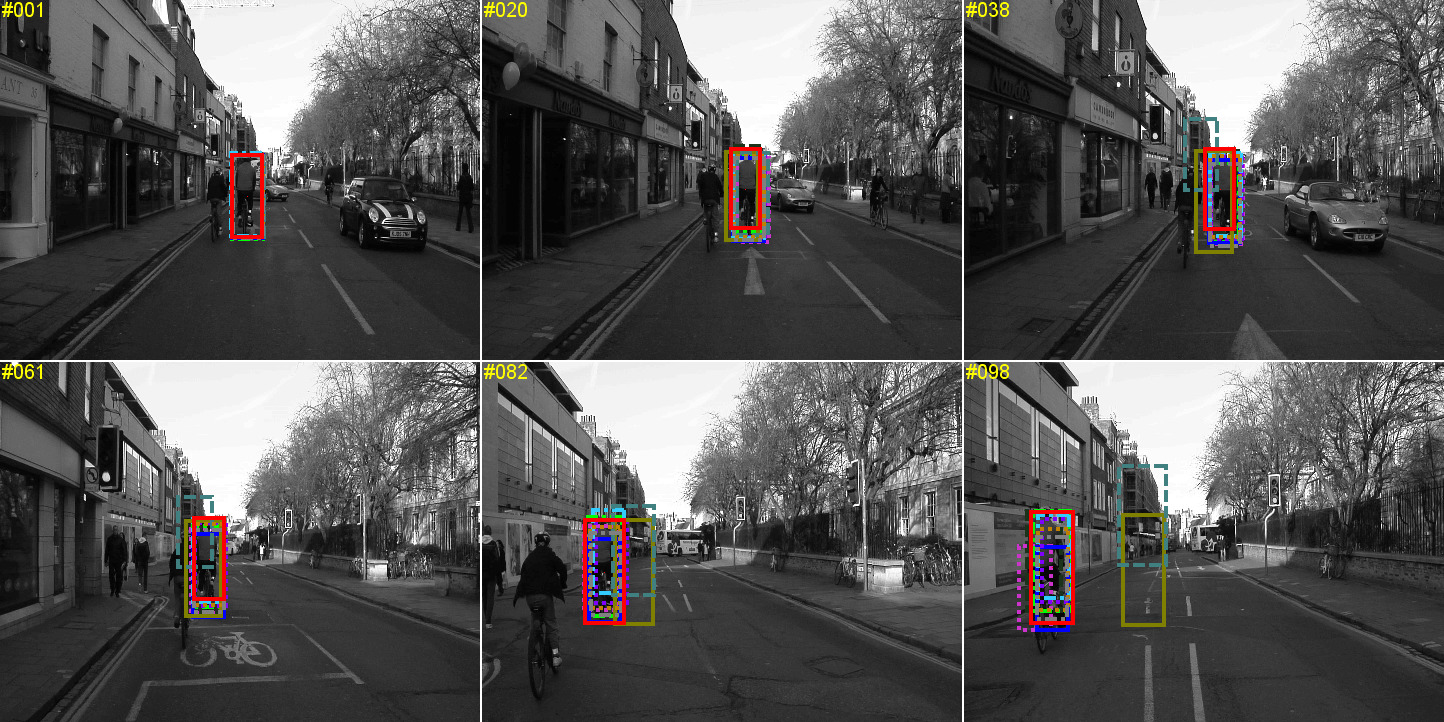}\\
 \caption{Tracking results of different trackers over some
representative frames from the ``CamSeq01'' video sequence
in the scenarios with body pose changes.
}
 \label{fig:tracking_CamSeq01} \vspace{-0.3cm}
\end{figure*}

\begin{figure*}[t]
\centering
\includegraphics[width=0.698\linewidth]{./Suppelmentary_File/NewFigs/NewLegendBar.PNG}\\
\includegraphics[width=0.698\linewidth]{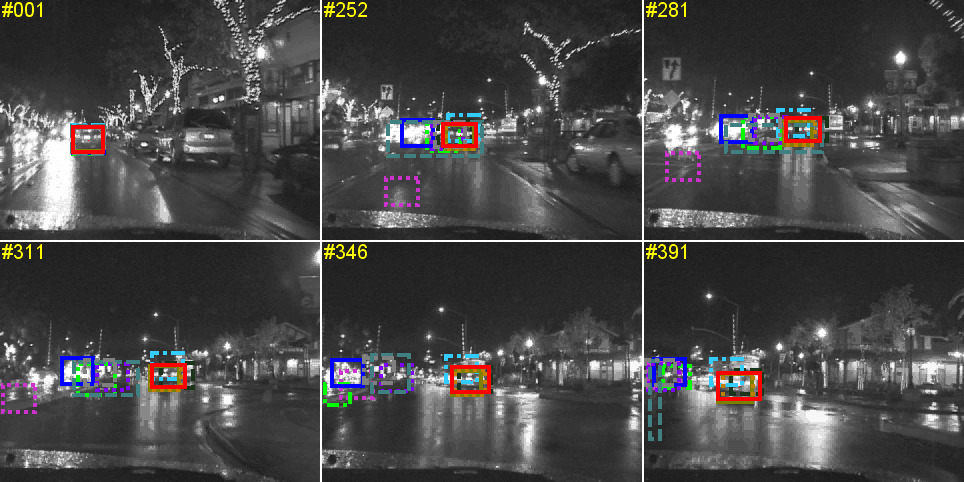}\\
 \caption{Tracking results of different trackers over some
representative frames from the ``car11'' video sequence
in the scenarios with varying lighting conditions and background clutters.
}
 \label{fig:tracking_car11} \vspace{-0.3cm}
\end{figure*}

\section{Discussion}
\label{sec:exp_discussion}

Based on the obtained experimental results, we observe that
the proposed tracking algorithm has the following properties. First,
after the buffer size exceeds a certain value (around $300$ in our
experiments), the tracking performance keeps stable with an increasing
buffer size, as shown in Fig.~\ref{fig:buffersize_particlenum}. This is desirable since we do not need a large buffer
size to achieve promising performance.
Second, in contrast to many existing particle filtering-based trackers
whose running time is typically linear in the number of particles, our method's
running time is sublinear in the number of
particles, as shown in Fig.~\ref{fig:buffersize_particlenum}.
Moreover, its tracking performance rapidly improves and
finally converge to a certain value, as shown in
Fig.~\ref{fig:buffersize_particlenum}.
Third, as shown in Fig.~\ref{fig:metric_non_metric} and Tab.~\ref{Tab:metric_success_rate}, the performance of
our metric learning with
no eigendecomposition is close to that of computationally expensive metric learning with
step-by-step eigendecomposition.
Fourth, using the structured metric learning is capable of
improving the tracking performance in CLE and VOR, as shown in Tab.~\ref{Tab:structured_metric_VOC_CLE}.
That is because the structured metric learning encodes the underlying
the structural interaction information on data samples, which plays an important role in robust visual tracking.
Fifth, based on linear representation with metric learning, it performs
better in tracking accuracy, as shown
in
Fig.~\ref{fig:regression}.
Sixth, it utilizes weighed reservoir
sampling to effectively maintain and update the foreground
and background sample buffers for metric learning, as shown in Fig.~\ref{fig:sampling}.
Seventh, compared with other state-of-the-art trackers, it is capable of
effectively
adapting to complicated appearance changes in the tracking process by
constructing
an effective metric-weighted linear representation with weighed
reservoir sampling, as shown in
Fig.~\ref{fig:exp_error_curve} and Tab.~\ref{Tab:quantitative}.

}
\end{spacing}

\end{document}